\documentclass[backref=page]{article}
\newcommand{\inclPaper}{1}
\newcommand{\inclAppendix}{1}

\usepackage{microtype}
\usepackage{graphicx}
\usepackage{subfig}%
\usepackage{booktabs} %

\usepackage[accepted]{icml2022}
\usepackage{hyperref}%
\definecolor{mydarkblue}{rgb}{0,0.08,0.45}
  \hypersetup{ %
    pdftitle={NOMU: Neural Optimization-based Model Uncertainty},
    pdfauthor={},
    pdfsubject={Proceedings of the International Conference on Machine Learning 2022},
    pdfkeywords={Neural Networks, Model Uncertainty},
    pdfborder=0 0 0,
    pdfpagemode=UseNone,
    colorlinks=true,
    linkcolor=mydarkblue,
    citecolor=mydarkblue,
    filecolor=mydarkblue,
    urlcolor=mydarkblue,
    }
    
    \hypersetup{pdfauthor=\icmlfullauthorlist}
\usepackage{amsmath}
\usepackage{amssymb}
\usepackage{mathtools}
\usepackage{amsthm}

\usepackage[capitalize,noabbrev]{cleveref}

\icmltitlerunning{NOMU: Neural Optimization-based Model Uncertainty}

\usepackage{ourcommands}

\ifOnlyAppendix{\includeonly{}}

\begin{document}

\ifOnlyPaper{%

\twocolumn[
\icmltitle{NOMU: Neural Optimization-based Model Uncertainty}

\icmlsetsymbol{equal}{*}

\begin{icmlauthorlist}
\icmlauthor{Jakob Heiss}{equal,eth,ai}
\icmlauthor{Jakob Weissteiner}{equal,ai,uzh}
\icmlauthor{Hanna Wutte}{equal,eth,ai}
\icmlauthor{Sven Seuken}{ai,uzh}
\icmlauthor{Josef Teichmann}{eth,ai}
\end{icmlauthorlist}

\icmlaffiliation{eth}{ETH Zurich}
\icmlaffiliation{uzh}{University of Zurich}
\icmlaffiliation{ai}{ETH AI Center}

\icmlcorrespondingauthor{Jakob Weissteiner}{weissteiner@ifi.uzh.ch}
\icmlkeywords{Neural Networks, Model Uncertainty}

\vskip 0.3in
]

\printAffiliationsAndNotice{\icmlEqualContribution} %

\begin{abstract}
We study methods for estimating model uncertainty for neural networks (NNs) in regression. To isolate the effect of model uncertainty, we focus on a noiseless setting with scarce training data. We introduce five important desiderata regarding model uncertainty that any method should satisfy. However, we find that established benchmarks often fail to reliably capture some of these desiderata, even those that are required by Bayesian theory. To address this, we introduce a new approach for capturing model uncertainty for NNs, which we call \emph{Neural Optimization-based Model Uncertainty (NOMU)}. The main idea of NOMU is to design a network architecture consisting of two connected sub-NNs, one for model prediction and one for model uncertainty, and to train it using a carefully-designed loss function. Importantly, our design enforces that NOMU satisfies our five desiderata. Due to its modular architecture, NOMU can provide model uncertainty for any given (previously trained) NN if given access to its training data. We evaluate NOMU in various regressions tasks and noiseless Bayesian optimization (BO) with costly evaluations. In regression, NOMU performs at least as well as state-of-the-art methods. In BO, NOMU even outperforms all considered benchmarks.
\end{abstract}

\section{Introduction}
\label{sec:Introduction}
Neural networks (NNs) are becoming increasingly important in machine learning applications \citep{LeCun2015}. In many domains, it is essential to be able to quantify the \emph{model uncertainty (epistemic uncertainty)} of NNs \citep{neal2012bayesian,ghahramani2015probabilistic}. 
Good estimates of model uncertainty are indispensable in Bayesian optimization (BO) and active learning, where exploration is steered by (functions of) these uncertainty estimates. In recent years, BO has been successfully applied in practice to a wide range of problems, including robotics \citep{martinez2009bayesian}, sensor networks \citep{Srinivas_2012}, and drug development \citep{gomez2018automatic}. Better model uncertainty estimates for BO directly translate to improvements in these applications.

\begin{figure}[t!]
    \begin{center}
    \centerline{\includegraphics[width=1\columnwidth]{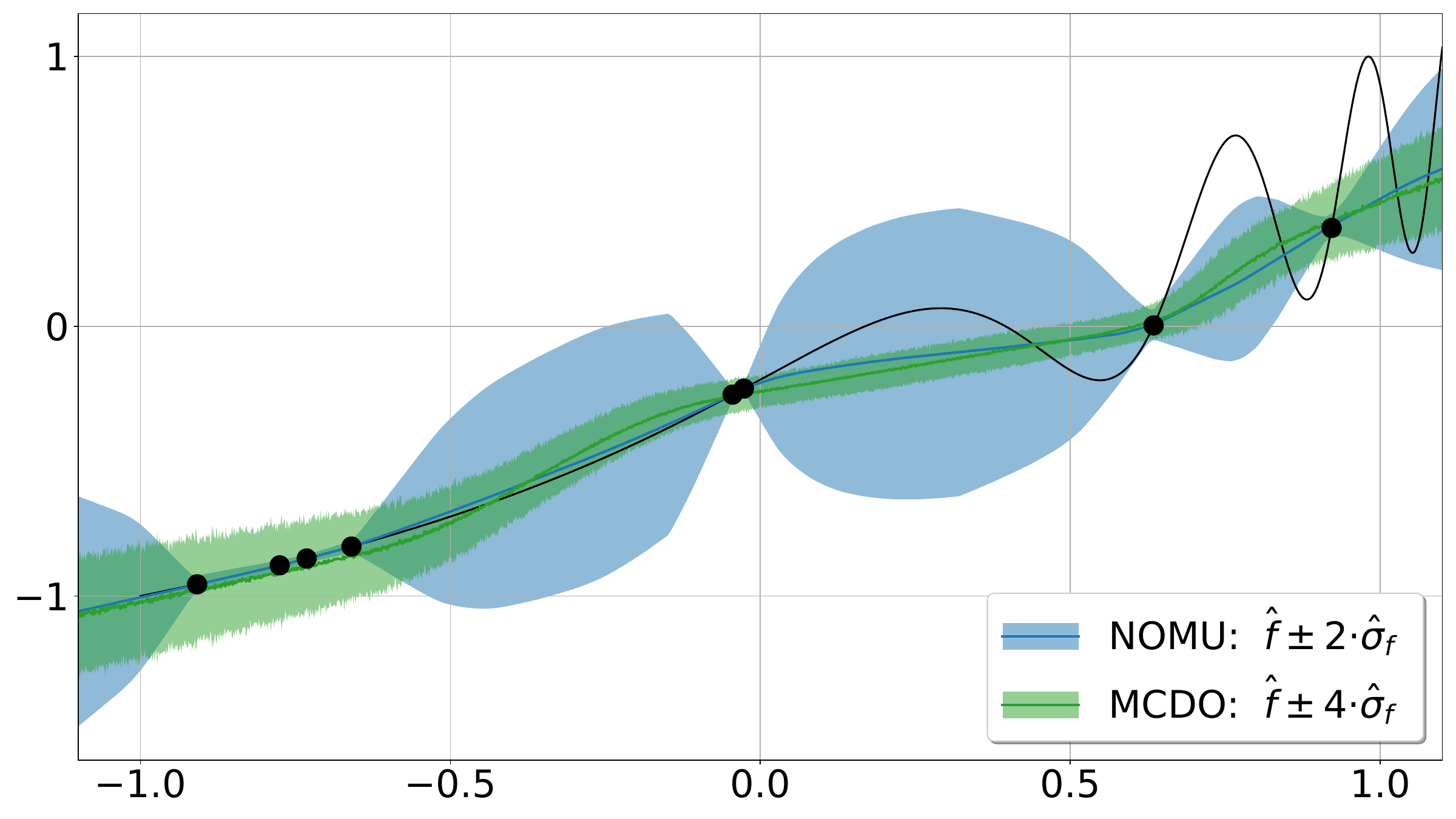}}
    \caption{Visualization of estimated model uncertainty $\sigmodelhat$. The unknown true function is depicted as a black solid line with training points as black dots. NOMU's model prediction $\fhat$ is shown as a solid blue line and its uncertainty bounds are shown as a blue shaded area. As a benchmark, MC Dropout is shown in green.}
    \label{fig:D41dJakobhfun}
    \end{center}
    \vskip -0.2in
\end{figure}

However, estimating model uncertainty well for NNs is still an open research problem. For settings with scarce training data and negligible data noise, where model uncertainty is the main source of uncertainty, we uncover deficiencies of widely used state-of-the-art methods for estimating model uncertainty for NNs. Prior work often only measures the performance in data noise dominant settings, and thus does not adequately isolate the pure model uncertainty, thereby overlooking the algorithms' deficiencies. However, in tasks such as BO with costly evaluations, where accurate estimates of model uncertainty are of utmost importance, these deficiencies can drastically decrease performance.

In this paper, we study the problem of estimating model uncertainty for NNs to obtain \emph{uncertainty bounds (UBs)} that estimate Bayesian credible bounds in a setting with negligible data noise and scarce training data. For this, we propose a novel algorithm (NOMU) that is specialized to such a setting.
Figure~\ref{fig:D41dJakobhfun} shows UBs for NOMU and the benchmark method MC Dropout.

\subsection{Prior Work on Model Uncertainty for NNs}\label{subsec:Uncertainty Quantification of Neural Networks}
Over the last decade, researchers have developed various methods to quantify model uncertainty for NNs. One strand of research considers Bayesian Neural Networks (BNNs), where distributions are placed over the NN's parameters \citep{graves2011practical,blundell2015weight,hernandez2015probabilistic}. However, variational methods approximating BNNs are usually computationally prohibitive and require careful hyperparameter tuning. Thus, BNNs are rarely used in practice \citep{wenzel2020good}.

In practice, \textit{ensemble methods} are more established:
\begin{itemize}[leftmargin=*,topsep=0pt,partopsep=0pt, parsep=0pt]
\item \citet{gal2016dropout} proposed \textit{Monte Carlo dropout (MCDO)} to estimate model uncertainty via stochastic forward passes. Interestingly, they could show that training a NN with dropout can also be interpreted as variational inference approximating a BNN.
\item \citet{lakshminarayanan2017simple} experimentally evaluated ensembles of NNs and showed that they perform as well as or better than BNNs. They proposed using \emph{deep ensembles (DE)}, which use NNs with two outputs for model prediction and data noise, and they estimate model uncertainty via the empirical standard deviation of the ensemble. DE is the most established state-of-the art ensemble method and has been shown to consistently outperform other ensemble methods \citep{ovadia2019can,fort2019deep,gustafsson2020evaluating,ashukha2020pitfalls}.
\item Recently, \citet{wenzel2020hyperparameter} proposed \emph{hyper deep ensembles (HDE)}, an extension of DE where additional diversity is created via different hyperparameters, and they showed that HDE outperforms DE.
\end{itemize}

Despite the popularity of MCDO, DE and HDE, our empirical results %
suggest that none of them reliably capture all essential features of model uncertainty: MCDO yields tubular bounds that do not narrow at observed data points (which can already be observed in \Cref{fig:D41dJakobhfun}); DE and HDE can produce UBs that are sometimes unreasonably narrow in regions far from observed data or unreasonably wide at training points (as we will show in Section \ref{subsec:Regression}).

\subsection{Overview of our Contribution}\label{subsec:Our Contribution}
We present a new approach for estimating model uncertainty for NNs, which we call \emph{neural optimization-based model uncertainty (NOMU)}. In contrast to a fully Bayesian approach (e.g., BNNs), where approximating the posterior for a realistic prior is in general very challenging, we estimate posterior credible bounds by directly enforcing essential properties of model uncertainty. Specifically, we make the following contributions:
\begin{enumerate}[leftmargin=*,topsep=0pt,partopsep=0pt, parsep=0pt]
\item We first introduce five desiderata that we argue model UBs should satisfy (\Cref{subsubsec:Desiderata}).
\item We then introduce NOMU, which consists of a network architecture (\Cref{subsec:The Network Architecture}) and a carefully-designed loss function (\Cref{subsec:modeluncertaintyloss}),  %
such that the estimated UBs fulfill these five desiderata. NOMU is easy to implement, scales well to large NNs, and can be represented as a \emph{single} NN without the need for further ensemble distillation (in contrast to MCDO, DE and HDE). Because of its modular architecture, it can easily be used to obtain UBs for already trained NNs. 
 \item We experimentally evaluate NOMU in various regression settings: in scarce and noiseless settings to isolate model uncertainty (\Cref{subsubsec:ToyRegression,subsubsec:GenerativeTestbed}) and on real-word data-sets (\Cref{subsubsec:SolarIrradiance,subsubsec:UCI}). We show that NOMU performs well across all these settings while state-of-the-art methods (MCDO, DE, and HDE) exhibit several deficiencies.\footnote{We also conducted experiments with \citep{blundell2015weight}. However, we found that this method did not perform as well as the other considered benchmarks. Moreover, it was shown in \citep{gal2016dropout,lakshminarayanan2017simple} that \emph{deep ensembles} and \emph{MC dropout} outperform the methods by \citep{hernandez2015probabilistic} and \citep{graves2011practical}, respectively. Therefore, we do not include \citep{graves2011practical,blundell2015weight,hernandez2015probabilistic} in our experiments.}
\item Finally, we evaluate the performance of NOMU in high-dimensional Bayesian optimization (BO) and show that NOMU performs as well or better than all considered benchmarks (Section~\ref{subsec:BayesianOptimization}).
\end{enumerate}

Our source code is available on GitHub: \url{https://github.com/marketdesignresearch/NOMU}.

\subsection{Further Related Work}\label{subsec:Related Work}
\citet{nix1994estimating} were among the first to introduce NNs with two outputs: one for model prediction and one for \emph{data noise (aleatoric uncertainty)}, using the Gaussian negative log-likelihood as loss function. 
However, such a data noise output cannot be used as an estimator for model uncertainty (epistemic uncertainty); see \Cref{sec:appendix:Aleatoric Neural Networks: Aleatoric vs. Epistemic Uncertainty} for details.
To additionally capture model uncertainty, \citet{kendall2017uncertainties} combined the idea of \citet{nix1994estimating} with MCDO.

Similarly, NNs with two outputs for lower and upper UBs, trained on specifically-designed loss functions, were previously considered by \citet{khosravi2010lower} and \citet{pearce2018high}. However, the method by \citet{khosravi2010lower} again only accounts for data noise and does not consider model uncertainty. The method by \citet{pearce2018high} also does not take model uncertainty into account in the design of their loss function and only incorporates it via ensembles (as in DE).

Besides the state-of-the art ensemble methods HDE and DE, there exist several other papers on ensemble methods that, for example, promote their diversity on the function space \citep{wang2019function} or reduce their computational cost \citep{wen2020batchensemble,havasi2021training}.

For classification, \citet{malinin2018predictive} introduced prior networks, which explicitly model in-sample and out-of-distribution uncertainty, where the latter is realized by minimizing the reverse KL-distance to a selected flat point-wise defined prior. In a recent working paper (and concurrent to our work), \citet{malinin2020regression} report on progress extending their idea to regression. While the idea of introducing a separate loss for learning model uncertainty is related to NOMU, there are several important differences (loss, architecture, behavior of the model prediction, theoretical motivation; see \Cref{sec:NOMUvsPriorNetworks} for details).
Furthermore, their experiments suggest that DE still performs weakly better than their proposed method.

In contrast to BNNs, which perform approximate inference over the entire set of weights, Neural Linear Models (NLMs) perform \emph{exact} inference on only the last layer. NLMs have been extensively benchmarked in \citep{ober2019benchmarking} against MCDO and the method from \citep{blundell2015weight}. Their results suggest that MCDO and \citep{blundell2015weight} perform  competitive, even to carefully-tuned NLMs. 

Neural processes, introduced by \citet{conditional_neural_proc_pmlr-v80-garnelo18a,neural_proc}, have been used to express model uncertainty for image completion tasks, where one has access to thousands of different images interpreted as functions $f_i$ instead of input points $x_i$. See \Cref{sec:appendix:NOMU vs. Neural Processes} for a detailed comparison of their setting to the setting we consider in this paper.

\section{Preliminaries}\label{sec:Bayesian Uncertainty Framework}
In this section, we briefly review the classical Bayesian uncertainty framework for regression.

Let $\X\subset\R^d, Y\subset\R$ and let $f\colon \X \to \Y$ denote the unknown ground truth function. Let $\Dtr:=\{\left( \xtr_i,\ytr_i\right)\in\X\times\Y, i\in \fromto{\ntr}\},$ with $\ntr\in\N$ be i.i.d samples from the data generating process
$\label{modelAssumption}
    y = f(x) + \varepsilon,
$
where $\varepsilon|x\sim\mathcal{N}(0,\signoise^2(x))$. We use $\signoise$ to refer to the \emph{data noise (aleatoric uncertainty)}. We refer to $\left( \xtr_i,\ytr_i\right)$ as a \emph{training point} and to $\xtr_i$ as an \emph{input training point}.

In the remainder of this paper, we follow the classic Bayesian uncertainty framework by modelling the unknown ground truth function $f$ as a random variable. Hence, with a slight abuse of notation, we use the symbol $f$ to denote \textit{both} the unknown ground truth function as well as the corresponding random variable.
In \Cref{sec:DetailsNotation}, we provide a mathematically rigorous formulation of the considered Bayesian uncertainty framework.

Given a prior distribution for $f$ and known data noise $\signoise$, the posterior of $f$ and $y$ are well defined. The \emph{model uncertainty (epistemic uncertainty)} $\sigmodel(x)$ is the posterior standard deviation of $f(x)$, i.e.,
\begin{align}\label{eq:modelUncertainty}
\sigmodel(x):= \sqrt{\mathbb{V}[f(x)|\Dtr,x]}, \quad x\in \X.
\end{align}
Assuming independence between $f$ and $\varepsilon$, the variance of the predictive distribution of $y$ can be decomposed as
$\label{varianceAss}\mathbb{V}[y|\Dtr,x]=\sigmodel^2(x) + \signoise^2(x).$
We present our algorithm for estimating model uncertainty $\sigmodelhat$ for the case of zero or negligible data noise, i.e., $\signoise\approx0$ (see \Cref{sec:Extensions} for an extension to $\signoise\gg0$). For a given model prediction $\fhat$, the induced uncertainty bounds (UBs) are then given by $\left(\lUB_c(x),\uUB_c(x)\right):=\left(\fhat(x)\mp c~\sigmodelhat(x)\right)$, for $x\in\X$ and a calibration parameter $c\geq0$.\footnote{Note that our UBs estimate \emph{credible bounds (CBs)}~$\underline{C\!B}$ and $\overline{C\!B}$, which, for $\alpha \in [0,1]$, fulfill that $\mathbb{P}[f(x)\in [\underline{C\!B},\overline{C\!B}]|\Dtr,x]=\alpha$. For $\signoise\equiv0$, CBs are equal to \emph{predictive bounds} $\underline{P\!B},\overline{P\!B}$ with $\mathbb{P}[y \in [\underline{P\!B},\overline{P\!B}]|\Dtr,x]=\alpha$. See \Cref{sec:TheoreticalAnalysisAppendix} for an explanation.}

\section{The NOMU Algorithm}\label{sec:Neural Optimization-based Model Uncertainty}
We now present NOMU. We design NOMU to yield a model prediction $\fhat$ and a model uncertainty prediction $\sigmodelhat$, such that the resulting UBs $(\lUB_c,\uUB_c)$ fulfill five desiderata.
\subsection{Desiderata}\label{subsubsec:Desiderata}
\begin{enumerate}[label=D\arabic*,align=left, leftmargin=*,topsep=0pt]
    \item[\mylabel{itm:Axioms:trivial}{D1 (Non-Negativity)}]
	\textit{The upper/lower UB between two training points lies above/below the model prediction $\fhat$, i.e.,  $\lUB_c(x)\le\fhat(x)\le\uUB_c(x)$ for all $x\in\X$ and for $c\geq0$. Thus, $\sigmodelhat\geq0$.}
	\end{enumerate}
By definition, for any given prior, the exact posterior model uncertainty $\sigmodel$ is positive, and therefore Desideratum~\hyperref[itm:Axioms:trivial]{D1} should also hold for any estimate $\sigmodelhat$.
\begin{enumerate}[resume,label=D\arabic*,align=left, leftmargin=*,topsep=0pt]
		\item[\mylabel{itm:Axioms:ZeroUncertaintyAtData}{D2 (In-Sample)}]
	\textit{In the noiseless case ($\signoise\equiv0$), there is zero model uncertainty at each input training point $\xtr$, i.e., $\sigmodelhat(\xtr)=0$. Thus, $\uUB_c(\xtr)=\lUB_c(\xtr)=\fhat(\xtr)$ for $c\ge0$.}
\end{enumerate}
In \Cref{sec:desiderata:InSample}, we prove that, for any prior that does not contradict the training data, the exact $\sigmodel$ satisfies \hyperref[itm:Axioms:ZeroUncertaintyAtData]{D2}. Thus, \hyperref[itm:Axioms:ZeroUncertaintyAtData]{D2} should also hold for any estimate $\sigmodelhat$, and we argue that even in the case of non-zero small data noise, model uncertainty should be small at input training data points.%
\begin{enumerate}[resume,label=D\arabic*,align=left, leftmargin=*,topsep=0pt]
	\item[\mylabel{itm:Axioms:largeDistantLargeUncertainty}{D3 (Out-of-Sample)}]
	\textit{The larger the distance of a point $x\in\X$ to the input training points in $\Dtr$, the wider the UBs at $x$, i.e., model uncertainty~$\sigmodelhat$ increases out-of-sample.\footnote{Importantly, \hyperref[itm:Axioms:largeDistantLargeUncertainty]{D3} also promotes model uncertainty in gaps \emph{between} input training points.}}
\end{enumerate}
For \hyperref[itm:Axioms:largeDistantLargeUncertainty]{D3} it is often not obvious which metric to choose to measure distances.
Some aspects of this metric can be specified via the architecture (e.g., shift invariance for CNNs).
In many applications, it is best to \textit{learn} further aspects of this metric from training data, motivating our next desideratum.
\begin{enumerate}[resume,label=D\arabic*,align=left, leftmargin=*,topsep=0pt]
	\item[\mylabel{itm:Axioms:irregular}{D4 (Metric Learning)}]
	\textit{Changes in those features of $x$ that have high predictive power on the training set have a large effect on the distance metric used in \hyperref[itm:Axioms:largeDistantLargeUncertainty]{D3}.}\footnote{\label{footnote:LearningTheMetricISImportant}
	 Consider the task of learning facial expressions from images. For this, eyes and mouth are important features, while background color is not. 
	 A CNN automatically learns which features are important for model prediction.
	 The same features are also important for model uncertainty: Consider an image with pixel values similar to those of an image of the training data, but where mouth and eyes are very different. We should be substantially more uncertain about the model prediction for such an image than for one which is almost identical to a training image except that it has a different background color, even if this change of background color results in a huge Euclidean distance of the pixel vectors. \hyperref[itm:Axioms:irregular]{D4} requires that a more useful metric is learned instead.
	 }
	 \end{enumerate}
\hyperref[itm:Axioms:irregular]{D4} is not required for any application. However, specifically in \emph{deep} learning applications, where it is a priori not clear which features are important, \hyperref[itm:Axioms:irregular]{D4} is particularly desirable.%
\begin{enumerate}[resume,label=D\arabic*,align=left, leftmargin=*,topsep=0pt]
	\item[\mylabel{itm:Axioms:UncertaintyDecreasesWithMoreTrainingPoints}{D5 (Vanishing)}]
	\textit{As the number~$\ntr$ of training points (with $\xtr_i\overset{\text{i.i.d}}{\sim}\PP_X$) tends to infinity, model uncertainty vanishes for each $x$ in the support of the input data distribution~$\PP_X$, i.e., $\lim_{\ntr\to\infty}\sigmodelhat(x) = 0$ for a fixed $c\ge0$. Thus, for a fixed $c$, $\lim_{\ntr\to\infty} |\uUB_c(x)-\lUB_c(x)|=0$.
	}
\end{enumerate}
In \Cref{appendix:Desiderata}, we discuss all desiderata in more detail (see \Cref{subsec:A note on Desideratum D4} for a visualization of \hyperref[itm:Axioms:irregular]{D4}).
\begin{figure}[t!]
        \vskip 0.1in
        \begin{center}
        \centerline{
        \resizebox{1\columnwidth}{!}{
                \begin{tikzpicture}
                [cnode/.style={draw=black,fill=#1,minimum width=3mm,circle}]
                \node[cnode=gray,label=180:$\mathlarger{\mathlarger{\mathlarger{x \in \X}}}$] (x1) at (0.5,0) {};
                \node[cnode=gray] (x2+3) at (3,3) {};
                \node at (3,2.2) {$\mathlarger{\mathlarger{\mathlarger{\vdots}}}$};
                \node[cnode=gray] (x2+1) at (3,1) {};
                \node[cnode=gray] (x2-3) at (3,-3) {};
                \node at (3,-1.7) {$\mathlarger{\mathlarger{\mathlarger{\vdots}}}$};
                \node[cnode=gray] (x2-1) at (3,-1) {};
                \draw (x1) -- (x2+3);
                \draw (x1) -- (x2+1);
                \draw (x1) -- (x2-1);
                \draw (x1) -- (x2-3);
                \node[cnode=gray] (x3+3) at (6,3) {};
                \node at (6,2.2) {$\mathlarger{\mathlarger{\mathlarger{\vdots}}}$};
                \node[cnode=gray] (x3+1) at (6,1) {};
                \node[cnode=gray] (x3-3) at (6,-3) {};
                \node at (6,-1.7) {$\mathlarger{\mathlarger{\mathlarger{\vdots}}}$};
                \node[cnode=gray] (x3-1) at (6,-1) {};
                \foreach \y in {1,3}
                {   \foreach \z in {1,3}
                        {   \draw (x2+\z) -- (x3+\y);
                        }
                }
                \foreach \y in {1,3}
                {   \foreach \z in {1,3}
                        {   \draw (x2-\z) -- (x3-\y);
                        }
                }
                \node at (7.5,+3) {$\mathlarger{\mathlarger{\mathlarger{\ldots}}}$};
                \node at (7.5,+2) {$\mathlarger{\mathlarger{\mathlarger{\ldots}}}$};
                \node at (7.5,+1) {$\mathlarger{\mathlarger{\mathlarger{\ldots}}}$};
                \node at (7.5,-1) {$\mathlarger{\mathlarger{\mathlarger{\ldots}}}$};
                \node at (7.5,-2) {$\mathlarger{\mathlarger{\mathlarger{\ldots}}}$};
                \node at (7.5,-3) {$\mathlarger{\mathlarger{\mathlarger{\ldots}}}$};
                \node[cnode=gray] (x4+3) at (9,3) {};
                \node at (9,2.2) {$\mathlarger{\mathlarger{\mathlarger{\vdots}}}$};
                \node[cnode=gray] (x4+1) at (9,1) {};
                \node[cnode=gray] (x4-3) at (9,-3) {};
                \node at (9,-1.7) {$\mathlarger{\mathlarger{\mathlarger{\vdots}}}$};
                \node[cnode=gray] (x4-1) at (9,-1) {};
                \node[cnode=gray,label=360:$\mathlarger{\mathlarger{\mathlarger{\fhat(x)\in \Y}}}$] (x5+2) at (11.5,2) {};
                \node[cnode=gray,label=360:$\mathlarger{\mathlarger{\mathlarger{\sigmodelhatraw(x)\in \Rpz}}}$] (x5-2) at (11.5,-2) {};
                \draw (x4+3)--(x5+2);
                \draw (x4+1)--(x5+2);
                \draw[dashed,-Latex,arrows={[scale=1.15]}] (x4+3)--(11.52,-1.8);
                \draw[dashed,-Latex,arrows={[scale=1.15]}] (x4+1)--(11.4,-1.86);
                \draw (x4-1)--(x5-2);
                \draw (x4-3)--(x5-2);
                \draw [dotted, line width=0.4mm] (2.5,0.5) rectangle (9.5,+3.5);
                \draw [dotted, line width=0.4mm] (2.5,-3.5) rectangle (9.5,-0.5);
                \node at (1,+2.5) {$\mathlarger{\mathlarger{\mathlarger{\fhat}}}${\Large-network}};
                \node at (1,-2.5) {$\mathlarger{\mathlarger{\mathlarger{\sigmodelhatraw}}}${\Large-network}};
                \end{tikzpicture}
        }
        }
        \caption{NOMU's network architecture}
        \label{fig:nn_tik}
        \end{center}
        \vskip -0.2in
\end{figure}

\subsection{The Network Architecture}\label{subsec:The Network Architecture}
For NOMU, we construct a network $\NN_\theta$ with two outputs: \emph{the model prediction} $\fhat$ (e.g., mean prediction) and a \emph{raw} model uncertainty prediction $\sigmodelhatraw$. Formally: $\NN_\theta \colon\X \to  \Y\times \Rpz$, with $x \mapsto \NN_\theta(x):=(\fhat(x),\sigmodelhatraw(x))$. NOMU's architecture consists of two almost separate sub-networks: the $\fhat$-network and the $\sigmodelhatraw$-network (see \Cref{fig:nn_tik}). For each sub-network, any network architecture can be used (e.g., feed-forward NNs, CNNs). This makes NOMU highly modular and we can plug in any previously trained NN for $\fhat$, or we can train $\fhat$  simultaneously with the $\sigmodelhatraw$-network. The $\sigmodelhatraw$-network learns the raw model uncertainty and is connected with the $\fhat$-network through the last hidden layer (dashed lines in Figure \ref{fig:nn_tik}). This connection enables $\sigmodelhatraw$ to re-use features that are important for the model prediction~$\fhat$, implementing Desideratum \ref{itm:Axioms:irregular}.\footnote{To prevent that $\sigmodelhatraw$ impacts $\fhat$, the dashed lines should only make forward passes when trained.}

\begin{remark} 
NOMU's network architecture can be modified to realize \ref{itm:Axioms:irregular} in many different ways. For example, if low-level features were important for predicting the model uncertainty, one could additionally add connections from earlier hidden layers of the $\fhat$-network to layers of the $\sigmodelhatraw$-network. Furthermore, one can strengthen \ref{itm:Axioms:irregular} by increasing the regularization of the $\sigmodelhatraw$-network (see \Cref{subsec:A note on Desideratum D4}).
\end{remark}

After training $\NN_\theta$, we apply the readout map $\readout(z)=\sigmax(1-\exp(-\frac{\max(0,z)+\sigmin}{\sigmax})),\, \sigmin\ge0, \sigmax>0$ to the \emph{raw} model uncertainty output $\sigmodelhatraw$ to obtain NOMU's \emph{model uncertainty prediction}
\begin{align}\label{eq:modelUncertaintyPrediction}
\sigmodelhat(x):=\readout(\sigmodelhatraw(x)),\, \forall x\in \X.
\end{align}
The readout map $\readout$ monotonically interpolates between a minimal $\approx\sigmin$ and a maximal $\approx\sigmax$ model uncertainty (see \Cref{fig:app:readout_map} in \Cref{sec:readout_map} for a visualization). Here, $\sigmin$ is used for numerical stability, and $\sigmax$ defines the maximal model uncertainty far away from input training points (similarly to the prior variance for RBF-GPs). 
With NOMU's model prediction $\fhat$, its model uncertainty prediction $\sigmodelhat$ defined in \eqref{eq:modelUncertaintyPrediction}, and given a calibration parameter\footnote{Like all other methods, NOMU outputs \emph{relative} UBs that should be calibrated, e.g., via a parameter $c\ge0$. See also \citet{pmlr-v80-kuleshov18a} for a non-linear calibration method.%
} $c\in \Rpz$, we can now define for each $x\in\X$  NOMU's UBs as
\begin{align}\label{eq:UBs}
&\left(\lUB_c(x),\uUB_c(x)\right):=\left(\fhat(x)\mp c~\sigmodelhat(x)\right). %
\end{align}
It is straightforward to construct a \emph{single} NN that directly outputs the upper/lower UB, by extending the architecture shown in Figure~\ref{fig:nn_tik}: we monotonically transform and scale the output $\sigmodelhatraw(x)$ and then add/subtract this to/from the other output $\fhat(x)$. It is also straightforward to compute NOMU's UBs for any \emph{given, previously trained NN}, by attaching the $\sigmodelhatraw$-network to the trained NN, and only training the $\sigmodelhatraw$-network on the same training points as the original NN.
\begin{remark}\label{rem:different_readout}
The readout map $\readout$ can be modified depending on the subsequent use of the estimated UBs. For example, for BO over discrete domains (e.g., $\X=\{0,1\}^d$) \citep{baptista2018bayesian}, we propose the linearized readout map $\readout(z)=\sigmin+\max(0,z-\sigmin)-\max(0,z-\sigmax)$. With this $\readout$ and ReLU activations, one can encode NOMU's UBs as a mixed integer program (MIP) \citep{weissteiner2020deep}.
This enables optimizing the upper UB as acquisition function without the need for further approximation via ensemble distillations \citep{malinin2019ensemble}.
\end{remark}
\subsection{The Loss Function}\label{subsec:modeluncertaintyloss}
We now introduce the loss function $L^\hp$ we use for training NOMU's architecture. Let $\X$ be such that $0<\lambda_d(\X)<\infty$, where $\lambda_d$ denotes the $d$-dimensional Lebesgue measure. We train $\NN_{\theta}$ with loss $L^\hp$ and L2-regularization parameter $\lambda>0$, i.e.,  %
minimizing $\Lmu(\NN_{\theta})+\lambda\twonorm[{\theta}]^2$ via a gradient descent-based algorithm.
\begin{definition}[NOMU Loss]\label{def:Model Uncertainty Loss}
Let  $\hp:=(\musqr,\muexp,\cexp)\in \Rpz^3$ denote a tuple of hyperparameters. Given a training set $\Dtr$, the loss function $L^\hp$ is defined as
\begin{align}\label{eq:lmu2}
  L^\hp(\NN_\theta):=&\underbrace{\sum_{i=1}^{\ntr}(\fhat(\xtr_i
  )-\ytr_i)^2}_{\terma{}}+ \,\musqr\cdot\underbrace{\sum_{i=1}^{\ntr}\left(\sigmodelhatraw(\xtr_i)\right)^2}_{\termb{}} \notag\\ 
   &+ \muexp\cdot\underbrace{\frac{1}{\lambda_d(\X)}\int_{\X}e^{-\cexp\cdot \sigmodelhatraw(x)}\,dx}_{\termc{}}.
\end{align}
\end{definition}
In the following, we explain how the three terms of $L^\hp$ promote the desiderata we introduced in \Cref{subsubsec:Desiderata}. Note that the behaviour of $\sigmodelhatraw$ directly translates to that of $\sigmodelhat$.
\begin{itemize}[leftmargin=*,topsep=0pt,partopsep=0pt, parsep=0pt]
\item Term~\terma{} solves the regression task, i.e., learning a smooth function $\fhat$ given $\Dtr$. If $\fhat$ is given as a pre-trained NN, then this term can be omitted.
\item Term~\termb{} implements \ref{itm:Axioms:ZeroUncertaintyAtData} and \ref{itm:Axioms:UncertaintyDecreasesWithMoreTrainingPoints} (i.e.,  $\sigmodelhatraw(\xtr_i)\approx 0$). The hyperparameter $\musqr$ controls the amount of uncertainty at the training points.\footnote{In the noiseless case, in theory $\frac{\musqr}{\lambda}=\infty$; we set $\frac{\musqr}{\lambda}=10^7$. For small non-zero data noise, setting $\frac{\musqr}{\lambda}\ll\infty$ captures \textit{data noise induced model uncertainty} $\sigmodel(\xtr)>0$ (\Cref{sec:desiderata:InSample}).} The larger $\musqr$, the narrower the UBs at training points.
\item Term~\termc{} has two purposes. First, it implements \ref{itm:Axioms:trivial} (i.e., $\sigmodelhatraw\ge0$). Second, it pushes $\sigmodelhatraw$ towards infinity across the whole input space $\X$. However, due to the counteracting force of \termb{} as well as regularization, $\sigmodelhatraw$ increases continuously as you move away from the training data. The interplay of \termb{}, \termc{}, and regularization thus promotes \ref{itm:Axioms:largeDistantLargeUncertainty}. The hyperparameters $\muexp$ and $\cexp$ control the size and shape of the UBs. Concretely, the larger $\muexp$, the wider the UBs; the larger $\cexp$, the narrower the UBs at points $x$ with large $\sigmodelhat(x)$ and the wider the UBs at points $x$ with small $\sigmodelhat(x)$.
\end{itemize}
In \Cref{subsec:appendix:Qualitative Sensitivity Analysis}, we provide detailed visualizations on how the loss hyperparameters $\musqr$, $\muexp$, and $\cexp$ shape NOMU's model uncertainty estimate.
In the implementation of $L^\hp$, we approximate \termc{} via MC-integration using additional, \emph{artificial input points} ${\Daug:=\left\{x_i\right\}_{i=1}^l\stackrel{i.i.d}{\sim}\textrm{Unif}(\X)}$ by
$\Scale[1.2]{\frac{1}{l}}\cdot\hspace{-0.1cm}\sum_{x\in \Daug}\hspace{-0.05cm}e^{-\cexp\cdot \sigmodelhat(x)}$.
\begin{remark}
In \termc{}, instead of the Lebesgue-measure, one can also use a different measure $\nu$, i.e., $\frac{1}{\nu(\X)}\int_{\X}e^{-\cexp\cdot \sigmodelhatraw(x)}\,d\nu(x)$. This can be relevant in high dimensions, where meaningful data points often lie close to a lower-dimensional manifold \citep{cayton2005algorithms}; $\nu$ can then be chosen to concentrate on that region. In practice, this can be implemented by sampling from a large unlabeled data set $\Daug$ representing this region or learning the measure $\nu$ using GANs \citep{goodfellow2014gen}.
\end{remark}

\paragraph{Theory} In \Cref{appendix:Desiderata}, we prove that NOMU fulfills \hyperref[itm:Axioms:trivial]{D1}, \hyperref[itm:Axioms:ZeroUncertaintyAtData]{D2} and \hyperref[itm:Axioms:UncertaintyDecreasesWithMoreTrainingPoints]{D5} (\Cref{prop:NOMUD1,prop:NOMUD2,prop:NOMUD5}),
and discuss how NOMU fulfills \hyperref[itm:Axioms:largeDistantLargeUncertainty]{D3} and \hyperref[itm:Axioms:irregular]{D4}.
In \Cref{subsec:connectingNOMUtoPointwiseUncertaintyBounds}, we show that, under certain assumptions, NOMU's UBs can be interpreted as \textit{pointwise worst-case} UBs $\uUBp(x):=\sup_{f\in \HC_{\Dtr}}f(x)$ within a hypothesis class~$\HC_{\Dtr}$ of data-explaining functions.
In \Cref{subsec:Pointwise Uncertainty Bounds}, we explain how $\lUBp(x)$ and $\uUBp(x)$ estimate posterior CBs of BNNs (with a Gaussian prior on the weights), without performing challenging variational inference. However, while exact posterior CBs of BNNs lose \hyperref[itm:Axioms:irregular]{D4} as their width goes to infinity, NOMU's UBs are capable of retaining \hyperref[itm:Axioms:irregular]{D4} in this limit.  

\section{Experimental Evaluation}\label{sec:ExperimentsMain}
In this section, we experimentally evaluate NOMU's model uncertainty estimate in multiple synthetic and real-world regression settings (\Cref{subsec:Regression}) as well as in high-dimensional Bayesian optimization (\Cref{subsec:BayesianOptimization}).

\paragraph{Benchmarks} We compare NOMU against four benchmarks, each of which gives a model prediction $\fhat$ and a model uncertainty prediction $\sigmodelhat$ (see \Cref{subec:Benchmarks} for formulas). We calculate model-specific UBs at $x\in\X$ as $(\fhat(x)\mp c~\sigmodelhat(x))$ with calibration parameter $c \in \Rpz$ and use them to evaluate all methods. We consider three algorithms that are specialized to model uncertainty for NNs: (i) MC dropout (MCDO) (ii) deep ensembles (DE) and (iii) hyper deep ensembles (HDE) and a non-NN-based benchmark: (iv) Gaussian process (GP) with RBF kernel.

\subsection{Regression}
\label{subsec:Regression}
To develop intuition, we first study the \emph{model} UBs of all methods on synthetic test functions with 1D--2D scarce input training points without data noise (\Cref{subsubsec:ToyRegression}). We then propose a novel generative test-bed and evaluate NOMU within this setting (\Cref{subsubsec:GenerativeTestbed}). Next, we analyze NOMU on a real-world time series (\Cref{subsubsec:SolarIrradiance}). Finally, we evaluate NOMU on the real-world UCI data sets (\Cref{subsubsec:UCI}).
\begin{figure}[t!]
    \begin{center}
    \centerline{\includegraphics[width=1\columnwidth]{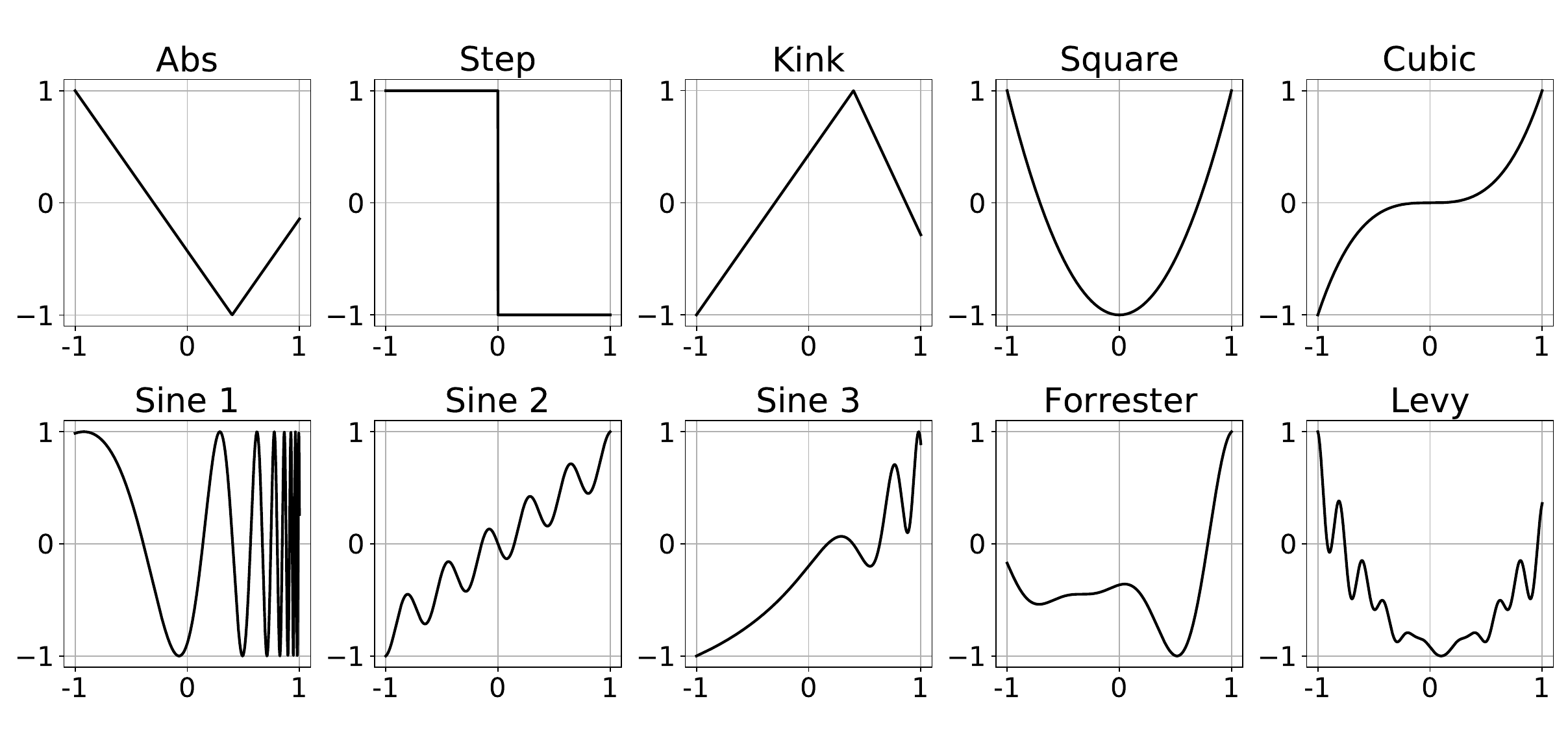}}
    \caption{1D synthetic test functions}
    \label{fig:1dfuns}
    \end{center}
    \vskip -0.2in
\end{figure}

\subsubsection{Toy Regression}\label{subsubsec:ToyRegression}
\paragraph{Setting} We consider ten different
1D functions whose graphs are shown in \Cref{fig:1dfuns}. 
Those include the popular Levy and Forrester function with multiple local optima.\footnote{\label{footnote:sfu}See \href{www.sfu.ca/~ssurjano/optimization.html}{sfu.ca/~ssurjano/optimization.html}.} All functions are transformed to $\X:=[-1,1]=:f(X)$. For each function, we conduct $500$ runs. In each run, we randomly sample eight noiseless input training points from $\X$,
such that the only source of uncertainty is model uncertainty.
For each run, we also generate $100$ \emph{test points} in the same fashion to assess the quality of the UBs.

\paragraph{Metrics} We report the \emph{average negative log (Gaussian) likelihood (NLL)}, minimized over the calibration parameter $c$, which we denote as $\MNLPD$. Following prior work \citep{khosravi2010lower,pmlr-v80-kuleshov18a,pearce2018high}, we further measure the quality of UBs by contrasting their \emph{mean width} $(\MW)$ with their \emph{coverage probability} $(\CP)$. 
Ideally, $\MW$ should be as small as possible, while $\CP$ should be close to $1$. Since $\CP$ is counter-acting $\MW$, we consider ROC-like curves, plotting $\MW$ against $CP$ for a range of calibration parameters $c$, and report the \emph{area under the curve (AUC)} (see \Cref{subsec:performanceMeasures} for details).

\paragraph{Algorithm Setup} For each of the two NOMU sub-networks, we use a feed-forward NN with three fully-connected hidden layers \`a $2^{10}$ nodes, ReLUs, and hyperparameters $\muexp=0.01, \musqr=0.1, \cexp=30$. In practice, the values for $\muexp, \musqr$, and $\cexp$ can be tuned on a validation set. However, for all synthetic experiments (\Cref{subsubsec:ToyRegression} and \Cref{subsubsec:GenerativeTestbed}), we use the same values, which lead to good results across all functions. Moreover, we set $\lambda=10^{-8}$ accounting for zero data-noise, $\sigmin=0.001$ and $\sigmax=2$.
In \Cref{subsec:appendix:Quantitative Sensitivity Analysis}, we provide an extensive sensitivity analysis for the hyperparameters $\muexp$, $\musqr$, $\cexp$, $\sigmin$ and $\sigmax$. This analysis demonstrates NOMU's robustness within a certain range of hyperparameter values.

Furthermore, to enable a fair comparison of all methods, we use generic representatives and do not optimize any of them for any test function. We also choose the architectures of all NN-based methods, such that the overall number of parameters is comparable. We provide all methods with the same prior information (specifically, this entails knowledge of zero data noise), and set the corresponding parameters accordingly. Finally, we set all remaining hyperparameters of the benchmarks to the values proposed in the literature. Details on all configurations are provided in \Cref{subsec:ConfigurationDetailsofBenchmarksRegression}.

\begin{figure}[t!]
    \begin{center}
    \centerline{\includegraphics[width=1\columnwidth]{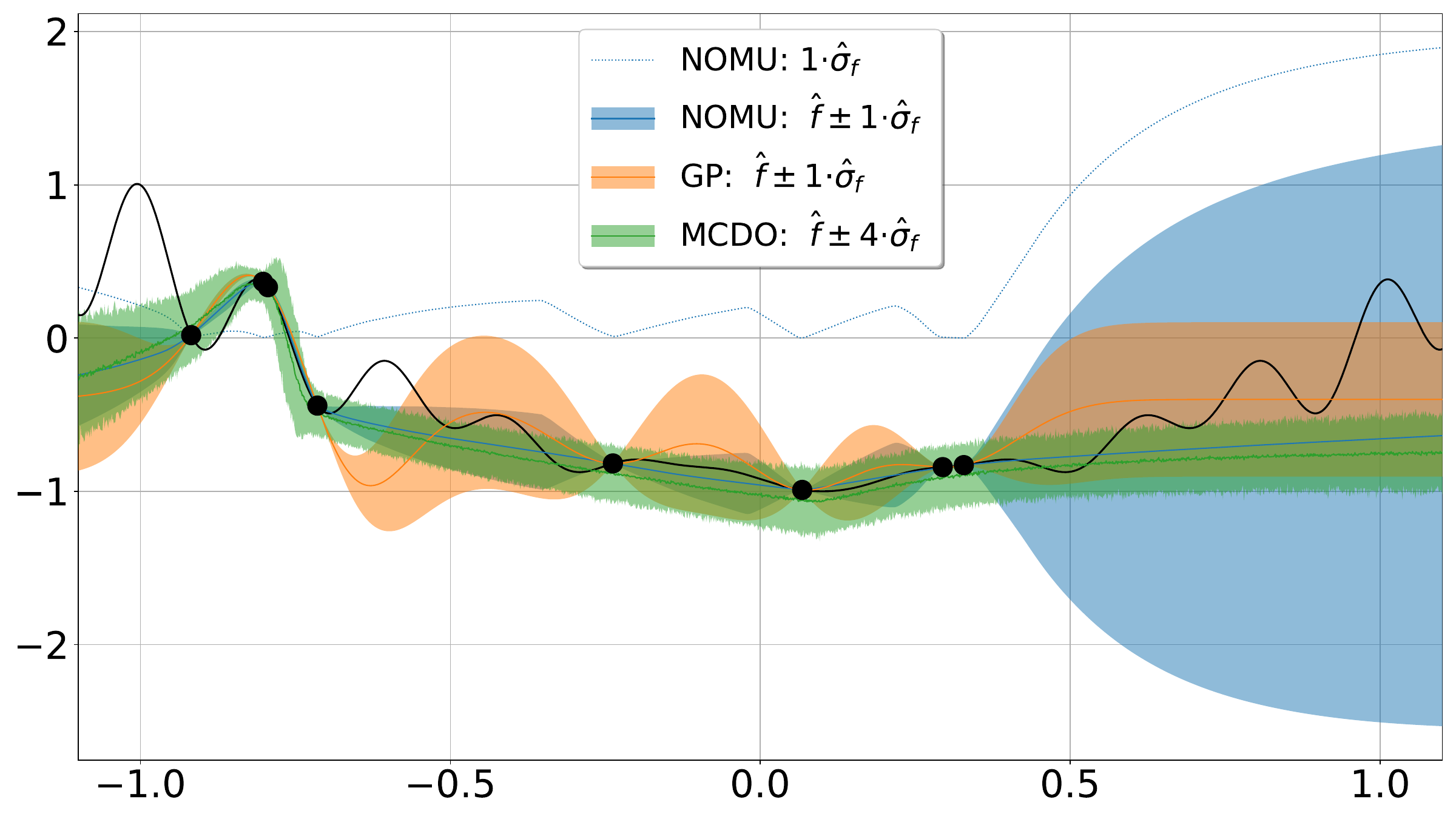}}
	\centerline{\includegraphics[width=1\columnwidth]{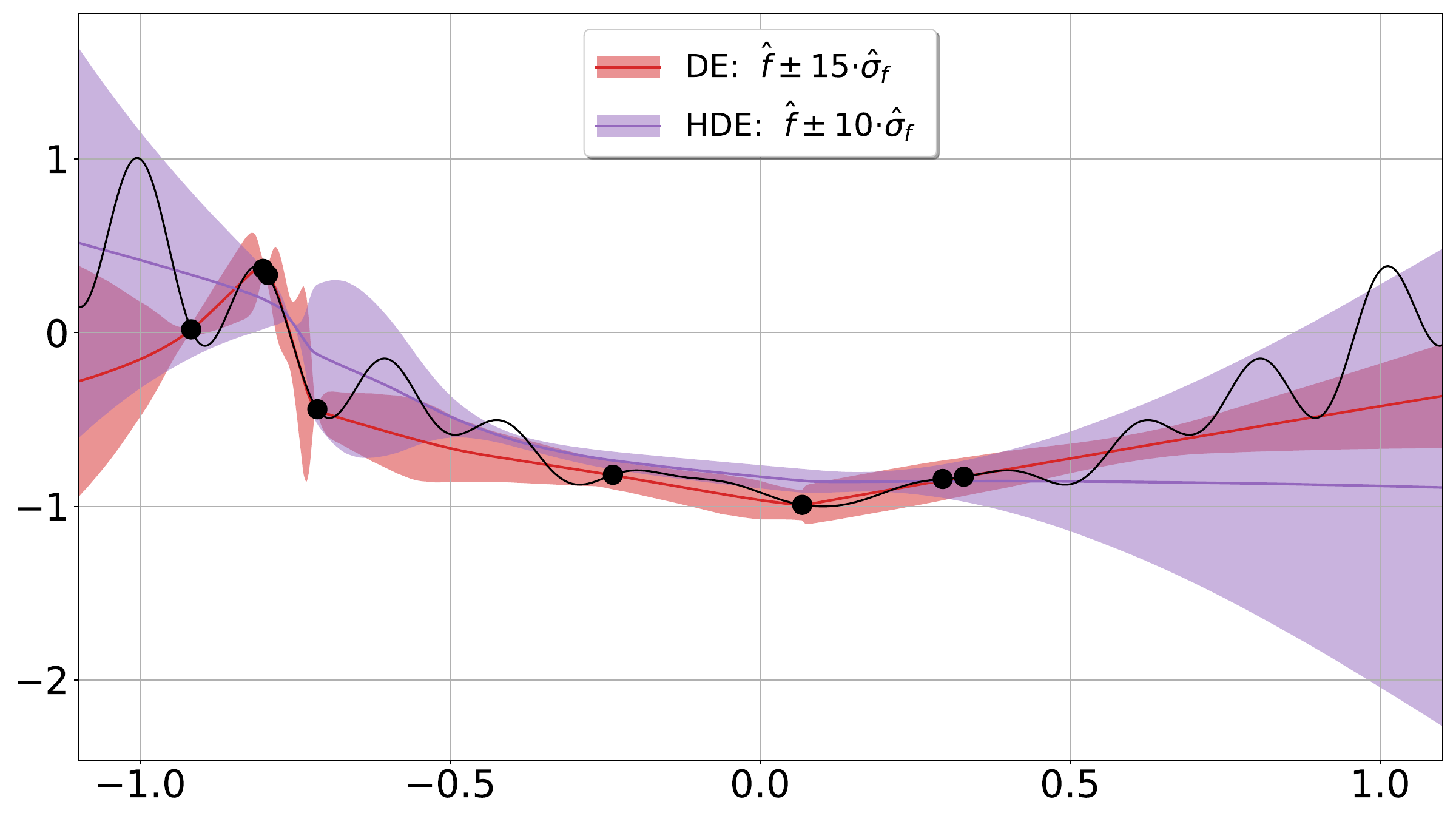}}
	\caption{UBs resulting from NOMU, GP, MCDO, and DE, HDE for the Levy function (solid black line). For NOMU, we also show $\sigmodelhat$ as a dotted blue line. Training points are shown as black dots.}
	\label{fig:1dboundsLevy}
	\end{center}
    \vskip -0.2in
\end{figure}
\paragraph{Results} \Cref{fig:1dboundsLevy} exemplifies our findings, showing typical UBs for the Levy function as obtained in one run. In \Cref{subsubsec:Uncertainty Bounds Detailed Characteristics} we provide further visualisations. We find that MCDO consistently yields tube-like UBs; in particular, its UBs do not narrow at training points, i.e., failing in \ref{itm:Axioms:ZeroUncertaintyAtData} even though MCDO's Bayesian interpretation requires \hyperref[itm:Axioms:ZeroUncertaintyAtData]{D2} to hold (see \Cref{sec:desiderata:InSample}). Moreover, it only fulfills \ref{itm:Axioms:largeDistantLargeUncertainty} to a limited degree. We frequently observe that DE leads to UBs of somewhat arbitrary shapes. This can be seen most prominently in \Cref{fig:1dboundsLevy} around $x\approx-0.75$ and at the edges of its input range, where DE's UBs are very different in width with no clear justification. Thus, also DE is limited in \ref{itm:Axioms:largeDistantLargeUncertainty}. In addition, we sometimes see that also DE's UBs do not narrow sufficiently at training points, i.e., not fulfilling \ref{itm:Axioms:ZeroUncertaintyAtData}.
HDE's UBs are even more random, i.e., predicting large model uncertainty at training points and sometimes zero model uncertainty in gaps between them (e.g., $x\approx-0.75$). %
In contrast, NOMU displays the behaviour it is designed to show. Its UBs nicely tighten at training points and expand in-between (\hyperref[itm:Axioms:trivial]{D1}\crefrangeconjunction\hyperref[itm:Axioms:largeDistantLargeUncertainty]{D3}, for \ref{itm:Axioms:irregular} see \cref{subsec:A note on Desideratum D4}). %
Like NOMU, the GP fulfills \ref{itm:Axioms:ZeroUncertaintyAtData} and \ref{itm:Axioms:largeDistantLargeUncertainty} well, but cannot account for \ref{itm:Axioms:irregular} (a fixed kernel does not depend on the model prediction). 
\Cref{tab:1dsynfunresults} provides the ranks achieved by each algorithm (see \Cref{subsubsec:Detailed Synthetic Functions Results in Regression} for corresponding metrics). We calculate the ranks based on the medians and a 95\% bootstrap CI of \AUC\, and \MNLPD{}. An algorithm loses one rank to each other algorithm that significantly dominates it. Winners are marked in grey. We observe that NOMU is the only algorithm that never comes in third place or worse.
Thus, while some algorithms do particularly well in capturing uncertainties of functions with certain characteristics (e.g., RBF-GPs for polynomials), NOMU is the only algorithm that consistently performs well. HDE's bad performance can be explained by its randomness and the fact that it sometimes predicts zero model uncertainty out-of-sample. For 2D, we provide results and visualizations in \Cref{subsubsec:Detailed Synthetic Functions Results in Regression} and \ref{subsubsec:Uncertainty Bounds Detailed Characteristics} highlighting similar characteristics of all the algorithms as in 1D.
\begin{table}[t!]
    \setlength\tabcolsep{2pt}
    \caption{Ranks (1=best to 5=worst) for \AUC\ and \MNLPD{}.}
    \label{tab:1dsynfunresults}
    \vskip 0.1in
    \begin{center}
    \begin{small}
    \begin{sc}
	\resizebox{1\columnwidth}{!}{
		\begin{tabular}{
				l
				c
				c
				c
				c
				c
				c
				c
				c
				c
				c
			}
			\toprule
			& \multicolumn{2}{c}{\textbf{NOMU}} &  \multicolumn{2}{c}{\textbf{GP}}  & \multicolumn{2}{c}{\textbf{MCDO}}  &  \multicolumn{2}{c}{\textbf{DE}} & \multicolumn{2}{c}{\textbf{HDE}}\\
			\textbf{Function} & \small\AUC & \small\MNLPD & \small\AUC  & \small\MNLPD & \small\AUC & \small\MNLPD & \small\AUC & \small\MNLPD & \small\AUC &\small\MNLPD\\
			\midrule
			Abs & \winc1&\winc1 & 3&\winc1 & 3&4 & \winc1&\winc1&5&5\\
    	    Step & 2&2 & 4&3 & 2&3 & \winc1&\winc1&5&5\\ 
			Kink &\winc1&\winc1 & 3&\winc1 & 4&4 & \winc1&\winc1&5&5\\
			Square & 2&2 & 2&\winc1 & 4&4 &2&3&5&5\\
			Cubic & 2&2 &\winc1&\winc1 & 3&4 & 3&3&5&5\\
			Sine 1 & 2&\winc1 & 2&\winc1 & \winc1&\winc1 & 2&\winc1&5&5\\
			Sine 2 & 2&2 & 3&\winc1 & \winc1&2 & 3&3&5&5\\
			Sine 3 & \winc1&\winc1 & 4&\winc1 & 3&4 & \winc1&\winc1&5&5\\
			Forrester & \winc1&2 & \winc1&\winc1 & 3&4  & 3&3&5&5\\
			Levy & \winc1&\winc1 & 4&3 & \winc1&4 & 3&\winc1&5&5\\
			\bottomrule 
		\end{tabular}
}
\end{sc}
\end{small}
\end{center}
\vskip -0.1in
\end{table}
\subsubsection{Generative Test-Bed}\label{subsubsec:GenerativeTestbed}
\paragraph{Setting} Instead of only relying on a \emph{limited} number of data-sets, we also evaluate all algorithms on a generative test-bed, which provides an unlimited number of test-functions.
The importance of using a test-bed (to avoid over-fitting on specific data-sets) has also been highlighted in a recent work by \citet{osband2021epistemic}.
For our test-bed, we generate $200$ different data-sets by randomly sampling 200 different test-functions from a BNN with i.i.d centered Gaussian weights and three fully-connected hidden layers with nodes $[2^{10},2^{11},2^{10}]$ and ReLU activations.
From each of these test-functions we uniformly at random sample $\ntr=8\cdot d$ input training points and $100\cdot d$ test data points, where $d$ refers to the input dimension. We train all algorithms on these training sets and determine the \NLPD\ on the corresponding test sets averaged over the 200 test-functions. This metric converges to the Kullback-Leibler divergence to the exact BNN-posterior (see \Cref{thm:appendix_dkl_approximation}). We calibrate by choosing per dimension an optimal value of $c$ in terms of average \NLPD, which does not depend on the test-function.
\begin{table}[t!]
    \setlength\tabcolsep{2pt}
    \caption{Average \NLPD\ {\scriptsize(without const. $\ln(2\pi)/2$)} and a 95\% CI over 200 BNN samples. Winners are marked in grey.}
    \label{tab:bnn_uniform}
	\vskip 0.1in
    \begin{center}
    \begin{small}
    \begin{sc}
	\resizebox{1\columnwidth}{!}{
		\begin{tabular}{
				l
				c
				c
				c
				c
				c
				c
				c
				c
				c
				c
			}
			\toprule
			\textbf{Function} & \multicolumn{1}{c}{\textbf{NOMU}} &  \multicolumn{1}{c}{\textbf{GP}}  & \multicolumn{1}{c}{\textbf{MCDO}}  &  \multicolumn{1}{c}{\textbf{DE}} & \multicolumn{1}{c}{\textbf{HDE}}\\
			\midrule
		    BNN1D & \winc-1.65$\pm$\scriptsize0.10 & -1.08$\pm$\scriptsize0.22 & -0.34$\pm$\scriptsize0.23 &  -0.38$\pm$\scriptsize0.36 & \hphantom{-}8.47$\pm$\scriptsize1.00\\
		    BNN2D &  \winc-1.16$\pm$\scriptsize0.05& -0.52$\pm$\scriptsize0.11 &-0.33$\pm$\scriptsize0.13 & -0.77$\pm$\scriptsize0.07 & \hphantom{-}9.11$\pm$\scriptsize0.39 \\
		    BNN5D & \winc-0.37$\pm$\scriptsize0.02 & -0.33$\pm$\scriptsize0.02 & -0.05$\pm$\scriptsize0.04 & -0.13$\pm$\scriptsize0.03 & \hphantom{-}8.41$\pm$\scriptsize1.00 \\
			\bottomrule 
		\end{tabular}
}
\end{sc}
\end{small}
\end{center}
\vskip -0.1in
\end{table}
\paragraph{Results} In \Cref{tab:bnn_uniform}, we provide the results for input dimensions 1, 2 and 5.
We see that NOMU outperforms all other algorithms including MCDO, which is a variational inference method to approximate this posterior \citep{gal2016dropout}.
See \Cref{sec:Appendix:GenerativeTestBed} for a detailed explanation of the experiment setting and further results in modified settings including a discussion of higher dimensional settings.

\subsubsection{Solar Irradiance Time Series}\label{subsubsec:SolarIrradiance}
\paragraph{Setting} Although the current version of NOMU is specifically designed for scarce settings without data noise, we are also interested to see how well it performs in settings where these assumptions are not satisfied. To this end, we now study a setting with many training points and small non-zero data noise. This allows us to analyze how well NOMU captures \ref{itm:Axioms:UncertaintyDecreasesWithMoreTrainingPoints}.
We consider the popular task of interpolating the solar irradiance data \citep{doi:10.1029/2009GL040142} also studied in \citep{gal2016dropout}. We scale the data to $X=[-1,1]$ and split it into 194 training and 197 test points. As in \citep{gal2016dropout}, we choose five intervals to contain only test points. Since the true function is likely of high frequency, we set $\lambda=10^{-19}$ for NOMU and the benchmarks' regularization accordingly. To account for small non-zero data noise, we set $\muexp=0.05$, $\sigmin=0.01$ and use otherwise the same hyperparameters as in \Cref{subsubsec:ToyRegression}.
\begin{figure}[b!]
    \begin{center}
    \centerline{\includegraphics[width=1\columnwidth, trim= 10 180 0 200, clip]{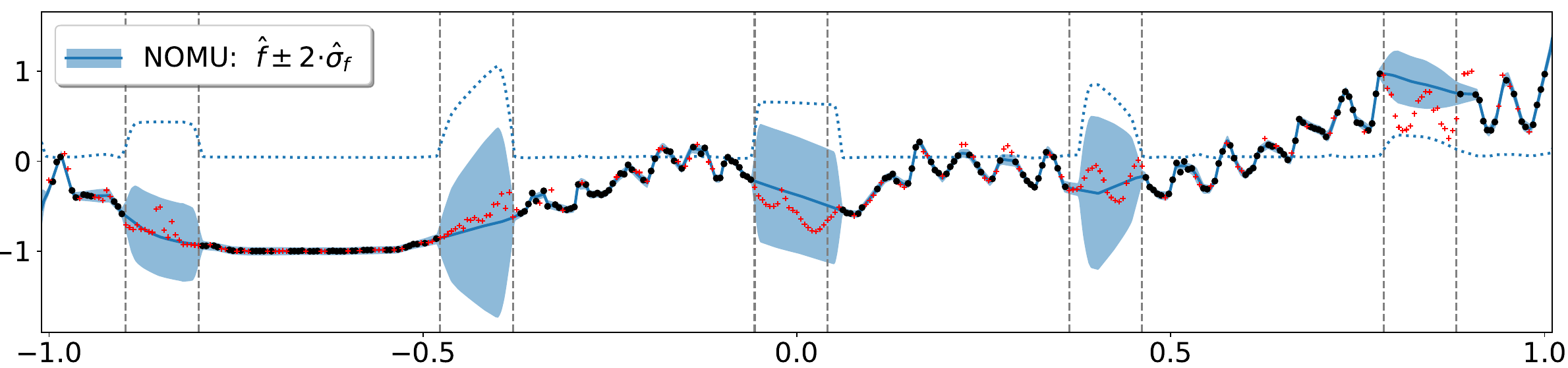}}
    \caption{NOMU's model prediction (solid), model uncertainty (dotted) and UBs (shaded area) on the solar irradiance data. Training and test points are shown as black dots and red crosses.}
    \label{fig:irradianceNOMU}
    \end{center}
    \vskip -0.2in
\end{figure}

\paragraph{Results} \Cref{fig:irradianceNOMU} visualizes NOMU's UBs. We see that NOMU manages to fit the training data well while capturing model uncertainty between input training points. In particular, large gaps \emph{between} input training points are successfully modeled as regions of high model uncertainty (\ref{itm:Axioms:largeDistantLargeUncertainty}). Moreover, in regions highly populated by input training points, NOMU's model uncertainty vanishes as required by \ref{itm:Axioms:UncertaintyDecreasesWithMoreTrainingPoints}. Plots for the other algorithms are provided in \Cref{subsubsec:Detailed Real-Data Time Series Results}, where we observe similar patterns as in \Cref{subsubsec:ToyRegression}.

\subsubsection{UCI Data Sets}\label{subsubsec:UCI}
\paragraph{Setting} Recall that NOMU is specifically tailored to noiseless settings with scarce input training data. However, even the current version of NOMU, which does not explicitly model data noise, already performs on par with existing benchmarks on real-world regression tasks \textit{with} data noise. To demonstrate this, we test its performance on (a) the UCI data sets proposed in \citet{hernandez2015probabilistic}, a common benchmark for uncertainty quantification in noisy, real-world regression, and (b) the UCI gap data set extension proposed in \citet{foong2019inbetween}. We consider exactly the same experiment setup as proposed in these works, with a $70/20/10$-train-validation-test split, equip NOMU with a shallow architecture of $50$ hidden nodes, and train it for $400$ epochs. Validation data are used to calibrate the constant $c$ on NLL. See \Cref{subsec:DetailsUCIExperiment} for details on NOMU's setup.

\paragraph{Results} \Cref{tab:uci_nll_detailed} reports NLLs on test data, averaged across 20 splits, compared to the current state-of-the-art. We report the \NLPD\ for MCDO \citep{gal2016dropout} and DE \citep{lakshminarayanan2017simple} from the original papers. Moreover, we reprint the best results (for comparable network sizes) of neural linear models (NLMs) with and without hyperparameter optimization (HPO) from \citep{ober2019benchmarking} (NLM-HPO, NLM), linearized laplace (LL) \cite{foong2019inbetween} and a recent strong MCDO baseline (MCDO2) from \citep{mukhoti2018importance}.
It is surprising that, even though the current design of NOMU does not yet explicitly incorporate data noise, it already performs comparably to state-of-the-art results. We obtain similar results for the UCI gap data. See \Cref{subsec:DetailsUCIExperiment} for more details on this experiment.

\begin{table}[t!]
    \caption{Average \NLPD\ and a $95\%$ normal-CI over 20 runs for \textsc{UCI} data sets. Winners are marked in grey.}
    \label{tab:uci_nll_detailed}
    \robustify\bfseries
    \setlength\tabcolsep{1pt}
    \vskip 0.1in
    \begin{center}
    \begin{small}
    \begin{sc}
    \resizebox{1\columnwidth}{!}{
    \begin{tabular}{lcccccccc}
    \toprule
    \textbf{Dataset}& NOMU & DE & MCDO &  MCDO2 & LL & NLM-HPO & NLM\\
    \midrule
    Boston & 2.68 $\pm$\scriptsize 0.11 & \winc2.41 $\pm$\scriptsize 0.49 & \winc2.46 $\pm$\scriptsize 0.11 &   \winc2.40 $\pm$\scriptsize 0.07 & 2.57 $\pm$\scriptsize 0.09 & \winc2.58 $\pm$\scriptsize 0.17 &     3.63 $\pm$\scriptsize 0.39\\
    Concrete & \winc3.05 $\pm$\scriptsize 0.06 & \winc3.06 $\pm$\scriptsize 0.35 & 3.04 $\pm$\scriptsize0.03 & \winc2.97 $\pm$\scriptsize 0.03 & \winc3.05 $\pm$\scriptsize 0.07 & 3.11 $\pm$\scriptsize 0.09 & 3.12 $\pm$\scriptsize 0.09 \\
    Energy & \winc0.77 $\pm$\scriptsize 0.06 & 1.38 $\pm$\scriptsize 0.43 & 1.99 $\pm$\scriptsize 0.03 & 1.72 $\pm$\scriptsize 0.01 & 0.82 $\pm$\scriptsize 0.05 & \winc0.69 $\pm$\scriptsize 0.05 & \winc0.69 $\pm$\scriptsize 0.05 \\
    Kin8nm & -1.08 $\pm$\scriptsize 0.01 & \winc-1.20 $\pm$\scriptsize 0.03 & -0.95 $\pm$\scriptsize 0.01 & -0.97 $\pm$\scriptsize 0.00 & \winc-1.23 $\pm$\scriptsize 0.01 & -1.12 $\pm$\scriptsize 0.01 & -1.13 $\pm$\scriptsize 0.01 \\
    Naval & -5.63 $\pm$\scriptsize 0.39 & -5.63 $\pm$\scriptsize 0.09 & -3.80 $\pm$\scriptsize 0.01 & -3.91 $\pm$\scriptsize 0.01 & -6.40 $\pm$\scriptsize 0.11 & \winc-7.36 $\pm$\scriptsize 0.15 & \winc-7.35 $\pm$\scriptsize 0.01 \\
    CCPP & \winc2.79 $\pm$\scriptsize 0.01 & \winc2.79 $\pm$\scriptsize 0.07 & \winc2.80 $\pm$\scriptsize 0.01 & \winc2.79 $\pm$\scriptsize 0.01 & 2.83 $\pm$\scriptsize 0.01 & \winc2.79 $\pm$\scriptsize 0.01 & \winc2.79 $\pm$\scriptsize 0.01 \\
    Protein & \winc2.79 $\pm$\scriptsize 0.01 & 2.83 $\pm$\scriptsize 0.03 & 2.89 $\pm$\scriptsize 0.00 & 2.87 $\pm$\scriptsize 0.00 & 2.89 $\pm$\scriptsize 0.00 & \winc2.78 $\pm$\scriptsize 0.01 & 2.81 $\pm$\scriptsize 0.00 \\
    Wine & 1.08 $\pm$\scriptsize 0.04 & \winc0.94 $\pm$\scriptsize 0.23 & \winc0.93 $\pm$\scriptsize 0.01 & \winc0.92 $\pm$\scriptsize 0.01 & 0.97 $\pm$\scriptsize 0.03 & 0.96 $\pm$\scriptsize 0.01 & 1.48 $\pm$\scriptsize 0.09 \\
    Yacht & \winc1.38 $\pm$\scriptsize 0.28 & \winc1.18 $\pm$\scriptsize 0.41 & 1.55 $\pm$\scriptsize 0.05 & 1.38 $\pm$\scriptsize 0.01 & \winc1.01 $\pm$\scriptsize 0.09 & \winc1.17 $\pm$\scriptsize 0.13 & \winc1.13 $\pm$\scriptsize 0.09 \\
    \bottomrule
    \end{tabular}
    }
    \end{sc}
    \end{small}
    \end{center}
    \vskip -0.1in
\end{table}

\subsection{Bayesian Optimization}
\label{subsec:BayesianOptimization}
In this section, we assess the performance of NOMU in high-dimensional noiseless \emph{Bayesian optimization (BO)}.

\paragraph{Setting} In BO, the goal is to maximize an unknown expensive-to-evaluate function, given a budget of function queries. We use a set of test functions with different characteristics from the same library as before, but now in $5$ to $20$ dimensions $d$, transformed to $X=[-1,1]^d,\ f(X)=[ -1,1]$.\footnote{These functions are designed for minimization. We multiply them by $-1$ and equivalently maximize instead.} Additionally, we again use Gaussian BNNs to create a generative test-bed consisting of a large variety of test functions (see \Cref{subsubsec:GenerativeTestbed}).
For each test function, we randomly sample $8$ initial points~$\left(x_i,f(x_i)\right)$ and let each algorithm choose $64$ further function evaluations (one by one) using its upper UB as acquisition function. %
This corresponds to a setting where one can only afford $72$ expensive function evaluations in total.
We provide details regarding the selected hyperparameters for each algorithm in \Cref{subsubsec:Hyperparameters}. We measure the performance of each algorithm based on its \emph{final regret} $|\max_{x\in X} f(x)-\max_{i\in\fromto{72}} f(x_i)|/|\max_{x\in X} f(x)|$. 

For each algorithm, the UBs must be calibrated by choosing appropriate values of $c$. We do so in the following straightforward way: First, after observing the $8$ random initial points, we 
determine those two values of $c$ for which the resulting mean width (MW) of the UBs is $0.05$ and $0.5$, respectively (\textsc{MW scaling}).\footnote{We fix MW instead of $c$, since the scales of the algorithms vary by orders of magnitude.} We perform one BO run for both resulting initial values of $c$. Additionally, if in a BO run, an algorithm were to choose a point~$x_{i'}$ very close to an already observed point $x_i$, we dynamically increase $c$ to make it select a different one instead (\textsc{Dynamic c}; see \Cref{subsubsec:Calibration} for details). A value of $0.05$ in \textsc{MW scaling} corresponds to small model uncertainty, such that exploration is mainly due to \textsc{Dynamic c}. Smaller values than $0.05$ thus lead to similar outcomes. In contrast, a value of $0.5$ corresponds to large model uncertainties, such that \textsc{Dynamic c} is rarely used. Only for the \enquote{plain GP (pGP)} we use neither \textsc{MW scaling} nor \textsc{Dynamic c}, as pGP uses its default calibration ($c$ is determined by the built-in hyperparameter optimization in every step). However, a comparison of GP and pGP suggests that \textsc{MW scaling} and \textsc{Dynamic c} surpass the built-in calibration (see \Cref{tab:BOresults}). As a baseline, we also report random search (RAND).

\sisetup{output-exponent-marker=\ensuremath{\mathrm{e}}}
\begin{table}[t!]
    \setlength\tabcolsep{2pt}
    \caption{BO results: average final regrets per dimension and ranks for each individual function (1=best to 7=worst).}
    \label{tab:BOresults}
    \vskip 0.1in
    \begin{center}
    \begin{small}
    \begin{sc}
    \resizebox{1\columnwidth}{!}{
    \begin{tabular}{
    		l
    		c
    		c
    		c
    		c
    		c
    		c
    		c
    	}
    		\toprule
    		\textbf{Function} & \textbf{NOMU}  & \textbf{GP}  &  \textbf{MCDO} & \textbf{DE} & \textbf{HDE} & \textbf{pGP}&\textbf{RAND} \\
    		
    		\midrule
    	    Levy5D & \winc 1 & \winc 1 & 6 & 3 & 3 &4 & 7  \\ 
            Rosenbrock5D & \winc 1 & \winc 1 & \winc 1 & \winc 1 & 2 &5 & 7  \\
            G-Function5D & 2 & 3 & \winc{1} & 4 & 2 & 3 & 7  \\
            Perm5D & 3 & \winc1 & \winc 1 & 5 & 7 &2 & 4  \\
            BNN5D & \winc{1} & \winc{1} & 4 & \winc{1} & 4& \winc{1}& 7  \\ 
    		\midrule
    		\emph{Average Regret 5D}& \winc\num{2.87e-02} & \num{5.03e-02}  & \num{4.70e-02}& \num{5.18e-02} & \num{7.13e-02}& \num{4.14e-02} & \num{1.93e-01} \\
    		\midrule
    		\midrule
            Levy10D & \winc 1 & 3 & 5 & 6 & \winc{1} &\winc{1} & 6  \\ 
            Rosenbrock10D & \winc{1} & \winc{1} & 2 & 6 & 3&2 & 7\\
            G-Function10D & 2 & 5 & \winc1 & 3 & 2& 5 & 7  \\
            Perm10D & 2 &  \winc{1} & 2 &6 & 2&2 & \winc{1}  \\
            BNN10D & \winc{1} & 2 & \winc{1} & \winc{1} & 3 & \winc1 & 7  \\ 
            \midrule
    		\emph{Average Regret 10D}& \winc \num{8.40e-02}   & \num{1.17e-01}  & \winc \num{6.96e-02}& \num{1.15e-01}& \num{9.32e-02}& \num{9.46e-02}& \num{2.35e-01}\\
    		\midrule\midrule
    	    Levy20D & \winc{1} & \winc{1} & 5 & 7 &\winc{1} &\winc{1} & 6  \\ 
            Rosenbrock20D & 2& 2 & 2 & 6 & \winc{1}& 4 & 6  \\
            G-Function20D & \winc{1} & 4 & 5 & \winc{1} &\winc{1} & 3 & 7  \\
            Perm20D & 3 & 5 & 3 & 2 &3 & 3 & \winc{1}  \\
            BNN20D & \winc{1} & 2 & 2 & 2 & 6 & \winc1 & 7  \\ 
    		\midrule
    		\emph{Average Regret 20D}& \winc\num{1.12e-01} &  \num{1.33e-01} & \num{1.39e-01}& \num{1.71e-01}& \num{1.37e-01}& \winc\num{1.17e-01} &\num{2.80e-01}\\
    		\midrule
    		\bottomrule 
    \end{tabular}%
    }
    \end{sc}
    \end{small}
    \end{center}
    \vskip -0.1in
\end{table}
\begin{figure}[b!]
    \begin{center}
    \centerline{\includegraphics[width=1\columnwidth, trim=33 0 50 15 ,clip]{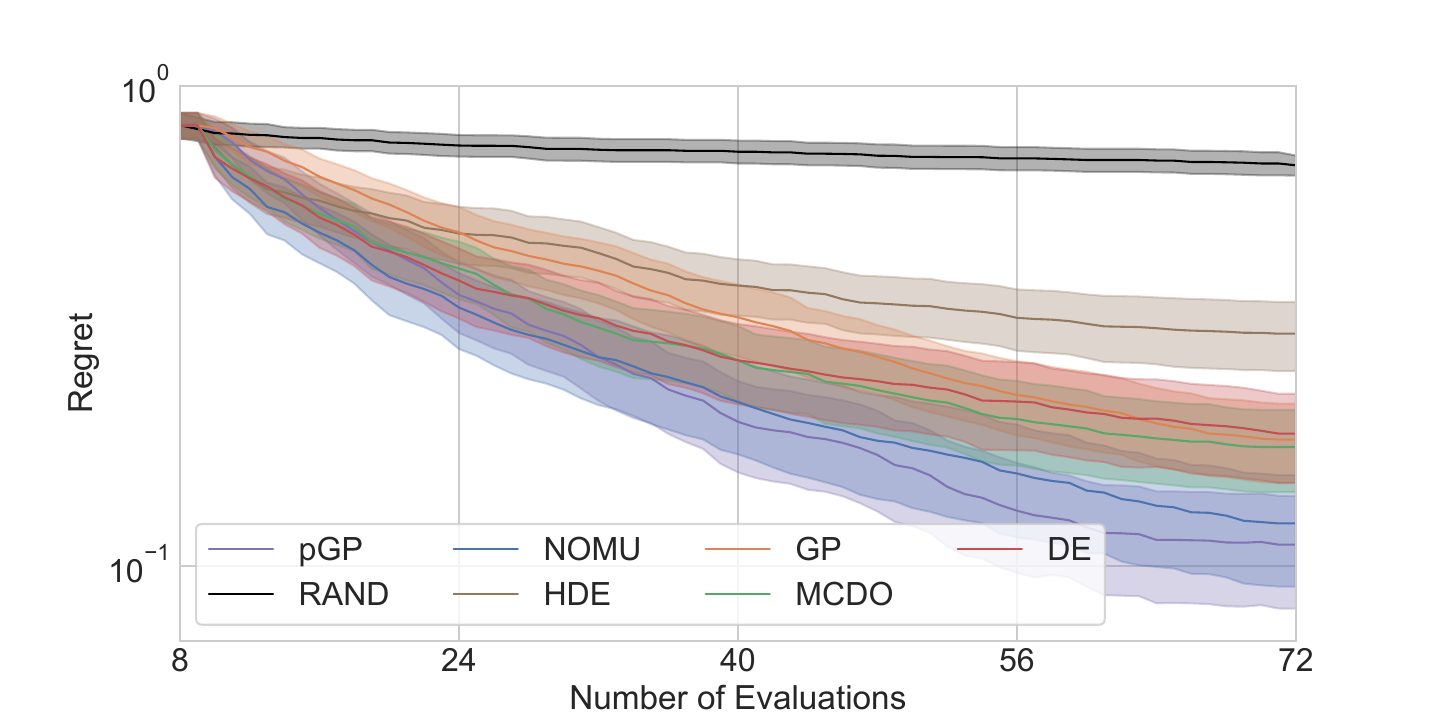}}
    \caption{Regret plot for BNN20D. For each BO step, we show the regrets averaged over 50 runs (solid lines) with 95\% CIs.}
    \label{fig:BNN20D}
    \end{center}
    \vskip -0.2in
\end{figure}
\paragraph{Results} In \Cref{tab:BOresults}, we present the BO results in 5D, 10D and 20D. We show the average final regret per dimension across the five functions. For each algorithm and dimension, we give the results corresponding to the MW scaling parameter (0.05 or 0.5) that minimizes the average final regret across that dimension (see \Cref{subsubsec:Detailed Results} for both MWs).
In practice, one often only knows the dimensionality of a given BO task, which is why we use the average final regret per dimension as the criterion for setting the optimal MW. For each individual function, we also present the ranks based on the final regret and a 95\% CI over 100 (5D) and 50 (10-20D) runs. We see that NOMU performs as well or better than all benchmarks in terms of average final regret.
By inspecting the ranks achieved for each individual function, we further observe that NOMU is never ranked worse than $3$\textsuperscript{rd}. In contrast, the performance of the benchmarks heavily depends on the test function; and each benchmark is ranked $4$\textsuperscript{th} and worse multiple times.
For Perm10D/20D we see that RAND performs best. However, due to a flat optimum of Perm, all algorithms achieve similar (very small) final regrets. 
Finally, we see that NOMU is always ranked first for the BNN test functions. \Cref{fig:BNN20D} shows the regret plot for BNN20D (see \Cref{subsubsec:Regret Plots} for all regret plots).
\section{Conclusion}
We have introduced NOMU, a new algorithm for estimating model uncertainty for NNs, specialized for scarce and noiseless settings.
By using a specific architecture and carefully-designed loss function, we have ensured that NOMU satisfies five important desiderata regarding model uncertainty that any method should satisfy.
However, when analyzing model uncertainty decoupled from data noise, we have experimentally uncovered that, perhaps surprisingly, established state-of-the-art methods fail to reliably capture some of the desiderata, even those that are required by Bayesian theory. In contrast, NOMU satisfies all desiderata, matches the performance of all benchmarks in regression tasks, and performs as well or better in noiseless BO tasks. We see great potential to further improve NOMU, for example by adapting the loss, or by modifying the connections between the two sub-NNs.
We also envision several extensions of NOMU, including its application to classification, employing different architectures (CNNs, GNNs, RNNs or Transformers), and incorporating data noise.

\section*{Acknowledgements}
We thank Marius H\"ogger and Aurelio Dolfini for insightful discussions and their excellent research assistance in implementing the Bayesian optimization experiments and the UCI data set experiments, respectively. Furthermore, we thank the anonymous reviewers for helpful comments. This paper is part of a project that has received funding from the European Research Council (ERC)
under the European Union’s Horizon 2020 research and innovation programme (Grant agreement No. 805542).

\bibliographystyle{icml2022}%
}

\ifAppendix{
\clearpage{}%

\twocolumn[
\icmltitle{NOMU: Neural Optimization-based Model Uncertainty}

\icmlsetsymbol{equal}{*}

\begin{icmlauthorlist}
\icmlauthor{Jakob Heiss}{equal,eth,ai}
\icmlauthor{Jakob Weissteiner}{equal,ai,uzh}
\icmlauthor{Hanna Wutte}{equal,eth,ai}
\icmlauthor{Sven Seuken}{ai,uzh}
\icmlauthor{Josef Teichmann}{eth,ai}
\end{icmlauthorlist}

\icmlaffiliation{eth}{ETH Zurich}
\icmlaffiliation{uzh}{University of Zurich}
\icmlaffiliation{ai}{ETH AI Center}

\icmlcorrespondingauthor{Jakob Weissteiner}{weissteiner@ifi.uzh.ch}
\icmlkeywords{Neural Networks, Model Uncertainty}

\vskip 0.3in
]

\printAffiliationsAndNotice{\icmlEqualContribution} %

\begin{abstract}
We study methods for estimating model uncertainty for neural networks (NNs) in regression. To isolate the effect of model uncertainty, we focus on a noiseless setting with scarce training data. We introduce five important desiderata regarding model uncertainty that any method should satisfy. However, we find that established benchmarks often fail to reliably capture some of these desiderata, even those that are required by Bayesian theory. To address this, we introduce a new approach for capturing model uncertainty for NNs, which we call \emph{Neural Optimization-based Model Uncertainty (NOMU)}. The main idea of NOMU is to design a network architecture consisting of two connected sub-NNs, one for model prediction and one for model uncertainty, and to train it using a carefully-designed loss function. Importantly, our design enforces that NOMU satisfies our five desiderata. Due to its modular architecture, NOMU can provide model uncertainty for any given (previously trained) NN if given access to its training data. We evaluate NOMU in various regressions tasks and noiseless Bayesian optimization (BO) with costly evaluations. In regression, NOMU performs at least as well as state-of-the-art methods. In BO, NOMU even outperforms all considered benchmarks.
\end{abstract}

\section{Introduction}
\label{sec:Introduction}
Neural networks (NNs) are becoming increasingly important in machine learning applications \citep{LeCun2015}. In many domains, it is essential to be able to quantify the \emph{model uncertainty (epistemic uncertainty)} of NNs \citep{neal2012bayesian,ghahramani2015probabilistic}. 
Good estimates of model uncertainty are indispensable in Bayesian optimization (BO) and active learning, where exploration is steered by (functions of) these uncertainty estimates. In recent years, BO has been successfully applied in practice to a wide range of problems, including robotics \citep{martinez2009bayesian}, sensor networks \citep{Srinivas_2012}, and drug development \citep{gomez2018automatic}. Better model uncertainty estimates for BO directly translate to improvements in these applications.

\begin{figure}[t!]
    \begin{center}
    \centerline{\includegraphics[width=1\columnwidth]{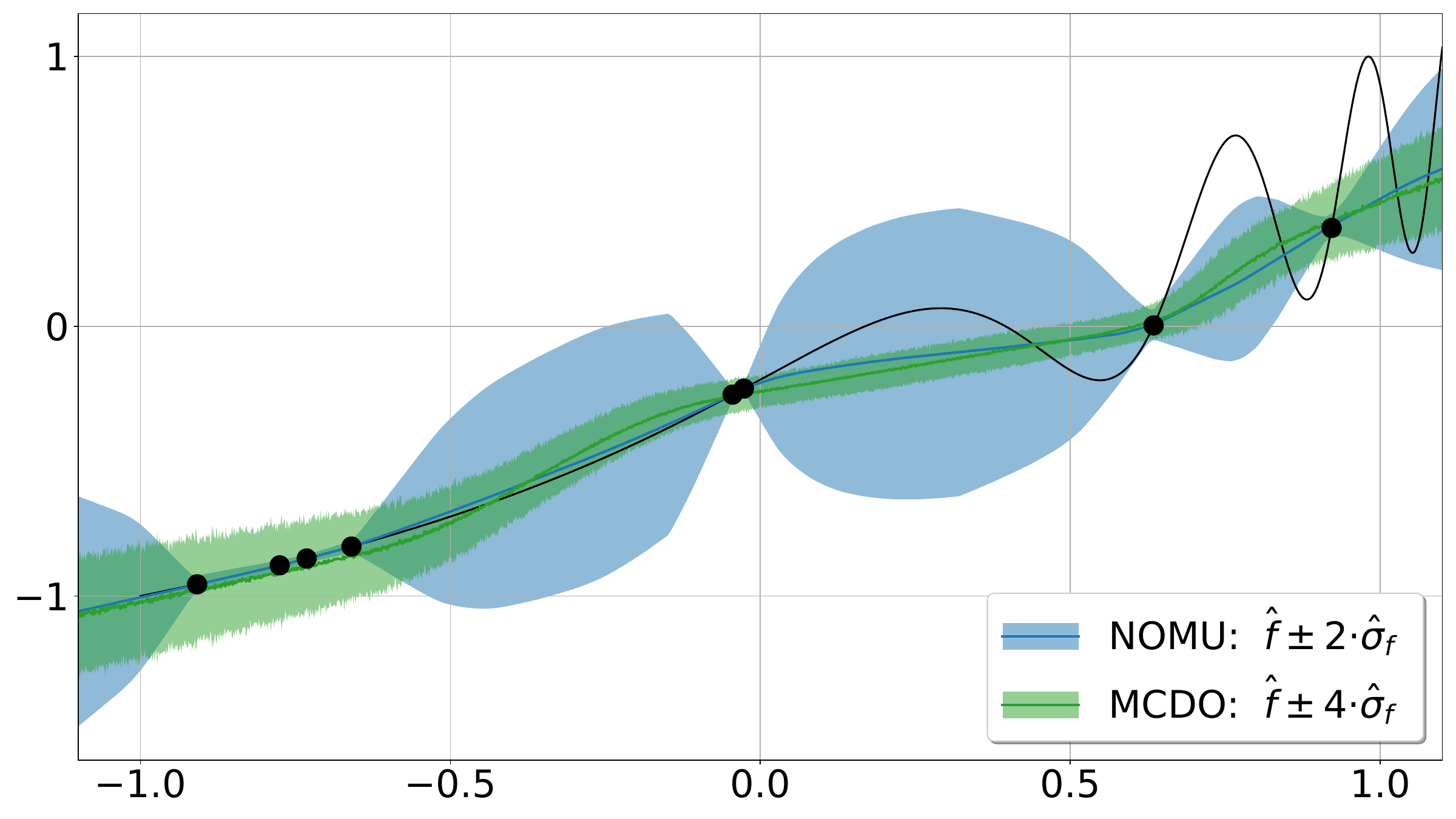}}
    \caption{Visualization of estimated model uncertainty $\sigmodelhat$. The unknown true function is depicted as a black solid line with training points as black dots. NOMU's model prediction $\fhat$ is shown as a solid blue line and its uncertainty bounds are shown as a blue shaded area. As a benchmark, MC Dropout is shown in green.}
    \label{fig:D41dJakobhfun}
    \end{center}
    \vskip -0.2in
\end{figure}

However, estimating model uncertainty well for NNs is still an open research problem. For settings with scarce training data and negligible data noise, where model uncertainty is the main source of uncertainty, we uncover deficiencies of widely used state-of-the-art methods for estimating model uncertainty for NNs. Prior work often only measures the performance in data noise dominant settings, and thus does not adequately isolate the pure model uncertainty, thereby overlooking the algorithms' deficiencies. However, in tasks such as BO with costly evaluations, where accurate estimates of model uncertainty are of utmost importance, these deficiencies can drastically decrease performance.

In this paper, we study the problem of estimating model uncertainty for NNs to obtain \emph{uncertainty bounds (UBs)} that estimate Bayesian credible bounds in a setting with negligible data noise and scarce training data. For this, we propose a novel algorithm (NOMU) that is specialized to such a setting.
Figure~\ref{fig:D41dJakobhfun} shows UBs for NOMU and the benchmark method MC Dropout.

\subsection{Prior Work on Model Uncertainty for NNs}\label{subsec:Uncertainty Quantification of Neural Networks}
Over the last decade, researchers have developed various methods to quantify model uncertainty for NNs. One strand of research considers Bayesian Neural Networks (BNNs), where distributions are placed over the NN's parameters \citep{graves2011practical,blundell2015weight,hernandez2015probabilistic}. However, variational methods approximating BNNs are usually computationally prohibitive and require careful hyperparameter tuning. Thus, BNNs are rarely used in practice \citep{wenzel2020good}.

In practice, \textit{ensemble methods} are more established:
\begin{itemize}[leftmargin=*,topsep=0pt,partopsep=0pt, parsep=0pt]
\item \citet{gal2016dropout} proposed \textit{Monte Carlo dropout (MCDO)} to estimate model uncertainty via stochastic forward passes. Interestingly, they could show that training a NN with dropout can also be interpreted as variational inference approximating a BNN.
\item \citet{lakshminarayanan2017simple} experimentally evaluated ensembles of NNs and showed that they perform as well as or better than BNNs. They proposed using \emph{deep ensembles (DE)}, which use NNs with two outputs for model prediction and data noise, and they estimate model uncertainty via the empirical standard deviation of the ensemble. DE is the most established state-of-the art ensemble method and has been shown to consistently outperform other ensemble methods \citep{ovadia2019can,fort2019deep,gustafsson2020evaluating,ashukha2020pitfalls}.
\item Recently, \citet{wenzel2020hyperparameter} proposed \emph{hyper deep ensembles (HDE)}, an extension of DE where additional diversity is created via different hyperparameters, and they showed that HDE outperforms DE.
\end{itemize}

Despite the popularity of MCDO, DE and HDE, our empirical results %
suggest that none of them reliably capture all essential features of model uncertainty: MCDO yields tubular bounds that do not narrow at observed data points (which can already be observed in \Cref{fig:D41dJakobhfun}); DE and HDE can produce UBs that are sometimes unreasonably narrow in regions far from observed data or unreasonably wide at training points (as we will show in Section \ref{subsec:Regression}).

\subsection{Overview of our Contribution}\label{subsec:Our Contribution}
We present a new approach for estimating model uncertainty for NNs, which we call \emph{neural optimization-based model uncertainty (NOMU)}. In contrast to a fully Bayesian approach (e.g., BNNs), where approximating the posterior for a realistic prior is in general very challenging, we estimate posterior credible bounds by directly enforcing essential properties of model uncertainty. Specifically, we make the following contributions:
\begin{enumerate}[leftmargin=*,topsep=0pt,partopsep=0pt, parsep=0pt]
\item We first introduce five desiderata that we argue model UBs should satisfy (\Cref{subsubsec:Desiderata}).
\item We then introduce NOMU, which consists of a network architecture (\Cref{subsec:The Network Architecture}) and a carefully-designed loss function (\Cref{subsec:modeluncertaintyloss}),  %
such that the estimated UBs fulfill these five desiderata. NOMU is easy to implement, scales well to large NNs, and can be represented as a \emph{single} NN without the need for further ensemble distillation (in contrast to MCDO, DE and HDE). Because of its modular architecture, it can easily be used to obtain UBs for already trained NNs. 
 \item We experimentally evaluate NOMU in various regression settings: in scarce and noiseless settings to isolate model uncertainty (\Cref{subsubsec:ToyRegression,subsubsec:GenerativeTestbed}) and on real-word data-sets (\Cref{subsubsec:SolarIrradiance,subsubsec:UCI}). We show that NOMU performs well across all these settings while state-of-the-art methods (MCDO, DE, and HDE) exhibit several deficiencies.\footnote{We also conducted experiments with \citep{blundell2015weight}. However, we found that this method did not perform as well as the other considered benchmarks. Moreover, it was shown in \citep{gal2016dropout,lakshminarayanan2017simple} that \emph{deep ensembles} and \emph{MC dropout} outperform the methods by \citep{hernandez2015probabilistic} and \citep{graves2011practical}, respectively. Therefore, we do not include \citep{graves2011practical,blundell2015weight,hernandez2015probabilistic} in our experiments.}
\item Finally, we evaluate the performance of NOMU in high-dimensional Bayesian optimization (BO) and show that NOMU performs as well or better than all considered benchmarks (Section~\ref{subsec:BayesianOptimization}).
\end{enumerate}

Our source code is available on GitHub: \url{https://github.com/marketdesignresearch/NOMU}.

\subsection{Further Related Work}\label{subsec:Related Work}
\citet{nix1994estimating} were among the first to introduce NNs with two outputs: one for model prediction and one for \emph{data noise (aleatoric uncertainty)}, using the Gaussian negative log-likelihood as loss function. 
However, such a data noise output cannot be used as an estimator for model uncertainty (epistemic uncertainty); see \Cref{sec:appendix:Aleatoric Neural Networks: Aleatoric vs. Epistemic Uncertainty} for details.
To additionally capture model uncertainty, \citet{kendall2017uncertainties} combined the idea of \citet{nix1994estimating} with MCDO.

Similarly, NNs with two outputs for lower and upper UBs, trained on specifically-designed loss functions, were previously considered by \citet{khosravi2010lower} and \citet{pearce2018high}. However, the method by \citet{khosravi2010lower} again only accounts for data noise and does not consider model uncertainty. The method by \citet{pearce2018high} also does not take model uncertainty into account in the design of their loss function and only incorporates it via ensembles (as in DE).

Besides the state-of-the art ensemble methods HDE and DE, there exist several other papers on ensemble methods that, for example, promote their diversity on the function space \citep{wang2019function} or reduce their computational cost \citep{wen2020batchensemble,havasi2021training}.

For classification, \citet{malinin2018predictive} introduced prior networks, which explicitly model in-sample and out-of-distribution uncertainty, where the latter is realized by minimizing the reverse KL-distance to a selected flat point-wise defined prior. In a recent working paper (and concurrent to our work), \citet{malinin2020regression} report on progress extending their idea to regression. While the idea of introducing a separate loss for learning model uncertainty is related to NOMU, there are several important differences (loss, architecture, behavior of the model prediction, theoretical motivation; see \Cref{sec:NOMUvsPriorNetworks} for details).
Furthermore, their experiments suggest that DE still performs weakly better than their proposed method.

In contrast to BNNs, which perform approximate inference over the entire set of weights, Neural Linear Models (NLMs) perform \emph{exact} inference on only the last layer. NLMs have been extensively benchmarked in \citep{ober2019benchmarking} against MCDO and the method from \citep{blundell2015weight}. Their results suggest that MCDO and \citep{blundell2015weight} perform  competitive, even to carefully-tuned NLMs. 

Neural processes, introduced by \citet{conditional_neural_proc_pmlr-v80-garnelo18a,neural_proc}, have been used to express model uncertainty for image completion tasks, where one has access to thousands of different images interpreted as functions $f_i$ instead of input points $x_i$. See \Cref{sec:appendix:NOMU vs. Neural Processes} for a detailed comparison of their setting to the setting we consider in this paper.

\section{Preliminaries}\label{sec:Bayesian Uncertainty Framework}
In this section, we briefly review the classical Bayesian uncertainty framework for regression.

Let $\X\subset\R^d, Y\subset\R$ and let $f\colon \X \to \Y$ denote the unknown ground truth function. Let $\Dtr:=\{\left( \xtr_i,\ytr_i\right)\in\X\times\Y, i\in \fromto{\ntr}\},$ with $\ntr\in\N$ be i.i.d samples from the data generating process
$\label{modelAssumption}
    y = f(x) + \varepsilon,
$
where $\varepsilon|x\sim\mathcal{N}(0,\signoise^2(x))$. We use $\signoise$ to refer to the \emph{data noise (aleatoric uncertainty)}. We refer to $\left( \xtr_i,\ytr_i\right)$ as a \emph{training point} and to $\xtr_i$ as an \emph{input training point}.

In the remainder of this paper, we follow the classic Bayesian uncertainty framework by modelling the unknown ground truth function $f$ as a random variable. Hence, with a slight abuse of notation, we use the symbol $f$ to denote \textit{both} the unknown ground truth function as well as the corresponding random variable.
In \Cref{sec:DetailsNotation}, we provide a mathematically rigorous formulation of the considered Bayesian uncertainty framework.

Given a prior distribution for $f$ and known data noise $\signoise$, the posterior of $f$ and $y$ are well defined. The \emph{model uncertainty (epistemic uncertainty)} $\sigmodel(x)$ is the posterior standard deviation of $f(x)$, i.e.,
\begin{align}\label{eq:modelUncertainty}
\sigmodel(x):= \sqrt{\mathbb{V}[f(x)|\Dtr,x]}, \quad x\in \X.
\end{align}
Assuming independence between $f$ and $\varepsilon$, the variance of the predictive distribution of $y$ can be decomposed as
$\label{varianceAss}\mathbb{V}[y|\Dtr,x]=\sigmodel^2(x) + \signoise^2(x).$
We present our algorithm for estimating model uncertainty $\sigmodelhat$ for the case of zero or negligible data noise, i.e., $\signoise\approx0$ (see \Cref{sec:Extensions} for an extension to $\signoise\gg0$). For a given model prediction $\fhat$, the induced uncertainty bounds (UBs) are then given by $\left(\lUB_c(x),\uUB_c(x)\right):=\left(\fhat(x)\mp c~\sigmodelhat(x)\right)$, for $x\in\X$ and a calibration parameter $c\geq0$.\footnote{Note that our UBs estimate \emph{credible bounds (CBs)}~$\underline{C\!B}$ and $\overline{C\!B}$, which, for $\alpha \in [0,1]$, fulfill that $\mathbb{P}[f(x)\in [\underline{C\!B},\overline{C\!B}]|\Dtr,x]=\alpha$. For $\signoise\equiv0$, CBs are equal to \emph{predictive bounds} $\underline{P\!B},\overline{P\!B}$ with $\mathbb{P}[y \in [\underline{P\!B},\overline{P\!B}]|\Dtr,x]=\alpha$. See \Cref{sec:TheoreticalAnalysisAppendix} for an explanation.}

\section{The NOMU Algorithm}\label{sec:Neural Optimization-based Model Uncertainty}
We now present NOMU. We design NOMU to yield a model prediction $\fhat$ and a model uncertainty prediction $\sigmodelhat$, such that the resulting UBs $(\lUB_c,\uUB_c)$ fulfill five desiderata.
\subsection{Desiderata}\label{subsubsec:Desiderata}
\begin{enumerate}[label=D\arabic*,align=left, leftmargin=*,topsep=0pt]
    \item[\mylabel{itm:Axioms:trivial}{D1 (Non-Negativity)}]
	\textit{The upper/lower UB between two training points lies above/below the model prediction $\fhat$, i.e.,  $\lUB_c(x)\le\fhat(x)\le\uUB_c(x)$ for all $x\in\X$ and for $c\geq0$. Thus, $\sigmodelhat\geq0$.}
	\end{enumerate}
By definition, for any given prior, the exact posterior model uncertainty $\sigmodel$ is positive, and therefore Desideratum~\hyperref[itm:Axioms:trivial]{D1} should also hold for any estimate $\sigmodelhat$.
\begin{enumerate}[resume,label=D\arabic*,align=left, leftmargin=*,topsep=0pt]
		\item[\mylabel{itm:Axioms:ZeroUncertaintyAtData}{D2 (In-Sample)}]
	\textit{In the noiseless case ($\signoise\equiv0$), there is zero model uncertainty at each input training point $\xtr$, i.e., $\sigmodelhat(\xtr)=0$. Thus, $\uUB_c(\xtr)=\lUB_c(\xtr)=\fhat(\xtr)$ for $c\ge0$.}
\end{enumerate}
In \Cref{sec:desiderata:InSample}, we prove that, for any prior that does not contradict the training data, the exact $\sigmodel$ satisfies \hyperref[itm:Axioms:ZeroUncertaintyAtData]{D2}. Thus, \hyperref[itm:Axioms:ZeroUncertaintyAtData]{D2} should also hold for any estimate $\sigmodelhat$, and we argue that even in the case of non-zero small data noise, model uncertainty should be small at input training data points.%
\begin{enumerate}[resume,label=D\arabic*,align=left, leftmargin=*,topsep=0pt]
	\item[\mylabel{itm:Axioms:largeDistantLargeUncertainty}{D3 (Out-of-Sample)}]
	\textit{The larger the distance of a point $x\in\X$ to the input training points in $\Dtr$, the wider the UBs at $x$, i.e., model uncertainty~$\sigmodelhat$ increases out-of-sample.\footnote{Importantly, \hyperref[itm:Axioms:largeDistantLargeUncertainty]{D3} also promotes model uncertainty in gaps \emph{between} input training points.}}
\end{enumerate}
For \hyperref[itm:Axioms:largeDistantLargeUncertainty]{D3} it is often not obvious which metric to choose to measure distances.
Some aspects of this metric can be specified via the architecture (e.g., shift invariance for CNNs).
In many applications, it is best to \textit{learn} further aspects of this metric from training data, motivating our next desideratum.
\begin{enumerate}[resume,label=D\arabic*,align=left, leftmargin=*,topsep=0pt]
	\item[\mylabel{itm:Axioms:irregular}{D4 (Metric Learning)}]
	\textit{Changes in those features of $x$ that have high predictive power on the training set have a large effect on the distance metric used in \hyperref[itm:Axioms:largeDistantLargeUncertainty]{D3}.}\footnote{\label{footnote:LearningTheMetricISImportant}
	 Consider the task of learning facial expressions from images. For this, eyes and mouth are important features, while background color is not. 
	 A CNN automatically learns which features are important for model prediction.
	 The same features are also important for model uncertainty: Consider an image with pixel values similar to those of an image of the training data, but where mouth and eyes are very different. We should be substantially more uncertain about the model prediction for such an image than for one which is almost identical to a training image except that it has a different background color, even if this change of background color results in a huge Euclidean distance of the pixel vectors. \hyperref[itm:Axioms:irregular]{D4} requires that a more useful metric is learned instead.
	 }
	 \end{enumerate}
\hyperref[itm:Axioms:irregular]{D4} is not required for any application. However, specifically in \emph{deep} learning applications, where it is a priori not clear which features are important, \hyperref[itm:Axioms:irregular]{D4} is particularly desirable.%
\begin{enumerate}[resume,label=D\arabic*,align=left, leftmargin=*,topsep=0pt]
	\item[\mylabel{itm:Axioms:UncertaintyDecreasesWithMoreTrainingPoints}{D5 (Vanishing)}]
	\textit{As the number~$\ntr$ of training points (with $\xtr_i\overset{\text{i.i.d}}{\sim}\PP_X$) tends to infinity, model uncertainty vanishes for each $x$ in the support of the input data distribution~$\PP_X$, i.e., $\lim_{\ntr\to\infty}\sigmodelhat(x) = 0$ for a fixed $c\ge0$. Thus, for a fixed $c$, $\lim_{\ntr\to\infty} |\uUB_c(x)-\lUB_c(x)|=0$.
	}
\end{enumerate}
In \Cref{appendix:Desiderata}, we discuss all desiderata in more detail (see \Cref{subsec:A note on Desideratum D4} for a visualization of \hyperref[itm:Axioms:irregular]{D4}).
\begin{figure}[t!]
        \vskip 0.1in
        \begin{center}
        \centerline{
        \resizebox{1\columnwidth}{!}{
                \begin{tikzpicture}
                [cnode/.style={draw=black,fill=#1,minimum width=3mm,circle}]
                \node[cnode=gray,label=180:$\mathlarger{\mathlarger{\mathlarger{x \in \X}}}$] (x1) at (0.5,0) {};
                \node[cnode=gray] (x2+3) at (3,3) {};
                \node at (3,2.2) {$\mathlarger{\mathlarger{\mathlarger{\vdots}}}$};
                \node[cnode=gray] (x2+1) at (3,1) {};
                \node[cnode=gray] (x2-3) at (3,-3) {};
                \node at (3,-1.7) {$\mathlarger{\mathlarger{\mathlarger{\vdots}}}$};
                \node[cnode=gray] (x2-1) at (3,-1) {};
                \draw (x1) -- (x2+3);
                \draw (x1) -- (x2+1);
                \draw (x1) -- (x2-1);
                \draw (x1) -- (x2-3);
                \node[cnode=gray] (x3+3) at (6,3) {};
                \node at (6,2.2) {$\mathlarger{\mathlarger{\mathlarger{\vdots}}}$};
                \node[cnode=gray] (x3+1) at (6,1) {};
                \node[cnode=gray] (x3-3) at (6,-3) {};
                \node at (6,-1.7) {$\mathlarger{\mathlarger{\mathlarger{\vdots}}}$};
                \node[cnode=gray] (x3-1) at (6,-1) {};
                \foreach \y in {1,3}
                {   \foreach \z in {1,3}
                        {   \draw (x2+\z) -- (x3+\y);
                        }
                }
                \foreach \y in {1,3}
                {   \foreach \z in {1,3}
                        {   \draw (x2-\z) -- (x3-\y);
                        }
                }
                \node at (7.5,+3) {$\mathlarger{\mathlarger{\mathlarger{\ldots}}}$};
                \node at (7.5,+2) {$\mathlarger{\mathlarger{\mathlarger{\ldots}}}$};
                \node at (7.5,+1) {$\mathlarger{\mathlarger{\mathlarger{\ldots}}}$};
                \node at (7.5,-1) {$\mathlarger{\mathlarger{\mathlarger{\ldots}}}$};
                \node at (7.5,-2) {$\mathlarger{\mathlarger{\mathlarger{\ldots}}}$};
                \node at (7.5,-3) {$\mathlarger{\mathlarger{\mathlarger{\ldots}}}$};
                \node[cnode=gray] (x4+3) at (9,3) {};
                \node at (9,2.2) {$\mathlarger{\mathlarger{\mathlarger{\vdots}}}$};
                \node[cnode=gray] (x4+1) at (9,1) {};
                \node[cnode=gray] (x4-3) at (9,-3) {};
                \node at (9,-1.7) {$\mathlarger{\mathlarger{\mathlarger{\vdots}}}$};
                \node[cnode=gray] (x4-1) at (9,-1) {};
                \node[cnode=gray,label=360:$\mathlarger{\mathlarger{\mathlarger{\fhat(x)\in \Y}}}$] (x5+2) at (11.5,2) {};
                \node[cnode=gray,label=360:$\mathlarger{\mathlarger{\mathlarger{\sigmodelhatraw(x)\in \Rpz}}}$] (x5-2) at (11.5,-2) {};
                \draw (x4+3)--(x5+2);
                \draw (x4+1)--(x5+2);
                \draw[dashed,-Latex,arrows={[scale=1.15]}] (x4+3)--(11.52,-1.8);
                \draw[dashed,-Latex,arrows={[scale=1.15]}] (x4+1)--(11.4,-1.86);
                \draw (x4-1)--(x5-2);
                \draw (x4-3)--(x5-2);
                \draw [dotted, line width=0.4mm] (2.5,0.5) rectangle (9.5,+3.5);
                \draw [dotted, line width=0.4mm] (2.5,-3.5) rectangle (9.5,-0.5);
                \node at (1,+2.5) {$\mathlarger{\mathlarger{\mathlarger{\fhat}}}${\Large-network}};
                \node at (1,-2.5) {$\mathlarger{\mathlarger{\mathlarger{\sigmodelhatraw}}}${\Large-network}};
                \end{tikzpicture}
        }
        }
        \caption{NOMU's network architecture}
        \label{fig:nn_tik}
        \end{center}
        \vskip -0.2in
\end{figure}

\subsection{The Network Architecture}\label{subsec:The Network Architecture}
For NOMU, we construct a network $\NN_\theta$ with two outputs: \emph{the model prediction} $\fhat$ (e.g., mean prediction) and a \emph{raw} model uncertainty prediction $\sigmodelhatraw$. Formally: $\NN_\theta \colon\X \to  \Y\times \Rpz$, with $x \mapsto \NN_\theta(x):=(\fhat(x),\sigmodelhatraw(x))$. NOMU's architecture consists of two almost separate sub-networks: the $\fhat$-network and the $\sigmodelhatraw$-network (see \Cref{fig:nn_tik}). For each sub-network, any network architecture can be used (e.g., feed-forward NNs, CNNs). This makes NOMU highly modular and we can plug in any previously trained NN for $\fhat$, or we can train $\fhat$  simultaneously with the $\sigmodelhatraw$-network. The $\sigmodelhatraw$-network learns the raw model uncertainty and is connected with the $\fhat$-network through the last hidden layer (dashed lines in Figure \ref{fig:nn_tik}). This connection enables $\sigmodelhatraw$ to re-use features that are important for the model prediction~$\fhat$, implementing Desideratum \ref{itm:Axioms:irregular}.\footnote{To prevent that $\sigmodelhatraw$ impacts $\fhat$, the dashed lines should only make forward passes when trained.}

\begin{remark} 
NOMU's network architecture can be modified to realize \ref{itm:Axioms:irregular} in many different ways. For example, if low-level features were important for predicting the model uncertainty, one could additionally add connections from earlier hidden layers of the $\fhat$-network to layers of the $\sigmodelhatraw$-network. Furthermore, one can strengthen \ref{itm:Axioms:irregular} by increasing the regularization of the $\sigmodelhatraw$-network (see \Cref{subsec:A note on Desideratum D4}).
\end{remark}

After training $\NN_\theta$, we apply the readout map $\readout(z)=\sigmax(1-\exp(-\frac{\max(0,z)+\sigmin}{\sigmax})),\, \sigmin\ge0, \sigmax>0$ to the \emph{raw} model uncertainty output $\sigmodelhatraw$ to obtain NOMU's \emph{model uncertainty prediction}
\begin{align}\label{eq:modelUncertaintyPrediction}
\sigmodelhat(x):=\readout(\sigmodelhatraw(x)),\, \forall x\in \X.
\end{align}
The readout map $\readout$ monotonically interpolates between a minimal $\approx\sigmin$ and a maximal $\approx\sigmax$ model uncertainty (see \Cref{fig:app:readout_map} in \Cref{sec:readout_map} for a visualization). Here, $\sigmin$ is used for numerical stability, and $\sigmax$ defines the maximal model uncertainty far away from input training points (similarly to the prior variance for RBF-GPs). 
With NOMU's model prediction $\fhat$, its model uncertainty prediction $\sigmodelhat$ defined in \eqref{eq:modelUncertaintyPrediction}, and given a calibration parameter\footnote{Like all other methods, NOMU outputs \emph{relative} UBs that should be calibrated, e.g., via a parameter $c\ge0$. See also \citet{pmlr-v80-kuleshov18a} for a non-linear calibration method.%
} $c\in \Rpz$, we can now define for each $x\in\X$  NOMU's UBs as
\begin{align}\label{eq:UBs}
&\left(\lUB_c(x),\uUB_c(x)\right):=\left(\fhat(x)\mp c~\sigmodelhat(x)\right). %
\end{align}
It is straightforward to construct a \emph{single} NN that directly outputs the upper/lower UB, by extending the architecture shown in Figure~\ref{fig:nn_tik}: we monotonically transform and scale the output $\sigmodelhatraw(x)$ and then add/subtract this to/from the other output $\fhat(x)$. It is also straightforward to compute NOMU's UBs for any \emph{given, previously trained NN}, by attaching the $\sigmodelhatraw$-network to the trained NN, and only training the $\sigmodelhatraw$-network on the same training points as the original NN.
\begin{remark}\label{rem:different_readout}
The readout map $\readout$ can be modified depending on the subsequent use of the estimated UBs. For example, for BO over discrete domains (e.g., $\X=\{0,1\}^d$) \citep{baptista2018bayesian}, we propose the linearized readout map $\readout(z)=\sigmin+\max(0,z-\sigmin)-\max(0,z-\sigmax)$. With this $\readout$ and ReLU activations, one can encode NOMU's UBs as a mixed integer program (MIP) \citep{weissteiner2020deep}.
This enables optimizing the upper UB as acquisition function without the need for further approximation via ensemble distillations \citep{malinin2019ensemble}.
\end{remark}
\subsection{The Loss Function}\label{subsec:modeluncertaintyloss}
We now introduce the loss function $L^\hp$ we use for training NOMU's architecture. Let $\X$ be such that $0<\lambda_d(\X)<\infty$, where $\lambda_d$ denotes the $d$-dimensional Lebesgue measure. We train $\NN_{\theta}$ with loss $L^\hp$ and L2-regularization parameter $\lambda>0$, i.e.,  %
minimizing $\Lmu(\NN_{\theta})+\lambda\twonorm[{\theta}]^2$ via a gradient descent-based algorithm.
\begin{definition}[NOMU Loss]\label{def:Model Uncertainty Loss}
Let  $\hp:=(\musqr,\muexp,\cexp)\in \Rpz^3$ denote a tuple of hyperparameters. Given a training set $\Dtr$, the loss function $L^\hp$ is defined as
\begin{align}\label{eq:lmu2}
  L^\hp(\NN_\theta):=&\underbrace{\sum_{i=1}^{\ntr}(\fhat(\xtr_i
  )-\ytr_i)^2}_{\terma{}}+ \,\musqr\cdot\underbrace{\sum_{i=1}^{\ntr}\left(\sigmodelhatraw(\xtr_i)\right)^2}_{\termb{}} \notag\\ 
   &+ \muexp\cdot\underbrace{\frac{1}{\lambda_d(\X)}\int_{\X}e^{-\cexp\cdot \sigmodelhatraw(x)}\,dx}_{\termc{}}.
\end{align}
\end{definition}
In the following, we explain how the three terms of $L^\hp$ promote the desiderata we introduced in \Cref{subsubsec:Desiderata}. Note that the behaviour of $\sigmodelhatraw$ directly translates to that of $\sigmodelhat$.
\begin{itemize}[leftmargin=*,topsep=0pt,partopsep=0pt, parsep=0pt]
\item Term~\terma{} solves the regression task, i.e., learning a smooth function $\fhat$ given $\Dtr$. If $\fhat$ is given as a pre-trained NN, then this term can be omitted.
\item Term~\termb{} implements \ref{itm:Axioms:ZeroUncertaintyAtData} and \ref{itm:Axioms:UncertaintyDecreasesWithMoreTrainingPoints} (i.e.,  $\sigmodelhatraw(\xtr_i)\approx 0$). The hyperparameter $\musqr$ controls the amount of uncertainty at the training points.\footnote{In the noiseless case, in theory $\frac{\musqr}{\lambda}=\infty$; we set $\frac{\musqr}{\lambda}=10^7$. For small non-zero data noise, setting $\frac{\musqr}{\lambda}\ll\infty$ captures \textit{data noise induced model uncertainty} $\sigmodel(\xtr)>0$ (\Cref{sec:desiderata:InSample}).} The larger $\musqr$, the narrower the UBs at training points.
\item Term~\termc{} has two purposes. First, it implements \ref{itm:Axioms:trivial} (i.e., $\sigmodelhatraw\ge0$). Second, it pushes $\sigmodelhatraw$ towards infinity across the whole input space $\X$. However, due to the counteracting force of \termb{} as well as regularization, $\sigmodelhatraw$ increases continuously as you move away from the training data. The interplay of \termb{}, \termc{}, and regularization thus promotes \ref{itm:Axioms:largeDistantLargeUncertainty}. The hyperparameters $\muexp$ and $\cexp$ control the size and shape of the UBs. Concretely, the larger $\muexp$, the wider the UBs; the larger $\cexp$, the narrower the UBs at points $x$ with large $\sigmodelhat(x)$ and the wider the UBs at points $x$ with small $\sigmodelhat(x)$.
\end{itemize}
In \Cref{subsec:appendix:Qualitative Sensitivity Analysis}, we provide detailed visualizations on how the loss hyperparameters $\musqr$, $\muexp$, and $\cexp$ shape NOMU's model uncertainty estimate.
In the implementation of $L^\hp$, we approximate \termc{} via MC-integration using additional, \emph{artificial input points} ${\Daug:=\left\{x_i\right\}_{i=1}^l\stackrel{i.i.d}{\sim}\textrm{Unif}(\X)}$ by
$\Scale[1.2]{\frac{1}{l}}\cdot\hspace{-0.1cm}\sum_{x\in \Daug}\hspace{-0.05cm}e^{-\cexp\cdot \sigmodelhat(x)}$.
\begin{remark}
In \termc{}, instead of the Lebesgue-measure, one can also use a different measure $\nu$, i.e., $\frac{1}{\nu(\X)}\int_{\X}e^{-\cexp\cdot \sigmodelhatraw(x)}\,d\nu(x)$. This can be relevant in high dimensions, where meaningful data points often lie close to a lower-dimensional manifold \citep{cayton2005algorithms}; $\nu$ can then be chosen to concentrate on that region. In practice, this can be implemented by sampling from a large unlabeled data set $\Daug$ representing this region or learning the measure $\nu$ using GANs \citep{goodfellow2014gen}.
\end{remark}

\paragraph{Theory} In \Cref{appendix:Desiderata}, we prove that NOMU fulfills \hyperref[itm:Axioms:trivial]{D1}, \hyperref[itm:Axioms:ZeroUncertaintyAtData]{D2} and \hyperref[itm:Axioms:UncertaintyDecreasesWithMoreTrainingPoints]{D5} (\Cref{prop:NOMUD1,prop:NOMUD2,prop:NOMUD5}),
and discuss how NOMU fulfills \hyperref[itm:Axioms:largeDistantLargeUncertainty]{D3} and \hyperref[itm:Axioms:irregular]{D4}.
In \Cref{subsec:connectingNOMUtoPointwiseUncertaintyBounds}, we show that, under certain assumptions, NOMU's UBs can be interpreted as \textit{pointwise worst-case} UBs $\uUBp(x):=\sup_{f\in \HC_{\Dtr}}f(x)$ within a hypothesis class~$\HC_{\Dtr}$ of data-explaining functions.
In \Cref{subsec:Pointwise Uncertainty Bounds}, we explain how $\lUBp(x)$ and $\uUBp(x)$ estimate posterior CBs of BNNs (with a Gaussian prior on the weights), without performing challenging variational inference. However, while exact posterior CBs of BNNs lose \hyperref[itm:Axioms:irregular]{D4} as their width goes to infinity, NOMU's UBs are capable of retaining \hyperref[itm:Axioms:irregular]{D4} in this limit.  

\section{Experimental Evaluation}\label{sec:ExperimentsMain}
In this section, we experimentally evaluate NOMU's model uncertainty estimate in multiple synthetic and real-world regression settings (\Cref{subsec:Regression}) as well as in high-dimensional Bayesian optimization (\Cref{subsec:BayesianOptimization}).

\paragraph{Benchmarks} We compare NOMU against four benchmarks, each of which gives a model prediction $\fhat$ and a model uncertainty prediction $\sigmodelhat$ (see \Cref{subec:Benchmarks} for formulas). We calculate model-specific UBs at $x\in\X$ as $(\fhat(x)\mp c~\sigmodelhat(x))$ with calibration parameter $c \in \Rpz$ and use them to evaluate all methods. We consider three algorithms that are specialized to model uncertainty for NNs: (i) MC dropout (MCDO) (ii) deep ensembles (DE) and (iii) hyper deep ensembles (HDE) and a non-NN-based benchmark: (iv) Gaussian process (GP) with RBF kernel.

\subsection{Regression}
\label{subsec:Regression}
To develop intuition, we first study the \emph{model} UBs of all methods on synthetic test functions with 1D--2D scarce input training points without data noise (\Cref{subsubsec:ToyRegression}). We then propose a novel generative test-bed and evaluate NOMU within this setting (\Cref{subsubsec:GenerativeTestbed}). Next, we analyze NOMU on a real-world time series (\Cref{subsubsec:SolarIrradiance}). Finally, we evaluate NOMU on the real-world UCI data sets (\Cref{subsubsec:UCI}).
\begin{figure}[t!]
    \begin{center}
    \centerline{\includegraphics[width=1\columnwidth]{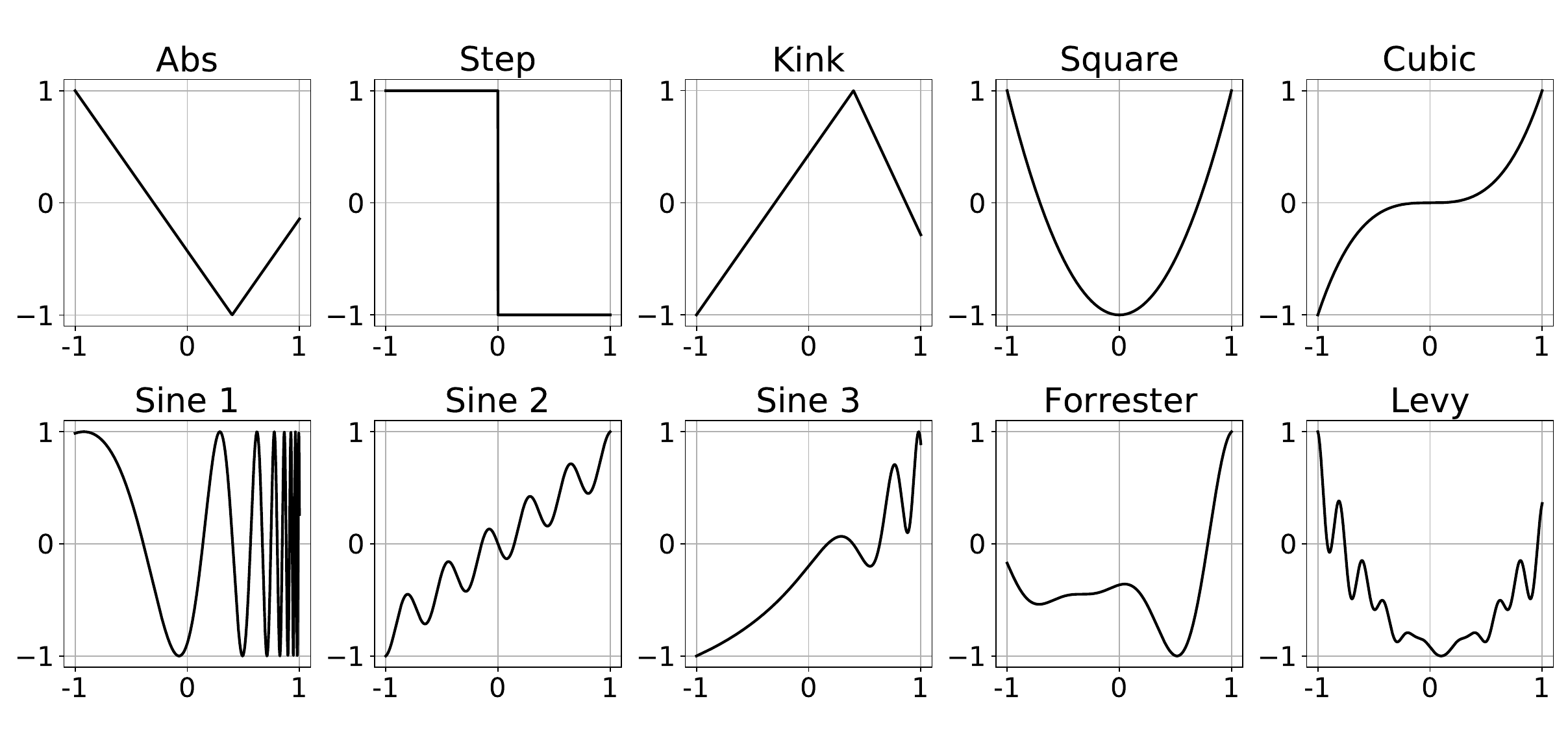}}
    \caption{1D synthetic test functions}
    \label{fig:1dfuns}
    \end{center}
    \vskip -0.2in
\end{figure}

\subsubsection{Toy Regression}\label{subsubsec:ToyRegression}
\paragraph{Setting} We consider ten different
1D functions whose graphs are shown in \Cref{fig:1dfuns}. 
Those include the popular Levy and Forrester function with multiple local optima.\footnote{\label{footnote:sfu}See \href{www.sfu.ca/~ssurjano/optimization.html}{sfu.ca/~ssurjano/optimization.html}.} All functions are transformed to $\X:=[-1,1]=:f(X)$. For each function, we conduct $500$ runs. In each run, we randomly sample eight noiseless input training points from $\X$,
such that the only source of uncertainty is model uncertainty.
For each run, we also generate $100$ \emph{test points} in the same fashion to assess the quality of the UBs.

\paragraph{Metrics} We report the \emph{average negative log (Gaussian) likelihood (NLL)}, minimized over the calibration parameter $c$, which we denote as $\MNLPD$. Following prior work \citep{khosravi2010lower,pmlr-v80-kuleshov18a,pearce2018high}, we further measure the quality of UBs by contrasting their \emph{mean width} $(\MW)$ with their \emph{coverage probability} $(\CP)$. 
Ideally, $\MW$ should be as small as possible, while $\CP$ should be close to $1$. Since $\CP$ is counter-acting $\MW$, we consider ROC-like curves, plotting $\MW$ against $CP$ for a range of calibration parameters $c$, and report the \emph{area under the curve (AUC)} (see \Cref{subsec:performanceMeasures} for details).

\paragraph{Algorithm Setup} For each of the two NOMU sub-networks, we use a feed-forward NN with three fully-connected hidden layers \`a $2^{10}$ nodes, ReLUs, and hyperparameters $\muexp=0.01, \musqr=0.1, \cexp=30$. In practice, the values for $\muexp, \musqr$, and $\cexp$ can be tuned on a validation set. However, for all synthetic experiments (\Cref{subsubsec:ToyRegression} and \Cref{subsubsec:GenerativeTestbed}), we use the same values, which lead to good results across all functions. Moreover, we set $\lambda=10^{-8}$ accounting for zero data-noise, $\sigmin=0.001$ and $\sigmax=2$.
In \Cref{subsec:appendix:Quantitative Sensitivity Analysis}, we provide an extensive sensitivity analysis for the hyperparameters $\muexp$, $\musqr$, $\cexp$, $\sigmin$ and $\sigmax$. This analysis demonstrates NOMU's robustness within a certain range of hyperparameter values.

Furthermore, to enable a fair comparison of all methods, we use generic representatives and do not optimize any of them for any test function. We also choose the architectures of all NN-based methods, such that the overall number of parameters is comparable. We provide all methods with the same prior information (specifically, this entails knowledge of zero data noise), and set the corresponding parameters accordingly. Finally, we set all remaining hyperparameters of the benchmarks to the values proposed in the literature. Details on all configurations are provided in \Cref{subsec:ConfigurationDetailsofBenchmarksRegression}.

\begin{figure}[t!]
    \begin{center}
    \centerline{\includegraphics[width=1\columnwidth]{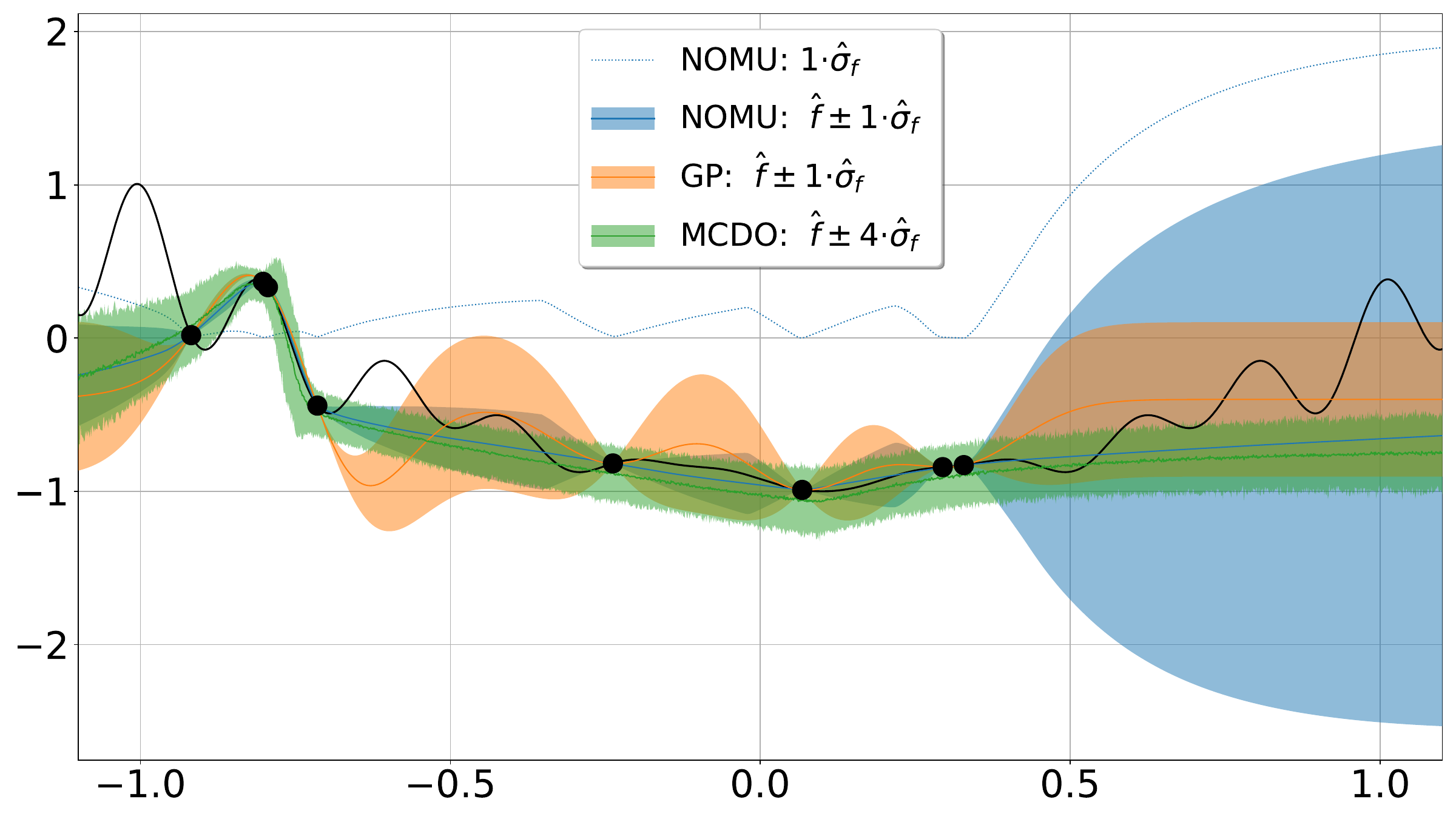}}
	\centerline{\includegraphics[width=1\columnwidth]{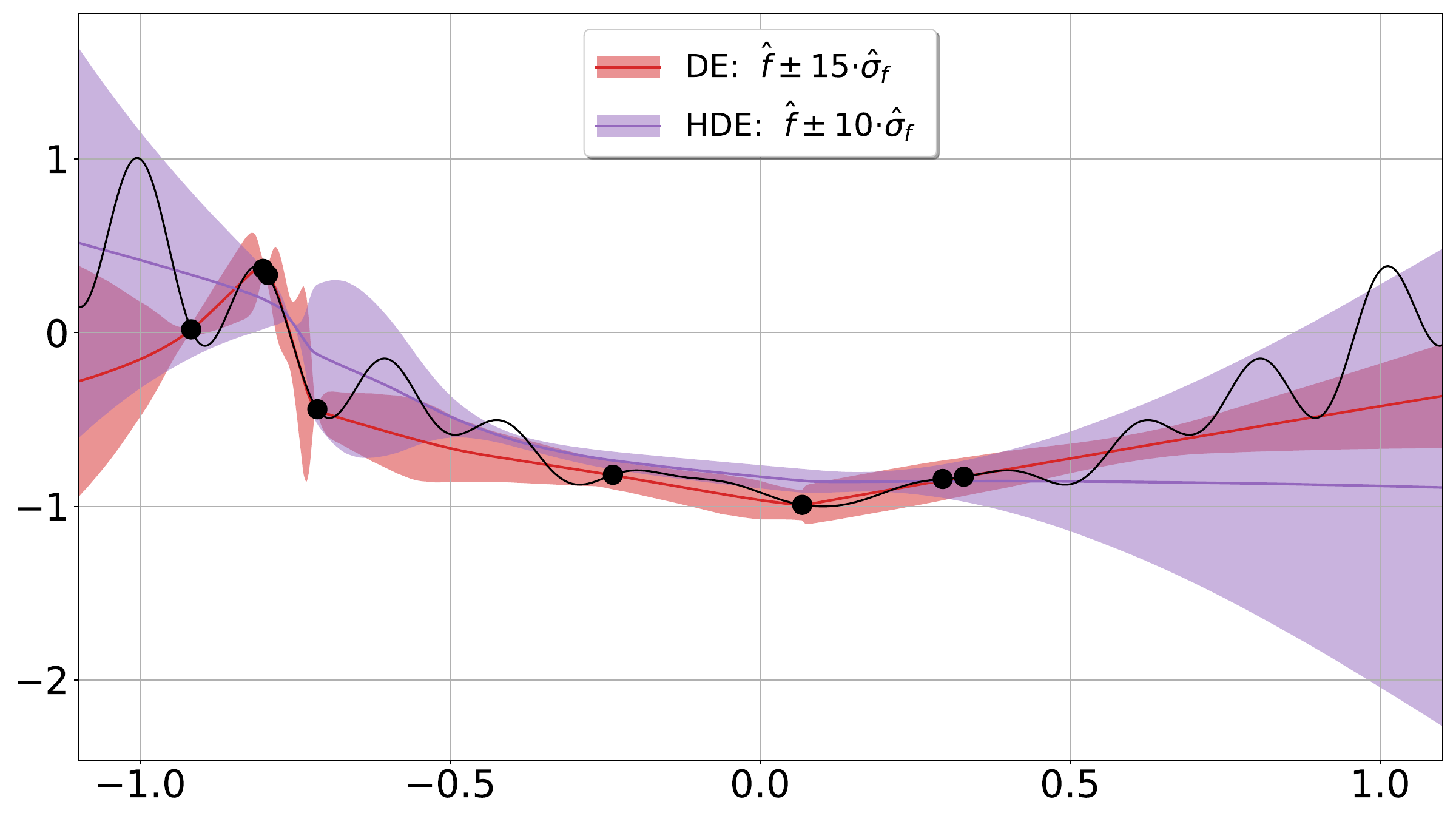}}
	\caption{UBs resulting from NOMU, GP, MCDO, and DE, HDE for the Levy function (solid black line). For NOMU, we also show $\sigmodelhat$ as a dotted blue line. Training points are shown as black dots.}
	\label{fig:1dboundsLevy}
	\end{center}
    \vskip -0.2in
\end{figure}
\paragraph{Results} \Cref{fig:1dboundsLevy} exemplifies our findings, showing typical UBs for the Levy function as obtained in one run. In \Cref{subsubsec:Uncertainty Bounds Detailed Characteristics} we provide further visualisations. We find that MCDO consistently yields tube-like UBs; in particular, its UBs do not narrow at training points, i.e., failing in \ref{itm:Axioms:ZeroUncertaintyAtData} even though MCDO's Bayesian interpretation requires \hyperref[itm:Axioms:ZeroUncertaintyAtData]{D2} to hold (see \Cref{sec:desiderata:InSample}). Moreover, it only fulfills \ref{itm:Axioms:largeDistantLargeUncertainty} to a limited degree. We frequently observe that DE leads to UBs of somewhat arbitrary shapes. This can be seen most prominently in \Cref{fig:1dboundsLevy} around $x\approx-0.75$ and at the edges of its input range, where DE's UBs are very different in width with no clear justification. Thus, also DE is limited in \ref{itm:Axioms:largeDistantLargeUncertainty}. In addition, we sometimes see that also DE's UBs do not narrow sufficiently at training points, i.e., not fulfilling \ref{itm:Axioms:ZeroUncertaintyAtData}.
HDE's UBs are even more random, i.e., predicting large model uncertainty at training points and sometimes zero model uncertainty in gaps between them (e.g., $x\approx-0.75$). %
In contrast, NOMU displays the behaviour it is designed to show. Its UBs nicely tighten at training points and expand in-between (\hyperref[itm:Axioms:trivial]{D1}\crefrangeconjunction\hyperref[itm:Axioms:largeDistantLargeUncertainty]{D3}, for \ref{itm:Axioms:irregular} see \cref{subsec:A note on Desideratum D4}). %
Like NOMU, the GP fulfills \ref{itm:Axioms:ZeroUncertaintyAtData} and \ref{itm:Axioms:largeDistantLargeUncertainty} well, but cannot account for \ref{itm:Axioms:irregular} (a fixed kernel does not depend on the model prediction). 
\Cref{tab:1dsynfunresults} provides the ranks achieved by each algorithm (see \Cref{subsubsec:Detailed Synthetic Functions Results in Regression} for corresponding metrics). We calculate the ranks based on the medians and a 95\% bootstrap CI of \AUC\, and \MNLPD{}. An algorithm loses one rank to each other algorithm that significantly dominates it. Winners are marked in grey. We observe that NOMU is the only algorithm that never comes in third place or worse.
Thus, while some algorithms do particularly well in capturing uncertainties of functions with certain characteristics (e.g., RBF-GPs for polynomials), NOMU is the only algorithm that consistently performs well. HDE's bad performance can be explained by its randomness and the fact that it sometimes predicts zero model uncertainty out-of-sample. For 2D, we provide results and visualizations in \Cref{subsubsec:Detailed Synthetic Functions Results in Regression} and \ref{subsubsec:Uncertainty Bounds Detailed Characteristics} highlighting similar characteristics of all the algorithms as in 1D.
\begin{table}[t!]
    \setlength\tabcolsep{2pt}
    \caption{Ranks (1=best to 5=worst) for \AUC\ and \MNLPD{}.}
    \label{tab:1dsynfunresults}
    \vskip 0.1in
    \begin{center}
    \begin{small}
    \begin{sc}
	\resizebox{1\columnwidth}{!}{
		\begin{tabular}{
				l
				c
				c
				c
				c
				c
				c
				c
				c
				c
				c
			}
			\toprule
			& \multicolumn{2}{c}{\textbf{NOMU}} &  \multicolumn{2}{c}{\textbf{GP}}  & \multicolumn{2}{c}{\textbf{MCDO}}  &  \multicolumn{2}{c}{\textbf{DE}} & \multicolumn{2}{c}{\textbf{HDE}}\\
			\textbf{Function} & \small\AUC & \small\MNLPD & \small\AUC  & \small\MNLPD & \small\AUC & \small\MNLPD & \small\AUC & \small\MNLPD & \small\AUC &\small\MNLPD\\
			\midrule
			Abs & \winc1&\winc1 & 3&\winc1 & 3&4 & \winc1&\winc1&5&5\\
    	    Step & 2&2 & 4&3 & 2&3 & \winc1&\winc1&5&5\\ 
			Kink &\winc1&\winc1 & 3&\winc1 & 4&4 & \winc1&\winc1&5&5\\
			Square & 2&2 & 2&\winc1 & 4&4 &2&3&5&5\\
			Cubic & 2&2 &\winc1&\winc1 & 3&4 & 3&3&5&5\\
			Sine 1 & 2&\winc1 & 2&\winc1 & \winc1&\winc1 & 2&\winc1&5&5\\
			Sine 2 & 2&2 & 3&\winc1 & \winc1&2 & 3&3&5&5\\
			Sine 3 & \winc1&\winc1 & 4&\winc1 & 3&4 & \winc1&\winc1&5&5\\
			Forrester & \winc1&2 & \winc1&\winc1 & 3&4  & 3&3&5&5\\
			Levy & \winc1&\winc1 & 4&3 & \winc1&4 & 3&\winc1&5&5\\
			\bottomrule 
		\end{tabular}
}
\end{sc}
\end{small}
\end{center}
\vskip -0.1in
\end{table}
\subsubsection{Generative Test-Bed}\label{subsubsec:GenerativeTestbed}
\paragraph{Setting} Instead of only relying on a \emph{limited} number of data-sets, we also evaluate all algorithms on a generative test-bed, which provides an unlimited number of test-functions.
The importance of using a test-bed (to avoid over-fitting on specific data-sets) has also been highlighted in a recent work by \citet{osband2021epistemic}.
For our test-bed, we generate $200$ different data-sets by randomly sampling 200 different test-functions from a BNN with i.i.d centered Gaussian weights and three fully-connected hidden layers with nodes $[2^{10},2^{11},2^{10}]$ and ReLU activations.
From each of these test-functions we uniformly at random sample $\ntr=8\cdot d$ input training points and $100\cdot d$ test data points, where $d$ refers to the input dimension. We train all algorithms on these training sets and determine the \NLPD\ on the corresponding test sets averaged over the 200 test-functions. This metric converges to the Kullback-Leibler divergence to the exact BNN-posterior (see \Cref{thm:appendix_dkl_approximation}). We calibrate by choosing per dimension an optimal value of $c$ in terms of average \NLPD, which does not depend on the test-function.
\begin{table}[t!]
    \setlength\tabcolsep{2pt}
    \caption{Average \NLPD\ {\scriptsize(without const. $\ln(2\pi)/2$)} and a 95\% CI over 200 BNN samples. Winners are marked in grey.}
    \label{tab:bnn_uniform}
	\vskip 0.1in
    \begin{center}
    \begin{small}
    \begin{sc}
	\resizebox{1\columnwidth}{!}{
		\begin{tabular}{
				l
				c
				c
				c
				c
				c
				c
				c
				c
				c
				c
			}
			\toprule
			\textbf{Function} & \multicolumn{1}{c}{\textbf{NOMU}} &  \multicolumn{1}{c}{\textbf{GP}}  & \multicolumn{1}{c}{\textbf{MCDO}}  &  \multicolumn{1}{c}{\textbf{DE}} & \multicolumn{1}{c}{\textbf{HDE}}\\
			\midrule
		    BNN1D & \winc-1.65$\pm$\scriptsize0.10 & -1.08$\pm$\scriptsize0.22 & -0.34$\pm$\scriptsize0.23 &  -0.38$\pm$\scriptsize0.36 & \hphantom{-}8.47$\pm$\scriptsize1.00\\
		    BNN2D &  \winc-1.16$\pm$\scriptsize0.05& -0.52$\pm$\scriptsize0.11 &-0.33$\pm$\scriptsize0.13 & -0.77$\pm$\scriptsize0.07 & \hphantom{-}9.11$\pm$\scriptsize0.39 \\
		    BNN5D & \winc-0.37$\pm$\scriptsize0.02 & -0.33$\pm$\scriptsize0.02 & -0.05$\pm$\scriptsize0.04 & -0.13$\pm$\scriptsize0.03 & \hphantom{-}8.41$\pm$\scriptsize1.00 \\
			\bottomrule 
		\end{tabular}
}
\end{sc}
\end{small}
\end{center}
\vskip -0.1in
\end{table}
\paragraph{Results} In \Cref{tab:bnn_uniform}, we provide the results for input dimensions 1, 2 and 5.
We see that NOMU outperforms all other algorithms including MCDO, which is a variational inference method to approximate this posterior \citep{gal2016dropout}.
See \Cref{sec:Appendix:GenerativeTestBed} for a detailed explanation of the experiment setting and further results in modified settings including a discussion of higher dimensional settings.

\subsubsection{Solar Irradiance Time Series}\label{subsubsec:SolarIrradiance}
\paragraph{Setting} Although the current version of NOMU is specifically designed for scarce settings without data noise, we are also interested to see how well it performs in settings where these assumptions are not satisfied. To this end, we now study a setting with many training points and small non-zero data noise. This allows us to analyze how well NOMU captures \ref{itm:Axioms:UncertaintyDecreasesWithMoreTrainingPoints}.
We consider the popular task of interpolating the solar irradiance data \citep{doi:10.1029/2009GL040142} also studied in \citep{gal2016dropout}. We scale the data to $X=[-1,1]$ and split it into 194 training and 197 test points. As in \citep{gal2016dropout}, we choose five intervals to contain only test points. Since the true function is likely of high frequency, we set $\lambda=10^{-19}$ for NOMU and the benchmarks' regularization accordingly. To account for small non-zero data noise, we set $\muexp=0.05$, $\sigmin=0.01$ and use otherwise the same hyperparameters as in \Cref{subsubsec:ToyRegression}.
\begin{figure}[b!]
    \begin{center}
    \centerline{\includegraphics[width=1\columnwidth, trim= 10 180 0 200, clip]{figures/main_paper/irradiance/NOMUred.pdf}}
    \caption{NOMU's model prediction (solid), model uncertainty (dotted) and UBs (shaded area) on the solar irradiance data. Training and test points are shown as black dots and red crosses.}
    \label{fig:irradianceNOMU}
    \end{center}
    \vskip -0.2in
\end{figure}

\paragraph{Results} \Cref{fig:irradianceNOMU} visualizes NOMU's UBs. We see that NOMU manages to fit the training data well while capturing model uncertainty between input training points. In particular, large gaps \emph{between} input training points are successfully modeled as regions of high model uncertainty (\ref{itm:Axioms:largeDistantLargeUncertainty}). Moreover, in regions highly populated by input training points, NOMU's model uncertainty vanishes as required by \ref{itm:Axioms:UncertaintyDecreasesWithMoreTrainingPoints}. Plots for the other algorithms are provided in \Cref{subsubsec:Detailed Real-Data Time Series Results}, where we observe similar patterns as in \Cref{subsubsec:ToyRegression}.

\subsubsection{UCI Data Sets}\label{subsubsec:UCI}
\paragraph{Setting} Recall that NOMU is specifically tailored to noiseless settings with scarce input training data. However, even the current version of NOMU, which does not explicitly model data noise, already performs on par with existing benchmarks on real-world regression tasks \textit{with} data noise. To demonstrate this, we test its performance on (a) the UCI data sets proposed in \citet{hernandez2015probabilistic}, a common benchmark for uncertainty quantification in noisy, real-world regression, and (b) the UCI gap data set extension proposed in \citet{foong2019inbetween}. We consider exactly the same experiment setup as proposed in these works, with a $70/20/10$-train-validation-test split, equip NOMU with a shallow architecture of $50$ hidden nodes, and train it for $400$ epochs. Validation data are used to calibrate the constant $c$ on NLL. See \Cref{subsec:DetailsUCIExperiment} for details on NOMU's setup.

\paragraph{Results} \Cref{tab:uci_nll_detailed} reports NLLs on test data, averaged across 20 splits, compared to the current state-of-the-art. We report the \NLPD\ for MCDO \citep{gal2016dropout} and DE \citep{lakshminarayanan2017simple} from the original papers. Moreover, we reprint the best results (for comparable network sizes) of neural linear models (NLMs) with and without hyperparameter optimization (HPO) from \citep{ober2019benchmarking} (NLM-HPO, NLM), linearized laplace (LL) \cite{foong2019inbetween} and a recent strong MCDO baseline (MCDO2) from \citep{mukhoti2018importance}.
It is surprising that, even though the current design of NOMU does not yet explicitly incorporate data noise, it already performs comparably to state-of-the-art results. We obtain similar results for the UCI gap data. See \Cref{subsec:DetailsUCIExperiment} for more details on this experiment.

\begin{table}[t!]
    \caption{Average \NLPD\ and a $95\%$ normal-CI over 20 runs for \textsc{UCI} data sets. Winners are marked in grey.}
    \label{tab:uci_nll_detailed}
    \robustify\bfseries
    \setlength\tabcolsep{1pt}
    \vskip 0.1in
    \begin{center}
    \begin{small}
    \begin{sc}
    \resizebox{1\columnwidth}{!}{
    \begin{tabular}{lcccccccc}
    \toprule
    \textbf{Dataset}& NOMU & DE & MCDO &  MCDO2 & LL & NLM-HPO & NLM\\
    \midrule
    Boston & 2.68 $\pm$\scriptsize 0.11 & \winc2.41 $\pm$\scriptsize 0.49 & \winc2.46 $\pm$\scriptsize 0.11 &   \winc2.40 $\pm$\scriptsize 0.07 & 2.57 $\pm$\scriptsize 0.09 & \winc2.58 $\pm$\scriptsize 0.17 &     3.63 $\pm$\scriptsize 0.39\\
    Concrete & \winc3.05 $\pm$\scriptsize 0.06 & \winc3.06 $\pm$\scriptsize 0.35 & 3.04 $\pm$\scriptsize0.03 & \winc2.97 $\pm$\scriptsize 0.03 & \winc3.05 $\pm$\scriptsize 0.07 & 3.11 $\pm$\scriptsize 0.09 & 3.12 $\pm$\scriptsize 0.09 \\
    Energy & \winc0.77 $\pm$\scriptsize 0.06 & 1.38 $\pm$\scriptsize 0.43 & 1.99 $\pm$\scriptsize 0.03 & 1.72 $\pm$\scriptsize 0.01 & 0.82 $\pm$\scriptsize 0.05 & \winc0.69 $\pm$\scriptsize 0.05 & \winc0.69 $\pm$\scriptsize 0.05 \\
    Kin8nm & -1.08 $\pm$\scriptsize 0.01 & \winc-1.20 $\pm$\scriptsize 0.03 & -0.95 $\pm$\scriptsize 0.01 & -0.97 $\pm$\scriptsize 0.00 & \winc-1.23 $\pm$\scriptsize 0.01 & -1.12 $\pm$\scriptsize 0.01 & -1.13 $\pm$\scriptsize 0.01 \\
    Naval & -5.63 $\pm$\scriptsize 0.39 & -5.63 $\pm$\scriptsize 0.09 & -3.80 $\pm$\scriptsize 0.01 & -3.91 $\pm$\scriptsize 0.01 & -6.40 $\pm$\scriptsize 0.11 & \winc-7.36 $\pm$\scriptsize 0.15 & \winc-7.35 $\pm$\scriptsize 0.01 \\
    CCPP & \winc2.79 $\pm$\scriptsize 0.01 & \winc2.79 $\pm$\scriptsize 0.07 & \winc2.80 $\pm$\scriptsize 0.01 & \winc2.79 $\pm$\scriptsize 0.01 & 2.83 $\pm$\scriptsize 0.01 & \winc2.79 $\pm$\scriptsize 0.01 & \winc2.79 $\pm$\scriptsize 0.01 \\
    Protein & \winc2.79 $\pm$\scriptsize 0.01 & 2.83 $\pm$\scriptsize 0.03 & 2.89 $\pm$\scriptsize 0.00 & 2.87 $\pm$\scriptsize 0.00 & 2.89 $\pm$\scriptsize 0.00 & \winc2.78 $\pm$\scriptsize 0.01 & 2.81 $\pm$\scriptsize 0.00 \\
    Wine & 1.08 $\pm$\scriptsize 0.04 & \winc0.94 $\pm$\scriptsize 0.23 & \winc0.93 $\pm$\scriptsize 0.01 & \winc0.92 $\pm$\scriptsize 0.01 & 0.97 $\pm$\scriptsize 0.03 & 0.96 $\pm$\scriptsize 0.01 & 1.48 $\pm$\scriptsize 0.09 \\
    Yacht & \winc1.38 $\pm$\scriptsize 0.28 & \winc1.18 $\pm$\scriptsize 0.41 & 1.55 $\pm$\scriptsize 0.05 & 1.38 $\pm$\scriptsize 0.01 & \winc1.01 $\pm$\scriptsize 0.09 & \winc1.17 $\pm$\scriptsize 0.13 & \winc1.13 $\pm$\scriptsize 0.09 \\
    \bottomrule
    \end{tabular}
    }
    \end{sc}
    \end{small}
    \end{center}
    \vskip -0.1in
\end{table}

\subsection{Bayesian Optimization}
\label{subsec:BayesianOptimization}
In this section, we assess the performance of NOMU in high-dimensional noiseless \emph{Bayesian optimization (BO)}.

\paragraph{Setting} In BO, the goal is to maximize an unknown expensive-to-evaluate function, given a budget of function queries. We use a set of test functions with different characteristics from the same library as before, but now in $5$ to $20$ dimensions $d$, transformed to $X=[-1,1]^d,\ f(X)=[ -1,1]$.\footnote{These functions are designed for minimization. We multiply them by $-1$ and equivalently maximize instead.} Additionally, we again use Gaussian BNNs to create a generative test-bed consisting of a large variety of test functions (see \Cref{subsubsec:GenerativeTestbed}).
For each test function, we randomly sample $8$ initial points~$\left(x_i,f(x_i)\right)$ and let each algorithm choose $64$ further function evaluations (one by one) using its upper UB as acquisition function. %
This corresponds to a setting where one can only afford $72$ expensive function evaluations in total.
We provide details regarding the selected hyperparameters for each algorithm in \Cref{subsubsec:Hyperparameters}. We measure the performance of each algorithm based on its \emph{final regret} $|\max_{x\in X} f(x)-\max_{i\in\fromto{72}} f(x_i)|/|\max_{x\in X} f(x)|$. 

For each algorithm, the UBs must be calibrated by choosing appropriate values of $c$. We do so in the following straightforward way: First, after observing the $8$ random initial points, we 
determine those two values of $c$ for which the resulting mean width (MW) of the UBs is $0.05$ and $0.5$, respectively (\textsc{MW scaling}).\footnote{We fix MW instead of $c$, since the scales of the algorithms vary by orders of magnitude.} We perform one BO run for both resulting initial values of $c$. Additionally, if in a BO run, an algorithm were to choose a point~$x_{i'}$ very close to an already observed point $x_i$, we dynamically increase $c$ to make it select a different one instead (\textsc{Dynamic c}; see \Cref{subsubsec:Calibration} for details). A value of $0.05$ in \textsc{MW scaling} corresponds to small model uncertainty, such that exploration is mainly due to \textsc{Dynamic c}. Smaller values than $0.05$ thus lead to similar outcomes. In contrast, a value of $0.5$ corresponds to large model uncertainties, such that \textsc{Dynamic c} is rarely used. Only for the \enquote{plain GP (pGP)} we use neither \textsc{MW scaling} nor \textsc{Dynamic c}, as pGP uses its default calibration ($c$ is determined by the built-in hyperparameter optimization in every step). However, a comparison of GP and pGP suggests that \textsc{MW scaling} and \textsc{Dynamic c} surpass the built-in calibration (see \Cref{tab:BOresults}). As a baseline, we also report random search (RAND).

\sisetup{output-exponent-marker=\ensuremath{\mathrm{e}}}
\begin{table}[t!]
    \setlength\tabcolsep{2pt}
    \caption{BO results: average final regrets per dimension and ranks for each individual function (1=best to 7=worst).}
    \label{tab:BOresults}
    \vskip 0.1in
    \begin{center}
    \begin{small}
    \begin{sc}
    \resizebox{1\columnwidth}{!}{
    \begin{tabular}{
    		l
    		c
    		c
    		c
    		c
    		c
    		c
    		c
    	}
    		\toprule
    		\textbf{Function} & \textbf{NOMU}  & \textbf{GP}  &  \textbf{MCDO} & \textbf{DE} & \textbf{HDE} & \textbf{pGP}&\textbf{RAND} \\
    		
    		\midrule
    	    Levy5D & \winc 1 & \winc 1 & 6 & 3 & 3 &4 & 7  \\ 
            Rosenbrock5D & \winc 1 & \winc 1 & \winc 1 & \winc 1 & 2 &5 & 7  \\
            G-Function5D & 2 & 3 & \winc{1} & 4 & 2 & 3 & 7  \\
            Perm5D & 3 & \winc1 & \winc 1 & 5 & 7 &2 & 4  \\
            BNN5D & \winc{1} & \winc{1} & 4 & \winc{1} & 4& \winc{1}& 7  \\ 
    		\midrule
    		\emph{Average Regret 5D}& \winc\num{2.87e-02} & \num{5.03e-02}  & \num{4.70e-02}& \num{5.18e-02} & \num{7.13e-02}& \num{4.14e-02} & \num{1.93e-01} \\
    		\midrule
    		\midrule
            Levy10D & \winc 1 & 3 & 5 & 6 & \winc{1} &\winc{1} & 6  \\ 
            Rosenbrock10D & \winc{1} & \winc{1} & 2 & 6 & 3&2 & 7\\
            G-Function10D & 2 & 5 & \winc1 & 3 & 2& 5 & 7  \\
            Perm10D & 2 &  \winc{1} & 2 &6 & 2&2 & \winc{1}  \\
            BNN10D & \winc{1} & 2 & \winc{1} & \winc{1} & 3 & \winc1 & 7  \\ 
            \midrule
    		\emph{Average Regret 10D}& \winc \num{8.40e-02}   & \num{1.17e-01}  & \winc \num{6.96e-02}& \num{1.15e-01}& \num{9.32e-02}& \num{9.46e-02}& \num{2.35e-01}\\
    		\midrule\midrule
    	    Levy20D & \winc{1} & \winc{1} & 5 & 7 &\winc{1} &\winc{1} & 6  \\ 
            Rosenbrock20D & 2& 2 & 2 & 6 & \winc{1}& 4 & 6  \\
            G-Function20D & \winc{1} & 4 & 5 & \winc{1} &\winc{1} & 3 & 7  \\
            Perm20D & 3 & 5 & 3 & 2 &3 & 3 & \winc{1}  \\
            BNN20D & \winc{1} & 2 & 2 & 2 & 6 & \winc1 & 7  \\ 
    		\midrule
    		\emph{Average Regret 20D}& \winc\num{1.12e-01} &  \num{1.33e-01} & \num{1.39e-01}& \num{1.71e-01}& \num{1.37e-01}& \winc\num{1.17e-01} &\num{2.80e-01}\\
    		\midrule
    		\bottomrule 
    \end{tabular}%
    }
    \end{sc}
    \end{small}
    \end{center}
    \vskip -0.1in
\end{table}
\begin{figure}[b!]
    \begin{center}
    \centerline{\includegraphics[width=1\columnwidth, trim=33 0 50 15 ,clip]{figures/main_paper/regret_plots/bnn_20d_mix_no3.pdf}}
    \caption{Regret plot for BNN20D. For each BO step, we show the regrets averaged over 50 runs (solid lines) with 95\% CIs.}
    \label{fig:BNN20D}
    \end{center}
    \vskip -0.2in
\end{figure}
\paragraph{Results} In \Cref{tab:BOresults}, we present the BO results in 5D, 10D and 20D. We show the average final regret per dimension across the five functions. For each algorithm and dimension, we give the results corresponding to the MW scaling parameter (0.05 or 0.5) that minimizes the average final regret across that dimension (see \Cref{subsubsec:Detailed Results} for both MWs).
In practice, one often only knows the dimensionality of a given BO task, which is why we use the average final regret per dimension as the criterion for setting the optimal MW. For each individual function, we also present the ranks based on the final regret and a 95\% CI over 100 (5D) and 50 (10-20D) runs. We see that NOMU performs as well or better than all benchmarks in terms of average final regret.
By inspecting the ranks achieved for each individual function, we further observe that NOMU is never ranked worse than $3$\textsuperscript{rd}. In contrast, the performance of the benchmarks heavily depends on the test function; and each benchmark is ranked $4$\textsuperscript{th} and worse multiple times.
For Perm10D/20D we see that RAND performs best. However, due to a flat optimum of Perm, all algorithms achieve similar (very small) final regrets. 
Finally, we see that NOMU is always ranked first for the BNN test functions. \Cref{fig:BNN20D} shows the regret plot for BNN20D (see \Cref{subsubsec:Regret Plots} for all regret plots).
\section{Conclusion}
We have introduced NOMU, a new algorithm for estimating model uncertainty for NNs, specialized for scarce and noiseless settings.
By using a specific architecture and carefully-designed loss function, we have ensured that NOMU satisfies five important desiderata regarding model uncertainty that any method should satisfy.
However, when analyzing model uncertainty decoupled from data noise, we have experimentally uncovered that, perhaps surprisingly, established state-of-the-art methods fail to reliably capture some of the desiderata, even those that are required by Bayesian theory. In contrast, NOMU satisfies all desiderata, matches the performance of all benchmarks in regression tasks, and performs as well or better in noiseless BO tasks. We see great potential to further improve NOMU, for example by adapting the loss, or by modifying the connections between the two sub-NNs.
We also envision several extensions of NOMU, including its application to classification, employing different architectures (CNNs, GNNs, RNNs or Transformers), and incorporating data noise.

\section*{Acknowledgements}
We thank Marius H\"ogger and Aurelio Dolfini for insightful discussions and their excellent research assistance in implementing the Bayesian optimization experiments and the UCI data set experiments, respectively. Furthermore, we thank the anonymous reviewers for helpful comments. This paper is part of a project that has received funding from the European Research Council (ERC)
under the European Union’s Horizon 2020 research and innovation programme (Grant agreement No. 805542).

\clearpage{}%

\bibliographystyle{icml2022}%
\icmltitlerunning{Appendix of NOMU: Neural Optimization-based Model Uncertainty}
\clearpage
\appendix
\section*{\centering Appendix}
\section{Theoretical Analysis of NOMU}\label{sec:TheoreticalAnalysisAppendix}
In this section, we
\begin{enumerate}[leftmargin=*,topsep=0pt,partopsep=0pt, parsep=0pt]
\item first provide a theoretical motivation for the design of NOMU and establish via \Cref{thm:approxOfUBs} a connection to pointwise worst-case UBs $\lUBp,\uUBp$ \textbf{(\Cref{subsec:connectingNOMUtoPointwiseUncertaintyBounds})}.
\item Next, we provide a Bayesian interpretation of those pointwise worst-case UBs $\lUBp,\uUBp$ and elaborate on \emph{relative uncertainties} \textbf{(\Cref{subsec:Pointwise Uncertainty Bounds})}.
\item Then, we discuss the case if $\HC_{\Dtr}$ is not \emph{upwards directed} as assumed in \Cref{thm:approxOfUBs} and further show how we deal with this challenge \textbf{(\Cref{subsec:A note on Theorem})}.
\item Finally, we prove \Cref{thm:approxOfUBs} \textbf{(\Cref{subsec:ProofofTheorem})}.
\end{enumerate}

\subsection{Relating NOMU to Pointwise Worst-Case Uncertainty Bounds}\label{subsec:connectingNOMUtoPointwiseUncertaintyBounds}
In this section, we provide a theoretical motivation for the design of NOMU. To this end, we first define worst-case UBs. We then state a theorem connecting these worst case bounds and NOMU's UBs via NOMU's loss function. In what follows, we assume zero data noise~$\signoise=0$.

Consider a \emph{hypothesis class} $\HC$ given as a subset of a Banach space of functions $f:X\to Y$.
Furthermore, let $\HC_{\Dtr}:=\{f\in \HC:f(\xtr_i)=\ytr_i, i =1,\ldots,\ntr\}$ denote the set of all functions from the hypothesis class that fit through the training points and let $\fhat\in \HC_{\Dtr}$ be a prediction function (e.g., a NN trained on $\Dtr$).
Worst-case bounds within the class $\HC_{\Dtr}$ can be defined \textbf{p}oint\textbf{w}ise for each $x\in \X$ as:
\begin{align}
    \lUBp(x):=\inf_{f\in \HC_{\Dtr}}f(x),\label{eq:upperSupBound}\\
    \uUBp(x):=\sup_{f\in \HC_{\Dtr}}f(x).\label{eq:lowerInfBound}
\end{align}
By definition, these UBs are the tightest possible bounds that cover every $f\in \HC_{\Dtr}$ (i.e., $\lUBp(x)\le f(x)\le \uUBp(x)\  \forall x \in \X$). 
From a Bayesian perspective, such bounds correspond to credible bounds for $\alpha=1$ if the support of the prior is contained in $\HC$. Interestingly, if $\HC$ is the class of regularized NNs, these bounds can also be interpreted as an approximation of credible bounds for $\alpha<1$ with respect to a Gaussian prior on the parameters of a NN (see Appendix A.2. for a derivation).\footnote{Standard BNNs also aim to approximate the posterior coming from exactly this prior.}

In applications like BO, when optimizing an acquisition function based on these pointwise-defined bounds, we require the UBs for \emph{all} $x \in \X$. Thus, numerically optimizing such an acquisition function is practically infeasible, as it would require solving the optimization problems from \eqref{eq:upperSupBound} millions of times.
In NOMU, we circumvent this problem by constructing the UBs for \textit{all} $x\in \X$ simultaneously. We can do so by only solving a \emph{single} optimization problem, i.e., minimizing the NOMU loss from \eqref{eq:lmu2}.

In the following theorem, we show that these pointwise-defined UBs can be computed by solving a single optimization problem under the following assumption.%

\begin{assumption}[Upwards Directed]\label{assumption:upwards direction}
For every $f_1,f_2 \in \HC_{\Dtr}$ there exists an $f \in \HC_{\Dtr}$ such that $f(x)\ge \max(f_1(x),f_2(x))$ for all $x \in X$.
\end{assumption}

\begin{theorem}[Single optimization problem]\label{thm:approxOfUBs}
Let $X=\prod_{i=1}^{d}[a_i,b_i]\subset\R^d,\, a_i<b_i$, let $Y=\R$, and let $\Dtr$ be a nonempty set of training points. Furthermore, let $\HC_{\Dtr}\subset (C(X,Y),\|\cdot\nolinebreak\|_\infty)$ be compact and \emph{upwards directed} and $\fhat\in \HC_{\Dtr}$. Then, for every strictly-increasing and continuous $u:\R\to\R$, it holds that
\begin{align}\label{eq:IntegralEqualityUB}
    \uUBp=\argmax_{h\in \HC_{\Dtr}}\int_{X}\hspace{-0.15cm}u(h(x)-\fhat(x))\,dx.\hspace{-0.15cm}
\end{align}
\end{theorem}
In words, $\uUBp$ can be calculated via the single optimization problem~\eqref{eq:IntegralEqualityUB} on $\HC_{\Dtr}$.\footnote{We formulate \Cref{thm:approxOfUBs} for upper UBs~$\uUBp$. The analogous statement also holds for lower UBs~$\lUBp$.}
\begin{proof}
See \Cref{subsec:ProofofTheorem}.
\end{proof}
In practice, \Cref{assumption:upwards direction} can be violated such that a straightforward calculation of the r.h.s. of \eqref{eq:IntegralEqualityUB} for an \emph{arbitrary} $u$ would result in unreasonable UBs. However, for a sensible choice of $u$, NOMU's UBs based on the r.h.s. of \eqref{eq:IntegralEqualityUB} still satisfy our Desiderata
\hyperref[itm:Axioms:trivial]{D1}\crefrangeconjunction\hyperref[itm:Axioms:UncertaintyDecreasesWithMoreTrainingPoints]{D5},
similar to $\uUBp$ (see \Cref{subsec:A note on Theorem} for a discussion).

The connection of NOMU's UBs to the pointwise worst-case bounds $\lUBp, \uUBp$ becomes clear by observing that minimizing NOMU's loss function $L^{\hp}$ (Equation~\eqref{eq:lmu2}) can be interpreted as solving the r.h.s of \eqref{eq:IntegralEqualityUB} for a specific choice of $u$, when $\HC$ is the class of regularized NNs. In detail:
\begin{itemize}[leftmargin=*,topsep=0pt,partopsep=0pt, parsep=0pt]
\item Term~\terma{} of the NOMU-loss \eqref{eq:lmu2} implements that $\fhat$ solves the regression task and thus $\fhat\in \HC_{\Dtr}$ up to numerical precision (if the regularization $\lambda$ is small enough).
\item Term~\termb{} enforces $\sigmodelhatraw(\xtr_i)\approx0$  and thus when defining $h:=\fhat+\sigmodelhatraw$, we directly obtain $h(\xtr_i)\approx\ytr_i$  corresponding to the constraint $h\in\HC_{\Dtr}$ in \eqref{eq:IntegralEqualityUB}.
\item While terms~\terma{} and \termb{} enforce the constraints of \eqref{eq:IntegralEqualityUB}, term~\termc{} is the objective function of \eqref{eq:IntegralEqualityUB} for the specific choice of ${u(z):=-e^{-\cexp z}, \cexp \in \Rpz}$.%
\end{itemize}

\subsection{Bayesian Interpretation of Pointwise Worst-Case Uncertainty Bounds}\label{subsec:Pointwise Uncertainty Bounds}
In this section, we provide a Bayesian interpretation of the pointwise worst-case UBs $\lUBp,\uUBp$ and elaborate on \emph{relative uncertainties}.

In the following, we denote by $\NNf:\X\to\Y$  a (standard) NN for model predictions. Note that $\NNf$ does not represent the whole NOMU architecture but can be used as $\fhat$-sub-network in the NOMU architecture~$\NN_\theta$. Furthermore, we consider the hypothesis class of regularized NNs, i.e., $\HC:=\left\{\NNf: \twonorm[\theta]\leq\gamma\right\}$. Recall that one needs to assume that the prior is fully concentrated on $\HC$ in order to interpret the pointwise UBs $\lUBp,\uUBp$ as $\alpha{=}1$ CBs. In the following, we will present an alternative Bayesian interpretation of $\uUBp$.

Many other approaches (MC dropout, BNNs) assume a Gaussian prior on the parameters of the NNs, i.e, $\theta\sim\mathcal{N}(0,\sigma^2_{\theta}I)$, and try to approximate the corresponding posterior CBs. Interestingly, $\lUBp$ and $\uUBp$ can also be seen as approximations of $\alpha{<}1$ CBs in the case of such a Gaussian prior on the parameters. This can be seen as follows:

Let the data generating process be given as $y=\NNf+\varepsilon$, with $\varepsilon\sim\mathcal{N}(0,\signoise^2)$.\footnote{For simplicity we assume homoskedastic noise in this section.} For a Gaussian prior on the parameters $\theta\sim\mathcal{N}(0,\sigma^2_{\theta}I)$ the negative log posterior can be written as
\begin{align}
    -\log(p(\theta|\Dtr))=&\frac{1}{2\signoise^2}\overbrace{\sum_{i=1}^{\ntr} \left(\NNf(\xtr_i)-\ytr_i\right)^2}^{=:L(\theta)}\nonumber\\
    &+\frac
{\twonorm[\theta]^2}{2\sigma_{\theta}^2}+\ntr\log(\signoise)+C_{\ntr,\sigma_{\theta}}.
\end{align}
for a constant $C_{\ntr,\sigma_{\theta}}:=\frac{\ntr}{2}\log(2\pi)+\frac{1}{2}\log(2\pi\sigma^2_{\theta})$. Then the pointwise upper UBs can be reformulated to
\begin{align}
    \uUBp(x)&\overset{\text{def.}}{=}\sup_{f\in \HC_{\Dtr}}f(x)=\lim\limits_{\signoise\to 0}\sup_{f\in \tilde{\HC}_{\Dtr}^{\signoise}}f(x)
    \label{eq:upperSupBoundAppendix}
\end{align}
with
\begin{subequations}
\begin{align}
\tilde{\HC}_{\Dtr}^{\signoise}:&=\left\{\NNf: \frac{\sigma_{\theta}^2}{\signoise^2}L(\theta) +\twonorm[\theta]^2\leq\gamma\right\}\\
&=\left\{\NNf: \log(p(\theta|\Dtr))\geq
\tilde{\gamma}_{\signoise}\right\}\label{subeq:BayesianHDtr}
\end{align}
\end{subequations}

where $\tilde{\gamma}_{\signoise}:=-\frac{\gamma}{2\sigma_{\theta}^2}-\ntr\log(\signoise)-C_{\ntr,\sigma_{\theta}}$.

Therefore, for small data noise $\signoise\approx0$ we obtain from \cref{eq:upperSupBoundAppendix,subeq:BayesianHDtr} that
\begin{align}\label{eq:pseudoHPDI}
\uUBp(x)\approx\sup_\theta\left(\NNf(x):p(\theta|\Dtr)>e^{\tilde{\gamma}_{\signoise}}\right).
\end{align}
In words, from a Bayesian point of view, we seek the highest value $\NNf(x)$ for which the posterior density $p(\theta|\Dtr)>e^{\tilde{\gamma}_{\signoise}}$, which can be seen as a heuristic to approximate CBs analogously to the MAP on the parameter-space as a popular heuristic to approximate the posterior mean.

\begin{remark}
If $p(\NNf(x)|\Dtr)=p(\theta|\Dtr)$ and this posterior is unimodal, \cref{eq:pseudoHPDI} is a classical \emph{highest posterior density interval (HPDI)}. However, typically $p(\NNf(x)|\Dtr)\neq p(\theta|\Dtr)$. Thus, NOMU is only a heuristic to approximate Gaussian BNN posterior CBs (analogously to the MAP). \citet{HeissPart3MultiTask} show that such heuristics can perform better than exact BNNs (e.g., in contrast to BNNs, NOMU fulfills \hyperref[itm:Axioms:irregular]{D4} also for infinite-width).
\end{remark}

\subsubsection{From Absolute to Relative Model Uncertainty}\label{subsubsec:relativeModelUncertainty}
If the prior scale (e.g., $\gamma$ or $\sigma_{\theta}$ in the above mentioned approaches) is known, in theory no further calibration is needed and one can interpret the resulting UBs in \emph{absolute} terms \citep{pmlr-v80-kuleshov18a}. However, typically, the prior scale is unknown and the resulting UBs can only be interpreted in terms of \emph{relative} model uncertainty (i.e., how much more model uncertainty does one have at one point $x$ compared to any other point $x'$?) and in a second step careful calibration of the resulting UBs is required.

Thus, there are two (almost) independent problems: First, the fundamental concept of how to estimate relative model uncertainty (such as MC dropout, deep ensembles, hyper deep ensembles or NOMU) and second, the calibration of these UBs. In this paper, we do not mix up these two challenges and only focus on the first one. Furthermore, desiderata \hyperref[itm:Axioms:trivial]{D1}\crefrangeconjunction\hyperref[itm:Axioms:irregular]{D4} and our metrics \AUC\ and \MNLPD\ do not depend on the scaling of the uncertainty. Whenever  we use \NLPD\ as a metric, we make sure to calibrate the uncertainties by a scalar $c$ to ensure that our evaluations do not depend on the scaling of the uncertainty predictions.

\subsection{A Note on \Cref{thm:approxOfUBs} }\label{subsec:A note on Theorem}
In practice, the set~$\HC_{\Dtr}$ often is not upwards directed for typical NN-architectures and \Cref{eq:IntegralEqualityUB} of \Cref{thm:approxOfUBs} is not fulfilled in general.
Indeed, $h^*:=\argmax_{h\in \HC_{\Dtr}}\int_{X}u(h(x)-\fhat(x))\,d(x)$ can be much more overconfident than $\uUBp$. 
However, in the following we will motivate why for a suitably chosen $u$ the relative model uncertainty of $h^*$ still fulfills the desiderata
\ref{itm:Axioms:trivial}, \ref{itm:Axioms:ZeroUncertaintyAtData} and \ref{itm:Axioms:largeDistantLargeUncertainty}.

\begin{remark}
In our proposed final algorithm from Section~\ref{sec:Neural Optimization-based Model Uncertainty} NOMU, we incorporated \ref{itm:Axioms:irregular} by  
modifying the network architecture as presented in \Cref{subsec:The Network Architecture}.
\end{remark}

\paragraph{Problem} First, we give an example of a specific NN-architecture, where \Cref{thm:approxOfUBs} is not fulfilled due to a violation of the upwards directed assumption. Note that from a Lagrangian perspective it is equivalent to add a term $-\lambda\twonorm[\theta]^2, \lambda\ge0$ to the target function of \eqref{eq:IntegralEqualityUB} instead of bounding $\twonorm[\theta]^2\leq\gamma^2$ as a constraint in $\HC_{\Dtr}$. Moreover, for a specific NN-architecture it is shown that regularizing $\twonorm[\theta]^2$ is equivalent to regularizing the $L_2$-norm of the second derivative of the function $f=\NNf$ \citep{heiss2019implicit1}, i.e., regularizing $\int_{\X}f''(x)^2\,dx$. 
If we choose for example $u=id:\R\to\R$, then increasing $h$ in-between the two close points in the middle of \Cref{fig:Wrong_u_breaks_D2}, would improve the target function of \eqref{eq:IntegralEqualityUB} $\int u(h(x)-\fhat(x))\,dx$ less than the additional regularization cost resulting from the additional second derivative when increasing $h$.
Therefore, $h^*$ will be below the mean prediction $\fhat$ in this region, break \ref{itm:Axioms:trivial}, and $h^*\neq\uUBp$.

\begin{figure}[h!]
    \vskip 0.2in
    \begin{center}
    \centerline{
    \includegraphics[page=1,trim= 180 180 180 175, clip, width=\columnwidth]{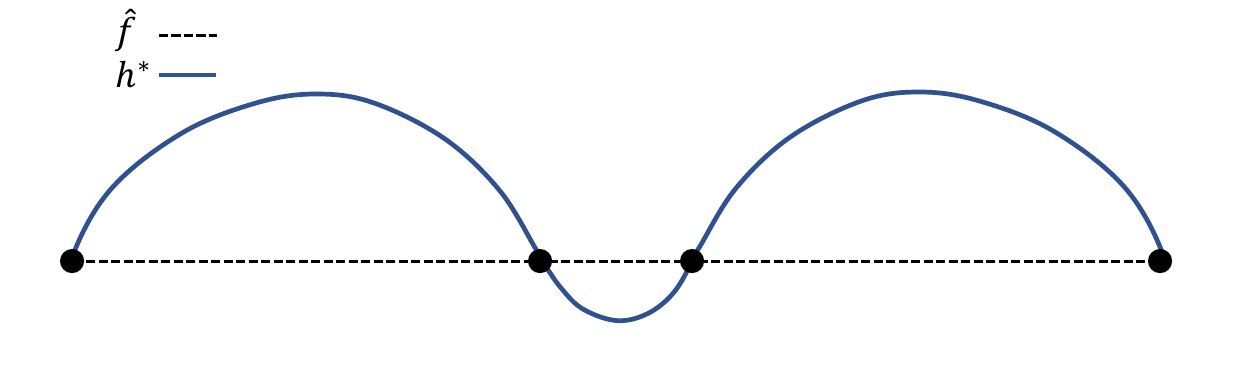}}
    \caption{$h^*$ would not fulfill \ref{itm:Axioms:trivial}, if $u$ was chosen as the identity and the architecture from \citet{heiss2019implicit1} was used.}
    \label{fig:Wrong_u_breaks_D2}
    \end{center}
    \vskip -0.2in
\end{figure}

\paragraph{Solution} However, if e.g., we choose $u:x\mapsto-e^{-\cexp x}$ with $\cexp$ large enough, we highly penalize $h<\fhat$, and thus $\fhat\leq h^*$ \ref{itm:Axioms:trivial} is fulfilled.
Since $h^* \in \HC_{\Dtr}$, it follows that $h^*(\xtr_i)-\fhat(\xtr_i)=0$ for all training points. This together with $\fhat\leq h^*$  implies that $\nabla\left( h^*-\fhat\right)(\xtr_i)=0$ at the training points (in the interior of $\X$), which subsequently
interrupts the downwards trend of $h^*-\fhat$ from one side of a point to the other side of a point as depicted in \Cref{fig:Better_u_fulfills_D2}.

\begin{figure}[h!]
    \vskip 0.2in
    \begin{center}
    \centerline{
   \includegraphics[page=2,trim= 180 180 180 175, clip, width=\columnwidth]{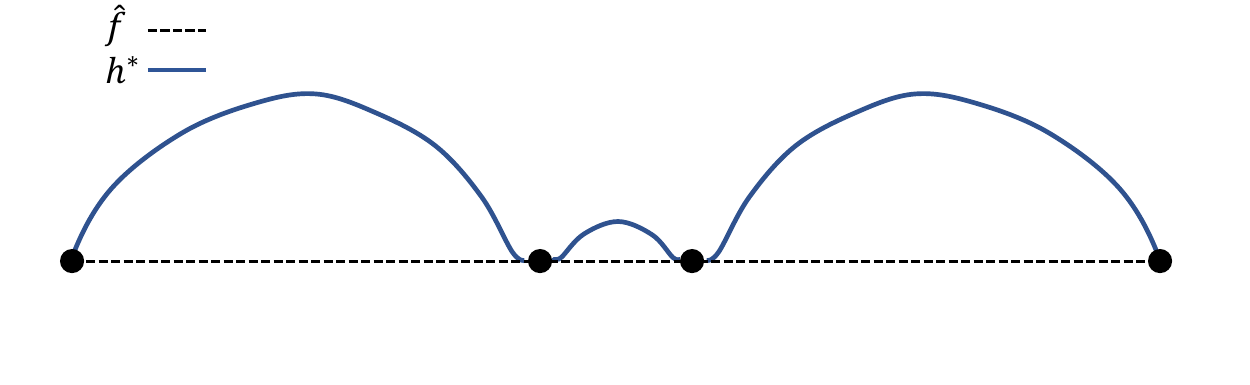}}
    \caption{$h^*$ fulfills \hyperref[itm:Axioms:trivial]{D1}\crefrangeconjunction\hyperref[itm:Axioms:largeDistantLargeUncertainty]{D3} if $u$ prevents it from getting negative.}
    \label{fig:Better_u_fulfills_D2}
    \end{center}
    \vskip -0.2in
\end{figure}

See \Cref{appendix:Desiderata} for a detailed discussion how NOMU fulfills all five desiderata.

\subsection{Proof of \Cref{thm:approxOfUBs}}\label{subsec:ProofofTheorem}
In the following, we prove \Cref{thm:gernelazidApproxOfUBs}, an even more general version of \Cref{thm:approxOfUBs}. The statement of \Cref{thm:approxOfUBs} follows from \Cref{thm:gernelazidApproxOfUBs} by setting
\begin{enumerate}[leftmargin=*,topsep=0pt,partopsep=0pt, parsep=0pt]
    \item $\X=\prod_{i=1}^{d}[a_i,b_i]\subset\R^d$ compact, $Y:=\R$
    \item $\mu=\lambda_d$, where $\lambda_d$ is the Lebesgue measure on $\R^d$ (which is finite on every compact set and has full support).
\end{enumerate}

\begin{theorem}\label{thm:gernelazidApproxOfUBs}
Let $\X$ be a nonempty topological space equipped with a finite measure $\mu$ with full support\footnote{I.e. every nonempty open set has non-zero measure.}, let $\Y$ be a normed and totally ordered space and let $\Dtr$ denote a set of observations. Moreover, let $\HC_{\Dtr}\subset (C(X,Y),\|\cdot\|_\infty)$ be compact, and $\fhat\in \HC_{\Dtr}$. Assume further that $\HC_{\Dtr}$ is upwards directed (see \Cref{assumption:upwards direction}). Then, for every strictly-increasing and continuous $u:Y\to\R$ it holds that
\begin{align}\label{eq:IntegralEqualityUBAppendix}
 \uUBp=\argmax_{h\in \HC_{\Dtr}}\int_{X}\hspace{-0.2cm}u(h(x)-\fhat(x))\,d\mu(x).
\end{align}%
\end{theorem}

\begin{proof}
First note that since $\mu$ is finite and $h, \fhat$ and $\uUBp$ are bounded since $\HC_{\Dtr}$ is assumed to be compact with respect to the $\|\cdot\|_\infty$-topology, it holds that the integral in \eqref{eq:IntegralEqualityUBAppendix} is finite.

Let $h^* \in \HC_{\Dtr}$ denote an optimizer of \eqref{eq:IntegralEqualityUBAppendix}, which
exists since $\HC_{\Dtr}$ is assumed to be compact and the operator $h\mapsto\int_{X}\hspace{-0.05cm}u(h(x)-\fhat(x))\,d\mu (x)$ is continuous on the $\|\cdot\|_\infty$-topology.

We want to show that $h^*(x)=\uUBp(x)$ for every $x\in X$.

Note that per definition for all $x\in X$ and $h\in \HC_{\Dtr}$ it holds that
\begin{align}\label{eq:proofThm0}
\uUBp(x)=\sup\limits_{f\in \HC_{\Dtr}}f(x)\ge h(x).
\end{align}
Thus $\uUBp(x)\ge h^*(x)$ for all $x\in \X.$

For the reverse inequality assume on the contrary that there exists an $x^\prime\in X$ such that
\begin{align}
    \uUBp(x^\prime)>h^*(x^\prime).
\end{align}
Then, we define $f_{x^\prime}:=\argmax_{f\in \HC_{\Dtr}}f(x^\prime)$, which exists because of compactness and continuity. Since $f_{x^\prime}$ and $h^*$ are both continuous and $f_{x^\prime}(x^\prime)=\uUBp(x^\prime)>h^*(x^\prime)$ there exists a neighbourhood $U_{x^\prime}$ of $x^\prime$ such that 
\begin{align}\label{eq:proofThm1}
    f_{x^\prime}(x)>h^*(x) \text{ for all } x \in U_{x^\prime},
\end{align}

Using the \emph{upwards directed} property with $f_1:=f_{x^\prime}$ and $f_2:=h^*$, it follows that there exists a $\tilde{h} \in \HC_{\Dtr}$ with 
\begin{align}\label{eq:proofThm2}
\tilde{h}(x)\ge \max(f_{x^\prime}(x),h^*(x)) \textrm{ for all } x\in \X.
\end{align}
Using \eqref{eq:proofThm1} together with \eqref{eq:proofThm2} implies further that
\begin{align}
   &\tilde{h}(x)\ge h^*(x) \text{ for all } x \in \X \text{ and }\\
   &\tilde{h}(x) > h^*(x) \text{ for all } x \in U_{x^\prime}.
\end{align}
However, since $u$ is strictly increasing and $\mu(U_{x^\prime})>0$ by the full support assumption, we get that
\begin{align}
    \int_{X}\hspace{-0.05cm} u(\tilde{h}(x)-\fhat(x))\,d\mu(x) > \hspace{-0.05cm}\int_{X}\hspace{-0.05cm}u(h^*(x)-\fhat(x))\,d\mu(x),
\end{align}
which is a contradiction to the assumption that $h^*$ is the optimizer of \eqref{eq:IntegralEqualityUBAppendix}. Therefore, it holds that $\uUBp(x)\le h^*(x)$ for all $x \in \X$.

In total we get that $\uUBp(x)=h^*(x)$ for all $x \in \X$, which concludes the proof.
\end{proof}

\begin{remark} 
For our algorithm we select
$\HC:=\left\{\NNf: \twonorm[\theta]\leq\gamma\right\}$ to be the class of regularized NNs with a continuous activation function. Thus, the assumption of $\HC_{\Dtr}$ being compact and a subset of $C(X,Y)$ is fulfilled.
\end{remark}

\section{Experiments}\label{sec:Experiments}

\subsection{Benchmark Methods}\label{subec:Benchmarks}
In this section, we give a brief overview of each considered benchmark algorithm.
\subsubsection{Gaussian Process (GP)}\label{subsubsec:GaussianProcess}
A GP defines a distribution over a set of functions $\{f:\X\to \Y\}$ where every finite collection of function evaluations follows a Gaussian distribution (see \citet{williams2006gaussian} for a survey on GPs). A GP is completely specified by a \emph{mean} $m:\X\to\Y$ and a \emph{kernel} function $k_\hp:\X\times\X\to\Y$, where $\hp$ denotes a tuple of hyper-parameters. More formally, for any finite set of $k\in \N$ input points $\boldsymbol{x}:=\fromto[x_1]{x_k},\, x_i \in \X$ and for $\boldsymbol{f}:=(f_1,\ldots,f_k)\in \Y^k$ with $ f_i:=f(x_i)$ it holds that
\begin{align}
   \boldsymbol{f}\sim \mathcal{N}_k\left(\boldsymbol{m}(\boldsymbol{x}),\gramm{\boldsymbol{x}}{\boldsymbol{x}}\right),
\end{align}
i.e., it follows a $k$-dimensional Gaussian distribution with covariance (or Gramian) matrix $\left[\gramm{\boldsymbol{x}}{\boldsymbol{x}}\right]_{i,j}:=k_\hp(x_i,x_j),$ and mean vector $\boldsymbol{m}(\boldsymbol{x}):=\left(m(x_1),\ldots,m(x_k)\right)$. Let
\begin{align}\label{GP}
f\sim \GP(m(\cdot),k_\hp(\cdot,\cdot)),
\end{align}
denote a GP with mean function $m$ and kernel $k_\hp$.

In the following, we summarize the main steps of GP regression for a 1D-setting and $m\equiv0$.
\begin{enumerate}[topsep=0pt,itemsep=0pt,partopsep=0pt, parsep=0pt,labelindent=0pt,wide]
\item \emph{Define probabilistic model:}
\begin{align}
y=f(x)+\varepsilon.
\end{align}
\item \emph{Specify prior (kernel and data noise):}
\begin{align}
   f&\sim \GP(0,k_\hp(\cdot,\cdot)),\\
   \varepsilon|x&\sim \mathcal{N}\left(0,\signoise^2(x)\right).
\end{align}
\item Calculate \emph{likelihood} for training points $\boldsymbol{x}$ and ${\boldsymbol{y}:=\fromto[y_1]{y_k}}$:
\begin{align}
    \boldsymbol{y}|\boldsymbol{x},\hp \sim \mathcal{N}_k\left(\boldsymbol{0},\gramm{\boldsymbol{x}}{\boldsymbol{x}}+\diag(\signoise^2(\boldsymbol{x})\right),
\end{align}

with $\signoise^2(\boldsymbol{x}):=\left(\signoise^2(x_1),\ldots,\signoise^2(x_k)\right)$.
\item\emph{Optimize kernel hyper-parameters (optional):}
\begin{align}
\hat{\hp}\in \argmax\limits_{\hp} p(\boldsymbol{y}|\boldsymbol{x},\hp).
\end{align}
\item Calculate \emph{posterior predictive distribution} for new point (\xnew):
\begin{align}\label{predpost}
    f(\xnew)|\xnew,\boldsymbol{y},\boldsymbol{x},\hat{\hp}\sim\mathcal{N}\left(\muhat(\xnew),\sigmodelhat^2(\xnew)\right),
\end{align}
where for $A:=\left(\grammhat{\boldsymbol{x}}{\boldsymbol{x}}+\diag(\signoise^2(\boldsymbol{x})\right)$ and $\kernelvechat{\xnew}{\boldsymbol{x}}:=\left(k_{\hat{\hp}}(\xnew,x_1),\ldots,k_{\hat{\hp}}(\xnew,x_k)\right)$ the parameters are given as
\begin{align}
    \muhat(\xnew) := &\kernelvechat{\xnew}{\boldsymbol{x}}A^{-1}f(\boldsymbol{x})\label{eq:meanPredGPR}\\
    \sigmodelhat^2(\xnew) := & k_{\hat{\hp}}(\xnew,\xnew)-\kernelvechat{\xnew}{\boldsymbol{x}} A^{-1} \kernelvechat{\xnew}{\boldsymbol{x}}^T \label{eq:sigPredGPR}.
\end{align}
\end{enumerate}
Setting the data noise to zero $\signoise\equiv0$ in \eqref{eq:meanPredGPR} and \eqref{eq:sigPredGPR} yields the mean prediction $\fhat$ and the model uncertainty prediction $\sigmodelhat^2$ as
\begin{align}
    \hspace{-.1cm}\fhat(\xnew)\hspace{-.05cm}&=\hspace{-.05cm}\boldsymbol{k}_{\hat{\hp}}(\xnew,\boldsymbol{x})\left(\grammhat{\boldsymbol{x}}{\boldsymbol{x}}\right)^{-1}\hspace{-.1cm}f(\boldsymbol{x}),\label{eq:meanPredGPRModelUncertainty}\\
    \hspace{-.1cm}\sigmodelhat^2(\xnew)\hspace{-.05cm}&=\hspace{-.05cm}k_{\hat{\hp}}(\xnew,\xnew)\hspace{-.05cm}-\hspace{-.05cm}\kernelvechat{\xnew}{\boldsymbol{x}}\hspace{-.05cm}\left(\grammhat{\boldsymbol{x}}{\boldsymbol{x}}\right)^{-1}\kernelvechat{\xnew}{\boldsymbol{x}}^T\hspace{-.1cm},\hspace{-.05cm}\label{eq:sigPredGPRModelUncertainty}
\end{align}
which are then used to define the Gaussian process's UBs by $\left(\fhat(x)\mp c~\sigmodelhat(x)\right)$ with a calibration parameter $c \in \Rpz$.

\subsubsection{Monte Carlo Dropout (MCDO)}
Let $\NNf=W^{K}\circ\activation\circ W^{K-1}\circ\ldots\circ\activation\circ W^{1}(x)$ be an NN with $K-1$ hidden layers, activation function $\activation$, and fixed parameters $\theta=\{W^{1},\ldots,W^{K}\}$, which have been trained with added dropout regularization. Furthermore, let $(p_1,\ldots,p_K)$ denote the dropout probability vector used when training $\NNf$, i.e, $p_i$ determines the probability for a single node in the $i$\textsuperscript{th} hidden layer $W^i$ to be dropped in each backpropagation step.\footnote{One could also use different probabilities $p_{ij}$ for each node within a hidden layer. The equations extend straightforwardly.}

To obtain model uncertainty one draws $M$ different NNs according to the dropout probability vector and represents model uncertainty using sample estimates of the mean and variance of the model predictions.\footnote{Alternatively, one could also determine the UBs using empirical upper and lower quantiles of the different model predictions.} These predictions are frequently termed \emph{stochastic forward passes}. More formally, given a dropout probability vector $(p_1,\ldots,p_K)$, one draws $M$ realisations $\{\theta^{(1)},\ldots,\theta^{(M)}\}$ of parameters $\theta$, where $\theta^{(m)}:=\{W^{1,(m)},\ldots,W^{K,(m)}\}$. $W^{k,(m)}$ is obtained from the original hidden layer $W^{k}$ by dropping each column with probability $p_i$, i.e., for $W^{k}\in \R^{d_\mathsmaller{{row}}\times d_\mathsmaller{{col}}}$ set
\begin{align}
&W^{k,(m)}=W^{k}\left[z^{(m)}_1,\ldots,z^{(m)}_{d_\mathsmaller{{col}}}\right],\, z^{(m)}_j\in \R^{d_\mathsmaller{{row}}}\\
&\textrm{where }z^{(m)}_j:=\begin{cases} \boldsymbol{0}, \textrm{ with probability } p_i\\
\boldsymbol{1}, \textrm{ with probability } 1-p_i.
\end{cases}
\end{align}
UBs that represent model uncertainty and \emph{known} data noise $\signoise^2$ are then estimated for each $x\in \X$ as
\begin{align}
    \fhat(x)&:=\dfrac{1}{M}\sum_{m=1}^M\NN^f_{\theta^{(m)}}(x),\label{eq:meanPredMCDO}\\
    \hat{\sigma}^2(x)&:=\underbrace{\dfrac{1}{M}\sum_{m=1}^M\left(\NN^f_{\theta^{(m)}}(x)-\fhat(x)\right)^2}_{\text{model uncertainty}} + \underbrace{\signoise^2(x)}_{\text{data noise}}.\label{eq:sigPredMCDO}
\end{align}
The model uncertainty prediction $\sigmodelhat^2$ is then given as
\begin{align}
    \sigmodelhat^2(x)&:=\dfrac{1}{M}\sum_{m=1}^M\left(\NN^f_{\theta^{(m)}}(x)-\fhat(x)\right)^2\label{eq:sigPredMCDOModelUncertainty}
\end{align}
which defines Mc dropout's UBs as $\left(\fhat(x)\mp c~\sigmodelhat(x)\right)$ with a calibration parameter $c \in \Rpz$.

\subsubsection{Deep Ensembles (DE)}\label{subsubsec:Deep Ensembles (DE)}
Deep ensembles consists of the following two steps:\footnote{\citet{lakshminarayanan2017simple} also considered in their paper a third step: \emph{adversarial training}. However, as the authors point out, the effectiveness of adversarial training drops quickly as the number of networks in the ensemble increases. Therefore, we do not consider adversarial training in this paper.}
\begin{enumerate}[leftmargin=*,topsep=0pt,itemsep=0pt,partopsep=0pt, parsep=0pt]
    \item Use a NN to define a predictive distribution $p_\theta(y|x)$, select a pointwise loss function (proper scoring rule) $\ell(p_\theta,(x,y))$, which measures the quality of the predictive distribution $p_\theta(y|x)$ for an observation $(x,y)$ and define the empirical loss used for training as
    \begin{align}
        L(\theta):=\sum_{(x,y)\in \Dtr}\ell(p_\theta,(x,y)).
    \end{align}
    \item Use an ensemble of NNs, each with different randomly initialized parameters to represent model uncertainty.
\end{enumerate}
Concretely, for regression, \citet{lakshminarayanan2017simple} use a NN $\NNf$ with two outputs: $\muhat[\theta]$ (mean prediction) and $\signoisehat[\theta]$ (data noise prediction) and train it using as pointwise loss function the Gaussian negative log-likelihood, i.e, 
\begin{align}
    p_\theta(y|x):=&\mathcal{N}\left(y;\muhat[\theta](x),\left(\signoisehat[\theta](x)\right)^2\right),\\
    \ell(p_\theta,(x,y)):=&\frac{\log\left(\left(\signoisehat[\theta](x)\right)^2\right)}{2} + \frac{(\muhat[\theta](x)-y)^2}{2\left(\signoisehat[\theta](x)\right)^2}\label{eq:scoringRule}.
\end{align}
To add model uncertainty, \citet{lakshminarayanan2017simple} use an ensemble of $M$ NNs $\{\NN^f_{\theta^{(1)}},\ldots,\NN^f_{\theta^{(M)}}\}$, where each NN outputs a mean and data noise prediction, i.e, for $x\in \X$ and $m \in \fromto[1]{M}$
\begin{align}
\NN^f _{\theta^{(m)}}(x):=\left(\muhat[\theta^{(m)}](x),\signoisehat[\theta^{(m)}](x)\right).    
\end{align}
This then defines the learned predictive distribution for each NN in regression as
\begin{align}
    p_{\theta^{(m)}}(y|x)=\mathcal{N}\left(y;\muhat[\theta^{(m)}](x),\left(\signoisehat[\theta^{(m)}](x)\right)^2\right).
\end{align}
Finally, the ensemble is treated as uniformly-weighted Gaussian mixture model, i.e.,
\begin{align}
    \frac{1}{M}\sum_{m=1}^Mp_{\theta^{(m)}}(y|x),
\end{align}
which is further approximated using a single Gaussian by matching the first and second moments. The deep ensemble predictions for the mean $\fhat$ and for the combined model uncertainty and data noise $\sigmahat^2$ are given as
\begin{align}
    \fhat(x)&:=\frac{1}{M}\sum_{m=1}^M\muhat[\theta^{(m)}](x),\label{eq:meanPredDE}\\
    \sigmahat^2(x)&:=\underbrace{\frac{1}{M}\sum_{m=1}^M\left(\signoisehat[\theta^{(m)}](x)\right)^2}_{\text{data noise}}\hspace{-0.1cm}+\underbrace{\frac{1}{M}\sum_{m=1}^M\left[\muhat[\theta^{(m)}](x)-\fhat(x)\right]^2}_{\text{model uncertainty}}\hspace{-0.1cm}.\label{eq:sigPredDE}
\end{align}\label{eq:MeanAndUncertaintyDE}%
The model uncertainty prediction $\sigmodelhat^2$ is then given as
\begin{align}
    \sigmodelhat^2(x)&:=\frac{1}{M}\sum_{m=1}^M\left[\muhat[\theta^{(m)}](x)-\fhat(x)\right]^2\hspace{-0.1cm}.\label{eq:sigPredDEModelUncertainty}
\end{align}
which defines deep ensembles' UBs as $\left(\fhat(x)\mp c~\sigmodelhat(x)\right)$ with a calibration parameter $c \in \Rpz$.

\begin{remark}
For known data noise $\signoise$, no estimation is required and one can use an NN $\NNf$ with only one output $\muhat[\theta]$. If additionally the data noise $\signoise$ is assumed to be homoskedastic, one can train $\NNf$ using the mean squared error (MSE) with suitably chosen L2-regularization parameter $\lambda$. To obtain predictive bounds instead of credible bounds, one can add $\signoise^2$ to \eqref{eq:sigPredDEModelUncertainty} at the end.
\end{remark}
\subsubsection{Hyper Deep Ensembles (HDE)}
Hyper deep ensembles (HDE) \citep{wenzel2020hyperparameter} is a simple extension of DE. In HDE, the ensemble is designed by varying not only weight initializations, but also hyperparameters. HDE involves (i) a random search over different hyperparameters and (ii) stratification across different random initializations. HDE inherits all the components (e.g., the architecture, or the loss function) from DE, which are presented in detail in \Cref{subsubsec:Deep Ensembles (DE)}. Specifically, formulas \eqref{eq:meanPredDE} and \eqref{eq:sigPredDE} for the mean prediction $\fhat$ and the model uncertainty prediction $\sigmodelhat^2(x)$ are the same for HDE.

Let $\NN^f_{\theta^{(m)},\hp^{(m)}}$ denote a NN for model predictions with weights $\theta^{(m)}$ and hyperparameters $\hp^{(m)}$ (e.g., the dropout rate). Furthermore, let $\texttt{rand\_search}(\kappa)$ denote the random search algorithm from \citep{bergstra2012random}, where $\kappa$ is the desired number of different NNs with randomly sampled hyperparameters.

The only difference of HDE compared to DE is the procedure how the ensemble is built, which we reprint in Algorithm~\ref{alg:hyper_deep_ens}. Algorithm~\ref{alg:hyper_deep_ens} uses as subprocedure Algorithm~\ref{alg:hyper_ens} from \citep{caruana2004ensemble}, which we present first.

Algorithm~\ref{alg:hyper_ens} greedily grows an ensemble among a given pre-defined set of models $\mathcal{M}$, until some target size $M$ is met, by selecting \emph{with-replacement} the NN leading to the best improvement of a certain score $\mathcal{S}$ on a validation set.\footnote{Note that we use a slightly different notation than \citet{wenzel2020good} here.}
\vskip 0.1in
\begin{algorithm}[ht]
    \caption{$\texttt{hyper\_ens}\,$ \citep{caruana2004ensemble}}
    \label{alg:hyper_ens}
    \begin{algorithmic}
    \INPUT{$\mathcal{M}$, $M$}
        \STATE Ensemble $\mathcal{E}:=\{\}$; Score $\mathcal{S}(\cdot)$; $s_{best}=+\infty$\;
        \WHILE{$|\mathcal{E}.\textrm{unique}()|\le M$}
        \STATE
        $\NN^f_{\theta^{*}}\in \argmin_{\NN^f_{\theta}\in \mathcal{M}}\mathcal{S}\left(
        \mathcal{E}\cup\{\NN^f_{\theta}\}\right)$\;
        \IF{$\mathcal{S}\left(\mathcal{E}\cup\{\NN^f_{\theta}\}\right)<s_{best}$}
        \STATE{
        $\mathcal{E} = \mathcal{E}\cup\{\NN^f_{\theta}\}$}
        \ELSE{\OUTPUT$\mathcal{E}$}
        \ENDIF
        \ENDWHILE
        \OUTPUT $\mathcal{E}$
    \end{algorithmic}
\end{algorithm}\vskip 0.1in
Note that the Algorithm~\ref{alg:hyper_ens} does not require the NNs to
have the same random weight initialization. However, \citet{wenzel2020hyperparameter} consider a fixed initialization for HDE to isolate the effect of just varying the hyperparameters.

Finally, Algorithm~\ref{alg:hyper_deep_ens} builds an ensemble of at most $M$ unique NNs which can exploit both sources of diversity: different \emph{random initializations} and different choices of \emph{hyperparameters}.
\begin{algorithm}[t!]
    \caption{$\texttt{hyper\_deep\_ens}\,$\citep{wenzel2020hyperparameter}}
    \label{alg:hyper_deep_ens}
    \begin{algorithmic}
        \INPUT{$M$, $\kappa$}
        \STATE $\mathcal{M}_0=\{\NN^f_{\theta^{(j)},\lambda^{(j)}}\}_{j=1}^{\kappa}\leftarrow \texttt{rand\_search}(\kappa)$
        \STATE $\mathcal{E}_0 \leftarrow \texttt{hyper\_ens}(\mathcal{M}_0,M)$
        \STATE $\mathcal{E}_{strat}=\{\}$
        \FORALL{$\NN^f_{\theta,\lambda} \in \mathcal{E}_0.\textrm{unique}()$}{\FORALL{$m \in \fromto[1]{M}$}\STATE
            $\theta^{\prime} \leftarrow \textrm{random initialization (seed number $m$)}$
            \STATE$\NN^f_{\theta^{(m)},\lambda} \leftarrow \textrm{train }\NN^f_{\theta^{\prime},\lambda}$
            \STATE$\mathcal{E}_{strat}=\mathcal{E}_{strat}\cup \{\NN^f_{\theta^{(m)},\lambda}\}$\;
            \ENDFOR}\ENDFOR
        \OUTPUT{$\texttt{hyper\_ens}(\mathcal{E}_{strat},M)$}
    \end{algorithmic}
\end{algorithm}
\subsection{Regression}\label{subsec:RegressionAppendix}
\subsubsection{Metrics}\label{subsec:performanceMeasures}
In this section, we provide details on the three metrics, which we use to assess the quality of UBs in the regression settings. In the following, let 
\begin{align}\label{eq:UBsAppednix}
    \left(\lUB_c(x),\uUB_c(x)\right):=\left(\fhat(x)\mp c~\sigmodelhat(x)\right)
\end{align}
denote UBs obtained from any of the considered models via a model prediction $\fhat$, a model uncertainty prediction $\sigmodelhat$, and a calibration parameter $c\in \Rpz$. Furthermore, we use the following shorthand notation:
\begin{align}\label{eq:UBsshort}
    \UBc(x):=\left(\lUB_c(x),\uUB_c(x)\right).
\end{align}
In the following, let $D:=\{\left( x_i,y_i\right)\in\X\times\Y, i\in \fromto{n}\},$ with $n\in\N$ denote a set of input-output points (depending on the specific purpose, $D$ would typically refer to a validation or test set).

\paragraph{Mean Width vs. Coverage Probability} We first formalize the concepts of mean width (MW) and coverage probability (CP). Then we define the metric \AUC{}. 

\begin{definition}[Coverage Probability]\label{def:coverageProbability}
Let $D$ denote a set of input-output points, and let $\UBc$ denote UBs for a calibration parameter $c$. Then the coverage probability is defined as
\begin{align}
    \hspace{-0.1cm}\CP\left(D\,|\,\UBc\right):=\frac{1}{|D|}\sum_{(x,y)\in D}\hspace{-0.25cm}\mathbbm{1}_{\left\{\lUB_c(x)\le y \le \uUB_c(x)\right\}}.
\end{align}
\end{definition}
\begin{definition}[Mean Width]\label{def:meanWidth}
Let $D$ denote a set of input-output points, and let $\UBc$ denote UBs for a calibration parameter $c$. Then the mean width is defined as
\begin{align}\label{eq:MeanWidth}
    &\MW\left(D\,|\,\UBc\right):=\frac{1}{|D|}\sum_{(x,y)\in D}\hspace{-0.2cm}|\uUB(x)-\lUB(x)|.
\end{align}
\end{definition}

\begin{remark}
Note that in Definition~\ref{def:meanWidth}, uncovered points at which UBs (for some fixed calibration parameter $c$) are narrow have a positive effect on overall mean width. In order not to reward such overconfident mispredictions, a possible remedy is to consider in \eqref{eq:MeanWidth} only the subset $\Dcapt\subset D$ of input-output points captured by the UBs, i.e.,
\[\Dcapt:=\left\{(x,y) \in D: \lUB(x)\le y\le \uUB(x)\right\}.\]
However, focusing on captured data points only punishes UBs that capture some points with large widths and almost cover others with small widths unnecessarily harshly compared to UBs for which the reverse is true. In other words, a slight change in calibration parameter $c$ can lead to very diverse evaluations of UBs that have been assessed almost equal under the original $c$.
Since ultimately we are interested in comparing UBs based on a range of calibration parameters (see \Cref{metric:AUC}) we decided to include all points in the calculation of \MW\ in our experiments.
\end{remark}

Ideally, $\MW$ should be as small as possible, while $\CP$ should be close to its maximal value $1$. Clearly, $\CP$ is counter-acting $\MW$. This naturally motivates considering ROC-like curves, plotting $\MW$ against $CP$ for a range of calibration parameters $c$, and comparing different UBs based on their \emph{area under the curves} ($\AUC$).

\begin{measure}[AUC]\label{metric:AUC}
Let $D$ denote a set of input-output points. Define further $c^*$ as the minimal calibration parameter achieving full coverage of $D$ for given UBs $\UBc$, i.e., 
\[c^*:=\argmin_{c\ge0}\{\CP\left(D\,|\,\UBc\right) = 1\}.\]
\AUC\ is then defined as the integral of the following curve
\[\{(\CP\left(D\,|\,\UBc\right),\,  \MW\left(D\,|\,\UBc\right))\, :\, c\in [0, c^*]\}.\footnote{In our experiments, we approximated this integral via the trapezoidal rule.}\]
\end{measure}

\paragraph{Negative Log-likelihood}
Next, we first define the second metric: \emph{average\footnote{We remark that \NLPD\ can be interpreted as average marginal predictive density over the set $D$, assuming these marginals are Gaussian with mean $\fhat$ and variance $c\sigmodelhat$. In particular, we refrain from posing assumptions of posterior independence, that we believe are highly flawed, i.e., \NLPD\ should not be interpreted as joint negative log predictive density.} negative log (Gaussian) likelihood (NLL)} and then present our third metric $\MNLPD{}$.

\begin{measure}[NLL]\label{metric:NLL}
Let $D$ denote a set of input-output points, and let $\UBc$ denote UBs for a calibration parameter $c$ with corresponding model prediction ${\fhat:\X\to \R}$ and model uncertainty prediction $\sigmodelhat:\X \to \Rpz$. Then \NLPD\ is defined as
\begin{align}\label{nlpd}
&\NLPD(D|\UBc):=\\
&\frac{1}{|D|}\sum_{(x,y) \in D}\left[\frac{\left(y-\fhat(x)\right)^2}{2 \left(c\sigmodelhat(x)\right)^2} +  \ln\left(c\sigmodelhat(x)\right)\right]+\ln(2\pi)/2 \nonumber
\end{align}
\end{measure}
The first term in \NLPD\ measures the error of the model prediction, where large errors can be attenuated by large uncertainties while the second term penalizes logarithmically larger uncertainties. Thus, \NLPD\ penalizes both over-confident ($c\sigmodelhat(x)\approx0$) wrong predictions as well as under-confident ($c\sigmodelhat(x)\gg0$) predictions.

We define the third metric as the minimal value of the \NLPD\ when varying the calibration parameter $c$.
\begin{table*}[ht]
	\renewcommand\arraystretch{1.0}
	\caption{Detailed results of the \AUC{} metric for NOMU, GP, MCDO, DE, HDE and HDE* (our HDE see \Cref{subsec:ConfigurationDetailsofBenchmarksRegression}) for all ten 1D synthetic functions. Shown are the medians and a 95\% bootstrap confidence interval over 500 runs. Winners are marked in grey.}
    \label{tab:1dsynfunresultsDetailsAUC}
	\vskip 0.1in
	\begin{center}
	\begin{small}
	\begin{sc}
	\resizebox{\textwidth}{!}{%
		\begin{tabular}{
				l
				S[table-format = 1.2]
                >{{[}} %
                S[table-format = 1.2,table-space-text-pre={[}]
                @{,\,} %
                S[table-format = 1.2,table-space-text-post={]}]
                <{{]}} %
				S[table-format = 1.2]
                >{{[}} %
                S[table-format = 1.2,table-space-text-pre={[}]
                @{,\,} %
                S[table-format = 1.2,table-space-text-post={]}]
                <{{]}} %
				S[table-format = 1.2]
                >{{[}} %
                S[table-format = 1.2,table-space-text-pre={[}]
                @{,\,} %
                S[table-format = 1.2,table-space-text-post={]}]
                <{{]}} %
				S[table-format = 1.2]
                >{{[}} %
                S[table-format = 1.2,table-space-text-pre={[}]
                @{,\,} %
                S[table-format = 1.2,table-space-text-post={]}]
                <{{]}} %
                S[table-format = 1.2]
                >{{[}} %
                S[table-format = 1.2,table-space-text-pre={[}]
                @{,\,} %
                S[table-format = 1.2,table-space-text-post={]}]
                <{{]}} %
                S[table-format = 1.2]
                >{{[}} %
                S[table-format = 1.2,table-space-text-pre={[}]
                @{,\,} %
                S[table-format = 1.2,table-space-text-post={]}]
                <{{]}} %
			}
			\toprule
			& \multicolumn{3}{c}{\textbf{NOMU}} &  \multicolumn{3}{c}{\textbf{GP}}  & \multicolumn{3}{c}{\textbf{MCDO}}  &  \multicolumn{3}{c}{\textbf{DE}}&\multicolumn{3}{c}{\textbf{HDE}}& \multicolumn{3}{c}{\textbf{HDE*}}\\
			\textbf{Function} & {\AUC$\downarrow$} & \multicolumn{2}{c}{95\%-CI}& {\AUC$\downarrow$} & \multicolumn{2}{c}{95\%-CI}& {\AUC} & \multicolumn{2}{c}{95\%-CI}& {\AUC$\downarrow$} & \multicolumn{2}{c}{95\%-CI}&{\AUC$\downarrow$} & \multicolumn{2}{c}{95\%-CI}&{\AUC$\downarrow$} & \multicolumn{2}{c}{95\%-CI}\\
			\midrule
			Abs & \winc0.07&0.07&0.08&0.14&0.13&0.16&0.16&0.15&0.17&\winc0.08&0.07&0.09&1.02&0.82&1.20&0.72&0.65&0.81\cr
    	    Step & 0.29&0.26&0.31&0.93&0.85&1.02&0.25&0.24&0.27&\winc0.14&0.13&0.15&1.90&1.61&2.67&0.82&0.66&0.99\cr
			Kink &\winc0.08&0.08&0.09&0.17&0.15&0.19&0.22&0.21&0.23&\winc0.09&0.07&0.10&1.81&1.40&2.31&1.01&0.9&1.24\cr
			Square & 0.12&0.11&0.13&0.16&0.13&0.19&0.38&0.36&0.41&0.14&0.12&0.15&1.72&1.50&2.23&1.86&1.54&2.53\cr
			Cubic & 0.07&0.07&0.08&\winc0.00&0.00&0.00&0.10&0.10&0.11&0.10&0.09&0.11&0.63&0.54&0.79&0.52&0.45&0.66\cr
			Sine 1& 1.10&1.03&1.16&1.14&1.09&1.2&\winc0.90&0.87&0.93&1.21&1.16&1.26&9.38&7.42&{11.4}&9.75&7.64&{12.1}\cr
			Sine 2 & 0.38&0.37&0.40&0.42&0.41&0.44&\winc0.36&0.35&0.36&0.50&0.47&0.53&3.18&2.82&4.08&2.43&2.12&3.38\cr
			Sine 3&\winc0.20&0.19&0.21&0.31&0.29&0.34&0.28&0.26&0.29&\winc0.20&0.19&0.21&1.43&1.19&1.73&1.33&1.10&1.57\cr
			Forrester& \winc0.19&0.18&0.20&\winc0.18&0.17&0.19&0.26&0.24&0.28&0.25&0.24&0.27&1.28&1.11&1.51&1.58&1.35&1.93\cr
			Levy &\winc0.40&0.39&0.42&0.53&0.51&0.56&\winc0.38&0.36&0.39&0.45&0.43&0.48&3.43&2.67&4.42&3.58&2.83&4.18\cr
			\bottomrule 
		\end{tabular}
	}
	\end{sc}
	\end{small}
	\end{center}
	\vskip -0.1in
\end{table*}

\begin{table*}[ht]
    \setlength\tabcolsep{2pt}
	\renewcommand\arraystretch{1.0}
	\caption{Detailed results of the \MNLPD{} metric {\scriptsize(without constant $\ln(2\pi)/2$)} for NOMU, GP, MCDO, DE, HDE and HDE* (our HDE see \Cref{subsec:ConfigurationDetailsofBenchmarksRegression}) for all ten 1D synthetic functions. Shown are the medians and a 95\% bootstrap confidence interval over 500 runs. Winners are marked in grey.}
		\label{tab:1dsynfunresultsDetailsNLPD}
	\begin{center}
	\begin{small}
	\begin{sc}
	\resizebox{\textwidth}{!}{%
		\begin{tabular}{
				l
				S[table-format = -1.2]
                >{{[}} %
                S[table-format = -1.2,table-space-text-pre={[}]
                @{,\,} %
                S[table-format = -1.2,table-space-text-post={]}]
                <{{]}} %
				S[table-format = -1.2]
                >{{[}} %
                S[table-format = -1.2,table-space-text-pre={[}]
                @{,\,} %
                S[table-format = -1.2,table-space-text-post={]}]
                <{{]}} %
				S[table-format = -1.2]
                >{{[}} %
                S[table-format = -1.2,table-space-text-pre={[}]
                @{,\,} %
                S[table-format = -1.2,table-space-text-post={]}]
                <{{]}} %
				S[table-format = -1.2]
                >{{[}} %
                S[table-format = -1.2,table-space-text-pre={[}]
                @{,\,} %
                S[table-format = -1.2,table-space-text-post={]}]
                <{{]}} %
                S[table-format = -1.2]
                >{{[}} %
                S[table-format = -1.2,table-space-text-pre={[}]
                @{,\,} %
                S[table-format = -1.2,table-space-text-post={]}]
                <{{]}} %
                S[table-format = -1.2]
                >{{[}} %
                S[table-format = -1.2,table-space-text-pre={[}]
                @{,\,} %
                S[table-format = -1.2,table-space-text-post={]}]
                <{{]}} %
			}
			\toprule
			& \multicolumn{3}{c}{\textbf{NOMU}} &  \multicolumn{3}{c}{\textbf{GP}}  & \multicolumn{3}{c}{\textbf{MCDO}}  &  \multicolumn{3}{c}{\textbf{DE}}&\multicolumn{3}{c}{\textbf{HDE}}&\multicolumn{3}{c}{\textbf{HDE*}}\\
			\textbf{Function} & {\MNLPD$\downarrow$} & \multicolumn{2}{c}{95\%-CI}& {\MNLPD$\downarrow$} & \multicolumn{2}{c}{95\%-CI}& {\MNLPD$\downarrow$} & \multicolumn{2}{c}{95\%-CI}& {\MNLPD$\downarrow$} & \multicolumn{2}{c}{95\%-CI}& {\MNLPD$\downarrow$} & \multicolumn{2}{c}{95\%-CI}& {\MNLPD$\downarrow$} & \multicolumn{2}{c}{95\%-CI}\\
			\midrule
			Abs & \winc-2.85&-2.95&-2.80&\winc-2.82&-2.90&-2.71&-1.76&-1.80&-1.69&\winc-2.91&-3.03&-2.76&-0.19&-0.36&-0.00&-0.72&-0.82&-0.56\cr
    	    Step &-1.04&-1.12&-0.97&-0.61&-0.81&-0.49&-0.78&-0.85&-0.69&\winc-2.49&-2.59&-2.36&0.61&0.38&0.83&-1.18&-1.42&-0.84\cr
			Kink & \winc-2.74&-2.80&-2.69&\winc-2.64&-2.75&-2.54&-1.28&-1.32&-1.22&\winc-2.81&-2.93&-2.68&0.56&0.26&0.83&-0.35&-0.46&-0.17\cr
			Square & -2.45&-2.51&-2.41&\winc-3.80&-4.21&-3.38&-0.56&-0.65&-0.52&-2.27&-2.37&-2.16&0.53&0.31&0.73&0.07&-0.17&0.22\cr
			Cubic & -2.98&-3.04&-2.94&\winc-6.03&-6.15&-5.92&-2.41&-2.47&-2.33&-2.65&-2.70&-2.59&-0.77&-0.93&-0.64&-1.42&-1.63&-1.15\cr
			Sine 1& \winc0.10&0.04&0.18&\winc-0.06&-0.13&0.01&\winc0.09&0.07&0.13&\winc0.09&0.04&0.16&1.97&1.71&2.28&1.65&1.45&1.84\cr
			Sine 2 & -1.35&-1.38&-1.32&\winc-1.48&-1.52&-1.42&-1.03&-1.06&-1.00&-0.92&-0.96&-0.88&1.21&0.95&1.65&0.45&0.26&0.73\cr
			Sine 3&\winc-1.55&-1.61&-1.47&\winc-1.69&-1.85&-1.55&-1.07&-1.12&-1.01&\winc-1.66&-1.75&-1.58&0.26&0.10&0.40&-0.02&-0.21&0.21\cr
			Forrester& -1.98&-2.03&-1.91&\winc-2.68&-2.86&-2.49&-1.10&-1.19&-1.03&-1.75&-1.79&-1.69&-0.11&-0.23&0.17&-0.50&-0.65&-0.28\cr
			Levy & \winc-1.12&-1.14&-1.07&-0.93&-0.98&-0.90&-0.76&-0.79&-0.73&\winc-1.04&-1.08&-0.97&1.35&0.99&1.53&0.74&0.53&0.96\cr 
			\bottomrule 
		\end{tabular}
	}
	\end{sc}
	\end{small}
	\end{center}
	\vskip -0.1in
\end{table*}

\begin{measure}[Minimum NLL]\label{metric:minnlpd}
Let $D$ denote a set of input-output points, and let $\UBc$ denote UBs for a calibration parameter $c$. Then $\MNLPD{}$ is defined as
\begin{align}\label{minnll}
    \MNLPD:=\min\limits_{c\in \Rpz }\,\NLPD(D|\UBc).
\end{align}
\end{measure}

\begin{remark}[Calibration]
Different approaches of obtaining UBs require rather different scaling to reach similar CP (as well as MW). For instance, we found that the calibration parameter $c$ required to achieve a certain value of \CP\ typically is larger for DE than those of the remaining methods by a factor of approximately 10. Therefore, in practice it is important to find a well-calibrated parameter c. Standard calibration techniques can be applied to NOMU (e.g., methods based on isotonic regression \citep{pmlr-v80-kuleshov18a}, or when assuming Gaussian marginals of the posterior one can select $c$ as the $\argmin\limits_{c\in \Rpz }\,\NLPD(D|\UBc)$ on a validation set $D$).
\end{remark}

\subsubsection{Toy Regression: Configuration Details}\label{subsec:ConfigurationDetailsofBenchmarksRegression}
All NN-based methods are fully-connected feed-forward NNs with ReLU activation functions, implemented in \textsc{tensorflow.keras} and trained for $2^{10}$ epochs of \textsc{tensorflow.keras'} Adam stochastic gradient descent with standard learning rate 0.001 and full batch size of all training points. Moreover, weights and biases are initialized uniformly in the interval $[-0.05, 0.05]$. 

\paragraph{NOMU Setup} For the MC-integration of the integral in the NOMU loss \eqref{eq:lmu2}, i.e., term \termc{}, we use $l=128$ and $l=256$ artificial input points (we draw new artificial input points in every gradient descent step) in 1D and 2D, respectively.\footnote{In 1D, we tried MC-integration based on uniform samples and MC-integration based on a deterministic grid, where both approaches led to qualitatively similar results and the latter is presented in this paper.} For numerical stability, we use the parameters~$\theta$ that gave the best training loss during the training which are not necessarily the ones after the last gradient descent step.

\paragraph{Deep Ensembles Setup} We consider the proposed number of five ensembles \citep{lakshminarayanan2017simple} with a single output $\muhat[\theta]$ only (accounting for zero data noise), each with three hidden layers consisting of $2^8, 2^{10}$ and $2^9$ nodes (resulting in $\approx$ 4 million parameters). We train them on standard regularized mean squared error (MSE) with regularization parameter $\lambda=10^{-8}/\ntr$ chosen to represent the same data noise assumptions as NOMU. 

\paragraph{MC Dropout Setup} The MC dropout network is set up with three hidden layers as well, with $2^{10}, 2^{11}$ and $2^{10}$ nodes (resulting in $\approx$ 4 million parameters). Both training and prediction of this model is performed with constant dropout probability $p:=p_i=0.2$, proposed in \cite{gal2016dropout}. We perform $100$ stochastic forward passes. To represent the same data noise assumptions as NOMU and deep ensembles we set the regularization parameter to $\lambda=(1-p)\cdot10^{-8}/\ntr=(1-0.2)\cdot10^{-8}/\ntr$ (based on equation (7) from \cite{gal2016dropout}). 

\paragraph{Gaussian Process Setup} Finally, we compare to a Gaussian process\footnote{We use \emph{GaussianProcessRegressor} from \textsc{scikit-learn}.} with RBF-kernel,
\begin{align}\label{rbfKernel}
    k_\hp(x,x^{\prime}):=\kappa\cdot e^{-\left(\frac{\twonorm[x-x^\prime]^2}{h^2}\right)},
\end{align}
with hyperparameters $\hp:=(\kappa,h)$, where both the prior variance parameter $\kappa$ and the length scale parameter $h$ (initialized as $\kappa=h=1$) are optimized in the default range $[10^{-5}, 10^{5}]$ using $10$ restarts and the data noise level\footnote{This is mainly used for numerical stability.} is set to $10^{-7}$.

\paragraph{Hyper Deep Ensembles Setup}
We consider the same network architectures as for DE with a single output $\muhat[\theta]$ only (accounting for zero data noise), each with three hidden layers consisting of $2^8, 2^{10}$ and $2^9$ nodes (resulting in $\approx$ 4 million parameters). We train them on standard regularized mean squared error (MSE).

Furthermore, we use the following proposed parameter values from \citet{wenzel2020hyperparameter} to build the ensembles: we consider $M:=5$ networks (termed $K$ in \citep{wenzel2020hyperparameter}) in the final ensemble and search over $\kappa=50$ networks resulting from $\texttt{random\_search}$. Moreover, as random hyperparameters we use the proposed L2-regularization and the $\texttt{dropout\_rate}$, which we draw as in \citep{wenzel2020hyperparameter} log-uniform from $(0.001,0.9)$ and $(10^{-8}\cdot 10^{-3},10^{-8}\cdot 10^{3})$, respectively. Furthermore, we use the proposed train/validation split of 80\%/20\% and as score $\mathcal{S}$ the \NLPD. Finally, as for MCDO and DE we scale the realized L2-regularization by $(1-\texttt{dropout\_rate})/\texttt{floor}(0.8\cdot\ntr)$ chosen to represent the same data noise assumptions as NOMU.

For our HDE version termed \textbf{HDE*}, we use a train/validation split of 70\%/30\%, draw the dropout probability from $(0.001,0.5)$ and continue training the final ensemble again on \emph{all} $\ntr$ training points (rescaling the realized L2-regularization again by $\texttt{floor}(0.8\cdot\ntr)/\ntr$).

\subsubsection{Toy Regression: Detailed Results }\label{subsubsec:Detailed Synthetic Functions Results in Regression}
\paragraph{Results 1D} In \Cref{tab:1dsynfunresultsDetailsAUC} (\AUC{})
and \Cref{tab:1dsynfunresultsDetailsNLPD} (\MNLPD{}), we present detailed results, which correspond to the presented ranks from \Cref{tab:1dsynfunresults} in the main paper. In \Cref{tab:1dsynfunresultsaggregated}, we present aggregate results, i.e., median values (incl. a $95\%$ bootstrap confidence interval (CI)) of \AUC\, and \MNLPD{} across \emph{all} $5000$ runs ($500$ runs for $10$ test functions) for each algorithm. We see that NOMU performs well on both metrics and yields numbers similar to those of DE. The GP is only competitive on \MNLPD. MCDO performs in both metrics worse than NOMU and DE. Not surprisingly, HDE, shown to be the state-of-the-art ensemble algorithm \citep{wenzel2020hyperparameter}, fails to produce reliable model uncertainty estimates in such a scarce and noiseless regression setting (see Appendix~\ref{subsubsec:Uncertainty Bounds Detailed Characteristics} for a discussion).
\begin{table}[t!]
    \caption{Aggregate results for NOMU, GP, MCDO, DE, HDE and HDE* (our HDE see \Cref{subsec:ConfigurationDetailsofBenchmarksRegression}) for the set of all ten 1D synthetic functions. Shown are the medians and a 95\% bootstrap CI of \AUC{} and \MNLPD{} {\scriptsize(without constant $\ln(2\pi)/2$)}  across all runs. Winners are marked in grey.}
    	\label{tab:1dsynfunresultsaggregated}
    \vskip 0.1in
    \begin{center}
    \begin{small}
    \begin{sc}
    \resizebox{1\columnwidth}{!}{
    	\begin{tabular}{
    			l
    			c
    			c
    			c
    			c
    		}
    		\toprule
    		\textbf{Method}& {\textbf{\AUC}$\downarrow$} & 95\%-CI& {\textbf{\MNLPD}$\downarrow$} &95\%-CI \\
    		\midrule
    		NOMU & \winc{0.21}& [0.21, 0.22] & \winc{-1.76}& [-1.81, -1.72]\\
    		GP &{0.33} &[0.32, 0.35] & \winc{-1.74}& [-1.81, -1.69]\\
    		MCDO & {0.30} &[0.29, 0.30]& {-1.02} &[-1.05, -0.99]\\
    		DE &\winc{0.22} &[0.21, 0.23] &  \winc{-1.75} &[-1.79, -1.70]\\
    		HDE & {2.04}& [1.90, 2.20] & { \hphantom{-}0.59} &[\hphantom{-}0.51, \hphantom{-}0.66]\\
    		HDE* & {1.73}& [1.63, 1.84] & { \hphantom{-}0.04}& [-0.03, \hphantom{-}0.10]\\
    		
    		\bottomrule 
    	\end{tabular}
    }
    \end{sc}
    \end{small}
    \end{center}
    \vskip -0.1in
\end{table}

\paragraph{Results 2D} Test functions for 2D input are taken from the same library as in the 1D setting.\footnote{See \href{https://www.sfu.ca/~ssurjano/optimization.html}{sfu.ca/~ssurjano/optimization.html}. All test functions are scaled to $X=[-1,1]^2$ and $f(X)=[-1,1]$.} Specifically, we select the following $11$ different test functions: \emph{Sphere}, \emph{G-Function}, \emph{Goldstein-Price}, \emph{Levy}, \emph{Bukin N.6}, \emph{Rosenbrock}, \emph{Beale}, \emph{Camel}, \emph{Perm}, \emph{Branin}, \emph{Styblisnki-Tang}.

In \Cref{tab:2dsynfunresultsaggregated}, we present median values of \AUC{} and \MNLPD{} across \emph{all} 5500 runs (500 runs for 11 test functions) for each algorithm.  We also provide a 95\% bootstrap confidence interval (CI) to assess statistical significance. We observe a similar ranking of the different algorithms as in 1D, however, now GP is ranked first in \AUC, followed by NOMU and DE; 
and the GP is also ranked first in \MNLPD{} followed by NOMU and DE, who share the second rank.

\begin{figure*}[ht!]
\makebox[\textwidth][c]{%
 \subfloat[Sine 3\label{subfig:jakobhfuntrend}]{%
		\includegraphics[width =1\columnwidth]{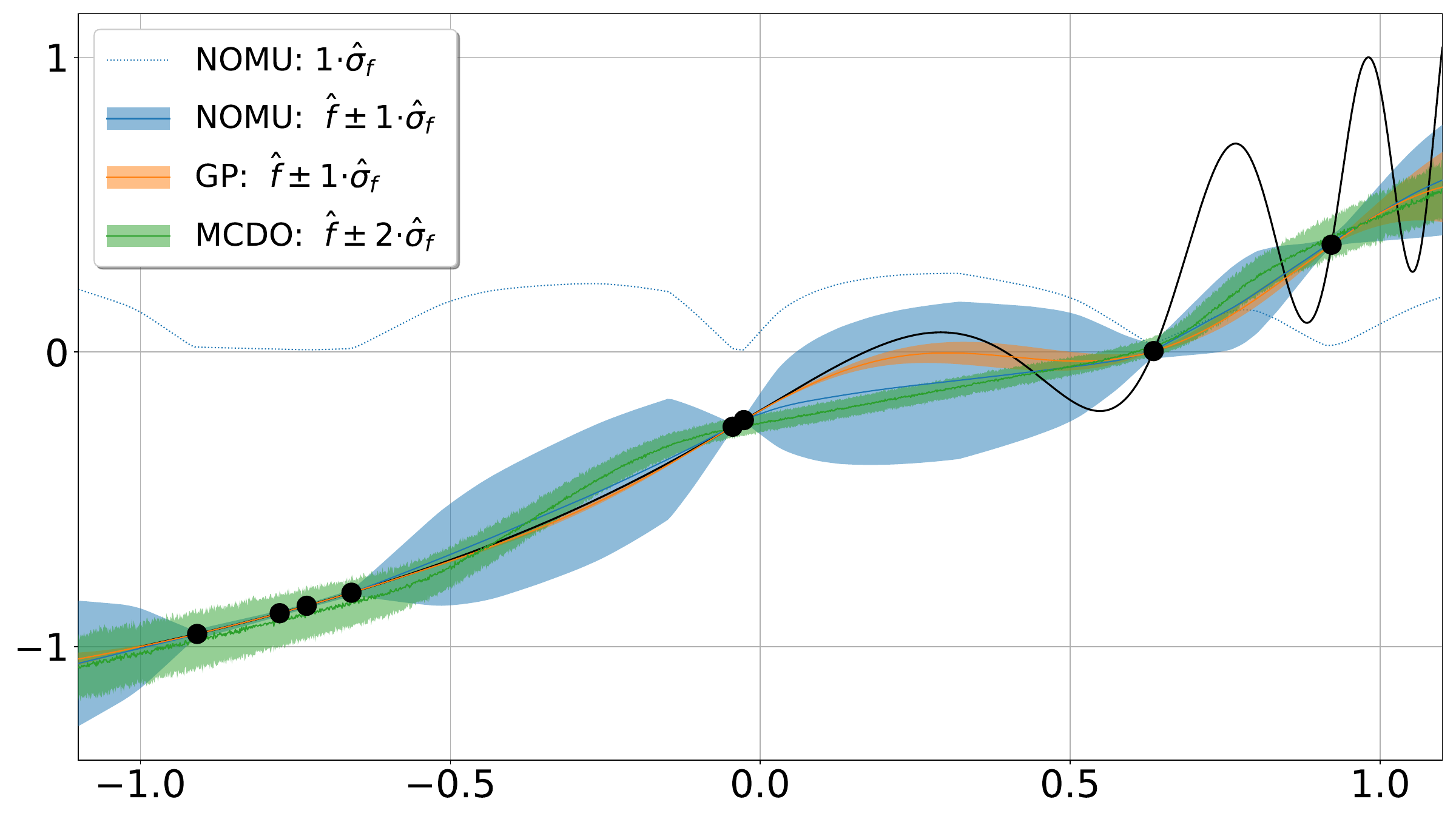}
		\includegraphics[width =1\columnwidth]{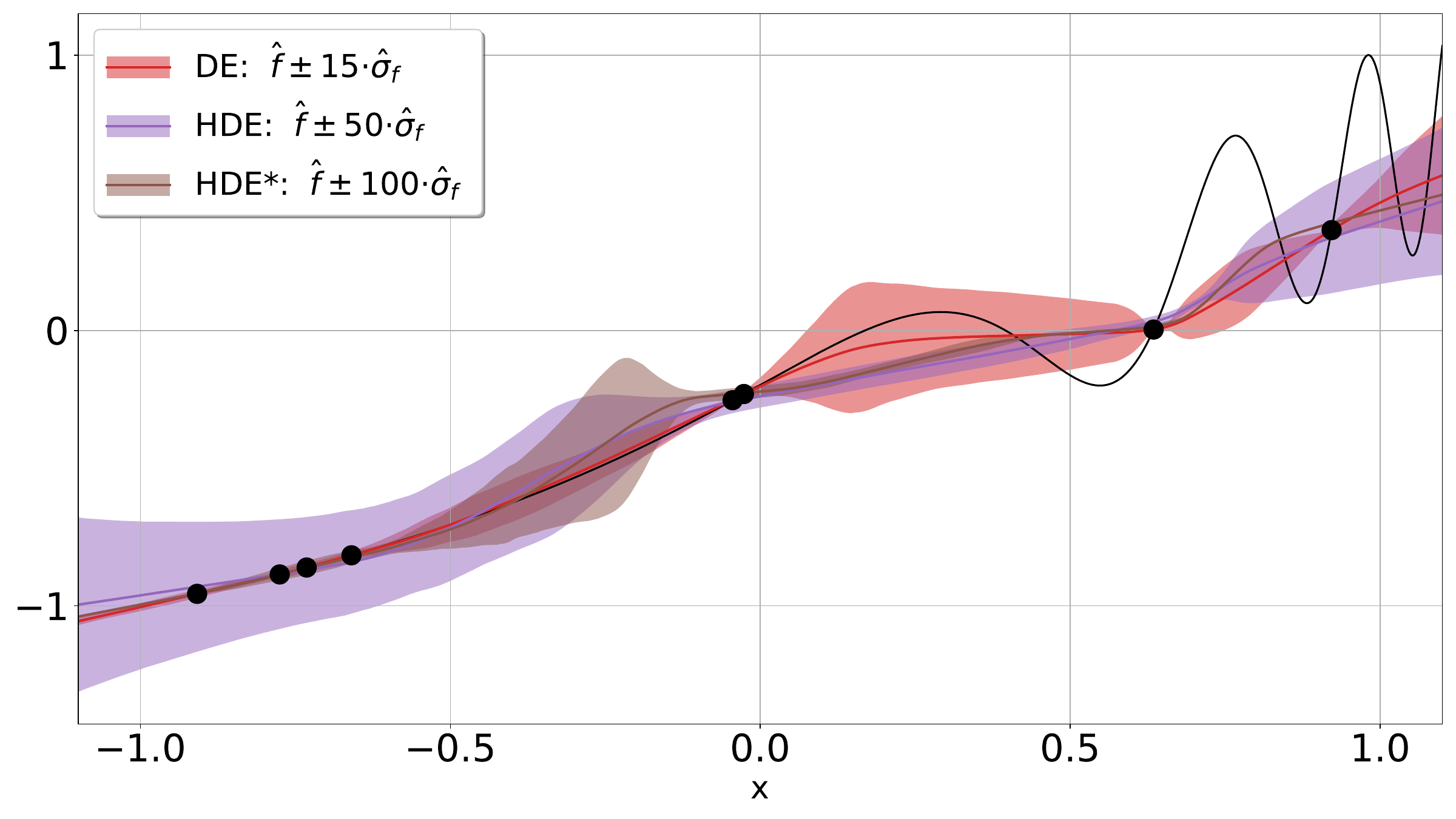}
	
    }
}
\makebox[\textwidth][c]{%
 \subfloat[Step\label{subfig:postoneg}]{%
		\includegraphics[width =1\columnwidth]{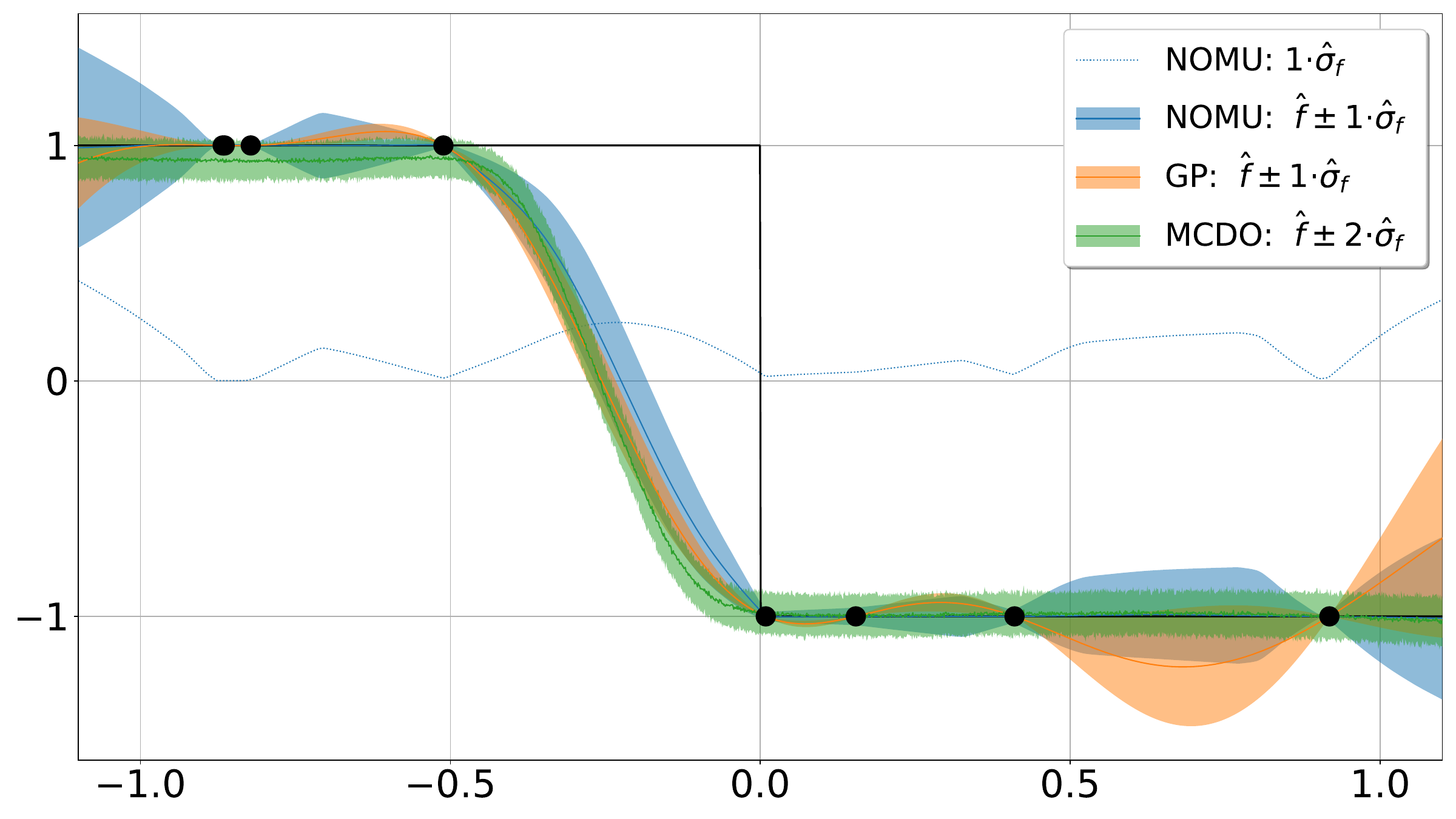}
		\includegraphics[width =1\columnwidth]{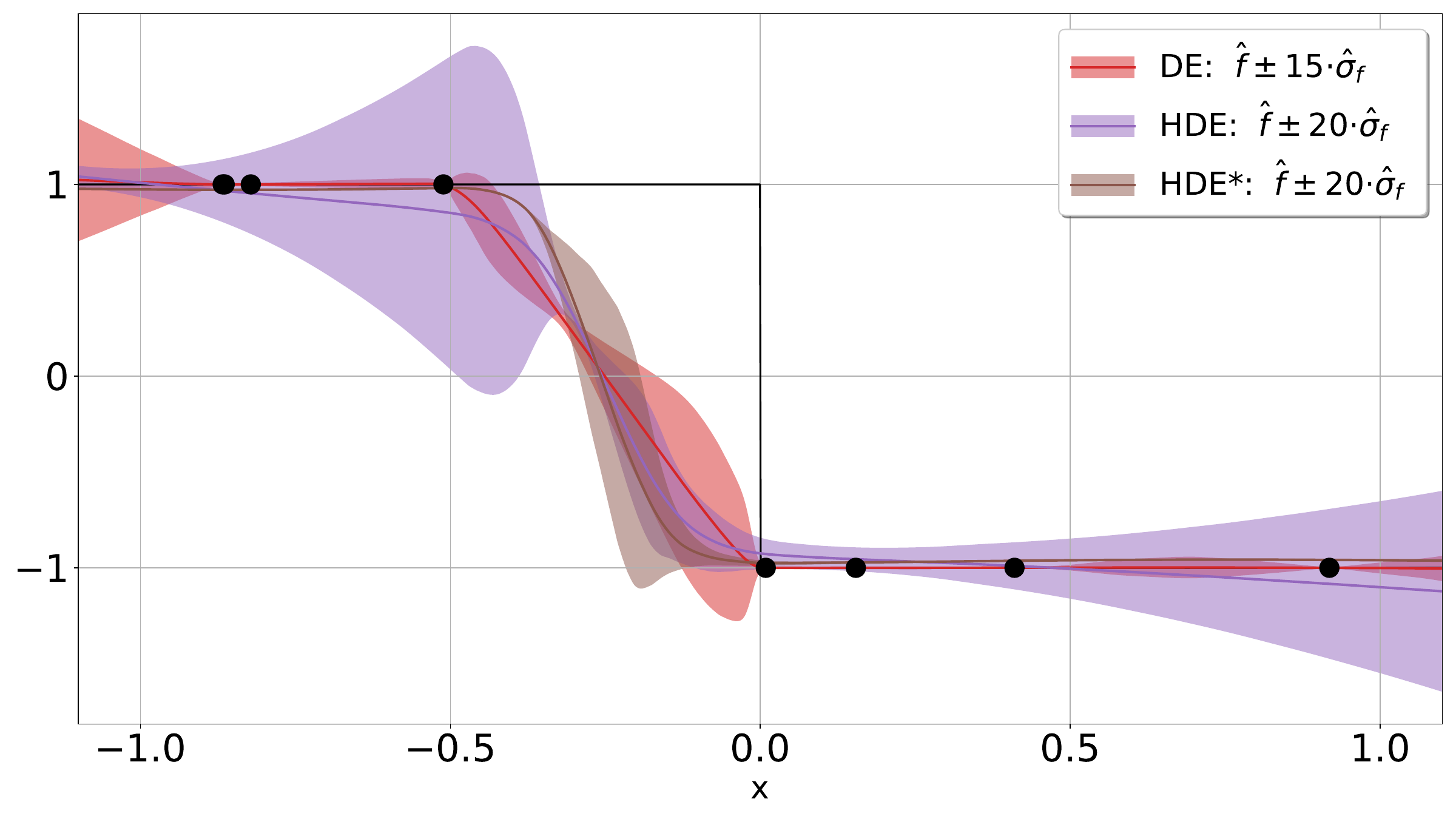}
	
    }
}
\makebox[\textwidth][c]{%
 \subfloat[Squared (GP would not behave reasonable for standard hyper-parameters in this instance, so we changed the range for the prior variance parameter $\kappa$ to be ${[10^{-5},4]}$.)\label{subfig:squared}]{%
		\includegraphics[width =1\columnwidth]{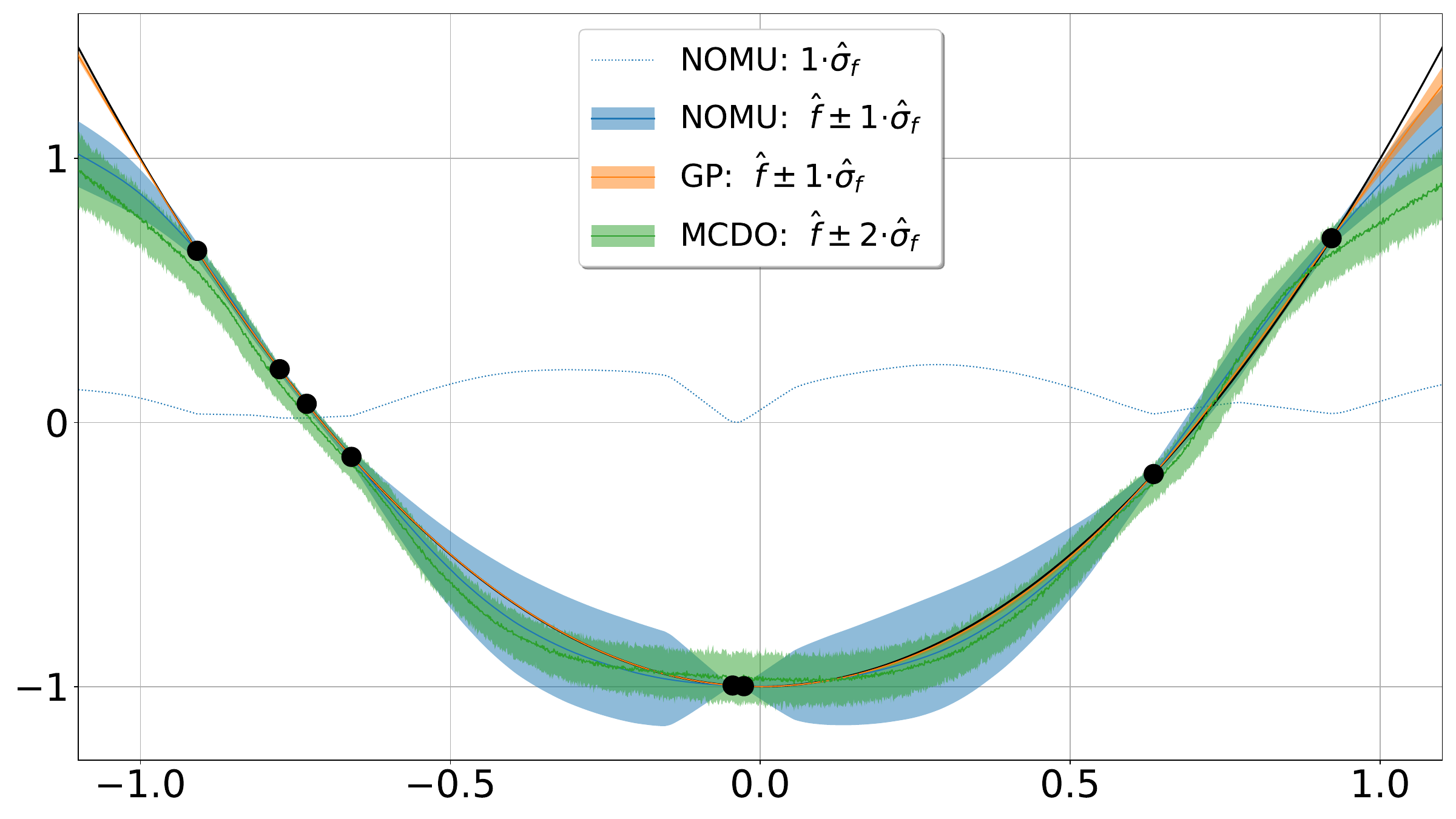}
		\includegraphics[width =1\columnwidth]{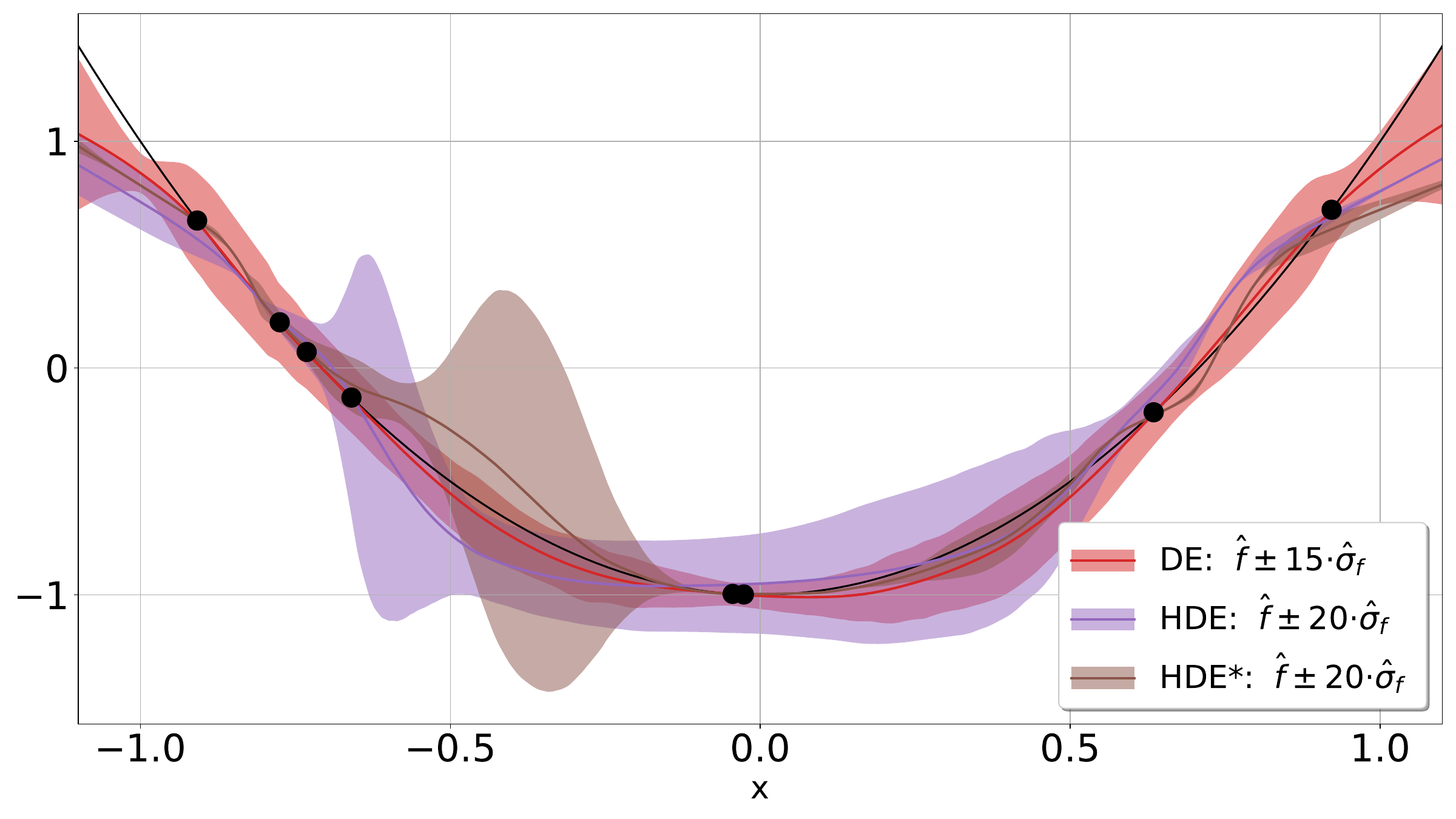}
	
    }
}
\makebox[\textwidth][c]{%
 \subfloat[Forrester\label{subfig:forrester}]{%
		\includegraphics[width =1\columnwidth]{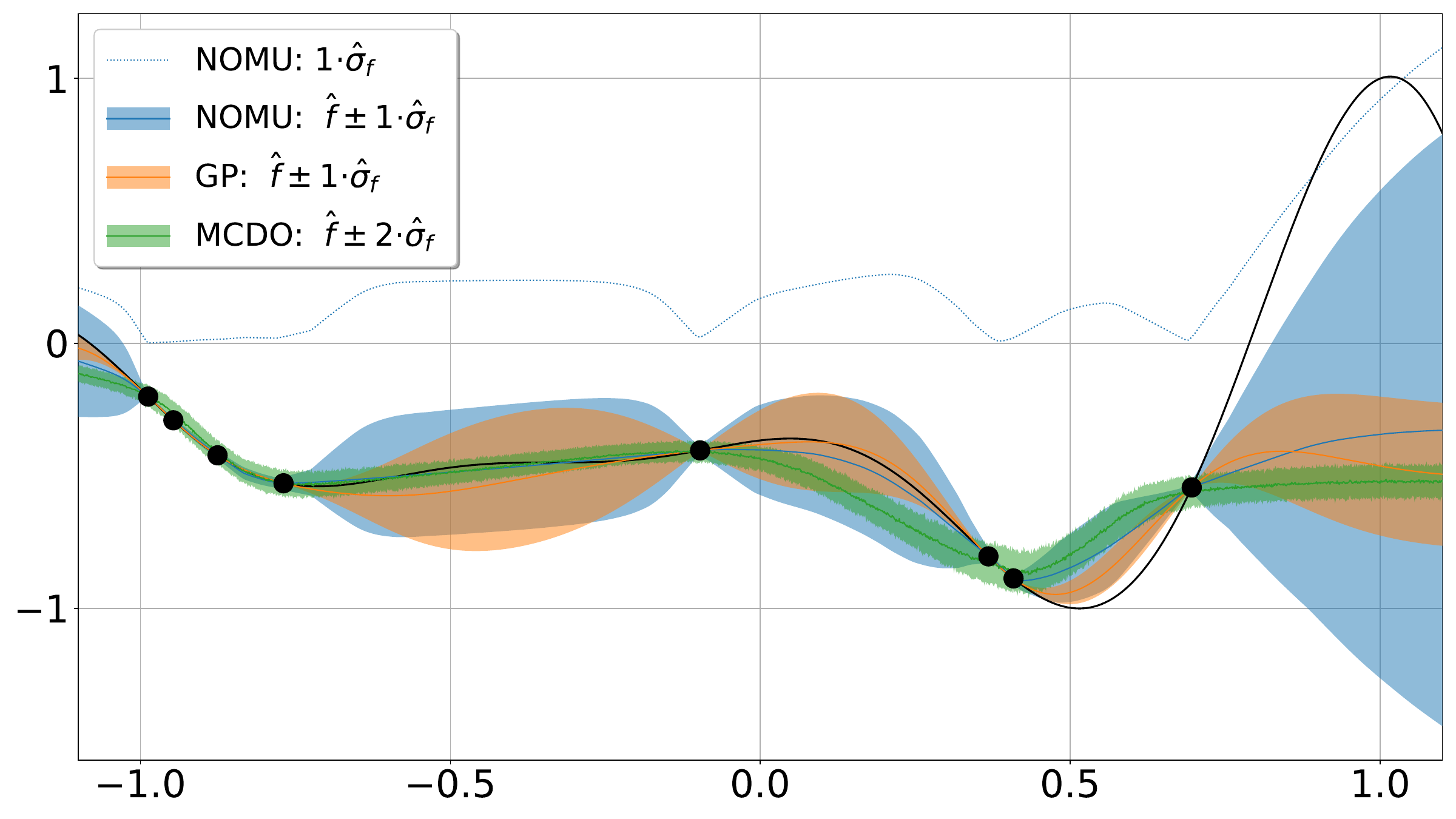}
		\includegraphics[width =1\columnwidth]{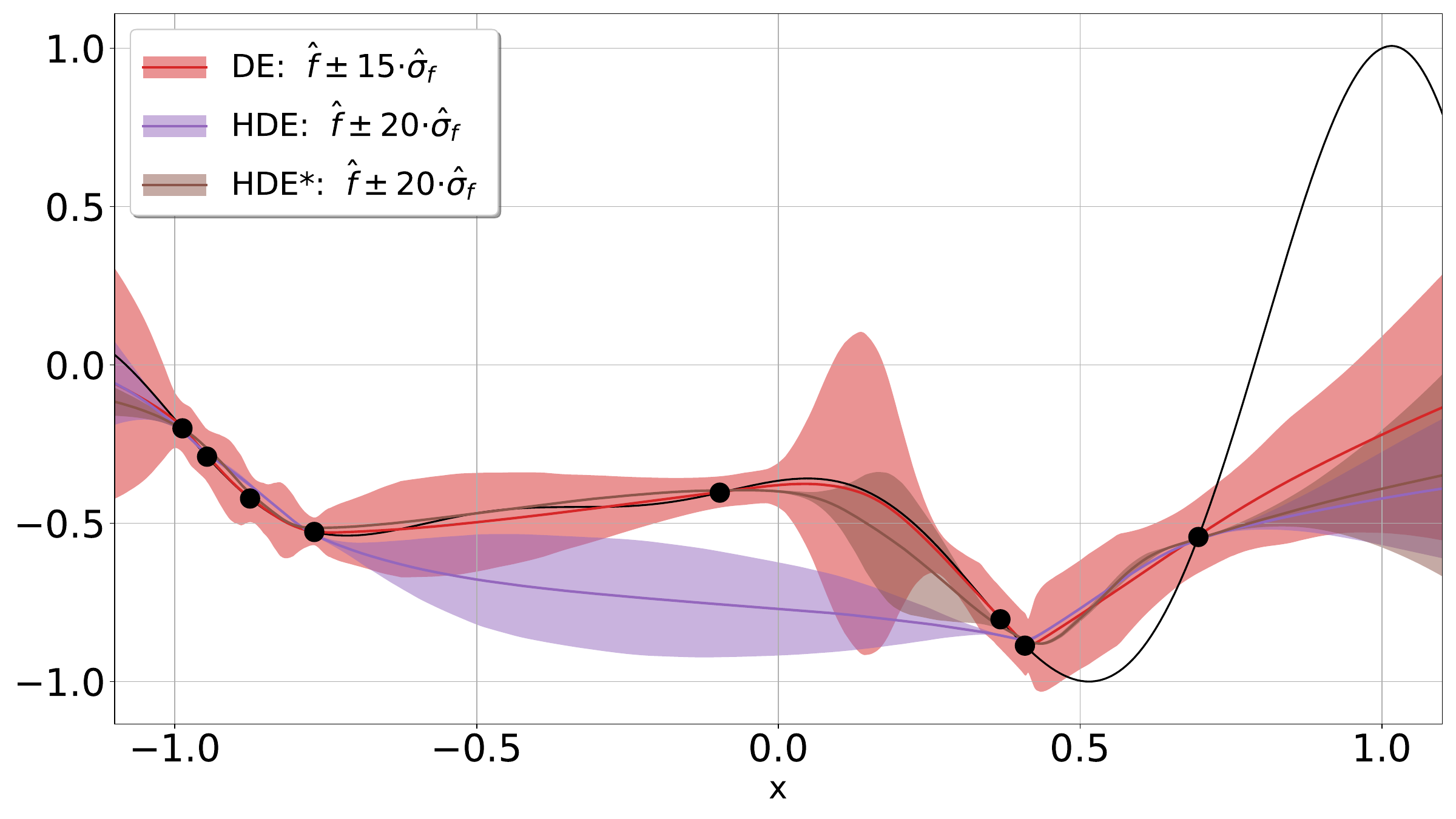}
	
    }
}
\phantomcaption
\end{figure*}
\begin{figure*}[ht!]
  \ContinuedFloat
\makebox[\textwidth][c]{%
 \subfloat[Kink\label{subfig:kink}]{%
		\includegraphics[width =1\columnwidth]{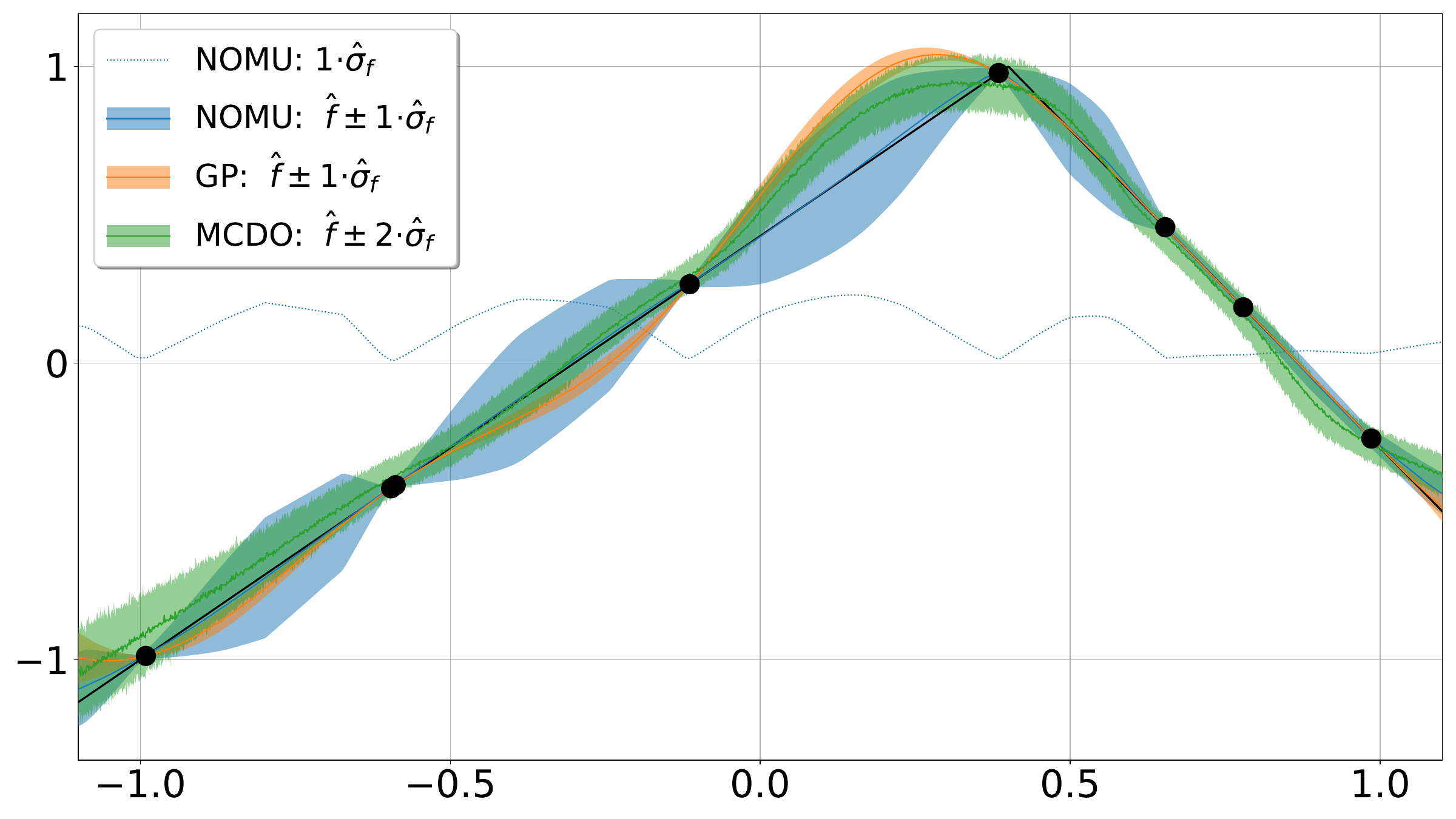}
		\includegraphics[width =1\columnwidth]{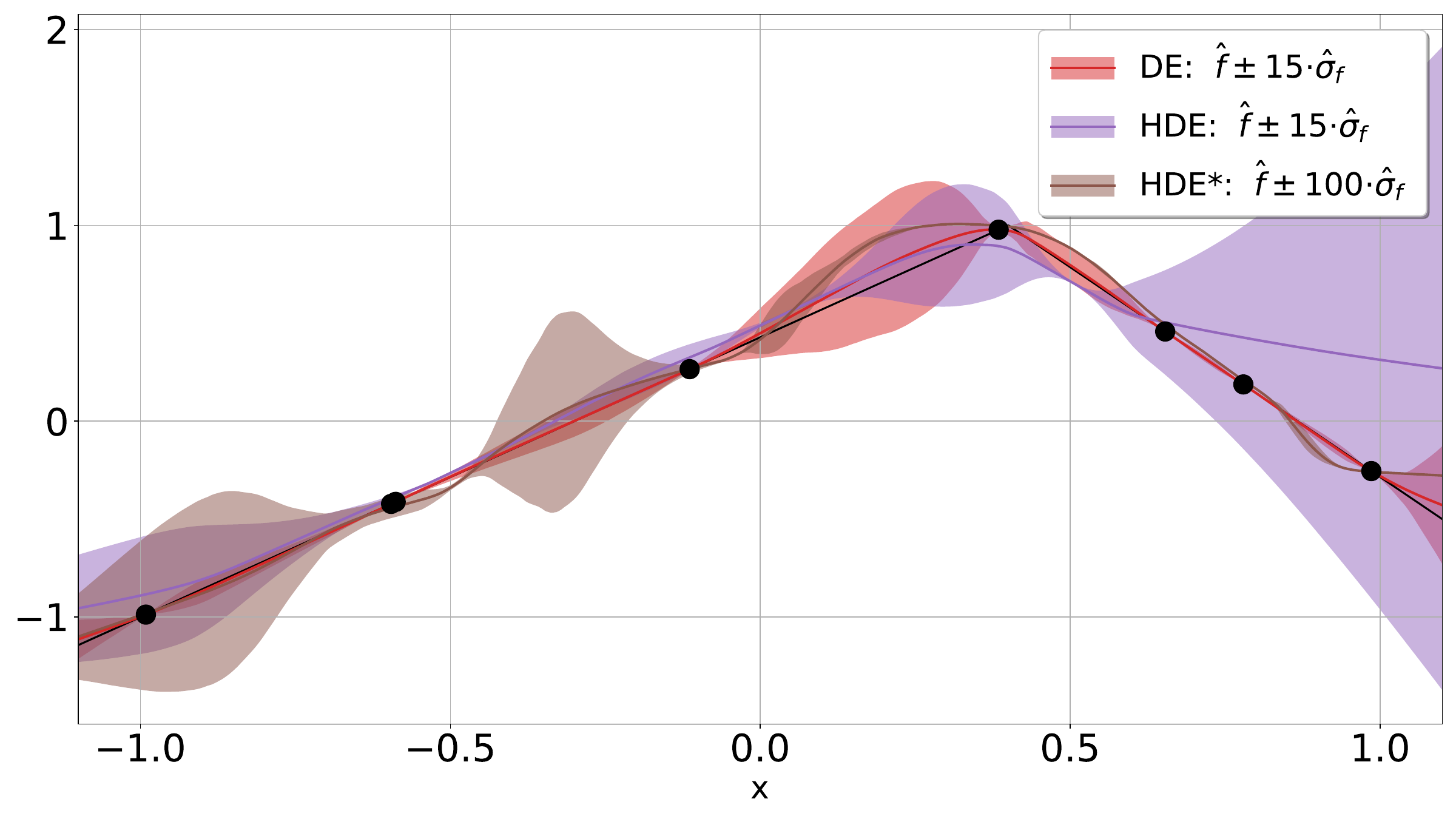}
	
    }
}
\makebox[\textwidth][c]{%
 \subfloat[Levy\label{subfig:levy}]{%
		\includegraphics[width =1\columnwidth]{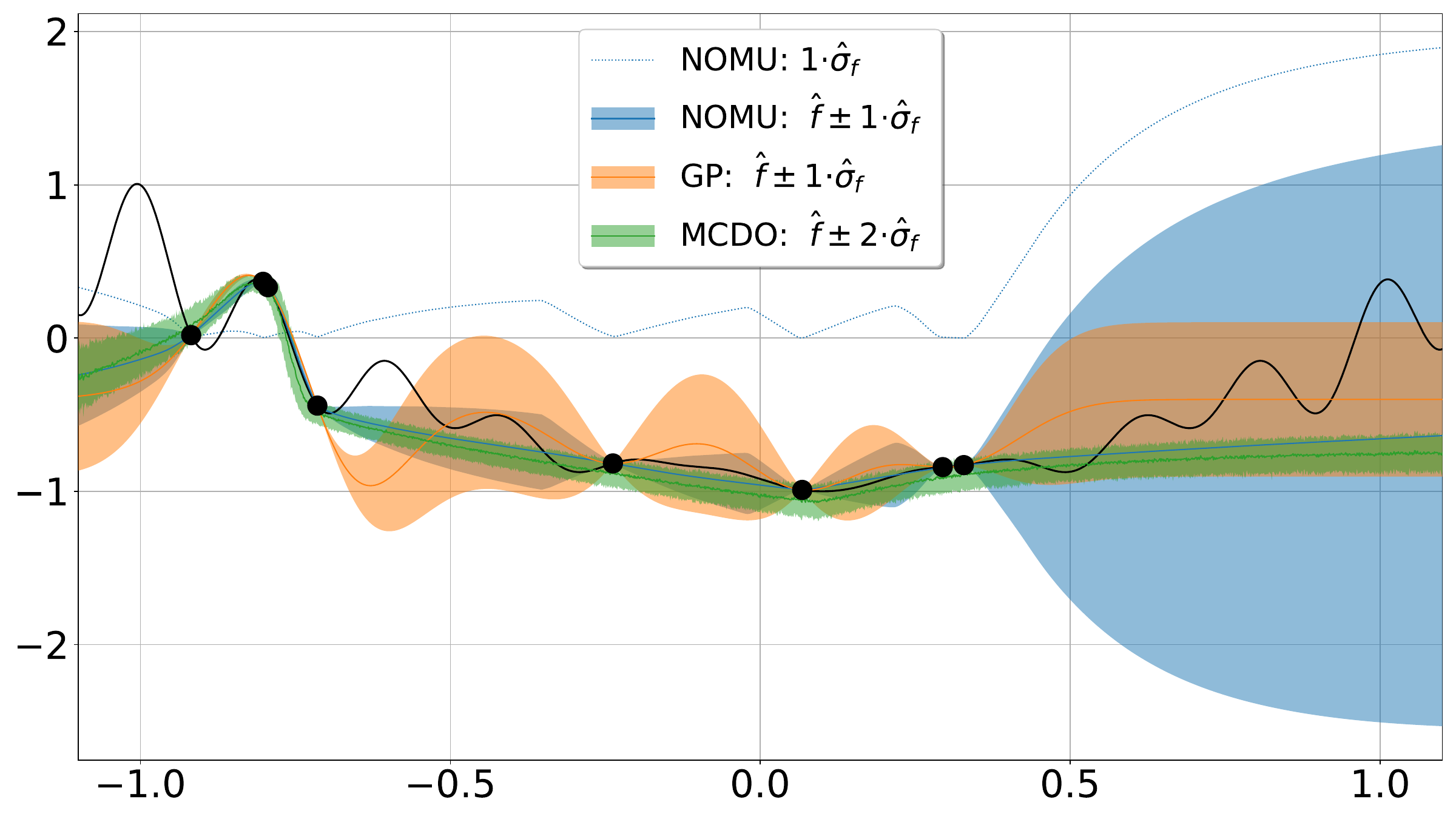}
		\includegraphics[width =1\columnwidth]{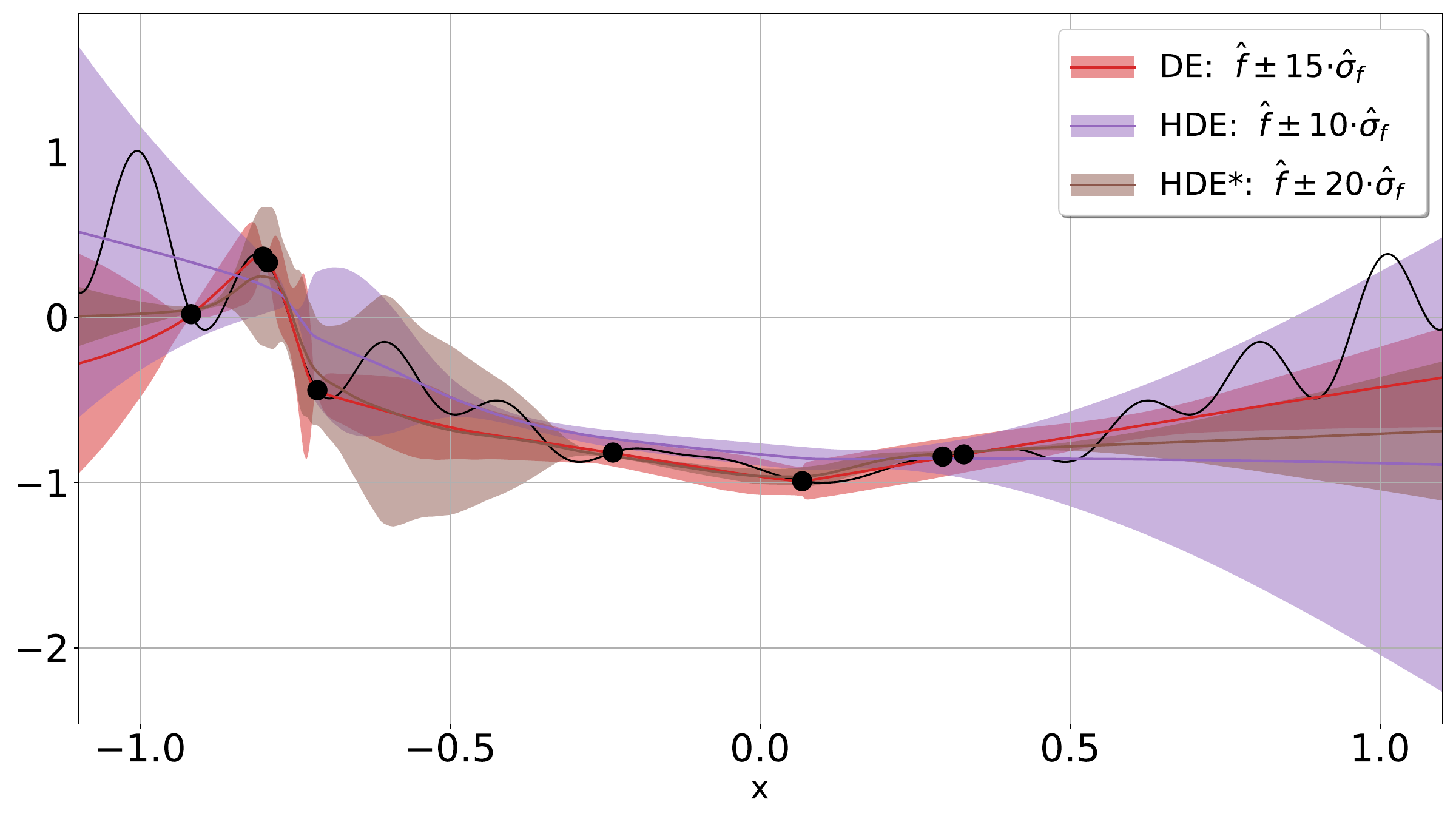}
	
    }
}
	\caption{Visualisations of resulting UBs for NOMU, GP, MCDO, DE, HDE and HDE* for a selection of test functions.}
	\label{fig:1dbounds}
\end{figure*}
\begin{table}[t!]
    \caption{Aggregate results for NOMU, GP, MCDO, DE, HDE and HDE* (our HDE see \Cref{subsec:ConfigurationDetailsofBenchmarksRegression}) for the set of all eleven 2D synthetic functions. Shown are the medians and a 95\% bootstrap CI of of \AUC{} and \MNLPD{} {\scriptsize(without constant $\ln(2\pi)/2$)} across all runs. Winners are marked in grey.}
	\label{tab:2dsynfunresultsaggregated}
	\vskip 0.1in
	\begin{center}
	\begin{small}
	\begin{sc}
	\resizebox{1\columnwidth}{!}{%
		\begin{tabular}{
				l
				c
				c
				c
				c
			}
			\toprule
				\textbf{Method}& {\textbf{\AUC}$\downarrow$} & 95\%-CI& {\textbf{\MNLPD}$\downarrow$} &95\%-CI \\
			\midrule
			NOMU & {0.42} &[0.41, 0.42] &{-0.99}& [-1.01, -0.96] \\
			GP &  \winc{0.36}& [0.36, 0.37] & \winc{-1.07}& [-1.09, -1.05] \\
			MCDO & {0.45}& [0.44, 0.45] & {-0.61}& [-0.62, -0.60]\\
			DE&  {0.42}& [0.41, 0.43] &  {-0.95}& [-0.97, -0.93] \\
			HDE& {5.06}& [4.75, 5.49] &  {\hphantom{-}1.74}& [\hphantom{-}1.66, \hphantom{-}1.84] \\
			HDE*& { 6.33}& [5.87, 6.75] &  {\hphantom{-}1.58}& [\hphantom{-}1.51, \hphantom{-}1.67] \\
			\bottomrule
		\end{tabular}
	}
	\end{sc}
	\end{small}
	\end{center}
	\vskip -0.1in
\end{table}
\subsubsection{Toy Regression: Detailed Characteristics of Uncertainty Bounds}\label{subsubsec:Uncertainty Bounds Detailed Characteristics}
\paragraph{UB Characteristics in 1D}
Within the one-dimensional setting, characteristics of NOMU UBs can be nicely compared to those of the benchmark approaches. \Cref{fig:1dbounds} exemplarily shows typical outcomes of the UBs on a number of selected test functions as obtained in one run.
\begin{itemize}[leftmargin=*,topsep=0pt,partopsep=0pt, parsep=0pt,itemsep=2pt]
\item \textbf{Deep Ensembles UBs:} Throughout the experiment, we observe that deep ensembles, while sometimes nicely capturing increased uncertainty in regions of high second derivative (e.g., around the kink in \Cref{subfig:kink}; pursuant in some form desiderata D4), still at times leads to UBs of somewhat arbitrary shapes. This can be seen most prominently in \Cref{subfig:postoneg} around $x\approx-0.25$, in \Cref{subfig:levy} around $x\approx -0.8$ or in \Cref{subfig:forrester} around $x\approx0.2$. Moreover, deep ensembles' UBs are very different in width, with no clear justification. This can be seen in \cref{subfig:postoneg} when comparing the UBs for $x\ge0$ against the UBs for $x\le0$ and at the edges of the input range in \Cref{subfig:jakobhfuntrend,subfig:levy}. In addition,
we frequently observe that deep ensembles UBs do not get narrow at training points, as for instance depicted in \Cref{subfig:squared} for $x<-0.5$, in \Cref{subfig:levy} for $x>0.2$, or in \Cref{subfig:forrester} for $x\in [-1,-0.7]$ and thus are not able to handle well small or zero data noise.
\item \textbf{Hyper Deep Ensembles UBs:}
Throughout our experiments, HDE produces less accurate UBs, which vary more from one run to another, than DE (recall that here we consider the setting of noiseless scarce regression). Possible reasons for this are:
\begin{enumerate}[leftmargin=*,topsep=0pt,partopsep=0pt, parsep=0pt,itemsep=0pt]
    \item The scarcity of training/validation data. HDE trains its NNs based on 80\% of the training points and uses the remaining 20\% to build an ensemble based on a score, whilst the other methods can use 100\% of the training points for training. In a scarce data setting this implies that first, the mean prediction of HDE does not fit through all the training points, and second, the scoring rule is less reliable.
    \item In a noiseless setting one already knows that the L2-regularization should be small, and thus optimizing this parameter is less useful here.
\end{enumerate}
Therefore, we believe HDE is less well suited in a noiseless scarce regression setting to reliably capture model uncertainty. This manifests in \Cref{fig:1dbounds}. HDE's uncertainty bounds often do not narrow at training points (\Cref{subfig:jakobhfuntrend,subfig:postoneg,subfig:squared,subfig:kink}), while they suggest unreasonably over-confident mean predictions in some regions (\Cref{subfig:forrester,subfig:postoneg}). While our HDE* (see \Cref{subsec:ConfigurationDetailsofBenchmarksRegression} Hyper Deep Ensembles) manages to better narrow at training data, it tends to more frequently result in unnecessarily narrow bounds (\Cref{subfig:jakobhfuntrend,subfig:squared,subfig:forrester}). Overall, both HDE and HDE* vary a lot from run to run and thus tend to yield bounds of seemingly arbitrary shapes, where in each run the desiderata are captured in different regions.
\item \textbf{MC Dropout UBs:} MCDO consistently yields tube-like UBs that do not narrow at training points. Interestingly, we remark that throughout the experiment MCDO samples frequently exhibit stepfunction-like shapes (e.g., see \Cref{subfig:forrester} at $x\approx0.5$ or \Cref{subfig:postoneg} for $x\in[-0.5,0]$). This effect intensifies with increasing training time.

\item\textbf{NOMU UBs:} In contrast, NOMU displays the behaviour it is designed to
show. Its UBs nicely tighten at training points and expand in between and thus NOMU fulfills \hyperref[itm:Axioms:trivial]{D1}\crefrangeconjunction\hyperref[itm:Axioms:largeDistantLargeUncertainty]{D3} across all considered synthetic test functions.

\item\textbf{Gaussian Process UBs:}  The quality of the RBF-Gaussian process' UBs (as judged in this simulated setting) naturally varies with the true underlying function. While the UBs nicely capture parts of the true function with low widths in \Cref{subfig:squared,subfig:forrester} they have a hard time accurately enclosing true functions that are not as conformant with the prior belief induced by the choice of the RBF kernel (e.g., \Cref{subfig:postoneg,subfig:kink}). Nonetheless, we also observe instances in which the training points are misleading to the GP's mean predictions even when considering ground truth functions for which this choice of kernel is very suitable. This manifests in over-confident mean predictions far away from the data generating true function (\Cref{subfig:jakobhfuntrend}) or over-swinging behavior of the fitted mean (\Cref{subfig:levy}). It is true that one could find better function-specific kernels. However, the RBF kernel is a good generic choice when evaluated across many different test functions without any specific prior information, which is why we choose this kernel for our comparison.
\end{itemize}

\paragraph{UB Characteristics in 2D} To visualize the UBs in 2D, we select as the ground truth the \emph{Styblinski} test function depicted in \Cref{fig:2dStyblinski}. In \Cref{fig:2dStyblinski_NOMU} (NOMU) and \Cref{fig:2dStyblinski_benchmarks} (benchmarks), we visualize the estimated model uncertainty as obtained in one run for this \emph{Styblinski} test function.
\begin{itemize}[leftmargin=*,topsep=0pt,partopsep=0pt, parsep=0pt,itemsep=0pt]
 \item \textbf{NOMU (\Cref{fig:2dStyblinski_NOMU}):} As in 1D, we observe that NOMU's UBs nicely tighten at input training points and expand in-between, with overall steady shapes. Specifically, NOMU's UBs are larger for extrapolation, e.g., $[0,1]\times[0.5,1]$, compared to regions which are \emph{surrounded} by input training data points, e.g., $[-0.25,0.25]\times[-0.25,-0.75]$, even though 
the distance to the closest input training point is approximately the same. Thus, NOMU's UBs do not only depend on the distance to the closest training point but rather on the arrangement of all surrounding training points.
\end{itemize}

\begin{figure}[t!]
\begin{center}
\resizebox{\columnwidth}{!}{
\includegraphics[trim= 10 5 50 10, clip]{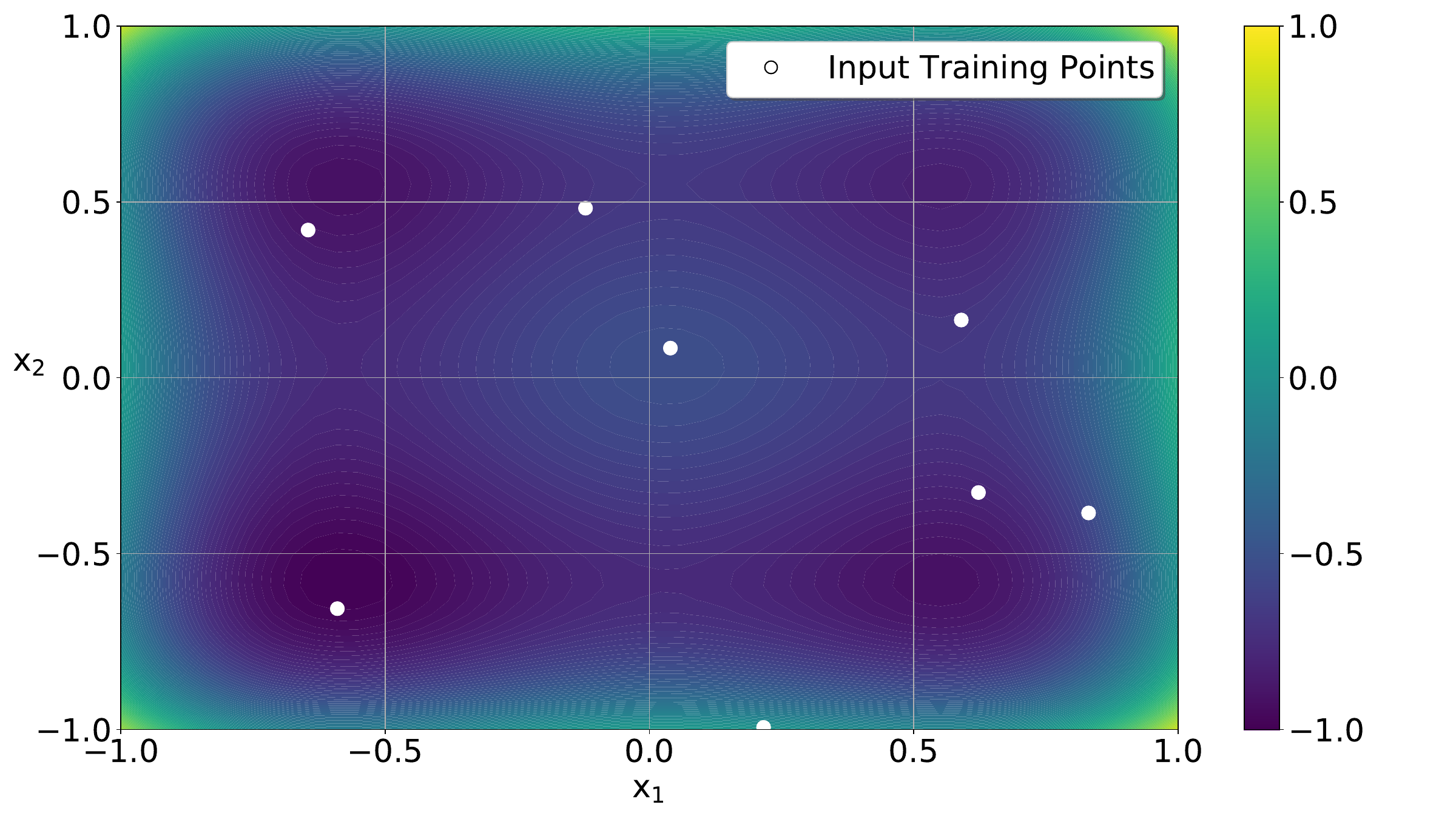}
}
\caption{Contour plot of the 2D Styblinski test function.}
\label{fig:2dStyblinski}
\end{center}
\vskip -0.2in
\end{figure}

\begin{figure}[t!]
\vskip 0.2in
\begin{center}
\resizebox{\columnwidth}{!}{
\includegraphics[trim= 10 5 80 10, clip]{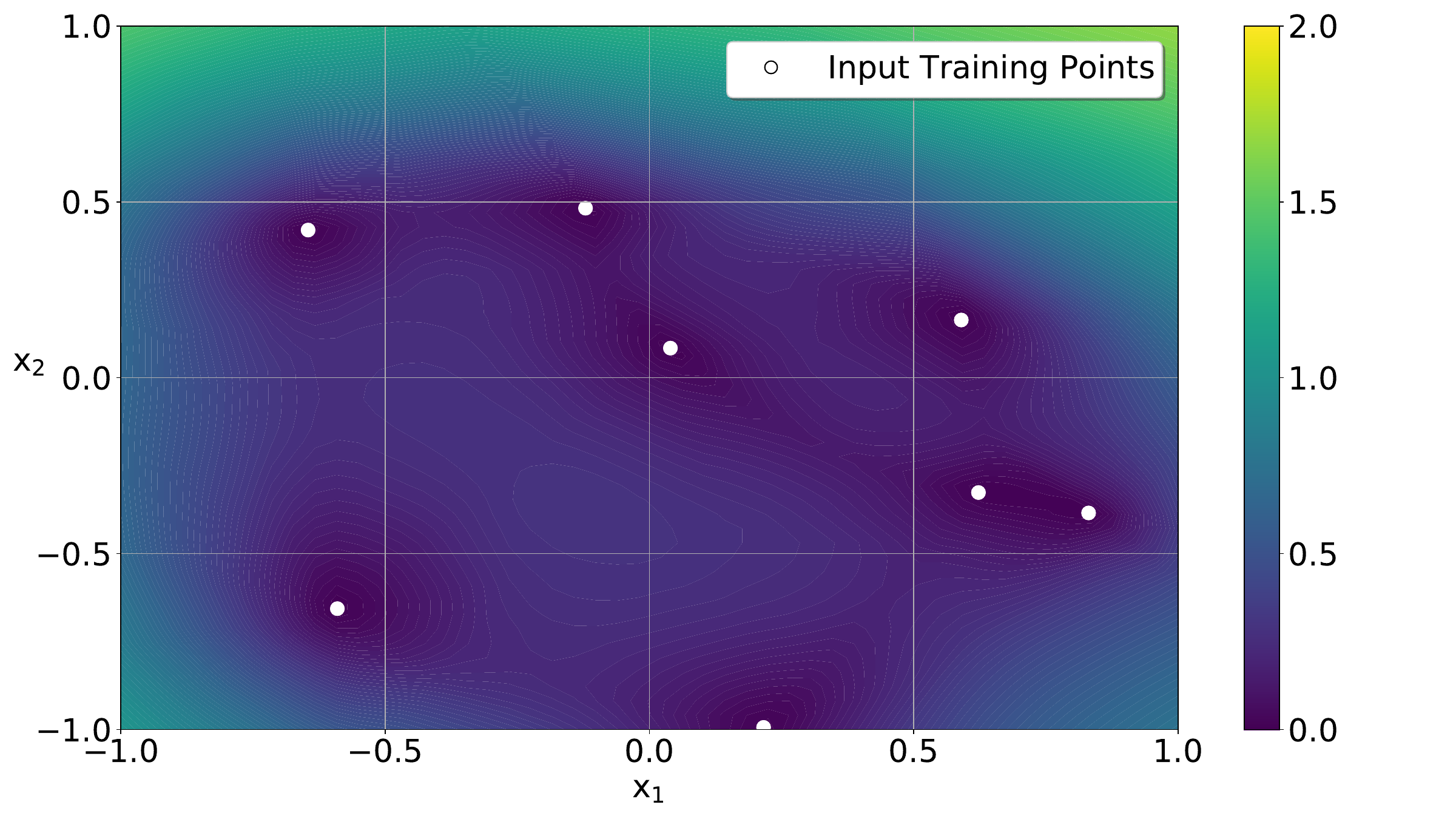}
}
\caption{Estimated model uncertainty $\sigmodelhat$ of NOMU for a single run of the Styblinski test function.}
\label{fig:2dStyblinski_NOMU}
\end{center}
\vskip -0.2in
\end{figure}

\begin{itemize}[leftmargin=*,topsep=0pt,partopsep=0pt, parsep=0pt,itemsep=0pt]
 \item \textbf{Benchmarks (\Cref{fig:2dStyblinski_benchmarks}):} As in 1D, we see that the RBF-based GP's UBs do have the expected smooth and isotropic shape with zero model uncertainty at input training points. Moreover, as in 1D, MCDO's UBs exhibit a tubular shape of equal size ($\approx$ 0--0.25) across the whole input space. Whilst DE nicely captures zero model uncertainty at input training points, it again exhibits the somehow arbitrary behaviour in areas with few input training points. Both HDE and HDE* yield model uncertainty estimates that are small at input training points, except for HDE for one input training point $(x_1,x_2)\approx (0.6,0.2)$, where the estimated model uncertainty is only small when continuing training on all training points (see HDE*). However, these estimates drastically fail to capture increased uncertainty in out-of-sample regions. For both algorithms, the model uncertainty estimate is unreasonably low in most regions of the input space and inexplicably high in a tiny fraction of the input space.
\end{itemize}
\begin{figure}[H]
\vskip 0.2in
\captionsetup[subfloat]{font=small,labelfont=small}
\centering
\subfloat[GP with calibration constant $c=1$. \label{subfig:2d_Styblinski_GPR}]{%
\includegraphics[trim= 10 5 80 10, clip, width=\columnwidth]{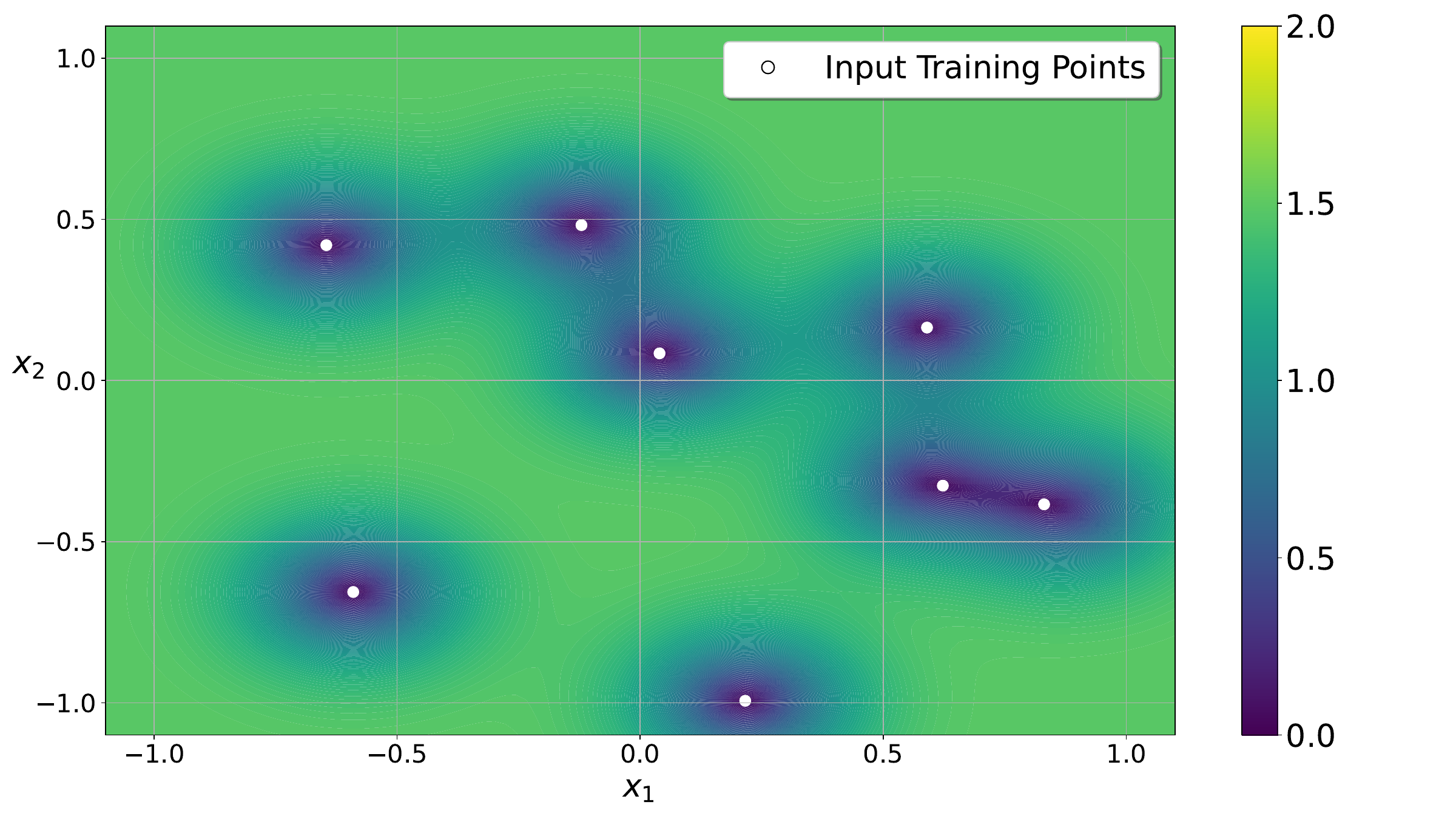}
}

\subfloat[{MCDO with calibration constant $c=10$.}\label{subfig:2d_Styblinski_MCDO}]{%
\includegraphics[trim= 10 5 80 10, clip, width=\columnwidth]{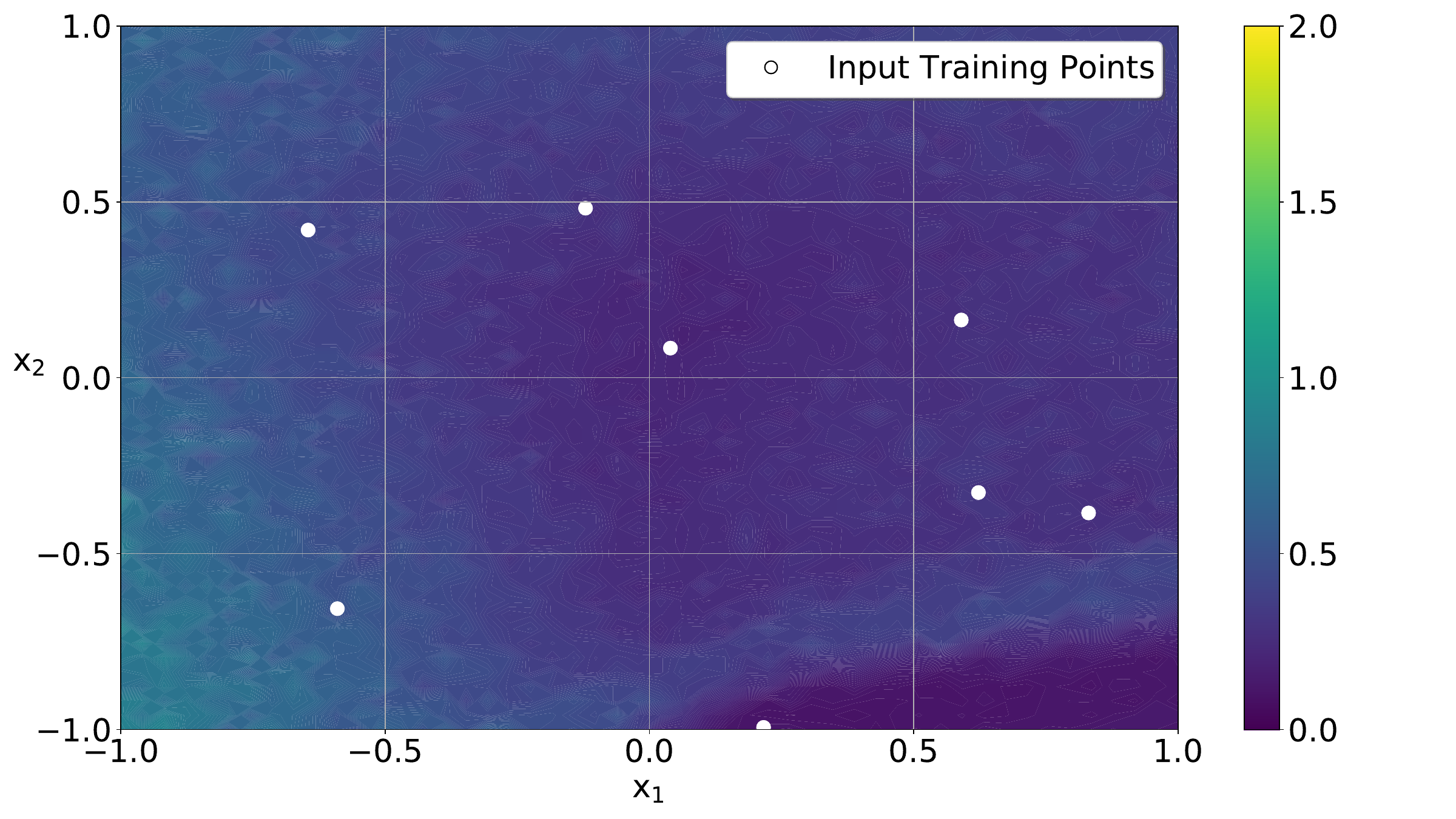}
}

\caption{Estimated model uncertainty of Gaussian process (GP) and MC dropout (MCDO) for a single run of the Styblinski test function.}
\end{figure}

\begin{figure}[H]
\ContinuedFloat

\subfloat[DE with calibration constant $c=15$.\label{subfig:2d_Styblinski_DE}]{%
\includegraphics[trim= 10 5 80 10, clip, width=\columnwidth]{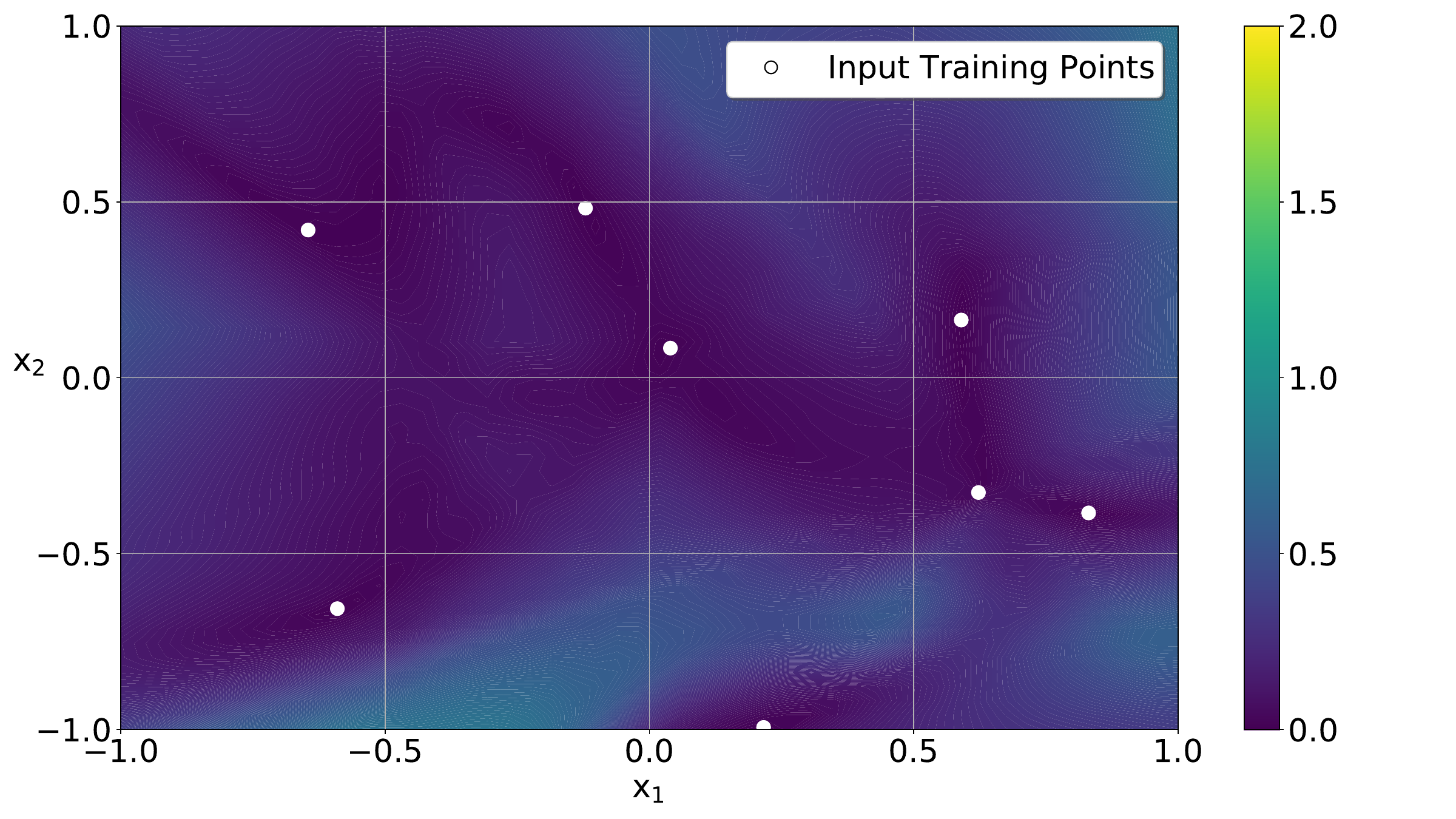}
}

\subfloat[HDE with calibration constant $c=30$.\label{subfig:2d_Styblinski_HDE}]{%
\includegraphics[trim= 10 5 80 10, clip, width=\columnwidth]{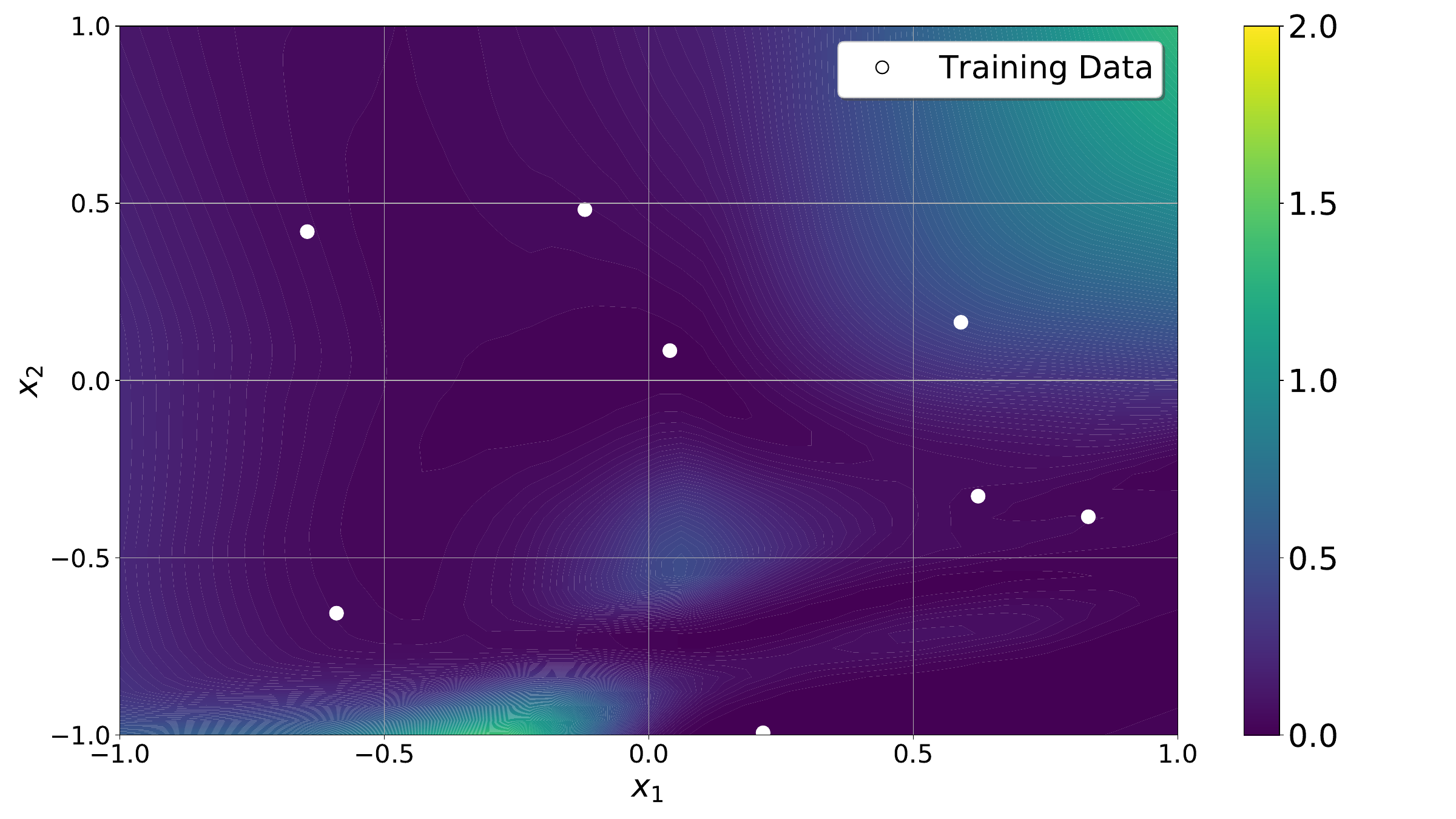}
}

\subfloat[HDE* with calibration constant $c=30$.\label{subfig:2d_Styblinski_HDEstar}]{%
\includegraphics[trim= 10 5 80 10, clip, width=\columnwidth]{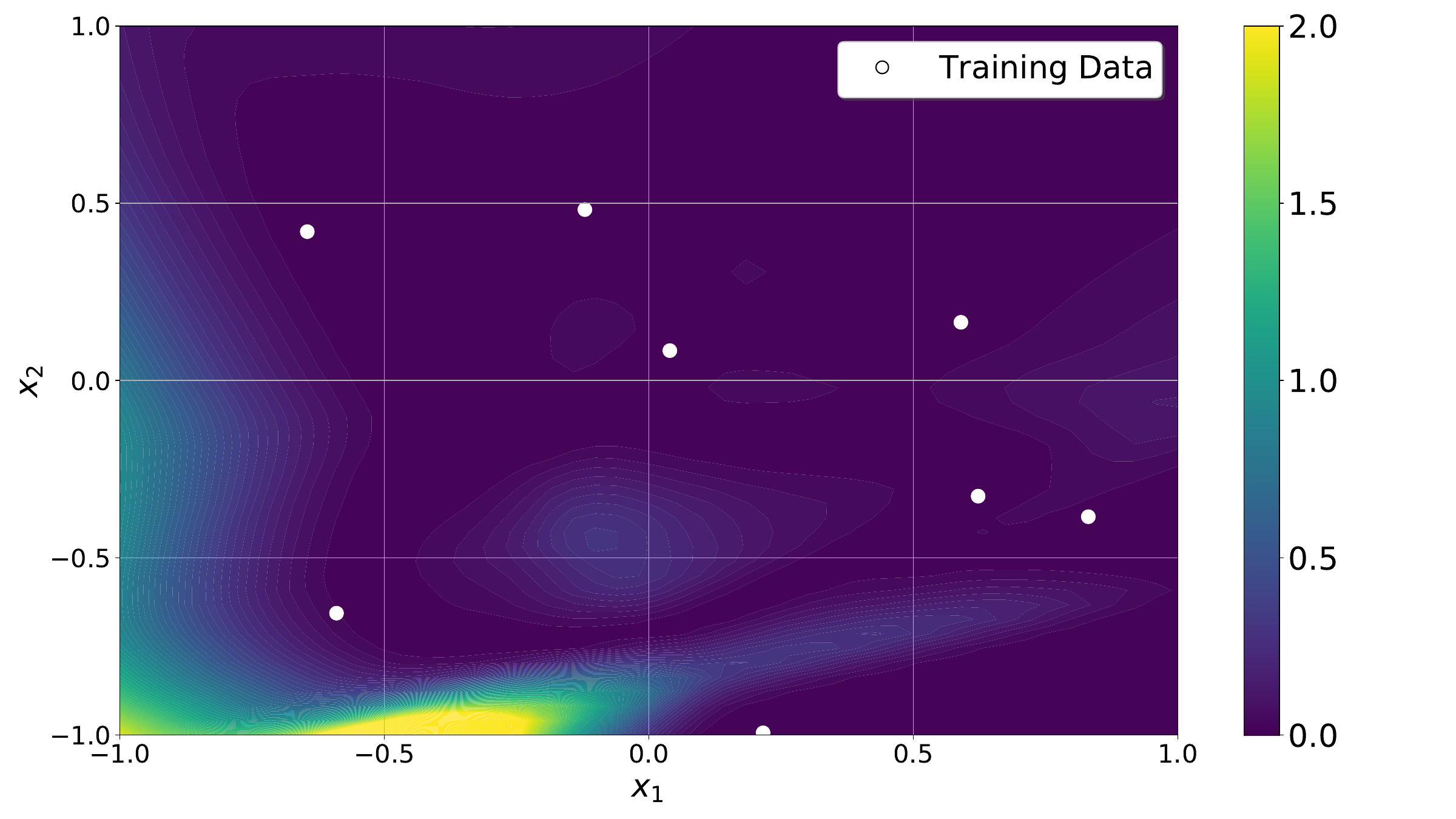}
}
\caption{(cont.) Estimated model uncertainty of deep ensembles (DE), hyper deep ensembles (HDE) and HDE* for a single run of the Styblinski test function.}
\label{fig:2dStyblinski_benchmarks}
\end{figure}

\subsubsection{Generative Test-Bed: Details}\label{sec:Appendix:GenerativeTestBed}
In this section, we give a detailed description of the generative test-bed setting and provide further results for up to $20$ dimensional input.

\paragraph{Detailed Setting}
\begin{enumerate}[leftmargin=*,topsep=0pt,partopsep=0pt, parsep=0pt,itemsep=0pt]
     \item We sample a test-function (i.e., the ground truth function) $f:[-1,1]^d\to \R$ from a BNN with i.i.d Gaussian weights.
     This means we sample i.i.d random parameters $\theta_i\sim\mathcal{N}(0,\sigma_{\text{prior}})$ and set $f=\NN_\theta$.
     The BNN is set up with three hidden layers with $2^{10}, 2^{11}$ and $2^{10}$ nodes, respectively.
     We choose $\sigma_{\text{prior}}$ such that $\E_\theta\left\lbrack\sqrt{\mathbbm{V}_x[\NN_\theta(x)\,|\, \theta]}\right\rbrack\approx 1.$%
     \footnote{$\E_\theta\left\lbrack\sqrt{\mathbbm{V}_x[\NN_\theta(x)\,|\, \theta]}\right\rbrack\approx 1$ holds true for $x\sim\text{Unif}\left([-1,1]^d\right)$. This might deviate (slightly) for Gaussian $x$.}
     Resulting values for $\sigma_{\text{prior}}$ are shown in \cref{tab:prior_variance}.
     
\begin{table}[h!]
    \caption{$\sigma_{\text{prior}}$ depending on the input dimension $d$.}
    \label{tab:prior_variance}
    \vskip 0.1in
    \begin{center}
    \begin{small}
    \begin{sc}
    \begin{tabular}{l
                S[table-format=2.3,detect-weight]}
		\toprule
	    $d$ & $\sigma_{\text{prior}}$  \\
	    \midrule
	    1  &  0.114\\
	    2  &  0.102\\
	    5  &  0.092\\
	    10 &  0.084\\
	    20 &  0.070\\
	    \bottomrule
	\end{tabular}
	\end{sc}
    \end{small}
    \end{center}
\end{table}

\item\label{itm:SamlingTheDataPoints} We sample $\ntr=8\cdot d$ input training points uniformly at random from the input space $[-1,1]^d$ for \Cref{tab:bnn_uniform,tab:Appendix:bnn_uniform}. In the case of \Cref{tab:Appendix:bnn_gauss}, we sample $\ntr=8\cdot d$ input training points from $d$-dimensional centered normal distribution with a covariance with $\min(5,d)$ eigenvalues of value $0.15$ and the remaining eigenvalues only have the value $0.001$.
     In both cases we get for input training point $\xtr_i$ the corresponding noiseless output-training point $\ytr_i=f(\xtr_i)=\NN_\theta(\xtr_i)$.
     And in both cases we create $\ntest=100\cdot d$ test-data points from the same distribution as the training data points.\footnote{Note that we sample the artificial data points for NOMU always uniformly from $[-1,1]^d$, since we assume that the low-dimensional manifold is often unknown in practice. NOMU could probably be improved if one tries to learn the distribution of the input data or if one already had some prior information about the distribution of the input data points.
     It is good to see how robustly NOMU can handle a different distribution of the input data points without the need to adapt the distribution of the artificial data points.}
     \item We calculate UBs of all considered algorithms including NOMU. For this, we use for all algorithms the same configuration as in the 1D and 2D toy regression settings (see \Cref{subsubsec:ToyRegression} and \Cref{subsec:ConfigurationDetailsofBenchmarksRegression}) except for NOMU, where we set $\sigmin=0.1$, $\sigmax=1$ and use $l=100\cdot d$ artificial data points, where $d$ denotes the respective dimension.
     \item We measure the quality of the UBs via \NLPD\ for a fine grid of different values for the calibration parameter $c$.
 \end{enumerate}
 We repeat these four steps $\mSeeds=200$-times. Then we choose for each method the value $c$ where the average \NLPD\ is minimal. Importantly, the chosen calibration constant $c$ \emph{only} depends on the input-dimension $d$ and on the algorithm but not on the randomly chosen test-function $f$.
 In \Cref{tab:Appendix:bnn_uniform} and \Cref{tab:Appendix:bnn_gauss}, we report the mean and a 95\% CI over these 200 runs for the chosen calibration constant $c$.
\begin{table}[t!]
    \setlength\tabcolsep{2pt}
    \caption{Uniform data generation. Shown are average \NLPD\ {\scriptsize(without constant $\ln(2\pi)/2$)} and a 95\% CI over 200 BNN samples.}
    \label{tab:Appendix:bnn_uniform}
    \vskip 0.1in
    \begin{center}
    \begin{small}
    \begin{sc}
    \resizebox{1\columnwidth}{!}{
    	\begin{tabular}{
    			l
    			c
    			c
    			c
    			c
    			c
    			c
    			c
    			c
    			c
    			c
    			c
    		}
    		\toprule
    		\textbf{Function} & \multicolumn{1}{c}{\textbf{NOMU}} &  \multicolumn{1}{c}{\textbf{GP}}  & \multicolumn{1}{c}{\textbf{MCDO}}  &  \multicolumn{1}{c}{\textbf{DE}} & \multicolumn{1}{c}{\textbf{HDE}}&
    		\multicolumn{1}{c}{\textbf{TUBE}}\\
    		\midrule
    	    BNN1D & \winc-1.65$\pm$\scriptsize0.10 & -1.08$\pm$\scriptsize0.22 & -0.34$\pm$\scriptsize0.23 &  -0.38$\pm$\scriptsize0.36 & \hphantom{-}8.47$\pm$\scriptsize1.00 &-0.86$\pm$\scriptsize0.19\\
    	    BNN2D &  \winc-1.16$\pm$\scriptsize0.05& -0.52$\pm$\scriptsize0.11 &-0.33$\pm$\scriptsize0.13 & -0.77$\pm$\scriptsize0.07 & \hphantom{-}9.10$\pm$\scriptsize0.39&-0.81$\pm$\scriptsize 0.06 \\
    	    BNN5D & \winc-0.37$\pm$\scriptsize0.02 & -0.33$\pm$\scriptsize0.02 & -0.05$\pm$\scriptsize0.04 & -0.13$\pm$\scriptsize0.03 & \hphantom{-}8.41$\pm$\scriptsize1.00&-0.33$\pm$\scriptsize0.02 \\
    	    BNN10D & \winc-0.09$\pm$\scriptsize0.01 &\winc-0.11$\pm$\scriptsize0.01 & \hphantom{-}0.06$\pm$\scriptsize0.02 & \hphantom{-}0.10$\pm$\scriptsize0.01 &\hphantom{-}6.44$\pm$\scriptsize 1.36&  \winc-0.12$\pm$\scriptsize 0.01\\
    	    BNN20D & \hphantom{-}0.15$\pm$\scriptsize0.01 & \winc\hphantom{-}0.06$\pm$\scriptsize0.01 & \hphantom{-}0.18$\pm$\scriptsize0.01 & \hphantom{-}0.30$\pm$\scriptsize0.01 & \hphantom{-}3.58$\pm$\scriptsize0.81&  \winc \hphantom{-}0.07$\pm$\scriptsize 0.01 \\
    		\bottomrule 
    	\end{tabular}
    }
    \end{sc}
    \end{small}
    \end{center}
    \vskip -0.1in
\end{table}
\paragraph{Discussion of the Results in \Cref{tab:Appendix:bnn_uniform}}
We see that for $d\leq5$ NOMU significantly outperforms all other considered benchmarks. For dimensions $d\geq 10$ it gets harder to find a good metric for measuring the quality of the uncertainty bounds:
For high dimensions almost no test data point sampled uniformly from $[-1,1]^d$ is close to the training data points, so they all have quite high uncertainty. We further verified this for dimensions $d=10$ and $d=20$ by introducing another algorithm we call \textsc{Tube}. \textsc{Tube} uses NOMU's mean prediction and assigns the same (calibrated) constant uncertainty $c$ to all test data points. As we can see in \Cref{tab:Appendix:bnn_uniform}, \textsc{Tube} is ranked first (en par with GP) in $d=10$ and $d=20$, highlighting that the metric \NLPD\ is flawed in this setting for dimensions $d\ge10$. However, for dimensions $d\leq5$ the trivial \textsc{Tube} algorithm is significantly outperformed by other methods.
Thus, the \NLPD\-metric for dimensions $d\geq10$ and uniformly distributed data on $[-1,1]^d$ mainly measures if \ref{itm:Axioms:largeDistantLargeUncertainty} is reliably fulfilled. \ref{itm:Axioms:largeDistantLargeUncertainty} is particularly reliably fulfilled for \textsc{Tube} and for GP. For BO however, especially \ref{itm:Axioms:ZeroUncertaintyAtData} is important to not get stuck in local optima, because in BO the test data-points are not uniformly distributed, and the predicted uncertainty close to the training points is particularly important.
Therefore, in this setting, we only include results for input dimensions 1-5D (see \Cref{tab:bnn_uniform}) in the main part of the paper.

\paragraph{Discussion of the Results in \Cref{tab:Appendix:bnn_gauss}} Another approach to circumvent the problem that the \NLPD\ metric loses its relevance for scenarios where the input test data points are all far away from the input training data points is to use a different distribution for $X$. The well known manifold-hypothesis \citep{cayton2005algorithms} states that in typical real world data-sets relevant input data points lie close to a low-dimensional manifold. Therefore, we run a further experiment where we sample from our training and test data from a Gaussian distribution that concentrates mainly close to a $\min(d,5)$-dimensional flat manifold as described in \cref{itm:SamlingTheDataPoints} and report the results in \Cref{tab:Appendix:bnn_gauss}. For this experiment, we see that dimension $d\geq10$ are interesting too, since \textsc{Tube} is significantly beaten by NOMU and GP. In \Cref{tab:Appendix:bnn_gauss}, one can see that NOMU is among the best methods for all considered dimensions.
\begin{table}[t!]
    \setlength\tabcolsep{2pt}
    \caption{5D-Gaussian data generation. Shown are average \NLPD\ {\scriptsize(without constant $\ln(2\pi)/2$)} and a 95\% CI over 200 BNN samples.}
    \label{tab:Appendix:bnn_gauss}
    \vskip 0.1in
    \begin{center}
    \begin{small}
    \begin{sc}
    \resizebox{1\columnwidth}{!}{
    	\begin{tabular}{
    			l
    			c
    			c
    			c
    			c
    			c
    			c
    			c
    			c
    			c
    			c
    			c
    		}
    		\toprule
    		\textbf{Function} & \multicolumn{1}{c}{\textbf{NOMU}} &  \multicolumn{1}{c}{\textbf{GP}}  & \multicolumn{1}{c}{\textbf{MCDO}}  &  \multicolumn{1}{c}{\textbf{DE}} & \multicolumn{1}{c}{\textbf{HDE}} &\multicolumn{1}{c}{\textbf{TUBE}}\\
    		\midrule
    	    BNN1D & \winc-1.91$\pm$\scriptsize0.13 & -1.13$\pm$\scriptsize0.27 & -0.49$\pm$\scriptsize0.16 &  -0.27$\pm$\scriptsize0.54 & \hphantom{-}8.95$\pm$\scriptsize1.18& -0.85$\pm$\scriptsize0.13\\
    	    BNN2D &  \winc-1.55$\pm$\scriptsize0.05& -0.85$\pm$\scriptsize0.12 &-0.62$\pm$\scriptsize0.13 & -1.15$\pm$\scriptsize0.08 & \hphantom{-}7.92$\pm$\scriptsize0.42& -0.95$\pm$\scriptsize0.11 \\
    	    BNN5D & \winc-0.77$\pm$\scriptsize0.02 & -0.72$\pm$\scriptsize0.02 & -0.47$\pm$\scriptsize0.03 & -0.55$\pm$\scriptsize0.03 & \hphantom{-}9.68$\pm$\scriptsize1.08& -0.68$\pm$\scriptsize0.02 \\
    	    BNN10D & \winc-1.36$\pm$\scriptsize0.01 &\winc-1.31$\pm$\scriptsize0.01 & -1.08$\pm$\scriptsize0.03 & -1.14$\pm$\scriptsize0.02 &\hphantom{-}5.16$\pm$\scriptsize 1.55& -1.25$\pm$\scriptsize0.01\\
    	    BNN20D & \winc-1.70$\pm$\scriptsize0.01 & \winc-1.72$\pm$\scriptsize0.01 & -1.52$\pm$\scriptsize0.02 & -1.53$\pm$\scriptsize0.01 & 0.99$\pm$\scriptsize0.39& -1.68$\pm$\scriptsize0.01 \\
    	    
    		\bottomrule 
    	\end{tabular}
    }
    \end{sc}
    \end{small}
    \end{center}
    \vskip -0.1in
\end{table}
\paragraph{Our Generative Test-Bed vs. \citet{osband2021epistemic}}
\citet{osband2021epistemic} measures the Kullback–Leibler divergence between the posterior of any method to the posterior of shallow GP with a fixed (deep) neural tangent kernel (NTK). Thus, they evaluate methods by their similarity to a NTK-GP posterior. Because of the fixed kernel the posterior of the NTK-GP does not fulfill \ref{itm:Axioms:irregular} at all. This implies that their evaluation metric does not reward \ref{itm:Axioms:irregular} at all. However, we think that \ref{itm:Axioms:irregular} is very important for real-world deep learning applications (see
\cref{footnote:LearningTheMetricISImportant} and \Cref{subsec:A note on Desideratum D4}).
The posterior of a finite-width BNN fulfills \ref{itm:Axioms:irregular}. Therefore, approximating such a posterior is more desirable. However, in contrast to the NTK-GP posterior in \cite{osband2021epistemic}, there is no closed formula for a finite-width BNN posterior. Thus, at first sight one cannot straightforwardly evaluate methods based on the Kullback–Leibler divergence between their posterior and this intractable BNN posterior as in \citet{osband2021epistemic}.

Nonetheless, in the following, we prove in \Cref{thm:appendix_dkl_approximation} that the metric we use in \Cref{tab:bnn_uniform,tab:Appendix:bnn_uniform,tab:Appendix:bnn_gauss} indeed converges (up to a constant)\footnote{If we are just interested in the relative performance of different methods compared to each other, the constant does not matter.} to the average Kullback–Leibler divergence $\adkl$ to the posterior of a finite width BNN as $\mSeeds$ tends to infinity. First, we define the average Kullback–Leibler divergence $\adkl$ and introduce some notation.
\begin{definition}[Avg Kullback–Leibler divergence]\label{def:average_dKL}
Let $\Dtr$ be a finite set of training points and 
consider a prior distribution\footnote{In our generative test-bed the prior distribution is a BNN prior with i.i.d centered Gaussian weights.} $\PP[f\in\cdot]$ on the function space $\{f:X\to Y\}$ and the corresponding posterior $\PP[f\in\cdot\ |\ \Dtr]$ on the function space. Then the marginal of the posterior $\PP[f(x)\in\cdot\ |\ \Dtr,x]$ is a measure on $\R$ for every given input data point $x\in X$ and every given training data set $\Dtr$. Let $\Q[\cdot \ |\ \Dtr,x]$ also be a measure on $\R$ for every given input data point $x\in X$ and every given training data set $\Dtr$.\footnote{In our context $\Q[\cdot \ |\ \Dtr,x]$ can be the approximation of the marginal posterior at $x\in X$ given training data $\Dtr$ obtained from any method such as NOMU, GP, MCDO, DE, HDE, etc.}
The average Kullback–Leibler divergence is then defined as
\[\adkl=\E_{\Dtr,x}\left[{\dkl\left(\PP[f(x)\in\cdot\ |\ \Dtr,x]\ ||\ \Q[\cdot \ |\ \Dtr,x]\right)}\right],\]
where the expectation is taken over $x$ and $\Dtr$ according to $\PP$, and $\dkl$ is the classical \href{https://en.wikipedia.org/wiki/Kullback\%E2\%80\%93Leibler_divergence}{Kullback–Leibler divergence} between two probability measures on $\R$.
This is equivalent to Equation~(1) from \cite{osband2021epistemic}.
\end{definition}
\begin{theorem}\label{thm:appendix_dkl_approximation} Using the notation from \Cref{def:average_dKL}, let $q(\cdot\ |\ \Dtr,x)$ be the density of $\Q[\cdot\ |\ \Dtr,x]$ on $\R$.
Let $(f_j(x_j),\Dtr_j,x_j)_{j\in \fromto{\mSeeds}}$ be i.i.d samples of $\PP[(f(x),\Dtr,x)\in \cdot]$.%
\footnote{\label{footnote:dkl_theorem_ntest_geq2}We formulate the theorem for the hardest case of $\ntest=1$. The convergence of

\scalebox{0.85}{\parbox{1\columnwidth}{
\[\lim_{\mSeeds\to\infty}\frac{1}{\mSeeds}\sum_{j=1}^{\mSeeds}\frac{1}{\ntest}\sum_{i=1}^{\ntest}-\log\left(q(f_j(x_{j,i})\ |\ \Dtr_j,x_{j,i})\right)
=\adkl+C_{\PP}\]
}}\newline
is obviously even faster for $\ntest>1$. For the proof it is only important that $\mSeeds\to\infty$.
}
Then, the average \NLPD{}

\resizebox{\columnwidth}{!}{\parbox{1\columnwidth}{
\[\lim_{\mSeeds\to\infty}\frac{1}{\mSeeds}\sum_{j=1}^{\mSeeds}-\log\left(q(f_j(x_j)\ |\ \Dtr_j,x_j)\right)
=\adkl+C_{\PP}\]
}}
converges ($\PP$-a.s.) to $\adkl$ up to an additive constant $C_{\PP}$, where $C_{\PP}$ only\footnote{$\PP$ does not depend on the chosen method. Only $\Q$ (and accordingly $q$) differs among the methods. Thus, $C_{\PP}$ does not change the ranking amongst different methods.} depends on $\PP$ and not on $\Q$.
\end{theorem}
\begin{proof}
Let $p$ be the density of $\PP[(f(x),\Dtr,x)\in \cdot]$ and $p(\cdot\ |\ \Dtr,x)$ the density on $\R$ of the marginal posterior $\PP[f(x)\in\cdot\ |\ \Dtr,x]$.
Further we write $p(\Dtr,x)$ for the density of of the marginal $\PP[(\Dtr,x)\in\cdot]$ evaluated at $(\Dtr,x)$.
Since
\[\lim_{\mSeeds\to\infty}\frac{1}{\mSeeds}\sum_{j=1}^{\mSeeds}-\log\left(q(f_j(x_j)\ |\ \Dtr_j,x_j)\right)\]
can be seen as a Monte-Carlo approximation, it converges ($\PP$-a.s.) to%

\resizebox{\columnwidth}{!}{\parbox{1\columnwidth}{
\begin{align*}
    &\E_{f(x),\Dtr,x}[-\log\left(q(f(x)\ |\ \Dtr,x)\right)]=\\
    &=\int -\log\left(q(f(x)\ |\ \Dtr,x)\right) p(f(x),\Dtr,x) d(f(x),\Dtr,x)=\\
    &=\int -\log\left(q(f(x)\ |\ \Dtr,x)\right) p(f(x)\ |\ \Dtr,x)p(\Dtr,x) d(f(x),\Dtr,x).
\end{align*}
}}%

By Fubini this is equal to%

\resizebox{\columnwidth}{!}{\parbox{1\columnwidth}{
\begin{align*}
    &\int\int -\log\left(q(f(x)\ |\ \Dtr,x)\right) p(f(x)\ |\ \Dtr,x) d(f(x)) p(\Dtr,x) d(\Dtr,x)=\\
    &=\E_{\Dtr,x}\Bigg\lbrack
    \underbrace{\int -\log\left(q(f(x)\ |\ \Dtr,x)\right) p(f(x)\ |\ \Dtr,x) d(f(x))}_{H\left(\PP[f(x)\in\cdot\ |\ \Dtr,x],\Q[\cdot\ |\ \Dtr,x]\right)}
    \Bigg\rbrack,
\end{align*}
}}%

where $H$ is the \href{https://en.wikipedia.org/wiki/Cross_entropy}{cross-entropy}. So this is further equal to

\resizebox{\columnwidth}{!}{\parbox{1\columnwidth}{
\begin{align*}
    &\E_{\Dtr,x}\left\lbrack
    \dkl\left(\PP[f(x)\in\cdot\ |\ \Dtr,x]\ ||\ \Q[\cdot\ |\ \Dtr,x]\right) 
    \right\rbrack
    +\\
    &\E_{\Dtr,x}\left\lbrack
    H\left(\PP[f(x)\in\cdot\ |\ \Dtr,x]\right) 
    \right\rbrack=\\
    &= \adkl + C_{\PP},
\end{align*}
}}

where $C_{\PP}=\E_{\Dtr,x}\left\lbrack
    H\left(\PP[f(x)\in\cdot\ |\ \Dtr,x]\right) 
    \right\rbrack$ only depends on $\PP$ and not on $\Q$ or $q$.
\end{proof}
To apply \Cref{thm:appendix_dkl_approximation} to the metrics reported in \Cref{tab:bnn_uniform,tab:Appendix:bnn_uniform,tab:Appendix:bnn_gauss} one has to apply \cref{footnote:dkl_theorem_ntest_geq2} and has to set $q(\cdot\ |\ \Dtr,x)$ to be the density corresponding to the already calibrated uncertainty at $x$ obtained from any method trained on $\Dtr$, i.e., for example let $c$ be a fitted calibration parameter, and $\fhat$ and $\sigmahat$ be the fitted model and model uncertainty prediction of NOMU then $q(\cdot\ |\ \Dtr,x):=\mathcal{N}(\cdot;\fhat(x),c\cdot \sigmahat(x))$ for $x\in X$.\footnote{In theory, more general posteriors than Gaussians could be used too, but within this paper we always assumed Gaussian marginals of the posterior as all the considered benchmarks also output Gaussian distributed approximations of the marginal posteriors.} In our setting, we made sure that the correct calibration constant $c$ is chosen in the limit $\mSeeds\to\infty$, since we chose one fixed value for $c$ per dimension and method and do not over-fit on specific seeds.

\begin{figure*}[htbp]
    \centering
    \resizebox{.887\textwidth}{!}{
	\includegraphics[trim= 0 180 0 200, clip,width=2\columnwidth]{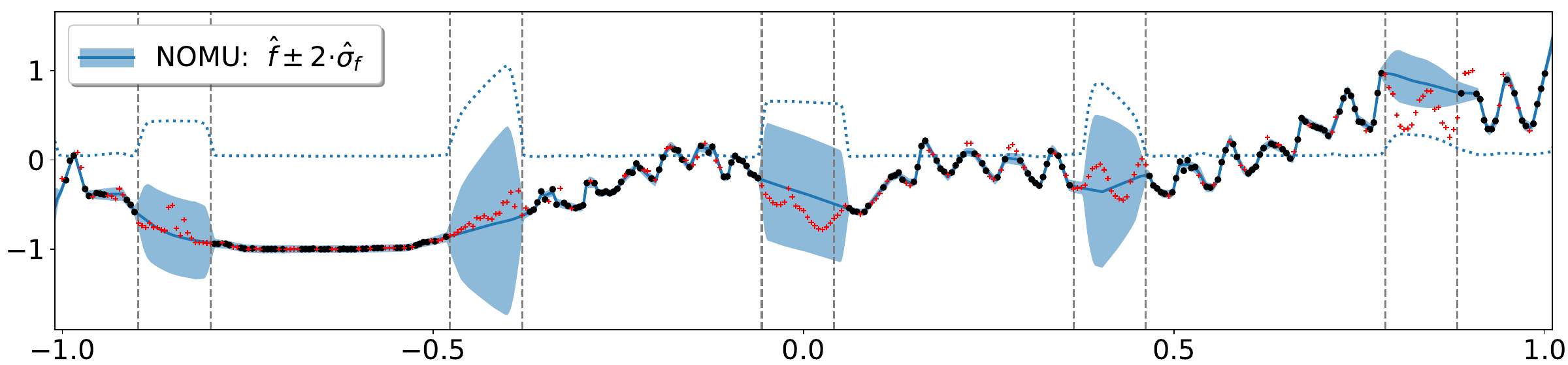}
	}
	\resizebox{.887\textwidth}{!}{
	\includegraphics[trim= 0 180 0 200, clip,width=2\columnwidth]{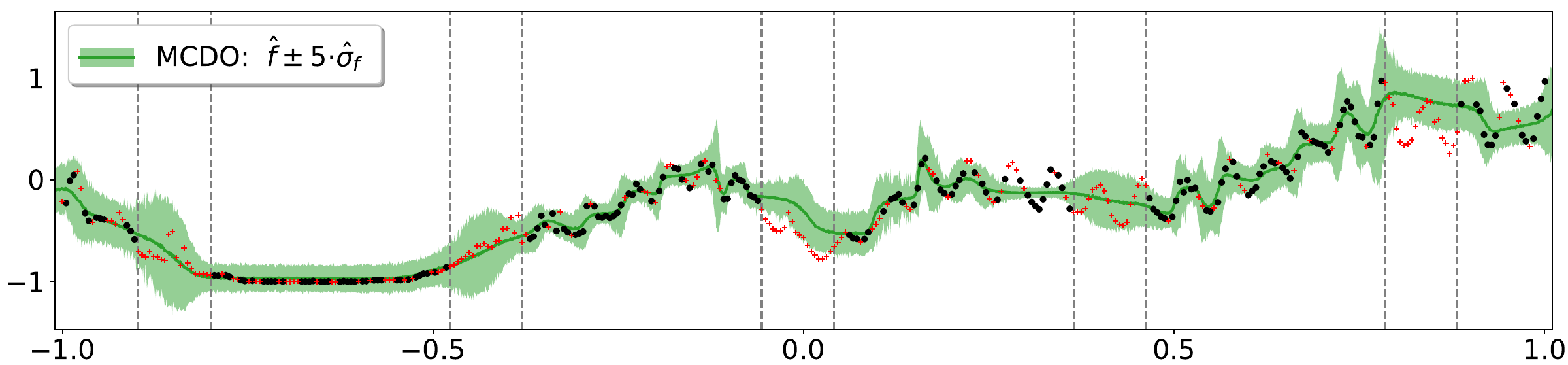}
	}
	\resizebox{.887\textwidth}{!}{
	\includegraphics[trim= 0 180 0 200, clip,width=2\columnwidth]{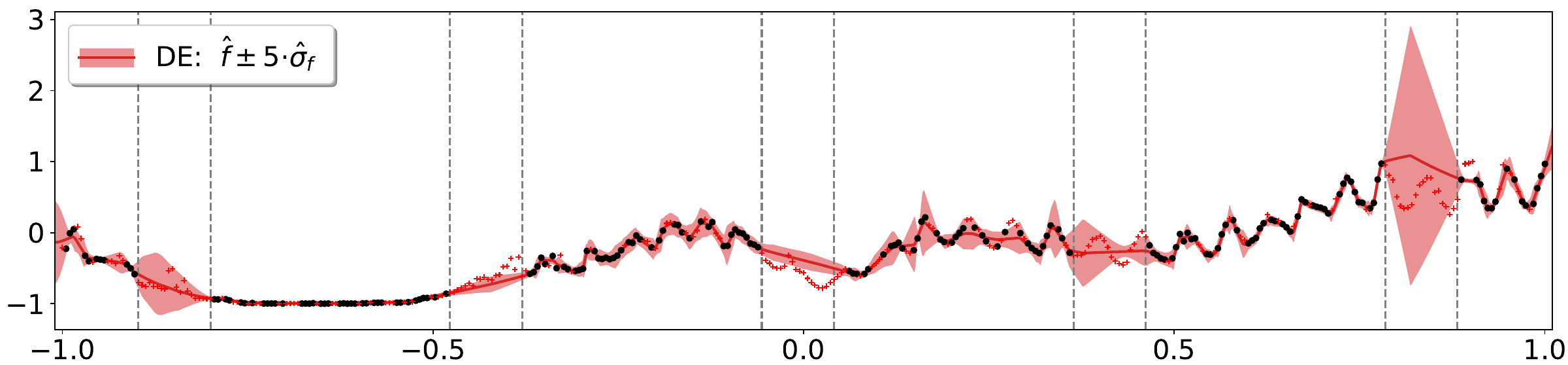}
	}
	\resizebox{.887\textwidth}{!}{
    \includegraphics[trim= 0 180 0 200, clip,width=2\columnwidth]{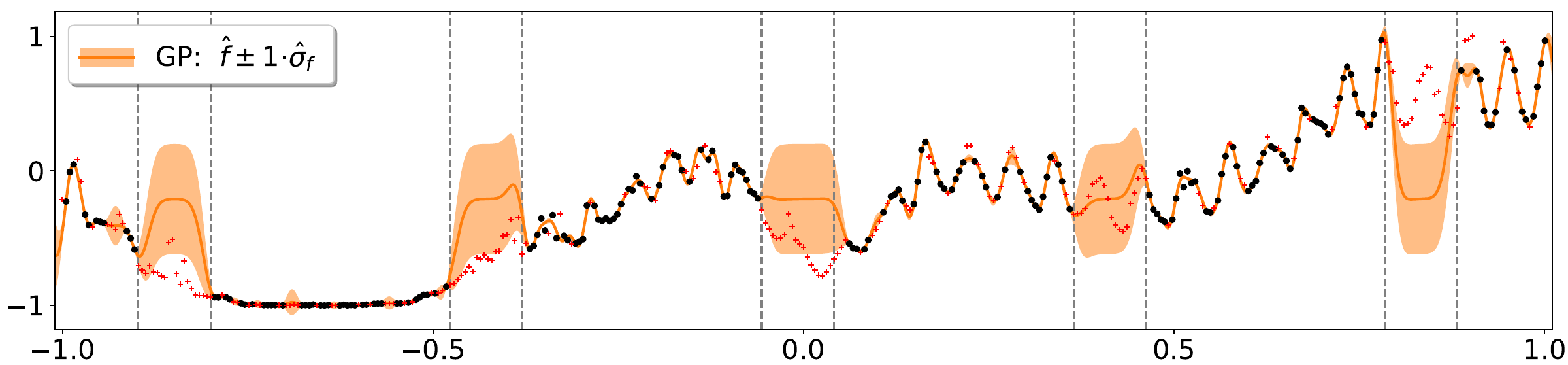}
    }
    \resizebox{.887\textwidth}{!}{
    \includegraphics[trim= 0 180 0 200, clip,width=2\columnwidth]{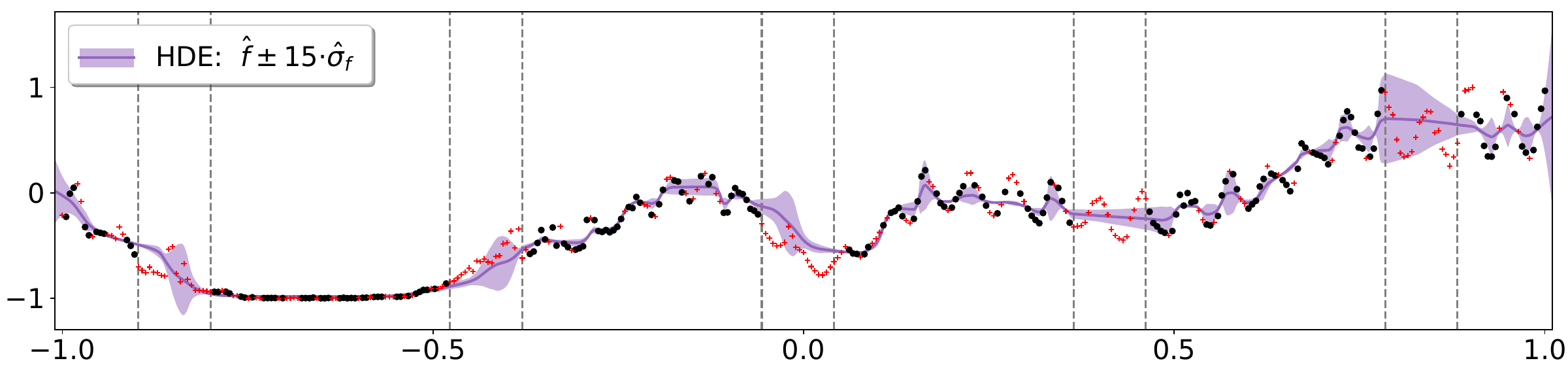}
    }
    \resizebox{.887\textwidth}{!}{
    \includegraphics[trim= 0 180 0 200, clip,width=2\columnwidth]{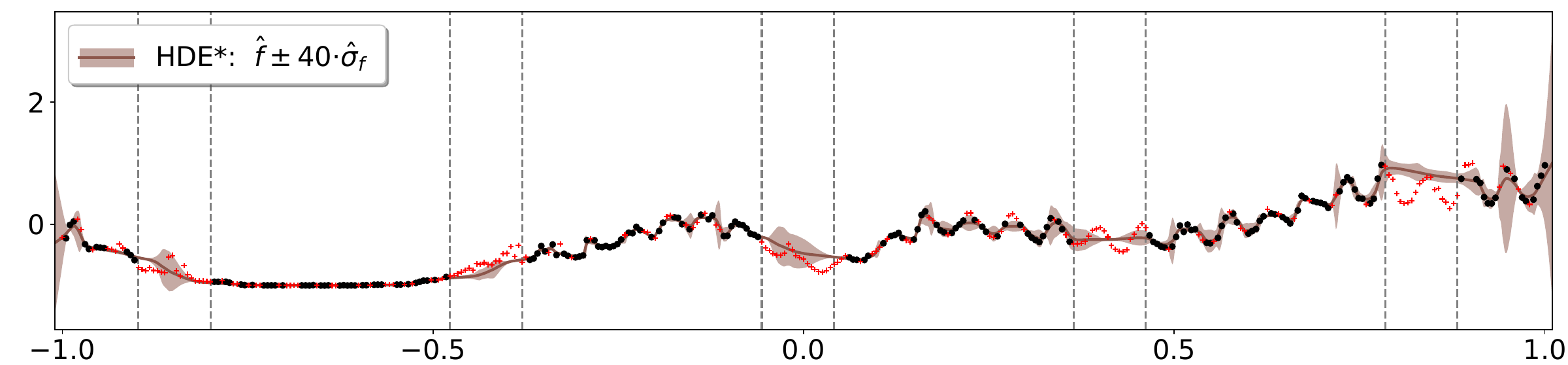}
    }
    \vskip -0.35cm
    \caption{
        Visualization of each algorithm's model prediction (solid lines) and UBs (shaded areas) on the solar irradiance data set. Training and test points are shown as black dots and red crosses, respectively. We present UBs obtained by NOMU (c=2) compared to the benchmarks MC dropout (MCDO) (c=5), deep ensembles (DE) (c=5), Gaussian process (GP) (c=1) and hyper deep ensembles (HDE) (c=15) and (HDE*) (c=40). Additionally, NOMU's estimated and scaled model uncertainty  $2\sigmodelhat$ is shown as a dotted line.
        }
    \label{fig:irradiance}
\end{figure*}

\subsubsection{Solar Irradiance Time Series: Details}\label{subsubsec:Detailed Real-Data Time Series Results}
\paragraph{Configuration Details}
We use for all algorithms the same configuration as in the 1D and 2D toy regression setting (see \Cref{subsubsec:ToyRegression} and \Cref{subsec:ConfigurationDetailsofBenchmarksRegression}) except that we now train all algorithms for $2^{14}$ epochs and set as L2-regularization $\lambda=10^{-19}$ for NOMU, $\lambda=10^{-19}/\ntr$ for DE, and $\lambda=(1-0.2)\cdot10^{-19}/\ntr$ for MCDO. For HDE, we accordingly draw the L2-regularization parameter log-uniform from $(10^{-19}\cdot 10^{-3},10^{-19}\cdot 10^{+3})$ (which affects the L2-regularization parameter of HDE* in the same way as in \Cref{subsec:ConfigurationDetailsofBenchmarksRegression}).

\paragraph{Results} \Cref{fig:irradiance} visualizes the UBs from NOMU and the benchmark algorithms. As becomes evident, NOMU manages to best fit the training data\footnote{\citet{gal2016dropout} consider the same data set to compare UBs of MC dropout and GPs. In this work however, NNs are trained for $10^6$ epochs possibly explaining why MC dropout more nicely fits the training data in their case.}, while keeping model uncertainty in-between training points. The corresponding metrics for this specific run are given in \Cref{tab:irradiance}.

\begin{table}[t!]
    \caption{Metrics for solar irradiance time series.}
    \label{tab:irradiance}
    \vskip 0.1in
    \begin{center}
    \begin{small}
    \begin{sc}
    \resizebox{0.6\columnwidth}{!}{
    \begin{tabular}{l
                S[table-format=2.3,detect-weight]
    			S[table-format=2.3,detect-weight]}
    	\toprule
        \textbf{Method} & {\textbf{\AUC}$\downarrow$}  & {\textbf{\MNLPD}$\downarrow$} \\
        \midrule
        {NOMU} &   0.32  &  \winc -1.44\\
        {GPR} & 0.46 &  -1.24\\
        {MCDO} & \winc 0.30&  -1.11\\
        {DE} & 0.35 &  -1.28\\
        {HDE} & 0.47 &  -0.98\\
        {HDE*} & 0.59 &  -1.08\\
        \bottomrule
    \end{tabular}
    }
    \end{sc}
    \end{small}
    \end{center}
\vskip -0.1in
\end{table}
\subsubsection{UCI Data Sets: Details}\label{subsec:DetailsUCIExperiment}
To get a feeling for how well NOMU (without data noise extension or hyperparameter tuning) performs in noisy regression with high-dimensional input, we test our algorithm on the popular UCI benchmark \citep{hernandez2015probabilistic} and the UCI gap data sets introduced in \citep{foong2019inbetween}.

The UCI gap data sets were designed to capture whether an algorithm's uncertainty estimate increases in between separated clusters of observations where epistemic uncertainty should be higher compared to regions with many data points. This is also required by \hyperref[itm:Axioms:largeDistantLargeUncertainty]{D3}. However, in \citep{foong2019inbetween} only those data points that had been removed to create gaps in the training data were used as test data for calculating \NLPD. Thus, \NLPD\ on UCI gap data fails by construction to capture the uncertainty quality outside of the gap: First, \hyperref[itm:Axioms:ZeroUncertaintyAtData]{D2} is not properly measured by this evaluation, since there are not many test-points in the test data-set that are close to training points.
Second, also \hyperref[itm:Axioms:largeDistantLargeUncertainty]{D3} is not properly measured, since the smaller gaps in between data points outside of the gap are not part of the test data-set. Note that \hyperref[itm:Axioms:largeDistantLargeUncertainty]{D3} also concerns these kinds of gaps.
To provide better evaluation of uncertainty, future work should focus on mixed test data-sets (moving some input points outside of the gap from the training set to the test set) similar to our train/test split from the experimental setting in \Cref{subsubsec:Detailed Real-Data Time Series Results}.

Nonetheless, we test NOMU on the UCI and UCI gap data sets using the same experiment setup as in the respective works. 

\paragraph{NOMU Setup}
In line with the literature, NOMU's main and side architectures were chosen as a single-hidden-layer NN with $50$ hidden nodes (except for the large protein data set, where the number of hidden nodes was $100$). NOMU was trained for $400$ (UCI) respectively $40$ (UCI gap) epochs, with L2-regularization $1e^{-09}$ and $1e^{-04}$ on the main- and side architectures respectively, using a batch size of $100$ in Adam stochastic gradient descent with learning rate $0.01$. NOMU used $\ell=100$ artificial data points randomly sampled on each batch and was refit on validation data after the constant $c$ had been calibrated on these. For the standard UCI data sets, we used the same loss parameters as in the remaining regression experiments, namely $\muexp=0.01$, $\musqr=0.1$, and $\cexp=30$. For the UCI gap data set, where uncertainty is only evaluated in gap regions, we chose $\muexp=0.1$ and $\musqr=0.01$, relaxing the requirement of small uncertainty at observed data points and strengthening the pull upward on our raw uncertainty output. Moreover, for numerical stability we use the following slight adaptation of the NOMU loss from \Cref{def:Model Uncertainty Loss}: \resizebox{\columnwidth}{!}{\parbox{1\columnwidth}{
\begin{align}
  L_{\mathsmaller{\text{stable}}}^\hp(\NN_\theta):=&\underbrace{\sum_{i=1}^{\ntr}(\fhat(\xtr_i
  )-\ytr_i)^2}_{\terma{}}+ \,\musqr\cdot\underbrace{\sum_{i=1}^{\ntr}\rho_\mathsmaller{\text{stable}}\left(\sigmodelhatraw(\xtr_i)\right)}_{\text{stable version of }\termb{}}+ \notag\\ 
   &\muexp\cdot\underbrace{\frac{1}{\lambda_d(\X)}\int_{\X}u_{\mathsmaller{\text{stable}}} \left(\sigmodelhatraw(x)\right)\,dx}_{\text{stable version of }\termc{}},
\end{align}
}}
where
\begin{equation}
\rho_\mathsmaller{\text{stable}}(r)=\begin{cases}
r^2 & ,|r|\leq 1\\
2|r|-1 & ,|r|>1
\end{cases}
\end{equation}
is the Huber-loss and
\begin{equation}
    u_\mathsmaller{\text{stable}}(r)=
    \begin{cases}
    u(r)=e^{-\cexp\cdot r} & ,r\geq0 \\
    -\cexp\cdot r & ,r<0.
    \end{cases}
\end{equation}
This stable loss $L_{\mathsmaller{\text{stable}}}^\hp$ behaves exactly the same as the standard NOMU-loss $L^\hp$ as long as the outputs of the NN stay in a reasonable range. Only if the gradient descent gets unstable, such that the NN outputs very extreme value, this stabilized loss assures a bounded gradient of the loss~$L_{\mathsmaller{\text{stable}}}^\hp$ with respect to the outputs of the NN~$\NN_\theta$.

\begin{table}[t!]
    \robustify\bfseries
    \setlength\tabcolsep{4pt}
    \caption{\textsc{UCI-Gap} average \NLPD{} and a $95\%$ normal-CI.}
    \label{tab:uci_gap_nll}
    \vskip 0.1in
    \begin{center}
    \begin{small}
    \begin{sc}
    \resizebox{1\columnwidth}{!}{
    \begin{tabular}{lccccc}
    \toprule
    \textbf{Dataset}& NOMU & LL & NLM-HPO & NLM \\
    \midrule
    Boston & \winc2.80 $\pm$\scriptsize 0.11 & \winc2.79 $\pm$\scriptsize 0.09 & \winc2.81 $\pm$\scriptsize 0.15 & 3.60 $\pm$\scriptsize 0.21 \\
    Concrete & \winc3.38 $\pm$\scriptsize 0.08 & \winc3.53 $\pm$\scriptsize 0.13 & 3.72 $\pm$\scriptsize 0.11 & 3.89 $\pm$\scriptsize 0.21 \\
    Energy & \winc3.71 $\pm$\scriptsize 1.62 & \winc6.49 $\pm$\scriptsize 5.50 & \winc3.78 $\pm$\scriptsize 2.27 & \winc3.40 $\pm$\scriptsize 1.99 \\
    Kin8nm & -0.97 $\pm$\scriptsize 0.02 & \winc-1.14 $\pm$\scriptsize 0.03 & -1.06 $\pm$\scriptsize 0.03 & -1.04 $\pm$\scriptsize 0.03 \\
    Naval & 31.81 $\pm$\scriptsize 28.0 & 15.66 $\pm$\scriptsize 8.34 & \winc1.17 $\pm$\scriptsize 1.56 & \winc1.48 $\pm$\scriptsize 1.64 \\
    CCPP & \winc2.91 $\pm$\scriptsize 0.04 & \winc2.89 $\pm$\scriptsize 0.03 & \winc2.96 $\pm$\scriptsize 0.09 & \winc2.90 $\pm$\scriptsize 0.05 \\
    Protein & \winc3.21 $\pm$\scriptsize 0.10 & \winc3.09 $\pm$\scriptsize 0.05 & \winc3.22 $\pm$\scriptsize 0.09 &     3.35 $\pm$\scriptsize 0.11 \\
    Wine & \winc0.99 $\pm$\scriptsize 0.02 & \winc0.96 $\pm$\scriptsize 0.01 & \winc0.98 $\pm$\scriptsize 0.01 & 1.75 $\pm$\scriptsize 0.07 \\
    Yacht & 2.15 $\pm$\scriptsize 0.30 & \winc1.33 $\pm$\scriptsize 0.29 & \winc1.72 $\pm$\scriptsize 0.33 & \winc1.44 $\pm$\scriptsize 0.21 \\
    \bottomrule
    \end{tabular}
    }
    \end{sc}
    \end{small}
    \end{center}
    \vskip -0.1in
\end{table}
\paragraph{Results UCI} \Cref{tab:uci_nll_detailed} in the main paper reports \NLPD\ on test data averaged over 20 runs as in \citep{hernandez2015probabilistic} for the UCI data set. It includes the following benchmark models:
\begin{itemize}[leftmargin=*]
    \item NLM-HPO and NLM correspond to BN(BO)-1 NL in Section D.1 respectively Section D.2. of \cite{ober2019benchmarking},
    \item LL corresponds to LL 1HL TANH in \cite{foong2019inbetween},
    \item MCDO2 represents Dropout (Convergence) in \cite{mukhoti2018importance},
    \item MCDO corresponds to Dropout in \cite{gal2016dropout} and DE to Deep Ensembles in \cite{lakshminarayanan2017simple}.
\end{itemize}
We used the standard deviations that were reported in these respective works to derive $95\%$-normal CIs.

\paragraph{Results UCI Gap}\Cref{tab:uci_gap_nll} lists Gaussian $95\%$-confidence intervals for \NLPD\ on the UCI gap data sets for NOMU as well as the values reported for benchmarks NLM, NLM-HPO and LL (Linearized Laplace) as listed above. These values result from different runs where in each run the gap were introduced in a different input dimension (see \citet{foong2019inbetween}). We see that NOMU performs comparably, and confirm the observation that LL tends to best capture uncertainty in the large gaps artificially introduced in the training data.

\subsection{Bayesian Optimization}
\subsubsection{BO: Configuration Details}\label{subsubsec:Hyperparameters}
In the following, we describe the detailed hyperparameter setup of all algorithms.

\paragraph{NOMU Setup} For NOMU, we set $\musqr=1$, $\sigmin=1\mathrm{e}{-6}$ and use $l=500$ artificial input points for 5D, 10D and 20D. Otherwise we use the exact same hyperparameters as in 1D and 2D regression (see in \Cref{subsec:Regression} the paragraph \textbf{Algorithm setup}).

\paragraph{Deep Ensembles Setup} For deep ensembles, we use the exact same  hyperparameters as in 1D and 2D regression (see \Cref{subsec:ConfigurationDetailsofBenchmarksRegression}).

\paragraph{MC Dropout Setup} For MC dropout, due to computational constraints, we were only able to use $10$ stochastic forward passes (instead of $100$ in the regression setting) However, this still results in an increase in compute time  by a factor 10 of every single acquisition function evaluation compared to NOMU. Otherwise, we use the exact same hyperparameters as in 1D and 2D regression (see \Cref{subsec:ConfigurationDetailsofBenchmarksRegression}) .

\paragraph{Gaussian Process Setup} For Gaussian processes (pGP and GP), we use the exact same setup as in 1D and 2D regression (see \Cref{subsec:ConfigurationDetailsofBenchmarksRegression}).

\paragraph{Hyper Deep Ensembles Setup}
For hyper deep ensembles, due to computational constraints, we were only able to use $\kappa=20$ (instead of $\kappa=50$ in the regression setting). However, this still results in an increase of training time by a factor of 5 in each single BO step compared to NOMU. Otherwise we use the exact same hyperparameters as in the 1D and 2D regression setting (see \Cref{subsec:ConfigurationDetailsofBenchmarksRegression}).

\paragraph{Acquisition Function Optimization}
For each algorithm and each BO step, we determine the next BO input point by maximizing the corresponding upper-bound acquisition function, i.e., $(\fhat(x)+ c~\sigmodelhat(x))$. We maximize this upper-bound acquisition function using the established \href{https://scipydirect.readthedocs.io/en/latest/reference.html}{\textsc{Direct} (Dividing Rectangles)} algorithm.

\paragraph{Gaussian BNN Test Functions}
Note that for the Gaussian BNN test functions in BO (see \Cref{tab:BOresults} and \Cref{tab:BOresultsDetails}, respectively) we do not have access to the ground truth global maximum. Thus, to determine the final regret, we use the established \href{https://scipydirect.readthedocs.io/en/latest/reference.html}{\textsc{Direct} (Dividing Rectangles)} algorithm. With \textsc{Direct}, we calculate a point $x\textsuperscript{direct}$ s.t. $f(x\textsuperscript{direct})\approx\max_{x\in X} f(x)$ and report $|f(x\textsuperscript{direct})-\max_{i\in\fromto{72}} f(x_i)|/f(x\textsuperscript{direct})$ in \Cref{tab:BOresults} and \Cref{tab:BOresultsDetails}, respectively.

\subsubsection{BO: Calibration}\label{subsubsec:Calibration}
\paragraph{\textsc{MW Scaling}}
In general, having a good understanding of the prior scale is very hard. Nevertheless, often it is easier to have at least some intuition in which order of magnitude the posterior UBs should be on average. When function values approximately range in $[-1,1]$, it is reasonable to require that after observing $8$ initial points the mean width (MW) of the UBs lies within the order of magnitudes $0.05$ and $0.5$. Hence our motivation for choosing the calibration parameter $c$ accordingly. 
An advantage of such a calibration is that it can be applied to every method equally, whereas there is in general no clear notion of setting the prior scale of two different methods (e.g., MC dropout and deep ensembles) to the same value.
Note, that we only use MW to calibrate $c$ directly after mean and variance predictions were fit based on the $8$ initial data points. So MW is not fixed when further data points are observed in subsequent steps of the BO.

\sisetup{output-exponent-marker=\ensuremath{\mathrm{e}}}
\begin{table*}
    \setlength\tabcolsep{2pt} %
    \renewcommand{\arraystretch}{1.3} %
    \caption{Results for 15 different BO tasks. Shown are the average final regrets and a 95\% normal-CI per dimension and for each individual function over 100 (5D) and 50 (10D and 20D) runs for an initial mean width (MW) of $0.05$ and $0.5$, respectively. Winners are marked in grey.}
    \label{tab:BOresultsDetails}
    \vskip 0.1in
    \begin{center}
	\begin{small}
	\begin{sc}
    \resizebox{\textwidth}{!}{
\begin{tabular}{lllllllllllll}
    \toprule
               &                        \multicolumn{1}{c}{\textbf{NOMU}} &                         \multicolumn{1}{c}{\textbf{NOMU}} &                          \multicolumn{1}{c}{\textbf{GP}} &                           \multicolumn{1}{c}{\textbf{GP}} &                          \multicolumn{1}{c}{\textbf{MCDO}} &                           \multicolumn{1}{c}{\textbf{MCDO}} &                          \multicolumn{1}{c}{\textbf{DE}} &                           \multicolumn{1}{c}{\textbf{DE}} &                         \multicolumn{1}{c}{\textbf{HDE}} &                          \multicolumn{1}{c}{\textbf{HDE}} &                         \multicolumn{1}{c}{\textbf{pGP}} &          \multicolumn{1}{c}{\textbf{RAND}} \\
              \textbf{Function} &                        \multicolumn{1}{c}{MW 0.05} &                         \multicolumn{1}{c}{MW 0.5} &                          \multicolumn{1}{c}{MW 0.05} &                           \multicolumn{1}{c}{MW 0.5} &                          \multicolumn{1}{c}{MW 0.05} &                           \multicolumn{1}{c}{MW 0.5} &                          \multicolumn{1}{c}{MW 0.05} &                           \multicolumn{1}{c}{MW 0.5} &                         \multicolumn{1}{c}{MW 0.05} &                          \multicolumn{1}{c}{MW 0.5} &                         \multicolumn{1}{c}{MW} &          \multicolumn{1}{c}{MW} \\
    \midrule
                Levy5D &                    $\num{2.12e-03}\pm\scriptsize\num{5.95e-04}$ & $\winc\num{1.52e-03}\pm\scriptsize\num{4.34e-04}$ & $\winc\num{1.10e-03}\pm\scriptsize\num{1.68e-04}$ &                    $\num{8.84e-03}\pm\scriptsize\num{1.15e-03}$ &                    $\num{1.98e-02}\pm\scriptsize\num{3.59e-03}$ &                    $\num{1.27e-02}\pm\scriptsize\num{2.57e-03}$ &                    $\num{7.09e-03}\pm\scriptsize\num{3.76e-03}$ &                    $\num{9.09e-02}\pm\scriptsize\num{9.16e-03}$ &                    $\num{4.48e-03}\pm\scriptsize\num{1.25e-03}$ &                    $\num{3.76e-03}\pm\scriptsize\num{1.22e-03}$ &                    $\num{6.25e-03}\pm\scriptsize\num{5.26e-04}$ & $\num{5.44e-02}\pm\scriptsize\num{4.22e-03}$ \\ 
          Rosenbrock5D & $\winc\num{1.75e-04}\pm\scriptsize\num{2.76e-05}$ & $\num{3.57e-04}\pm\scriptsize\num{6.39e-05}$ & $\num{3.62e-04}\pm\scriptsize\num{6.52e-05}$ & $\num{8.44e-04}\pm\scriptsize\num{1.62e-04}$ & $\num{2.83e-04}\pm\scriptsize\num{4.64e-05}$ & $\num{7.35e-04}\pm\scriptsize\num{1.08e-04}$ & $\winc\num{7.93e-04}\pm\scriptsize\num{7.74e-04}$ &                    $\num{5.96e-03}\pm\scriptsize\num{1.78e-03}$ & $\num{1.11e-03}\pm\scriptsize\num{7.25e-04}$ & $\num{5.02e-04}\pm\scriptsize\num{1.04e-04}$ & $\num{7.56e-04}\pm\scriptsize\num{8.13e-05}$ & $\num{4.73e-03}\pm\scriptsize\num{7.04e-04}$ \\ 
          G-Function5D &                    $\num{1.12e-01}\pm\scriptsize\num{1.99e-02}$ &                    $\num{1.06e-01}\pm\scriptsize\num{1.97e-02}$ &                    $\num{1.68e-01}\pm\scriptsize\num{1.62e-02}$ &                    $\num{2.60e-01}\pm\scriptsize\num{1.72e-02}$ & $\winc\num{2.66e-02}\pm\scriptsize\num{8.57e-03}$ &                    $\num{7.80e-02}\pm\scriptsize\num{8.82e-03}$ &                    $\num{1.88e-01}\pm\scriptsize\num{2.08e-02}$ &                    $\num{2.03e-01}\pm\scriptsize\num{2.20e-02}$ &                    $\num{1.35e-01}\pm\scriptsize\num{2.46e-02}$ &                    $\num{1.37e-01}\pm\scriptsize\num{2.59e-02}$ &                    $\num{1.78e-01}\pm\scriptsize\num{1.22e-02}$ & $\num{3.63e-01}\pm\scriptsize\num{1.09e-02}$ \\
                Perm5D &                    $\num{4.06e-04}\pm\scriptsize\num{1.70e-04}$ &                    $\num{2.58e-04}\pm\scriptsize\num{1.13e-04}$ & $\winc\num{7.76e-05}\pm\scriptsize\num{3.19e-05}$ &                    $\num{1.80e-03}\pm\scriptsize\num{4.44e-04}$ & $\winc\num{7.79e-05}\pm\scriptsize\num{4.19e-05}$ & $\winc\num{7.11e-05}\pm\scriptsize\num{1.45e-05}$ &                    $\num{6.96e-04}\pm\scriptsize\num{2.20e-04}$ &                    $\num{1.73e-03}\pm\scriptsize\num{4.55e-04}$ &                    $\num{2.74e-03}\pm\scriptsize\num{5.20e-04}$ &                    $\num{2.50e-03}\pm\scriptsize\num{5.58e-04}$ &                    $\num{2.06e-04}\pm\scriptsize\num{7.03e-05}$ & $\num{4.62e-04}\pm\scriptsize\num{1.20e-04}$ \\ 
                 BNN5D & $\winc\num{6.60e-02}\pm\scriptsize\num{4.03e-02}$ & $\winc\num{3.77e-02}\pm\scriptsize\num{1.89e-02}$ &   $\num{1.16e-01}\pm\scriptsize\num{8.63e-02}$ & $\winc\num{4.53e-02}\pm\scriptsize\num{3.04e-02}$ &   $\num{2.37e-01}\pm\scriptsize\num{8.04e-02}$ &   $\num{2.04e-01}\pm\scriptsize\num{9.29e-02}$ & $\winc\num{6.32e-02}\pm\scriptsize\num{5.28e-02}$ & $\winc\num{7.92e-02}\pm\scriptsize\num{5.34e-02}$ &   $\num{2.70e-01}\pm\scriptsize\num{8.78e-02}$ &   $\num{2.36e-01}\pm\scriptsize\num{8.27e-02}$ & $\winc\num{2.23e-02}\pm\scriptsize\num{6.42e-03}$ &   $\num{5.42e-01}\pm\scriptsize\num{9.21e-02}$ \\
                 \midrule
     \emph{Average Regret 5D} & $\winc\num{3.62e-02}\pm\scriptsize\num{8.99e-03}$ & $\winc\num{2.91e-02}\pm\scriptsize\num{5.46e-03}$ &   $\num{5.71e-02}\pm\scriptsize\num{1.76e-02}$ &   $\num{6.33e-02}\pm\scriptsize\num{6.99e-03}$ &   $\num{5.67e-02}\pm\scriptsize\num{1.62e-02}$ &   $\num{5.92e-02}\pm\scriptsize\num{1.87e-02}$ &   $\num{5.19e-02}\pm\scriptsize\num{1.14e-02}$ &   $\num{7.62e-02}\pm\scriptsize\num{1.17e-02}$ &   $\num{8.26e-02}\pm\scriptsize\num{1.82e-02}$ &   $\num{7.59e-02}\pm\scriptsize\num{1.73e-02}$ &   $\num{4.16e-02}\pm\scriptsize\num{2.76e-03}$ &   $\num{1.93e-01}\pm\scriptsize\num{1.86e-02}$ \\
     \midrule
     \midrule
               Levy10D & $\winc\num{6.27e-03}\pm\scriptsize\num{2.22e-03}$ & $\winc\num{6.27e-03}\pm\scriptsize\num{2.22e-03}$ & $\num{1.04e-02}\pm\scriptsize\num{2.10e-03}$ & $\num{1.04e-02}\pm\scriptsize\num{2.10e-03}$ &                    $\num{2.16e-02}\pm\scriptsize\num{3.71e-03}$ &                    $\num{2.17e-02}\pm\scriptsize\num{5.30e-03}$ &                    $\num{8.65e-02}\pm\scriptsize\num{2.01e-02}$ &                    $\num{1.46e-01}\pm\scriptsize\num{1.41e-02}$ & $\winc\num{6.21e-03}\pm\scriptsize\num{2.48e-03}$ & $\winc\num{5.81e-03}\pm\scriptsize\num{1.99e-03}$ & $\winc\num{6.16e-03}\pm\scriptsize\num{8.94e-04}$ & $\num{1.06e-01}\pm\scriptsize\num{7.64e-03}$ \\
         Rosenbrock10D & $\winc\num{2.34e-03}\pm\scriptsize\num{1.43e-03}$ &                    $\num{2.01e-03}\pm\scriptsize\num{6.91e-04}$ & $\winc\num{9.14e-04}\pm\scriptsize\num{1.25e-04}$ &                    $\num{5.67e-03}\pm\scriptsize\num{1.60e-03}$ &                    $\num{2.25e-03}\pm\scriptsize\num{2.58e-04}$ &                    $\num{4.96e-03}\pm\scriptsize\num{4.40e-04}$ &                    $\num{1.42e-02}\pm\scriptsize\num{4.43e-03}$ &                    $\num{7.65e-02}\pm\scriptsize\num{1.03e-02}$ &                    $\num{5.10e-03}\pm\scriptsize\num{1.44e-03}$ &                    $\num{5.00e-03}\pm\scriptsize\num{1.70e-03}$ &                    $\num{3.09e-03}\pm\scriptsize\num{5.77e-04}$ & $\num{2.82e-02}\pm\scriptsize\num{3.92e-03}$ \\
         G-Function10D &                    $\num{2.92e-01}\pm\scriptsize\num{3.85e-02}$ &                    $\num{3.53e-01}\pm\scriptsize\num{2.74e-02}$ &                    $\num{3.79e-01}\pm\scriptsize\num{2.24e-02}$ &                    $\num{4.21e-01}\pm\scriptsize\num{2.35e-02}$ & $\winc\num{1.79e-01}\pm\scriptsize\num{2.54e-02}$ &                    $\num{4.15e-01}\pm\scriptsize\num{1.12e-02}$ &                    $\num{3.24e-01}\pm\scriptsize\num{2.28e-02}$ &                    $\num{3.33e-01}\pm\scriptsize\num{2.06e-02}$ &                    $\num{2.53e-01}\pm\scriptsize\num{3.76e-02}$ &                    $\num{2.55e-01}\pm\scriptsize\num{3.72e-02}$ &                    $\num{3.80e-01}\pm\scriptsize\num{1.34e-02}$ & $\num{4.50e-01}\pm\scriptsize\num{6.48e-03}$ \\
               Perm10D & $\num{3.74e-04}\pm\scriptsize\num{1.26e-04}$ & $\num{4.27e-04}\pm\scriptsize\num{1.34e-04}$ & $\winc\num{3.59e-04}\pm\scriptsize\num{2.13e-04}$ & $\num{6.52e-04}\pm\scriptsize\num{2.04e-04}$ & $\num{3.37e-04}\pm\scriptsize\num{8.49e-05}$ & $\num{5.53e-04}\pm\scriptsize\num{1.33e-04}$ &                    $\num{1.01e-03}\pm\scriptsize\num{3.10e-04}$ &                    $\num{1.48e-03}\pm\scriptsize\num{3.92e-04}$ & $\num{4.07e-04}\pm\scriptsize\num{1.21e-04}$ & $\num{4.08e-04}\pm\scriptsize\num{1.48e-04}$ & $\num{5.57e-04}\pm\scriptsize\num{1.43e-04}$ & $\winc\num{1.81e-04}\pm\scriptsize\num{5.44e-05}$ \\
                BNN10D & $\winc\num{1.23e-01}\pm\scriptsize\num{3.76e-02}$ & $\winc\num{1.00e-01}\pm\scriptsize\num{3.43e-02}$ &   $\num{1.96e-01}\pm\scriptsize\num{5.16e-02}$ &   $\num{1.68e-01}\pm\scriptsize\num{4.09e-02}$ & $\winc\num{1.45e-01}\pm\scriptsize\num{3.60e-02}$ &   $\num{1.61e-01}\pm\scriptsize\num{2.86e-02}$ & $\winc\num{1.51e-01}\pm\scriptsize\num{4.27e-02}$ & $\winc\num{1.13e-01}\pm\scriptsize\num{4.00e-02}$ &   $\num{2.58e-01}\pm\scriptsize\num{4.96e-02}$ &   $\num{2.02e-01}\pm\scriptsize\num{3.79e-02}$ & $\winc\num{8.65e-02}\pm\scriptsize\num{2.29e-02}$ &   $\num{5.92e-01}\pm\scriptsize\num{5.34e-02}$ \\
                \midrule
    \emph{Average Regret 10D} & $\winc\num{8.48e-02}\pm\scriptsize\num{1.08e-02}$ &   $\num{9.24e-02}\pm\scriptsize\num{8.79e-03}$ &   $\num{1.17e-01}\pm\scriptsize\num{1.13e-02}$ &   $\num{1.21e-01}\pm\scriptsize\num{9.45e-03}$ & $\winc\num{6.97e-02}\pm\scriptsize\num{8.84e-03}$ &   $\num{1.21e-01}\pm\scriptsize\num{6.24e-03}$ &   $\num{1.15e-01}\pm\scriptsize\num{1.05e-02}$ &   $\num{1.34e-01}\pm\scriptsize\num{9.65e-03}$ &   $\num{1.05e-01}\pm\scriptsize\num{1.25e-02}$ &   $\num{9.36e-02}\pm\scriptsize\num{1.06e-02}$ &   $\num{9.52e-02}\pm\scriptsize\num{5.30e-03}$ &   $\num{2.35e-01}\pm\scriptsize\num{1.09e-02}$ \\
    \midrule
    \midrule
               Levy20D & $\winc\num{1.51e-02}\pm\scriptsize\num{1.69e-03}$ & $\winc\num{1.40e-02}\pm\scriptsize\num{1.63e-03}$ & $\winc\num{1.98e-02}\pm\scriptsize\num{8.61e-03}$ & $\winc\num{2.64e-02}\pm\scriptsize\num{1.12e-02}$ &                    $\num{4.27e-02}\pm\scriptsize\num{4.16e-03}$ &                    $\num{6.91e-02}\pm\scriptsize\num{9.00e-03}$ &                    $\num{1.88e-01}\pm\scriptsize\num{1.15e-02}$ &                    $\num{2.01e-01}\pm\scriptsize\num{1.07e-02}$ & $\winc\num{1.13e-02}\pm\scriptsize\num{4.60e-03}$ & $\winc\num{1.61e-02}\pm\scriptsize\num{5.47e-03}$ & $\winc\num{1.98e-02}\pm\scriptsize\num{7.88e-03}$ & $\num{1.48e-01}\pm\scriptsize\num{6.58e-03}$ \\
         Rosenbrock20D &                    $\num{3.47e-02}\pm\scriptsize\num{7.08e-03}$ &                    $\num{6.03e-03}\pm\scriptsize\num{9.93e-04}$ &                    $\num{8.94e-03}\pm\scriptsize\num{4.00e-03}$ &                    $\num{1.91e-02}\pm\scriptsize\num{4.04e-03}$ &                    $\num{7.41e-03}\pm\scriptsize\num{1.14e-03}$ &                    $\num{7.61e-03}\pm\scriptsize\num{2.15e-03}$ &                    $\num{6.46e-02}\pm\scriptsize\num{1.22e-02}$ &                    $\num{1.41e-01}\pm\scriptsize\num{1.29e-02}$ & $\winc\num{2.96e-03}\pm\scriptsize\num{1.36e-04}$ & $\winc\num{4.27e-03}\pm\scriptsize\num{1.33e-03}$ &                    $\num{1.09e-02}\pm\scriptsize\num{1.52e-03}$ & $\num{7.80e-02}\pm\scriptsize\num{8.38e-03}$ \\
         G-Function20D &                    $\num{4.16e-01}\pm\scriptsize\num{8.99e-03}$ &                    $\num{4.18e-01}\pm\scriptsize\num{1.70e-02}$ &                    $\num{4.59e-01}\pm\scriptsize\num{6.77e-03}$ &                    $\num{4.86e-01}\pm\scriptsize\num{3.95e-03}$ &                    $\num{4.70e-01}\pm\scriptsize\num{4.43e-03}$ &                    $\num{4.92e-01}\pm\scriptsize\num{1.58e-03}$ &                    $\num{4.18e-01}\pm\scriptsize\num{1.33e-02}$ & $\winc\num{4.02e-01}\pm\scriptsize\num{1.61e-02}$ & $\winc\num{3.79e-01}\pm\scriptsize\num{2.97e-02}$ & $\winc\num{3.73e-01}\pm\scriptsize\num{2.56e-02}$ &                    $\num{4.47e-01}\pm\scriptsize\num{9.37e-03}$ & $\num{4.86e-01}\pm\scriptsize\num{2.44e-03}$ \\
               Perm20D &                    $\num{7.45e-05}\pm\scriptsize\num{2.09e-05}$ &                    $\num{7.46e-05}\pm\scriptsize\num{2.50e-05}$ &                    $\num{2.20e-04}\pm\scriptsize\num{8.98e-05}$ &                    $\num{1.57e-04}\pm\scriptsize\num{7.50e-05}$ &                    $\num{9.98e-05}\pm\scriptsize\num{3.19e-05}$ &                    $\num{2.03e-04}\pm\scriptsize\num{8.86e-05}$ & $\num{3.54e-05}\pm\scriptsize\num{6.52e-06}$ &                    $\num{1.85e-04}\pm\scriptsize\num{6.90e-05}$ &                    $\num{7.98e-05}\pm\scriptsize\num{2.19e-05}$ &                    $\num{9.28e-05}\pm\scriptsize\num{3.89e-05}$ &                    $\num{9.44e-05}\pm\scriptsize\num{3.25e-05}$ & $\winc\num{1.90e-05}\pm\scriptsize\num{4.67e-06}$ \\
                Bnn20D & $\winc\num{1.58e-01}\pm\scriptsize\num{4.12e-02}$ & $\winc\num{1.23e-01}\pm\scriptsize\num{3.21e-02}$ &   $\num{1.84e-01}\pm\scriptsize\num{3.46e-02}$ & $\winc\num{1.42e-01}\pm\scriptsize\num{2.70e-02}$ &   $\num{1.77e-01}\pm\scriptsize\num{3.46e-02}$ &   $\num{2.45e-01}\pm\scriptsize\num{4.41e-02}$ &   $\num{1.89e-01}\pm\scriptsize\num{3.98e-02}$ &   $\num{3.05e-01}\pm\scriptsize\num{7.02e-02}$ &   $\num{2.93e-01}\pm\scriptsize\num{5.72e-02}$ &   $\num{3.05e-01}\pm\scriptsize\num{5.01e-02}$ & $\winc\num{1.11e-01}\pm\scriptsize\num{2.06e-02}$ &   $\num{6.85e-01}\pm\scriptsize\num{3.29e-02}$ \\
                \midrule
    \emph{Average Regret 20D} & $\winc\num{1.25e-01}\pm\scriptsize\num{8.57e-03}$ & $\winc\num{1.12e-01}\pm\scriptsize\num{7.27e-03}$ &   $\num{1.34e-01}\pm\scriptsize\num{7.31e-03}$ &   $\num{1.35e-01}\pm\scriptsize\num{5.95e-03}$ &   $\num{1.39e-01}\pm\scriptsize\num{7.03e-03}$ &   $\num{1.63e-01}\pm\scriptsize\num{9.02e-03}$ &   $\num{1.72e-01}\pm\scriptsize\num{9.03e-03}$ &   $\num{2.10e-01}\pm\scriptsize\num{1.48e-02}$ &   $\num{1.37e-01}\pm\scriptsize\num{1.29e-02}$ &   $\num{1.40e-01}\pm\scriptsize\num{1.13e-02}$ & $\winc\num{1.18e-01}\pm\scriptsize\num{4.81e-03}$ &   $\num{2.80e-01}\pm\scriptsize\num{6.94e-03}$ \\
    \midrule
    \bottomrule
    \end{tabular}     %
    }
    \end{sc}
    \end{small}
    \end{center}
    \vskip -0.1in
\end{table*}

\paragraph{\textsc{Dynamic c}}\label{par:dynamicC} The initial choice of $c$ can still be corrected in each BO step. Certainly, in the noiseless case it does not make sense to pick an input point $x_{i'}$ that is identical to an already observed input point $x_i, i<i'$, where nothing new is to be learned.
Therefore, we want our NN-agent to get \enquote{bored} if its acquisition function optimization would suggest to pick an input point $x_{i'}\approx x_i, i<i'$.
The idea of \textsc{Dynamic c} is to encourage the agent, in case it was \enquote{bored}, to become more \enquote{curious} to explore something new instead of picking a \enquote{boring} input point. This can be achieved by iteratively increasing $c$ until the acquisition function optimization suggests an input point $x_{i'}\not\approx x_i, \forall i<i'$. We then only apply the expensive function evaluation for $x_{i'}\not\approx x_i, \forall i<i'$ that we obtained after increasing $c$ enough to not \enquote{bore} the NN.
However, towards the last BO-steps, if we already know approximately where a global optimum is and we only want to fix the last digits of the optimizer, we have to evaluate the function closer to already observed points. 
In contrast, a very \enquote{young} NN (i.e., a network that has been trained on few training points) should get \enquote{bored} much more easily, since it is not sensible to exploit a given local optimum up to the last digits, when there is plenty of time to reach other, possibly better local optima. 

Thus, it makes sense to first explore on a large scale where the good local optima are approximately, then to find out which of them are the best ones and finally to exploit the best one in greater detail at the very end.

Therefore, we want the threshold $\delta_i$, that determines if $x_{i'}\overset{\delta_{i'}}{\approx} x_i \iff \twonorm[x_{i'}-x_i]\leq\delta_{i'}$ to decrease in each BO step. In our experiment, we choose an exponential decay of
\begin{align}
\label{eq:dc:exp}
\delta_{i} = \delta_{\nStart} \cdot \left(\dfrac{\delta_{\nEnd}}{\delta_{\nStart}}\right)^{(i-\nStart)/(\nEnd-\nStart)},
\end{align}
with $\delta_{\nStart}=\frac{1}{16}$ and $\delta_{\nEnd}=0.01$.

Concretely, we only evaluate $f$ at $x_{i'}$ if $\twonorm[x_{i'}-x_i]>\delta_{i'} \forall i<i'$ is fulfilled. Otherwise we double $c$ until it is fulfilled (With larger $c$ more emphasis is put on exploration, so there is a tendency that $x_{i'}$ will be further away from the observed input points the larger we pick $c$, if \ref{itm:Axioms:largeDistantLargeUncertainty} is fulfilled). After doubling $c$ 15 times without success, we evaluate $f$ at $x_{i'}$ no matter how close it is to the already observed input points (for methods that have severe troubles to fulfill \ref{itm:Axioms:largeDistantLargeUncertainty}, such as MCDO, even doubling $c$ infinite times would not help if the maximal uncertainty is within an $\delta_{i'}$-ball around an already observed input point).

\subsubsection{BO: Detailed Results}\label{subsubsec:Detailed Results}
In \Cref{tab:BOresultsDetails}, we present the mean final regrets, which correspond to the ranks shown in \Cref{tab:BOresults} in the main paper.

\subsubsection{BO: Regret Plots}\label{subsubsec:Regret Plots}
In \Cref{fig:5d_regret_plots,fig:10d_regret_plots,fig:20d_regret_plots}, we present the regret plots for each test function and both MW values.
\begin{figure*}[p]
\makebox[\textwidth][c]{%
\subfloat[Regret plot G-Function, 0.05 MW\label{subfig:gf5D_0.05}]{%
		\includegraphics[trim={30 0 60 30},clip,width =.83\columnwidth]{
		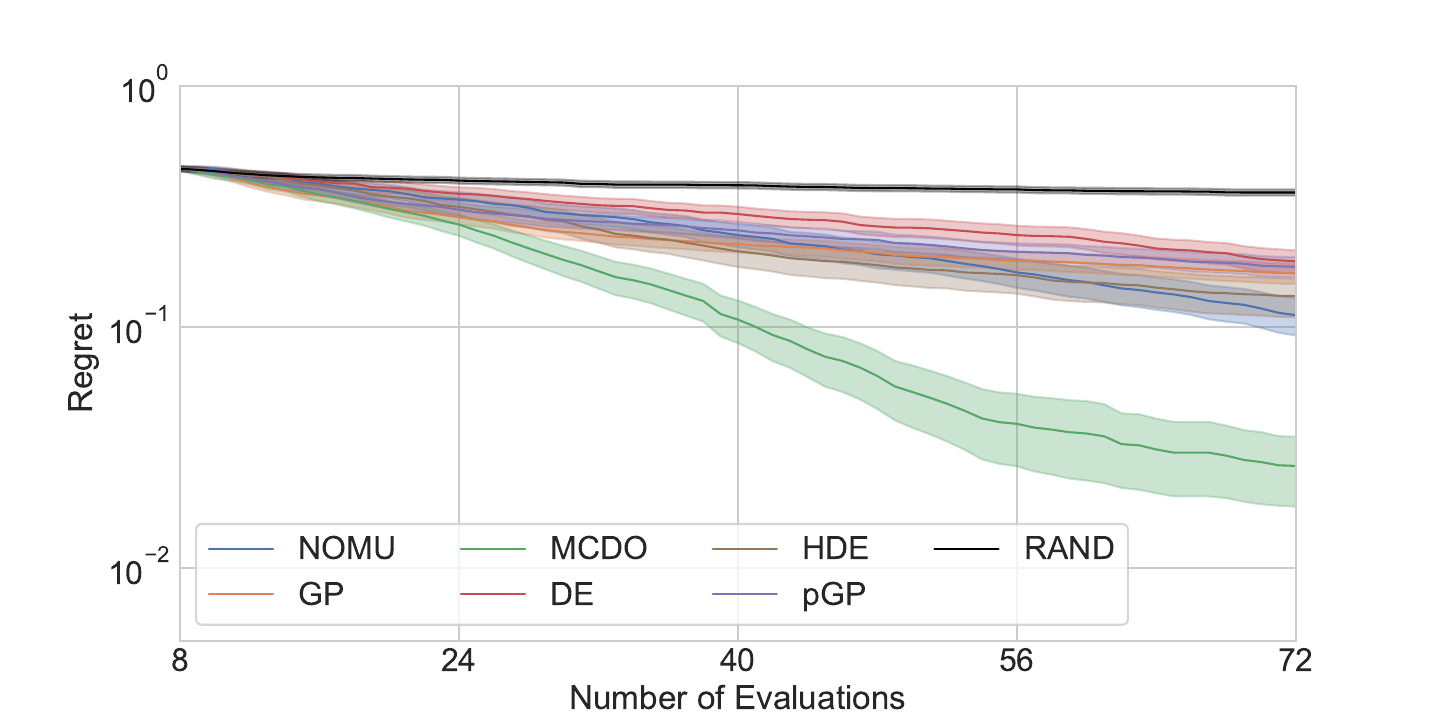}
		}
\hfill
\subfloat[Regret plot G-Function, 0.5 MW\label{subfig:gf5D_0.5}]{%
		\includegraphics[trim={30 0 60 30},clip,width =.83\columnwidth]{
		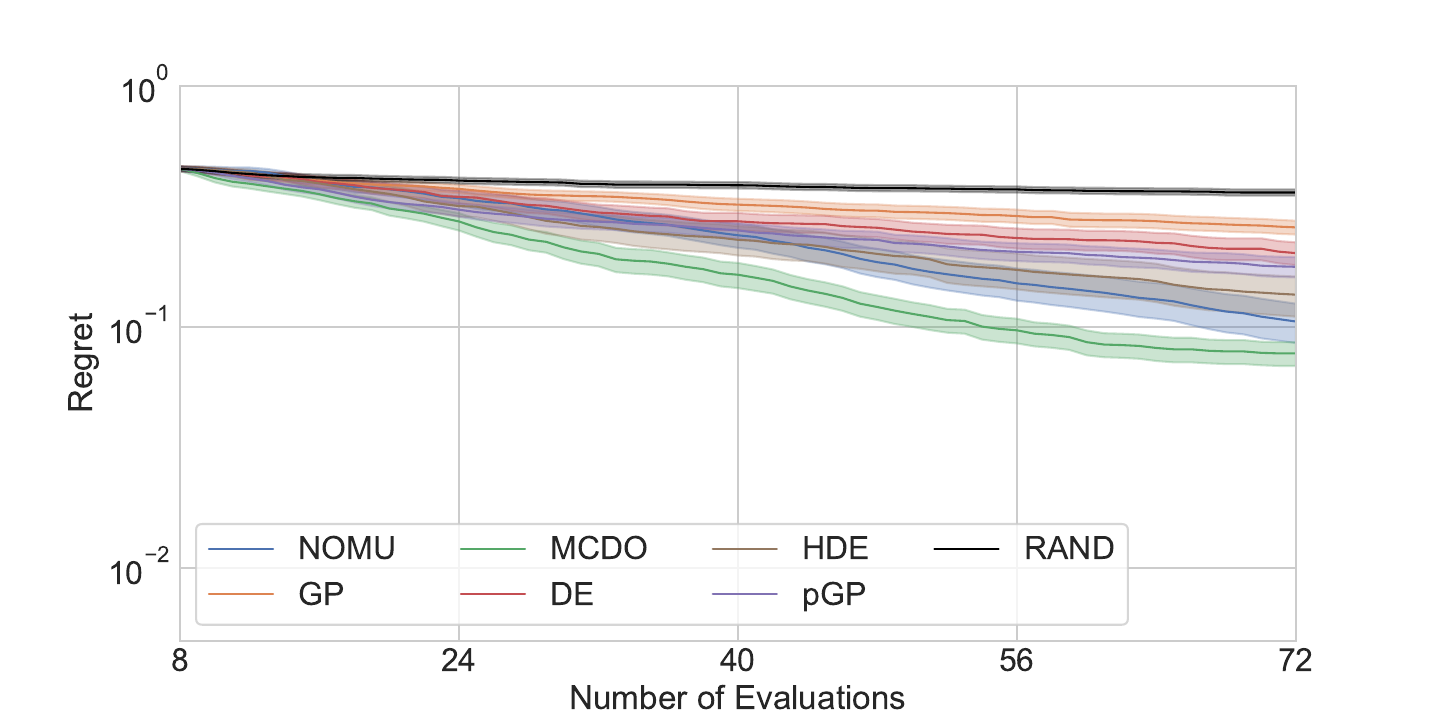}
	}
}
\makebox[\textwidth][c]{%
 \subfloat[Regret plot Levy, 0.05 MW\label{subfig:lev5D_0.05}]{%
		\includegraphics[trim={30 0 60 40},clip,width =.83\columnwidth]{
		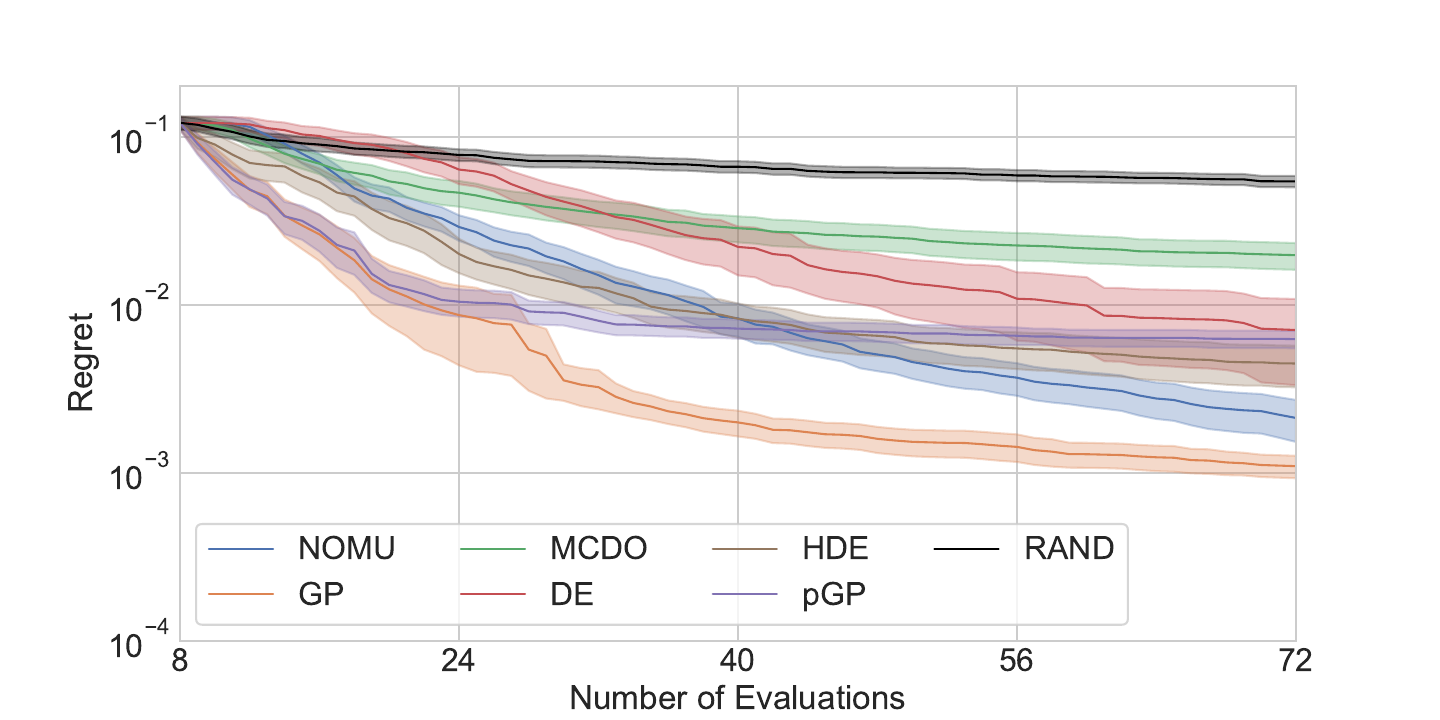}
}		
\hfill
\subfloat[Regret plot Levy, 0.5 MW\label{subfig:lev5D_0.5}]{%
		\includegraphics[trim={30 0 60 40},clip,width =.83\columnwidth]{
		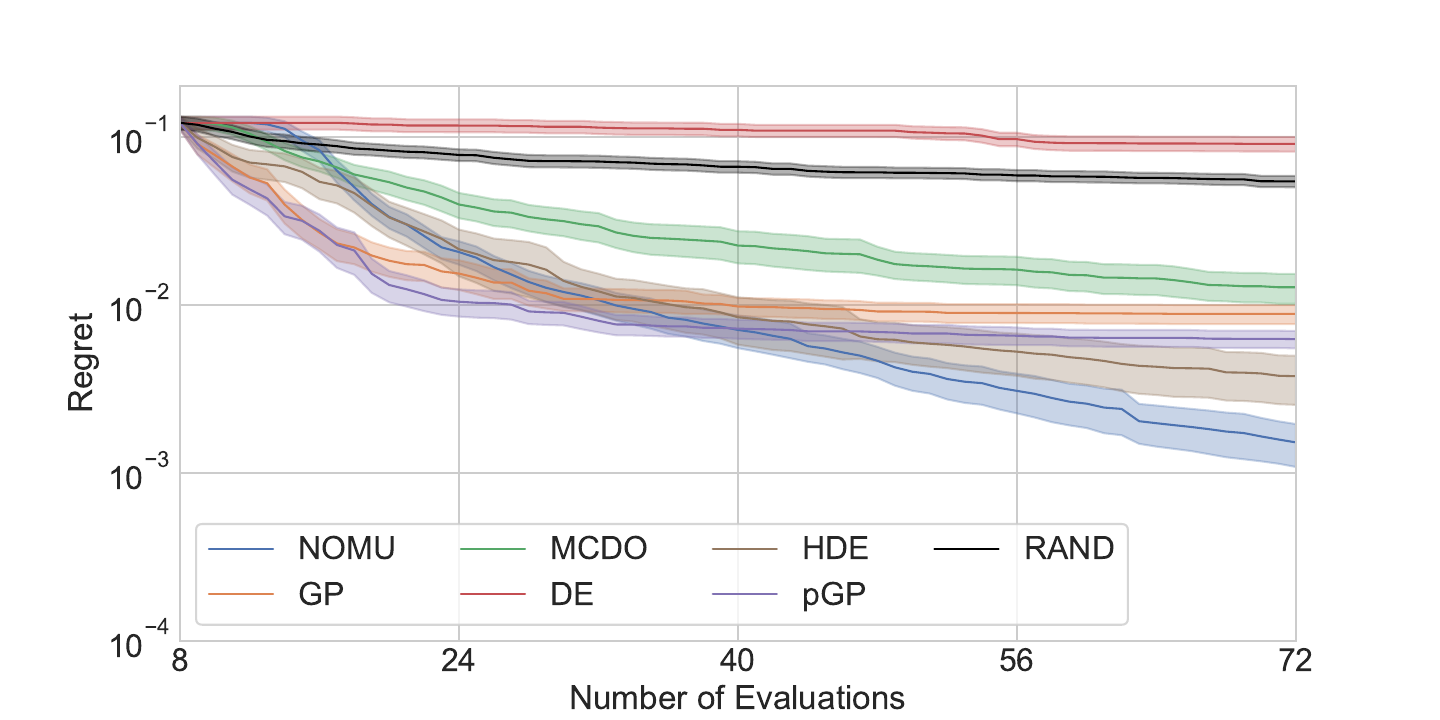}	
}
}
\makebox[\textwidth][c]{%
\subfloat[Regret plot Perm, 0.05 MW\label{subfig:per5D_0.05}]{%
		\includegraphics[trim={30 0 60 40},clip,width =.83\columnwidth]{
		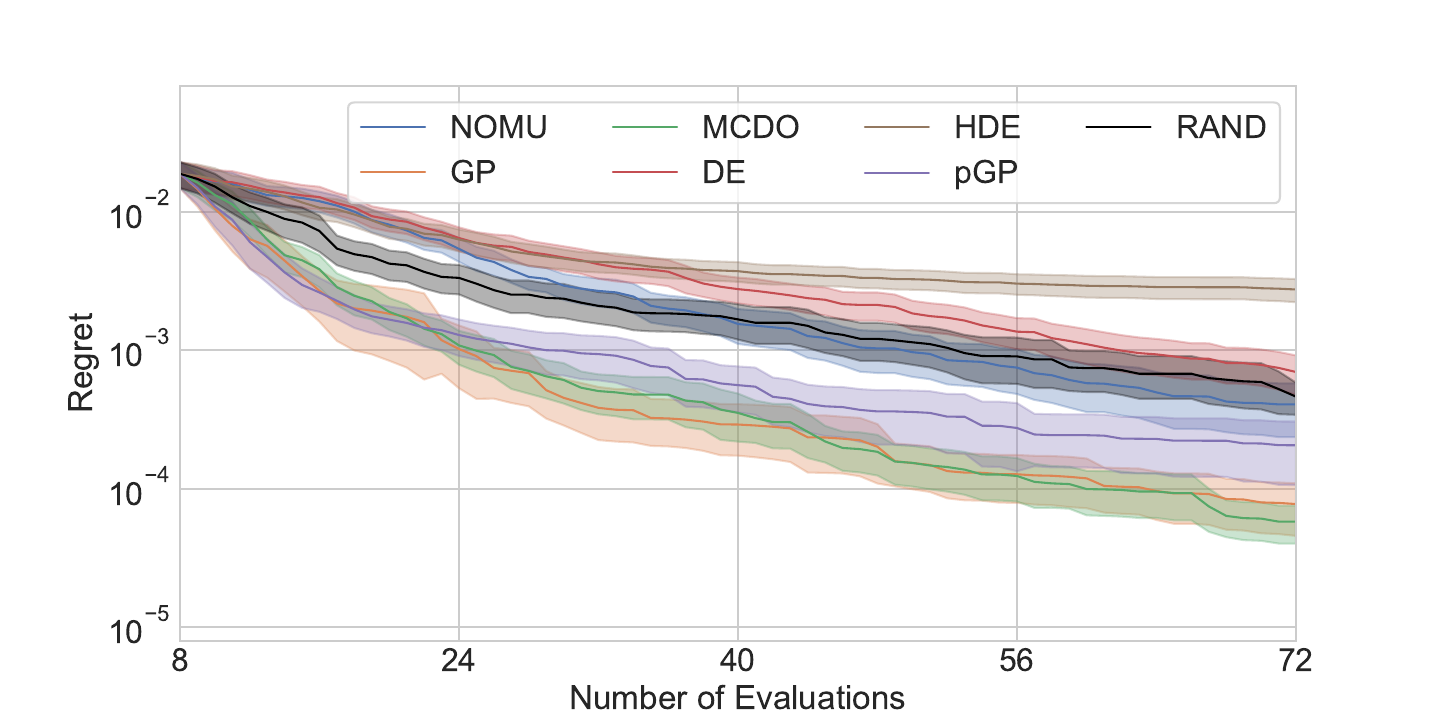}
	}
\hfill
\subfloat[Regret plot Perm, 0.5 MW\label{subfig:per5D_0.5}]{%
		\includegraphics[trim={30 0 60 40},clip,width =.83\columnwidth]{
		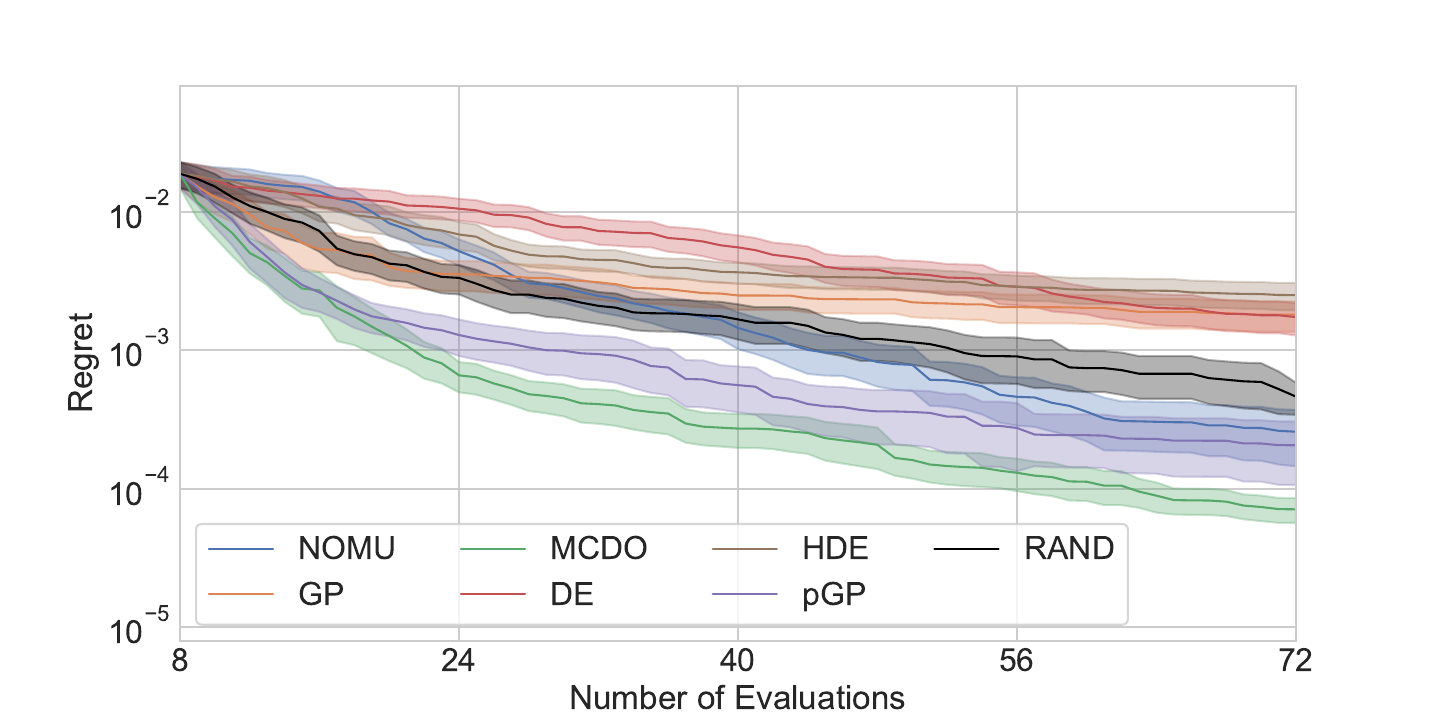}
	}
}
\makebox[\textwidth][c]{%
\subfloat[Regret plot Rosenbrock, 0.05 MW\label{subfig:ros5D_0.05}]{%
		\includegraphics[trim={30 0 60 30},clip,width =.83\columnwidth]{
		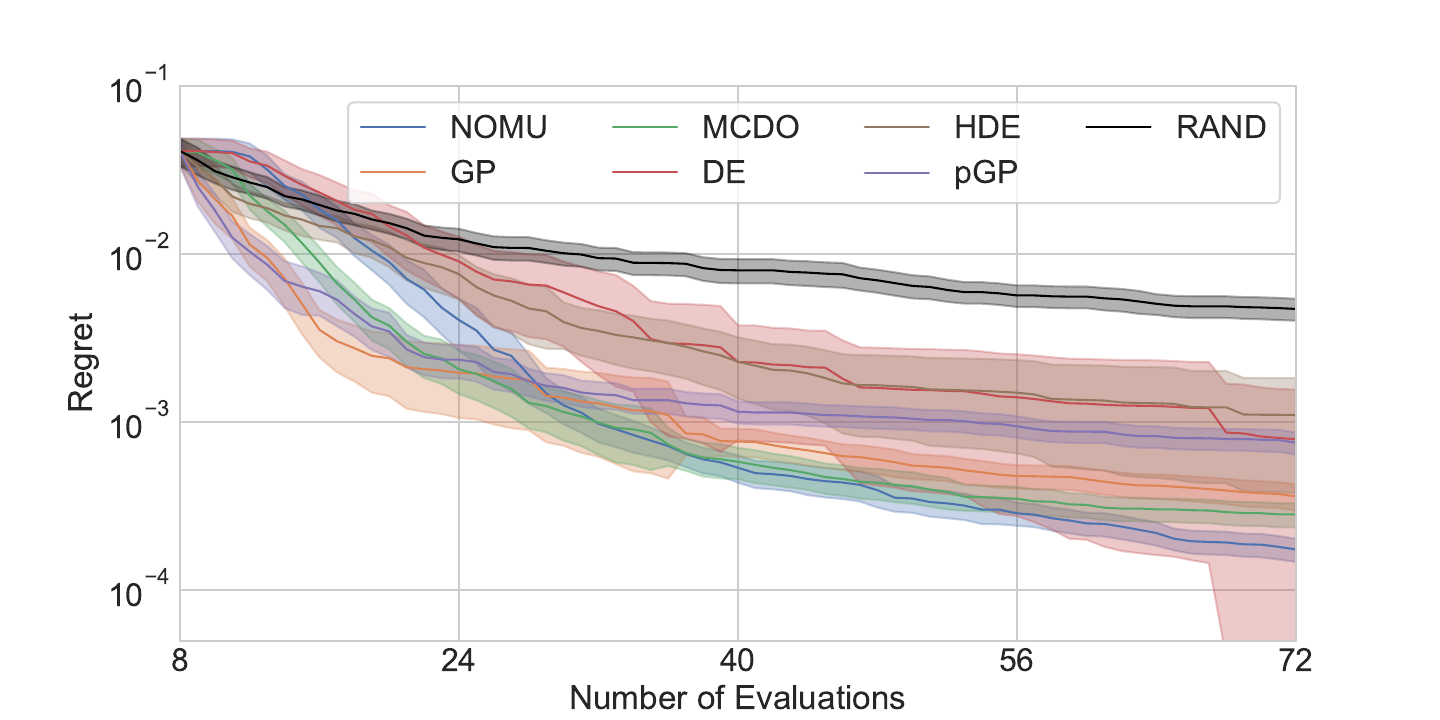}
	}
\hfill
\subfloat[Regret plot Rosenbrock, 0.5 MW\label{subfig:ros5D_0.5}]{%
		\includegraphics[trim={30 0 60 30},clip,width =.83\columnwidth]{
		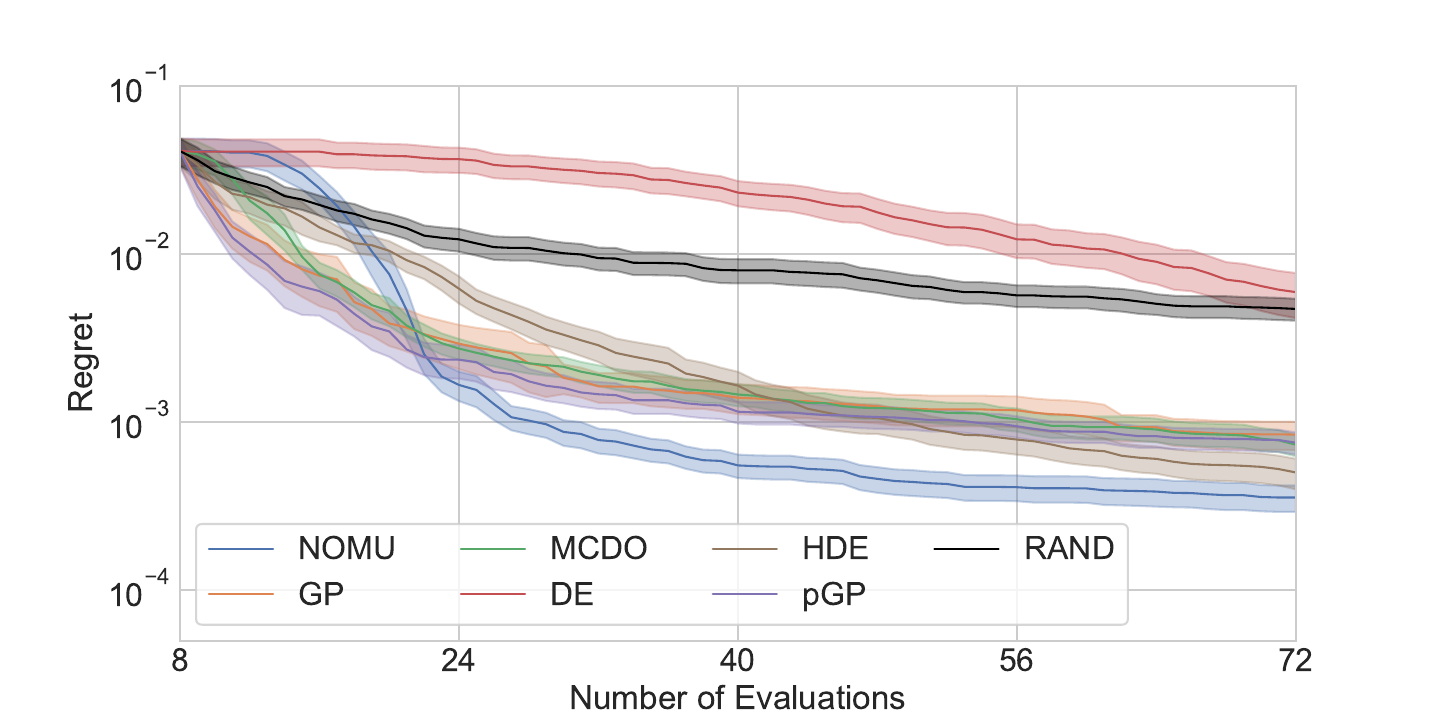}
	}
}
\makebox[\textwidth][c]{%
\subfloat[Regret plot BNN, 0.05 MW\label{subfig:bnn5D_0.05}]{%
		\includegraphics[trim={30 0 60 40},clip,width =.83\columnwidth]{
		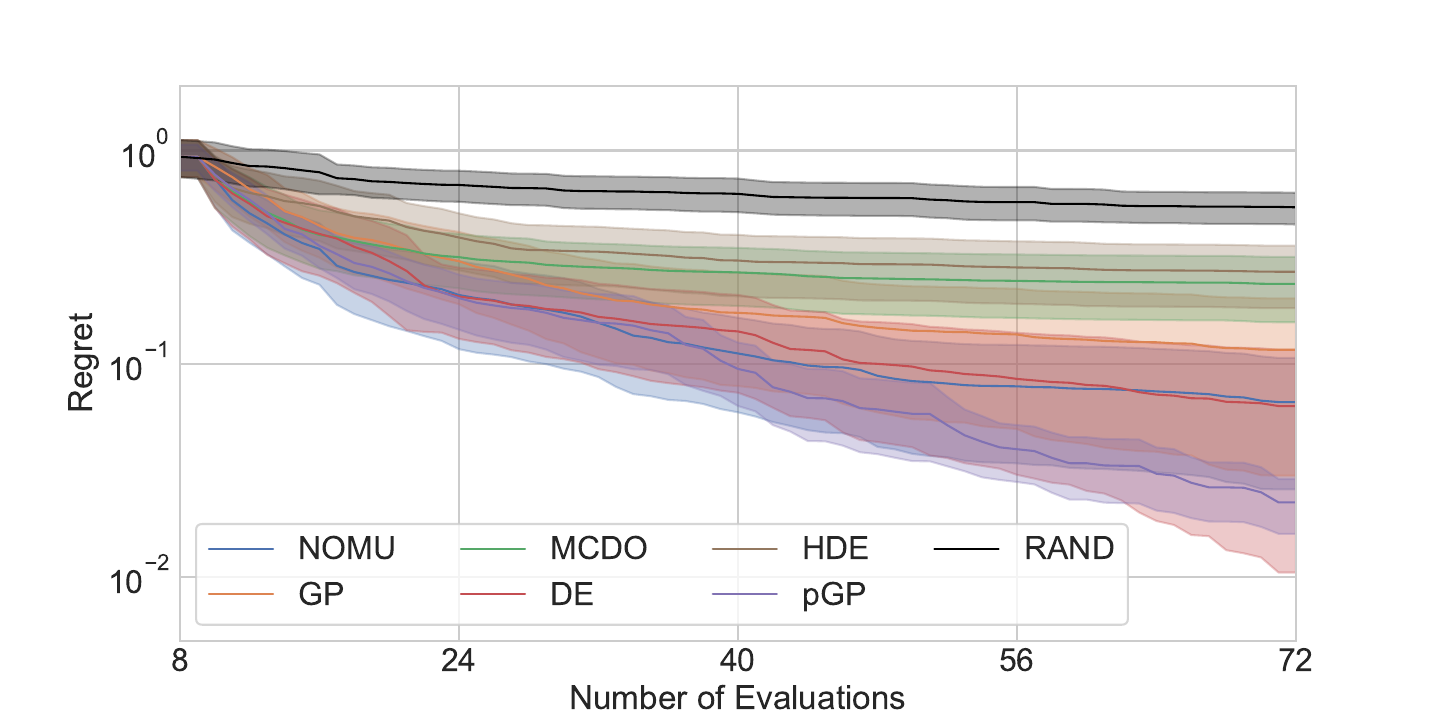}
	}
\hfill
\subfloat[Regret plot BNN, 0.5 MW\label{subfig:bnn5D_0.5}]{%
		\includegraphics[trim={30 0 60 40},clip,width =.83\columnwidth]{
		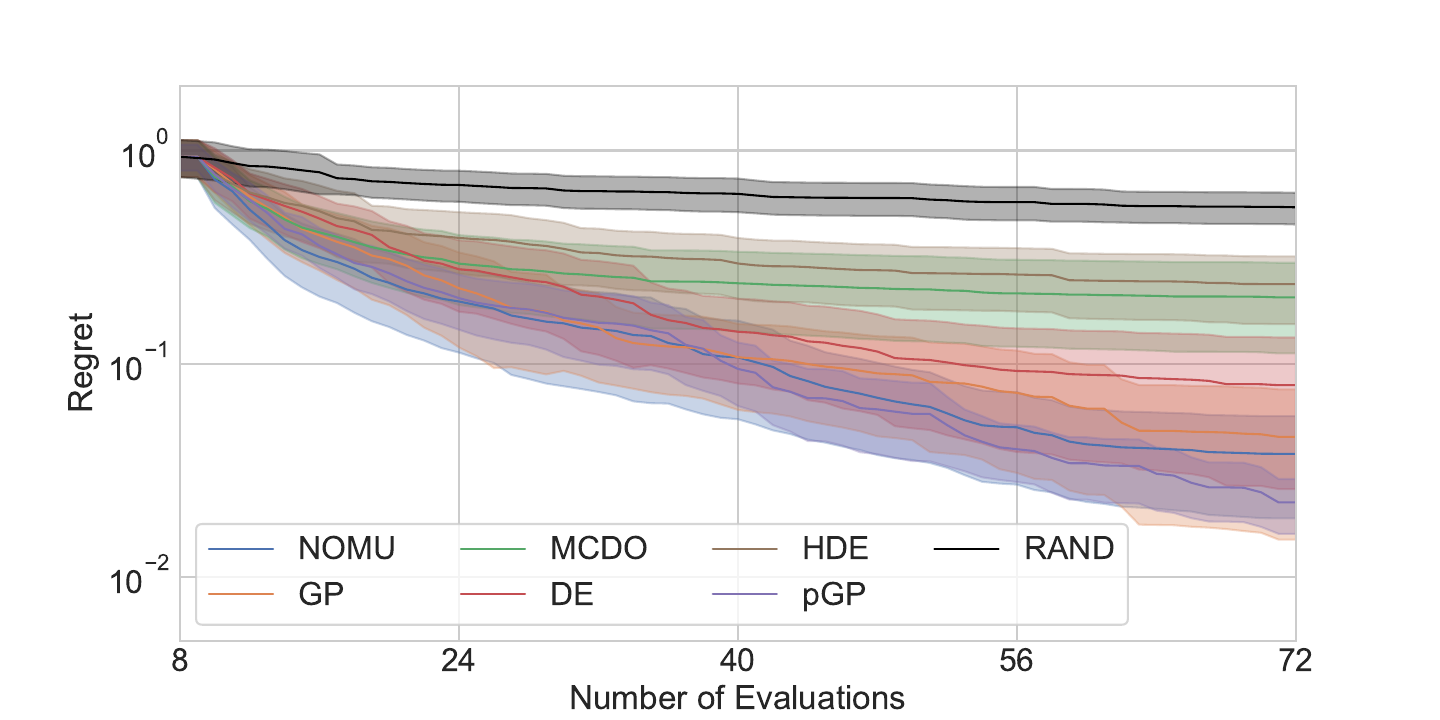}
	}
}
\vskip -0.2cm
	\caption{Regret plots for all 5D test functions and MWs of 0.05 and 0.5, respectively. We show regrets averaged over 100 runs (solid lines) with 95\% CIs.}
	\label{fig:5d_regret_plots}
\end{figure*}

\begin{figure*}[p]
\makebox[\textwidth][c]{%
\subfloat[Regret plot G-Function, 0.05 MW\label{subfig:gf10D_0.05}]{%
		\includegraphics[trim={30 0 60 40},clip,width =.83\columnwidth]{
		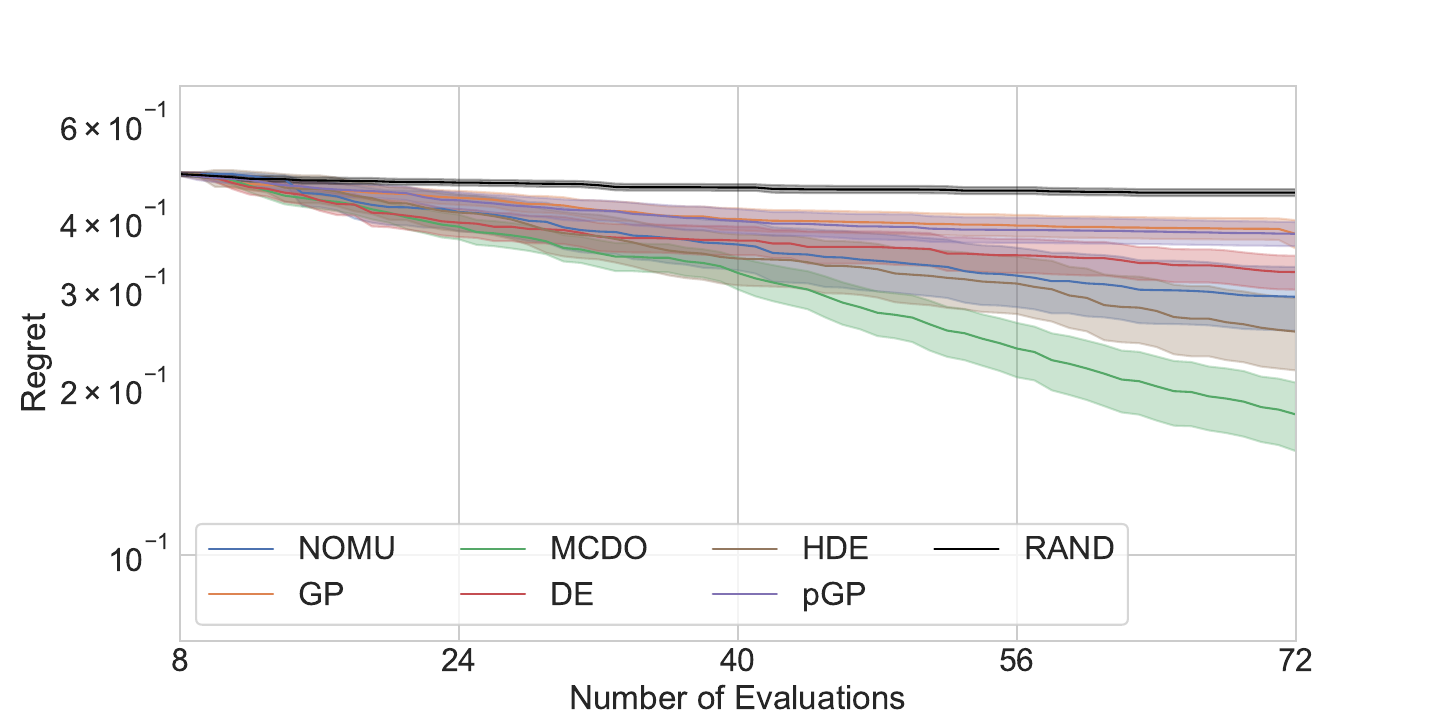}
		}
\hfill
\subfloat[Regret plot G-Function, 0.5 MW\label{subfig:gf10D_0.5}]{%
		\includegraphics[trim={30 0 60 40},clip,width =.83\columnwidth]{
		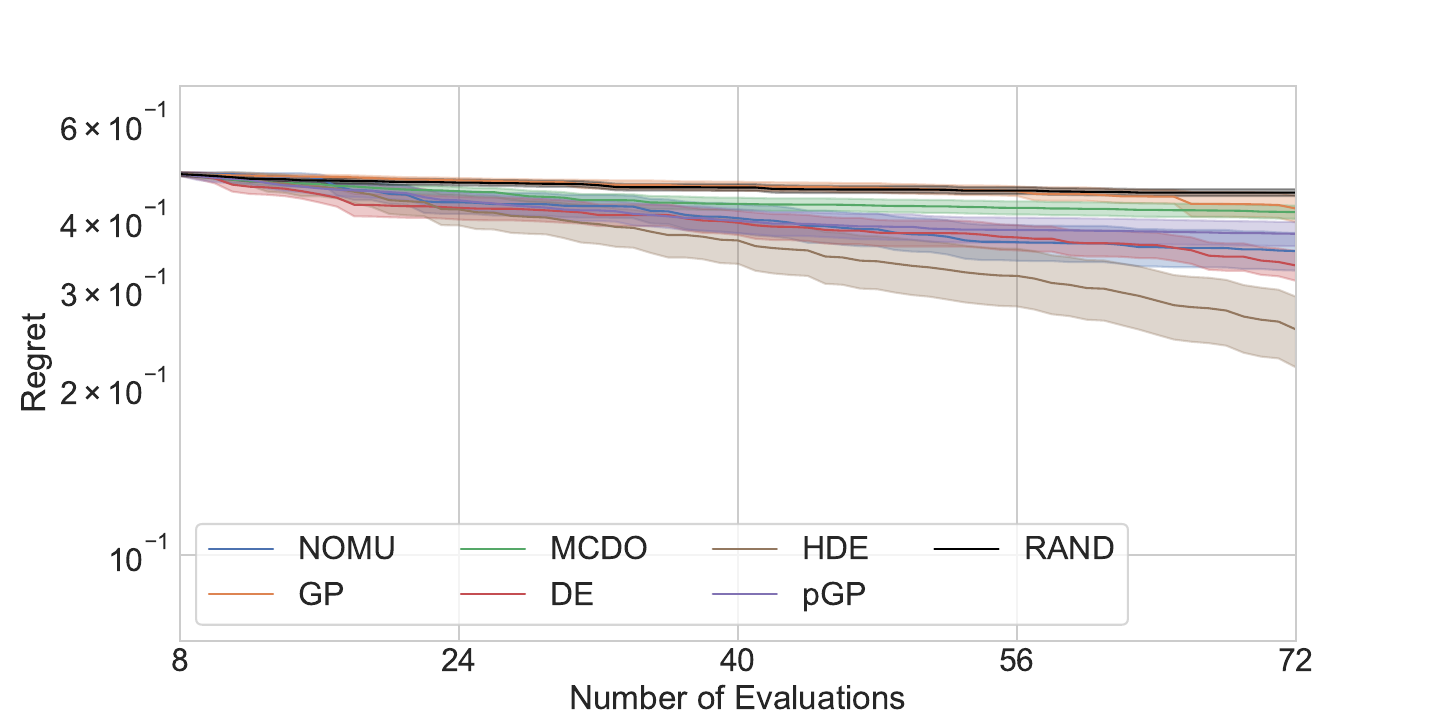}
	}
}
\makebox[\textwidth][c]{%
 \subfloat[Regret plot Levy, 0.05 MW\label{subfig:lev10D_0.05}]{%
		\includegraphics[trim={30 0 60 40},clip,width =.83\columnwidth]{
		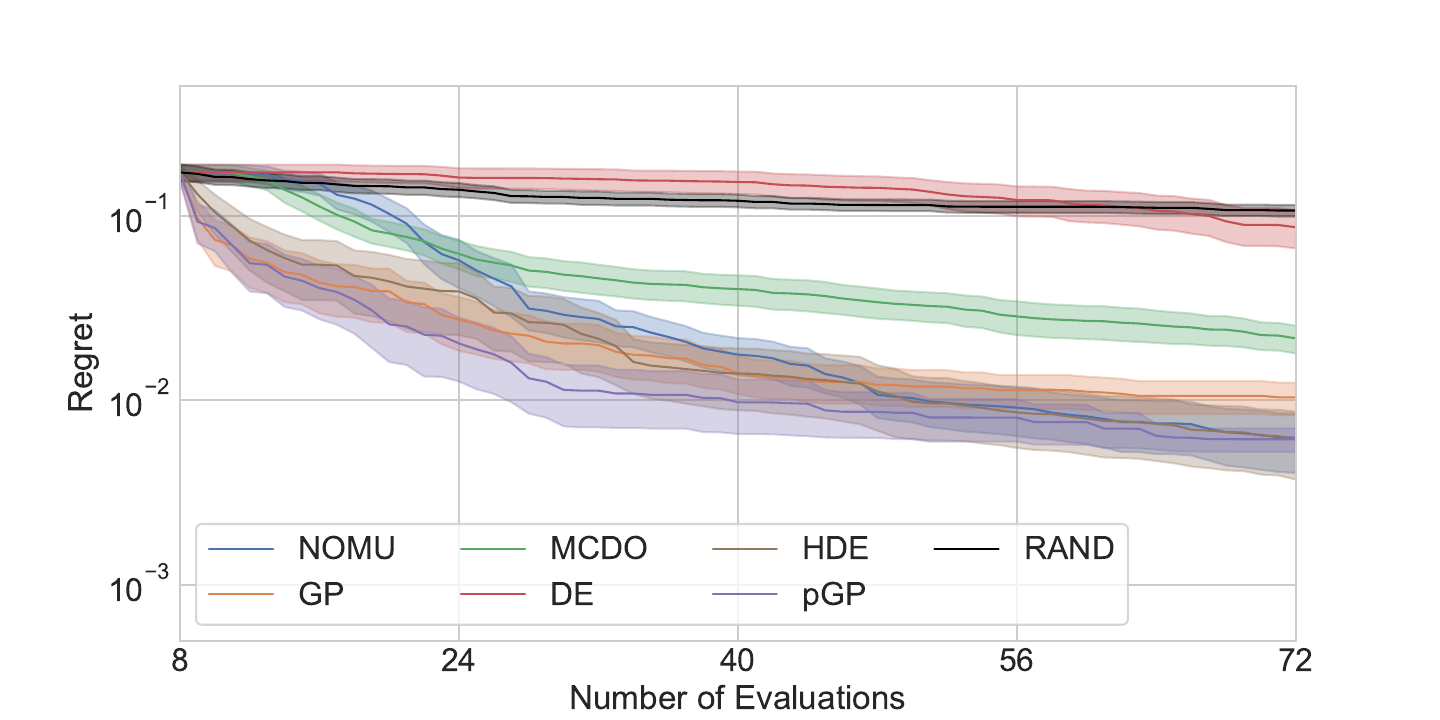}
}		
\hfill
\subfloat[Regret plot Levy, 0.5 MW\label{subfig:lev10D_0.5}]{%
		\includegraphics[trim={30 0 60 40},clip,width =.83\columnwidth]{
		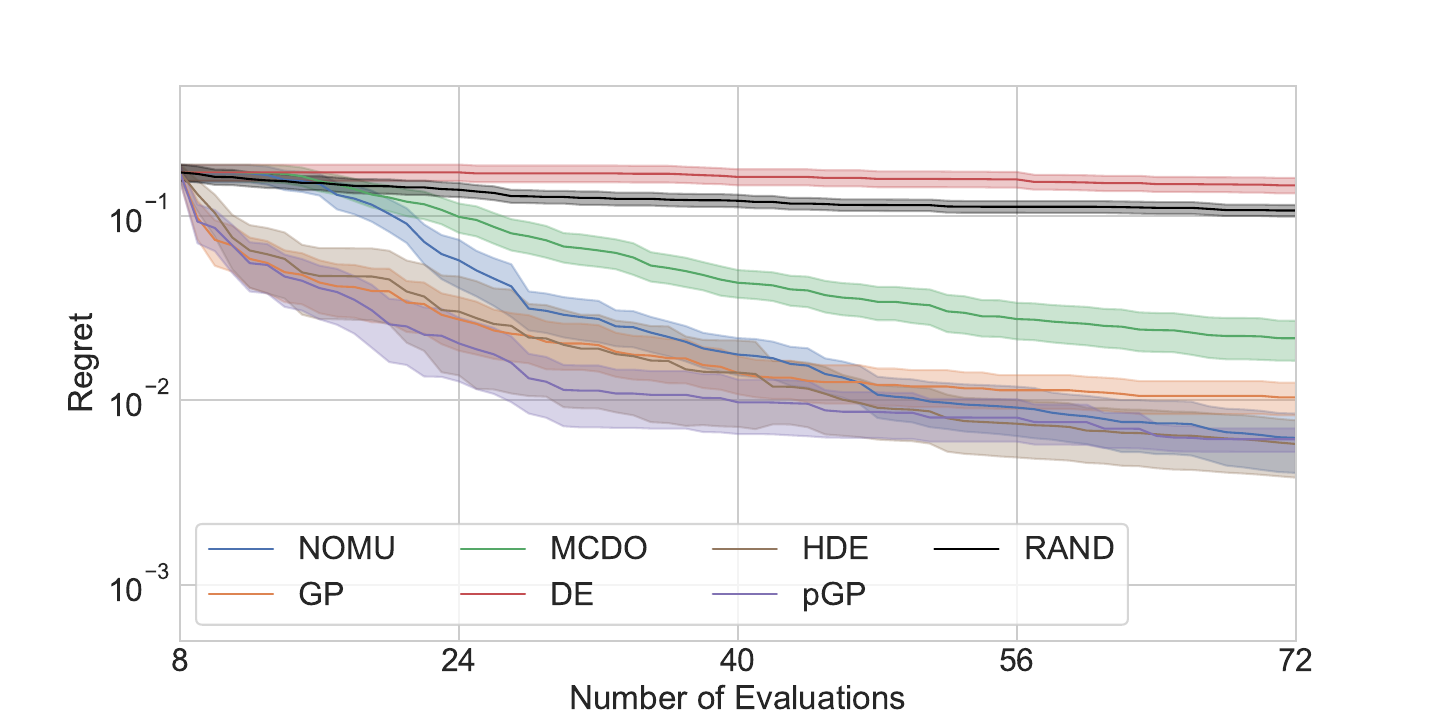}	
}
}
\makebox[\textwidth][c]{%
\subfloat[Regret plot Perm, 0.05 MW\label{subfig:per10D_0.05}]{%
		\includegraphics[trim={30 0 60 40},clip,width =.83\columnwidth]{
		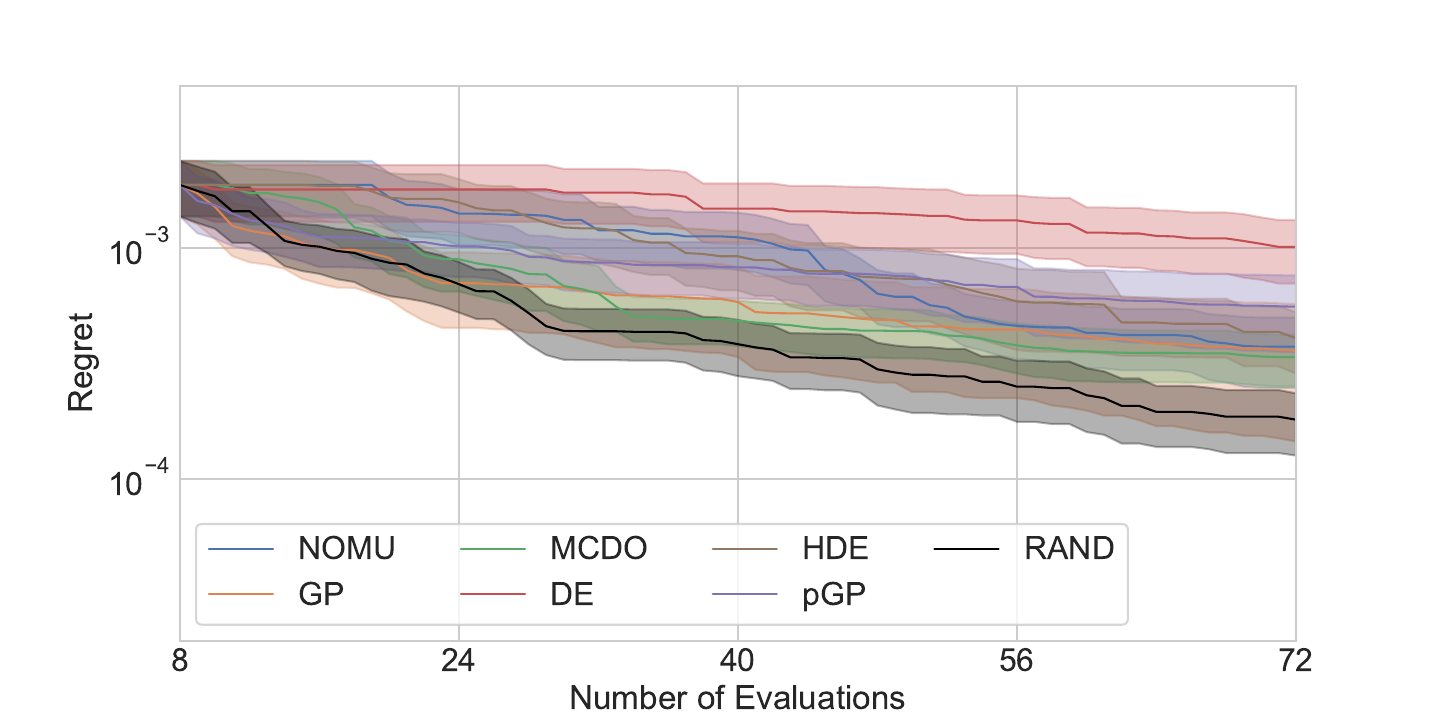}
	}
\hfill
\subfloat[Regret plot Perm, 0.5 MW\label{subfig:per10D_0.5}]{%
		\includegraphics[trim={30 0 60 40},clip,width =.83\columnwidth]{
		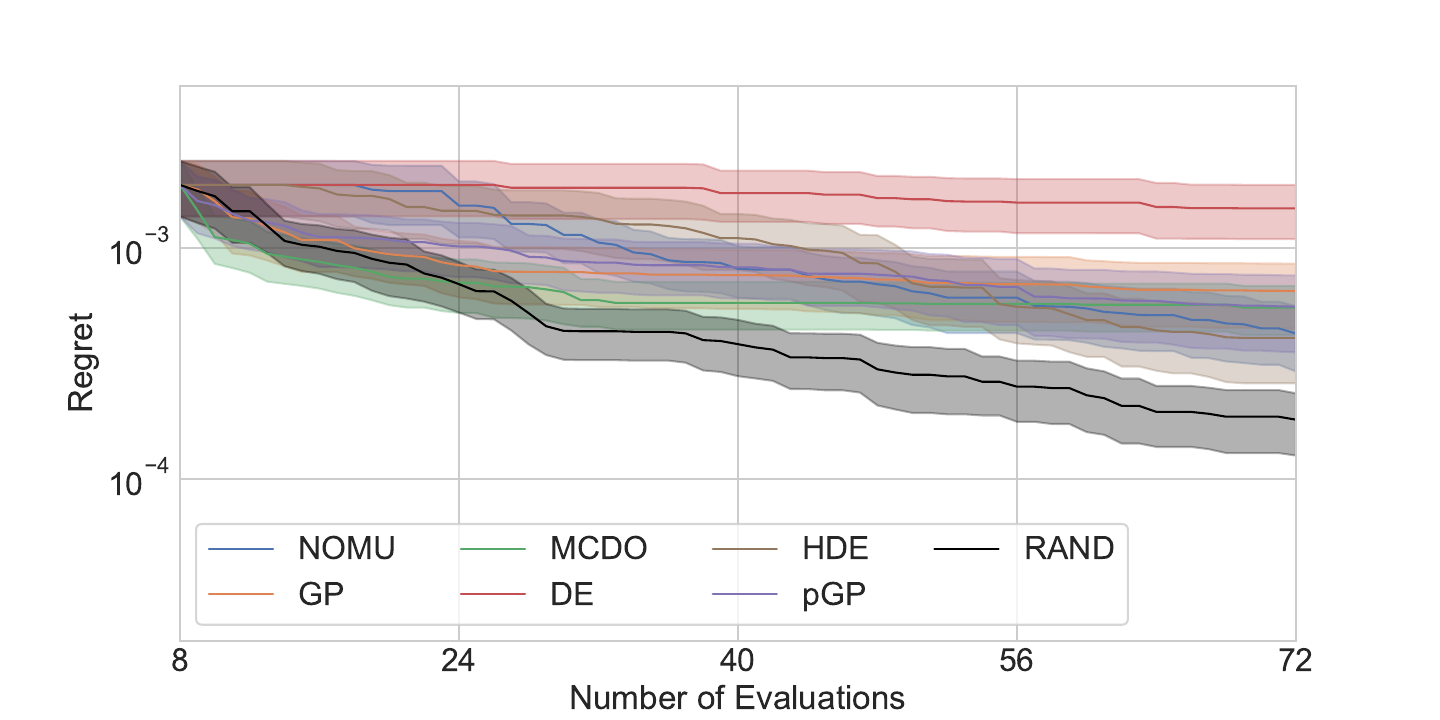}
	}
}
\makebox[\textwidth][c]{%
\subfloat[Regret plot Rosenbrock, 0.05 MW\label{subfig:ros10D_0.05}]{%
		\includegraphics[trim={30 0 60 30},clip,width =.83\columnwidth]{
		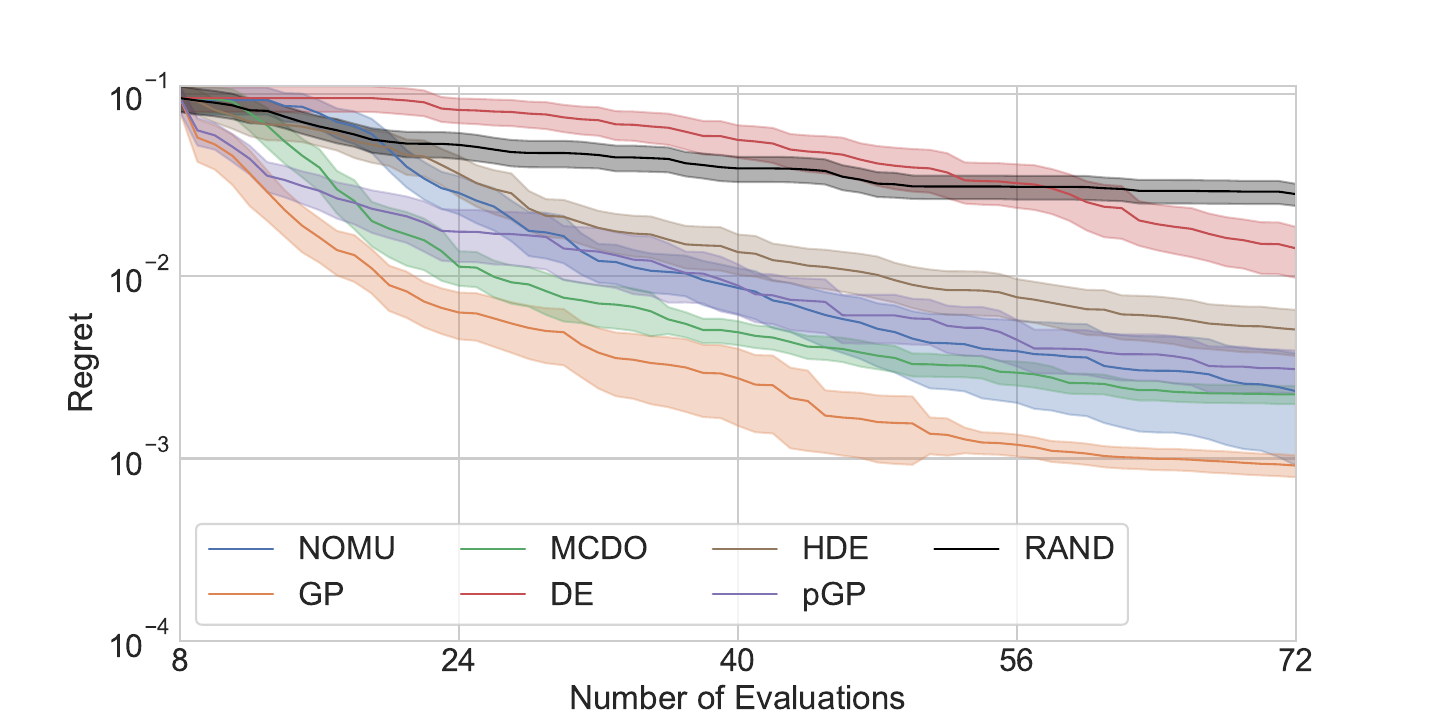}
	}
\hfill
\subfloat[Regret plot Rosenbrock, 0.5 MW\label{subfig:ros10D_0.5}]{%
		\includegraphics[trim={30 0 60 30},clip,width =.83\columnwidth]{
		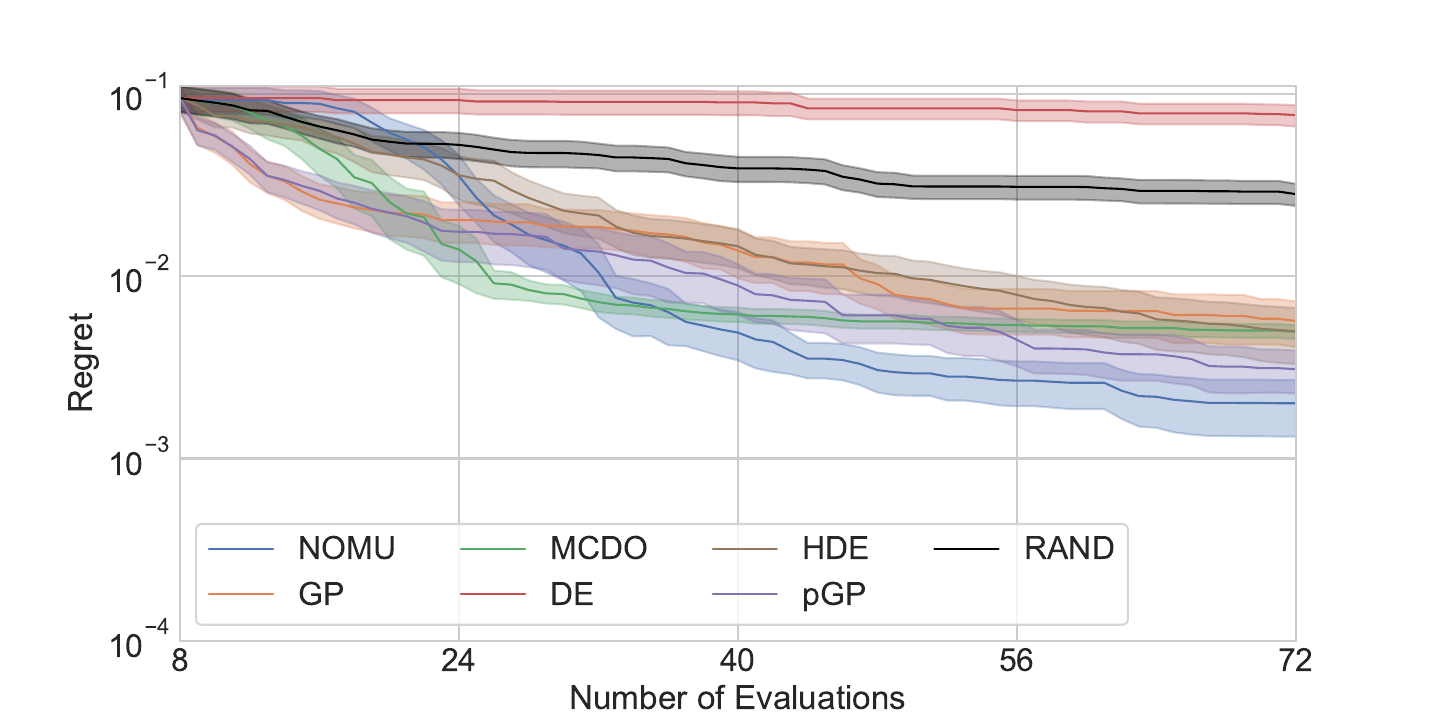}
	}
}
\makebox[\textwidth][c]{%
\subfloat[Regret plot BNN, 0.05 MW\label{subfig:bnn10D_0.05}]{%
		\includegraphics[trim={30 0 60 30},clip,width =.83\columnwidth]{
		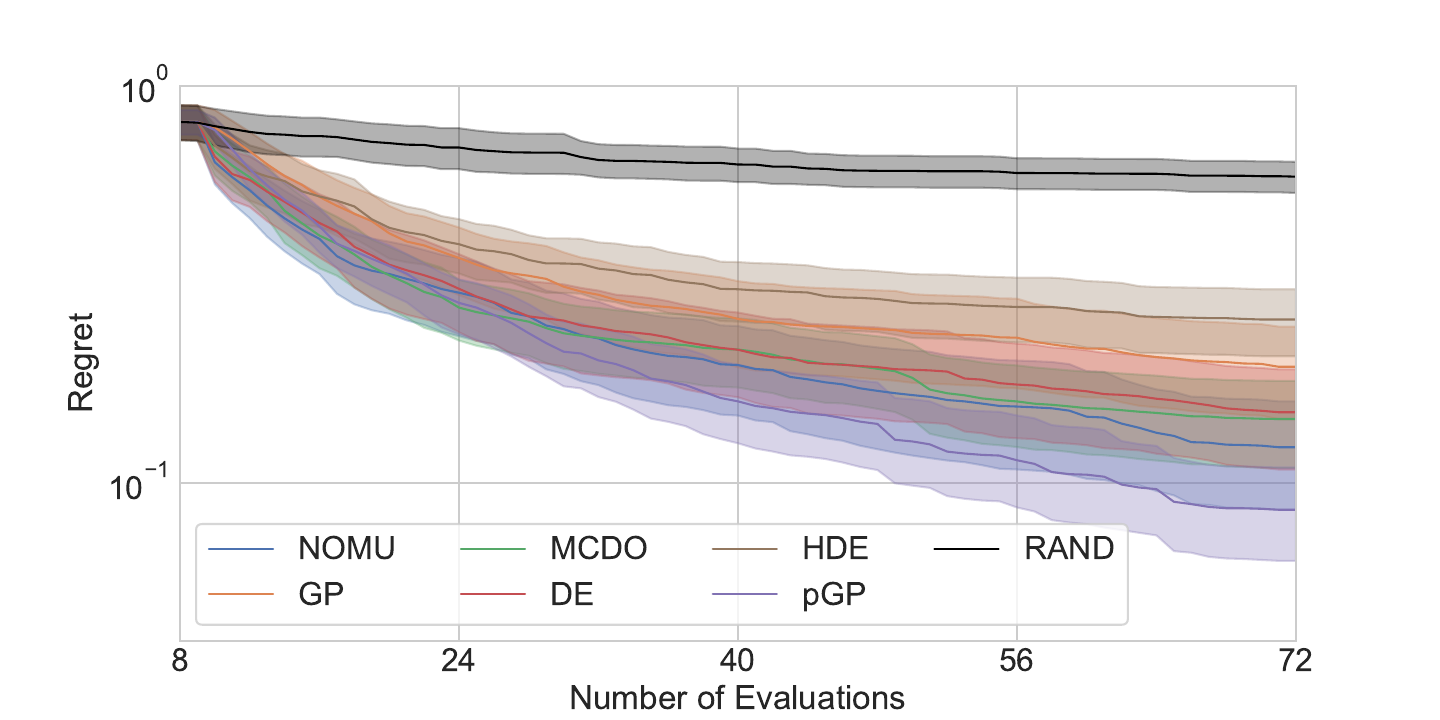}
	}
\hfill
\subfloat[Regret plot BNN, 0.5 MW\label{subfig:bnn10D_0.5}]{%
		\includegraphics[trim={30 0 60 30},clip,width =.83\columnwidth]{
		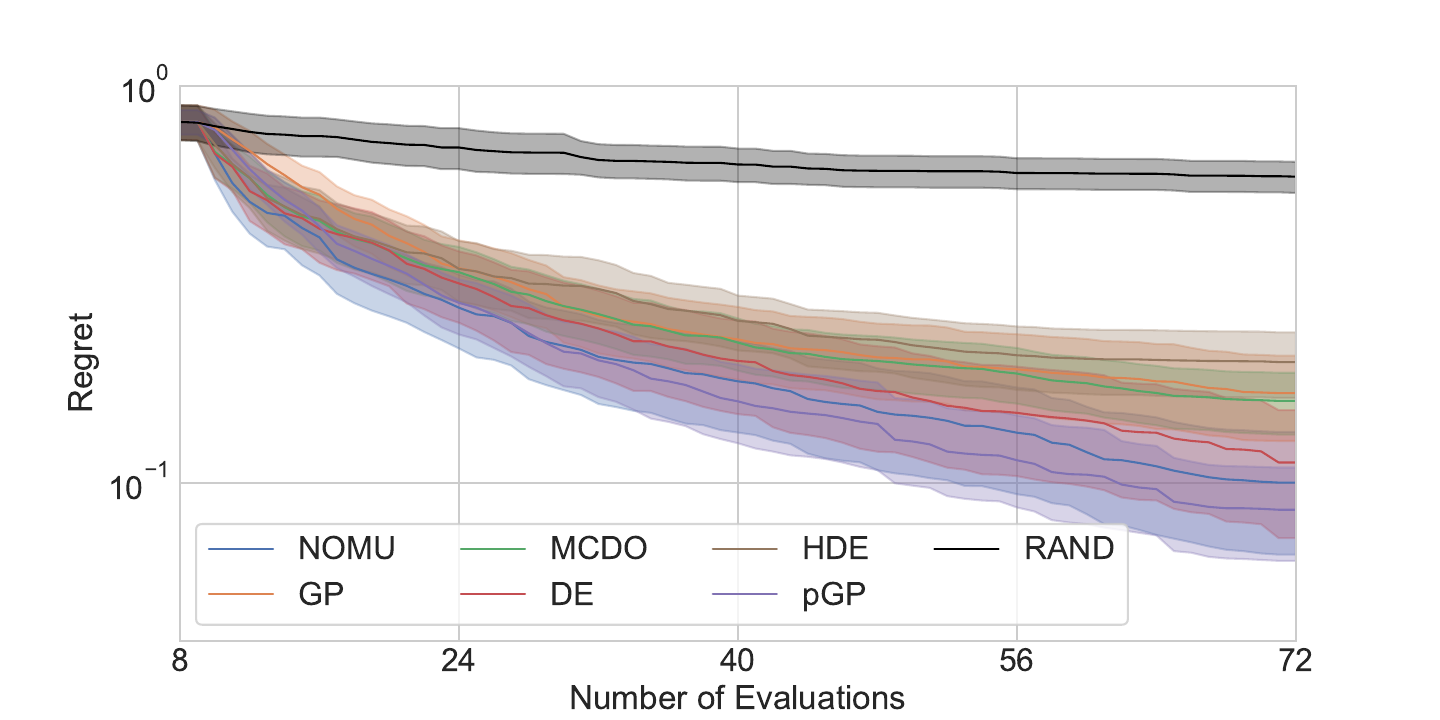}
	}
}
\vskip -0.2cm
	\caption{Regret plots for all 10D test functions and MWs of 0.05 and 0.5, respectively. We show regrets averaged over 100 runs (solid lines) with 95\% CIs.}
	\label{fig:10d_regret_plots}
\end{figure*}

\begin{figure*}[p]
\makebox[\textwidth][c]{%
\subfloat[Regret plot G-Function, 0.05 MW\label{subfig:gf20D_0.05}]{%
		\includegraphics[trim={30 0 60 30},clip,width =.83\columnwidth]{
		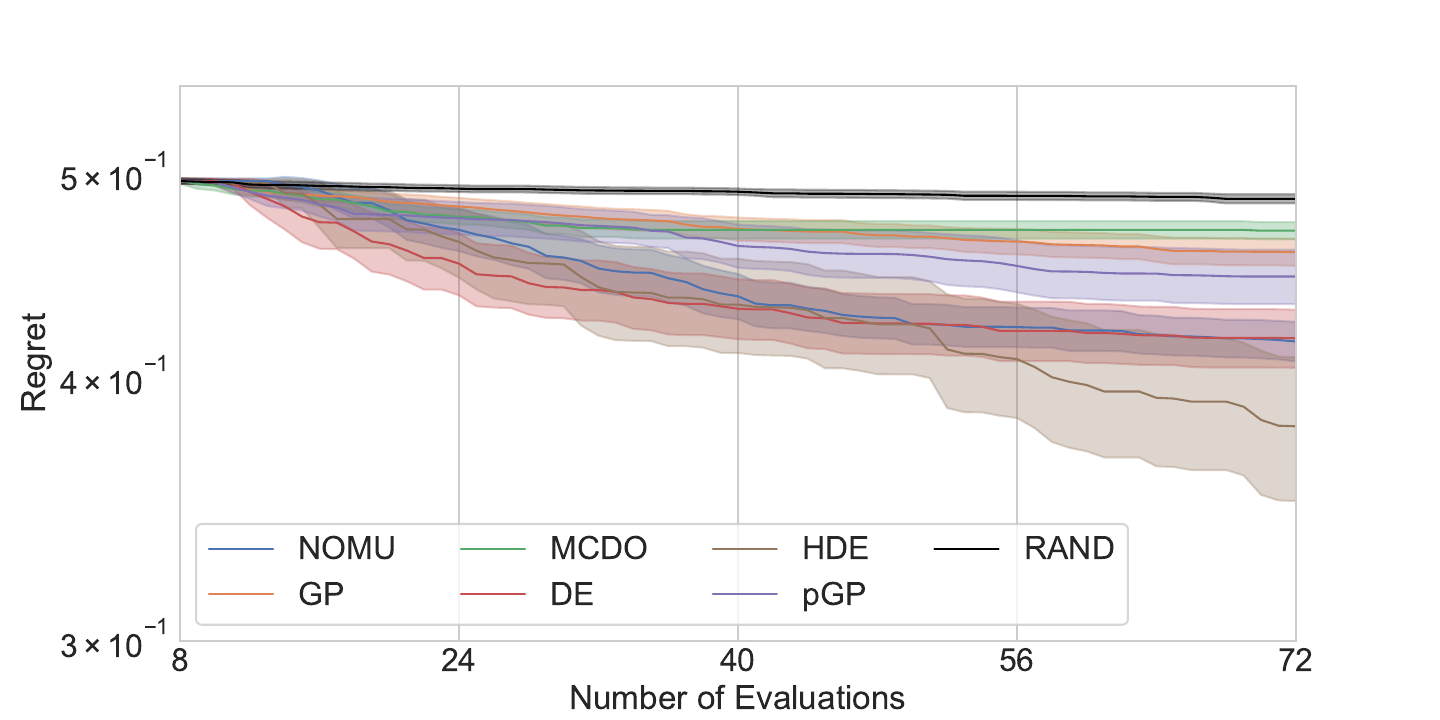}
		}
\hfill
\subfloat[Regret plot G-Function, 0.5 MW\label{subfig:gf20D_0.5}]{%
		\includegraphics[trim={30 0 60 30},clip,width =.83\columnwidth]{
		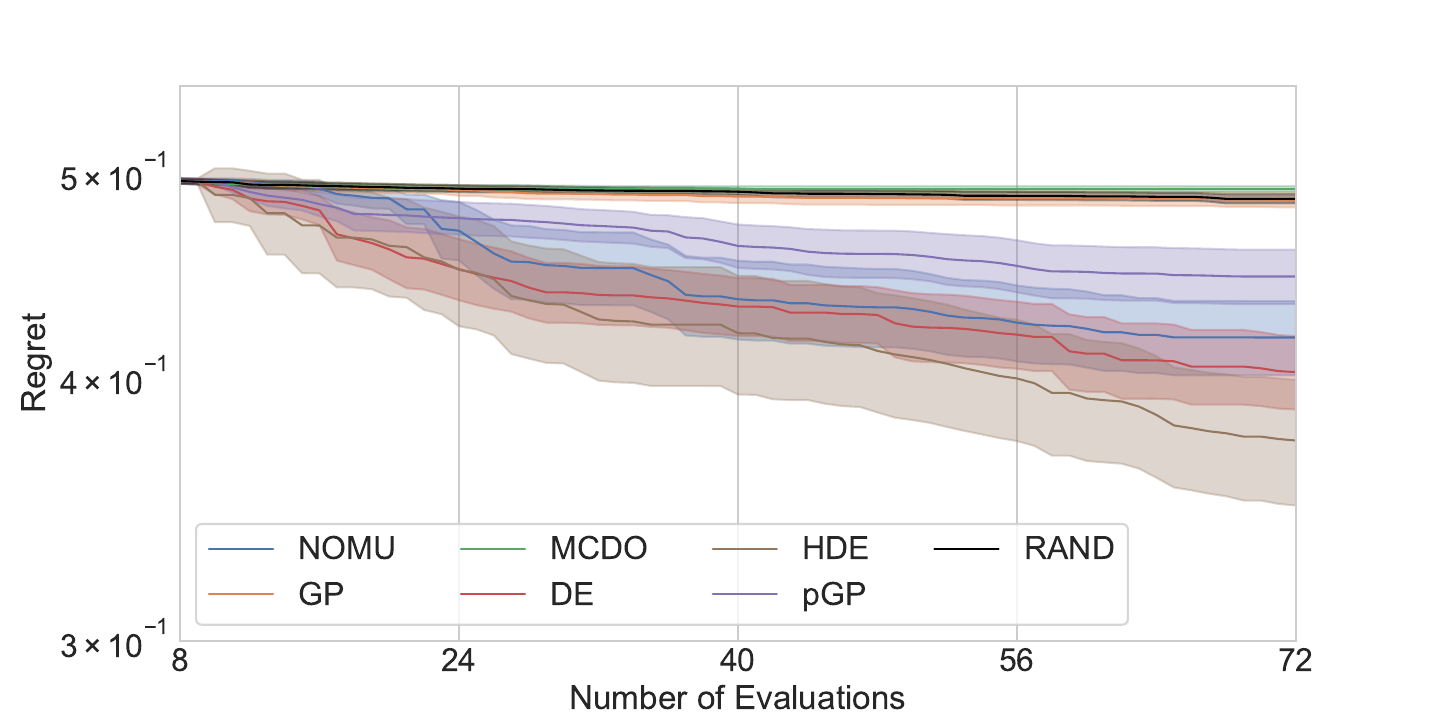}
	}
}
\makebox[\textwidth][c]{%
 \subfloat[Regret plot Levy, 0.05 MW\label{subfig:lev20D_0.05}]{%
		\includegraphics[trim={30 0 60 40},clip,width =.83\columnwidth]{
		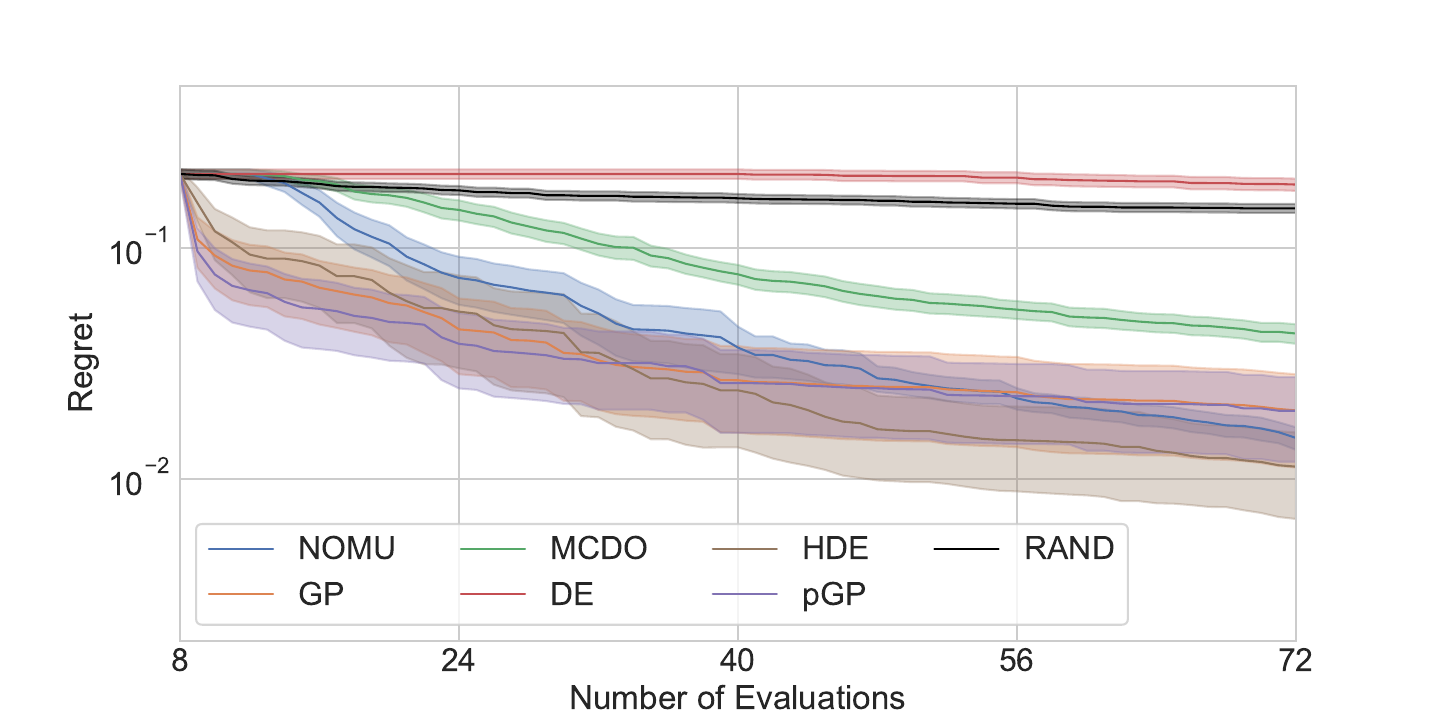}
}		
\hfill
\subfloat[Regret plot Levy, 0.5 MW\label{subfig:lev20D_0.5}]{%
		\includegraphics[trim={30 0 60 40},clip,width =.83\columnwidth]{
		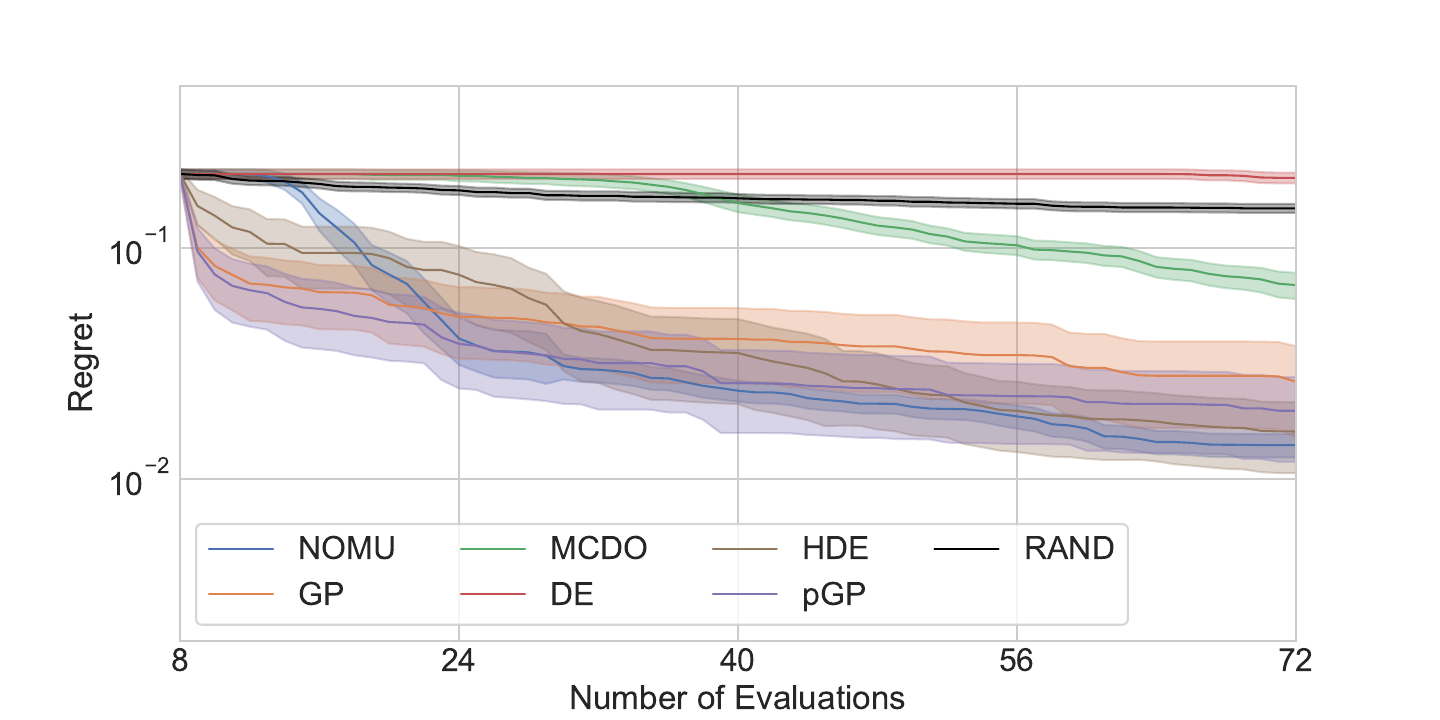}	
}
}
\makebox[\textwidth][c]{%
\subfloat[Regret plot Perm, 0.05 MW\label{subfig:per20D_0.05}]{%
		\includegraphics[trim={30 0 60 40},clip,width =.83\columnwidth]{
		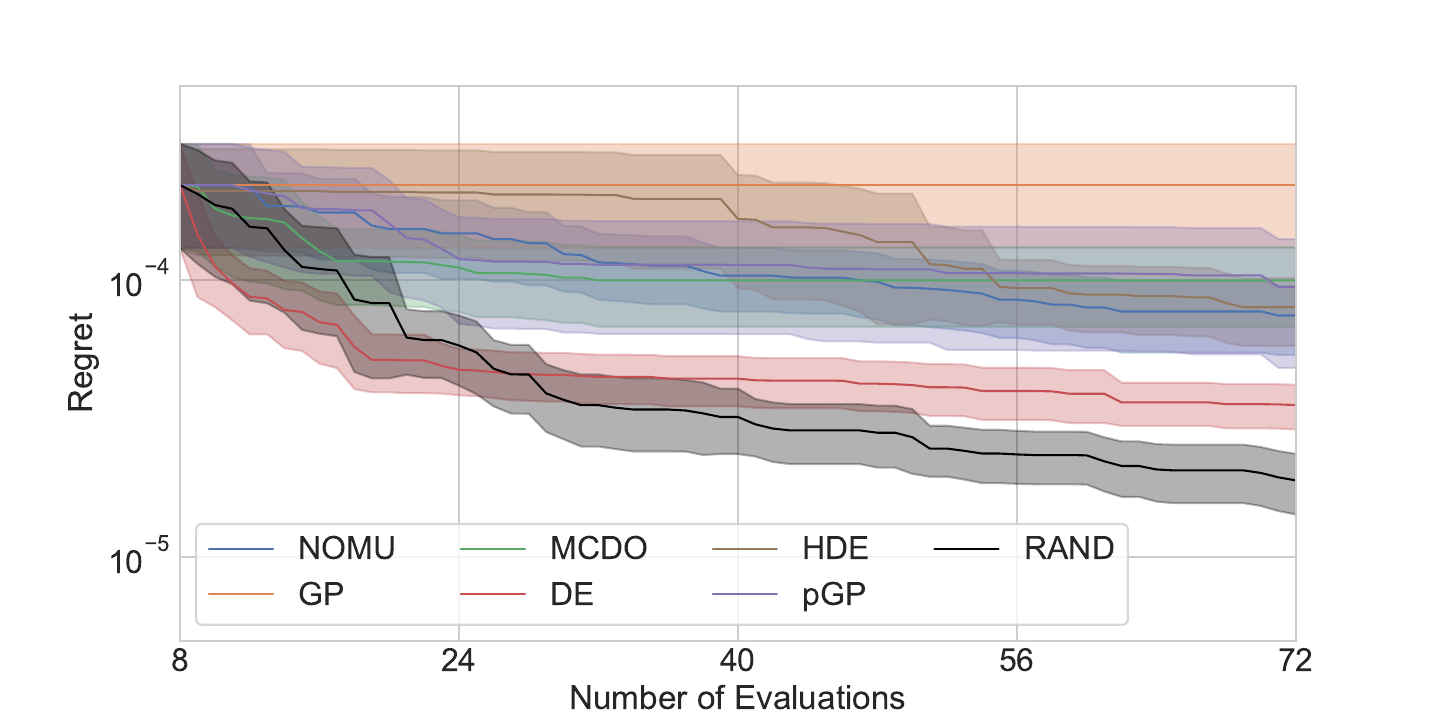}
	}
\hfill
\subfloat[Regret plot Perm, 0.5 MW\label{subfig:per20D_0.5}]{%
		\includegraphics[trim={30 0 60 40},clip,width =.83\columnwidth]{
		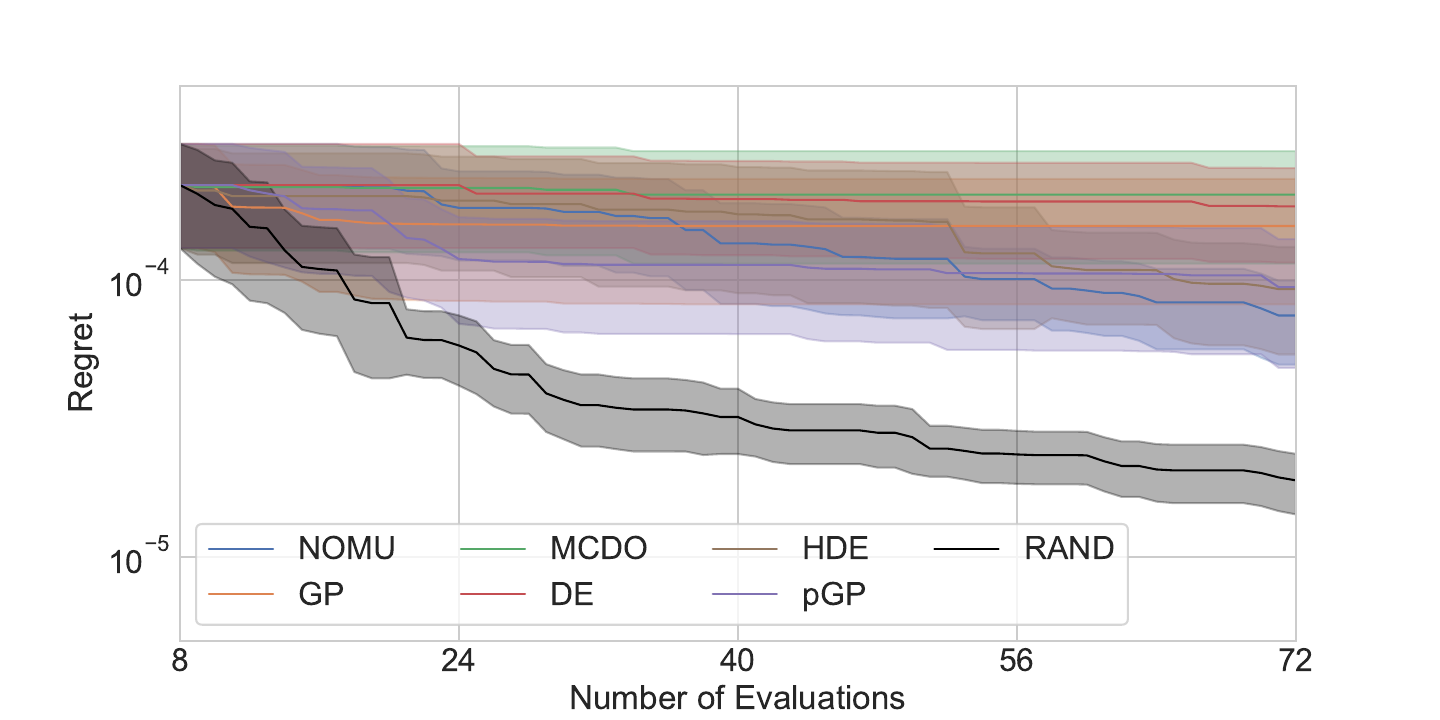}
	}
}
\makebox[\textwidth][c]{%
\subfloat[Regret plot Rosenbrock, 0.05 MW\label{subfig:ros20D_0.05}]{%
		\includegraphics[trim={30 0 60 40},clip,width =.83\columnwidth]{
		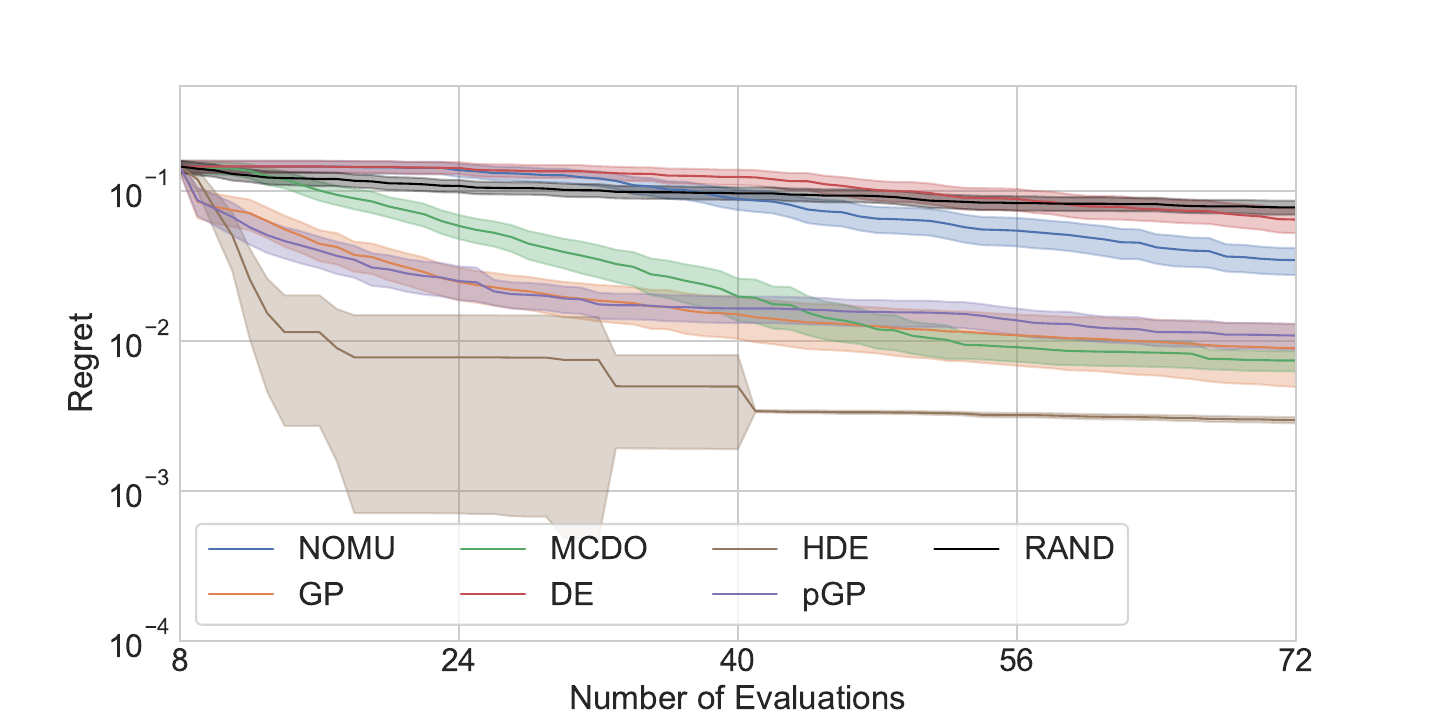}
	}
\hfill
\subfloat[Regret plot Rosenbrock, 0.5 MW\label{subfig:ros20D_0.5}]{%
		\includegraphics[trim={30 0 60 40},clip,width =.83\columnwidth]{
		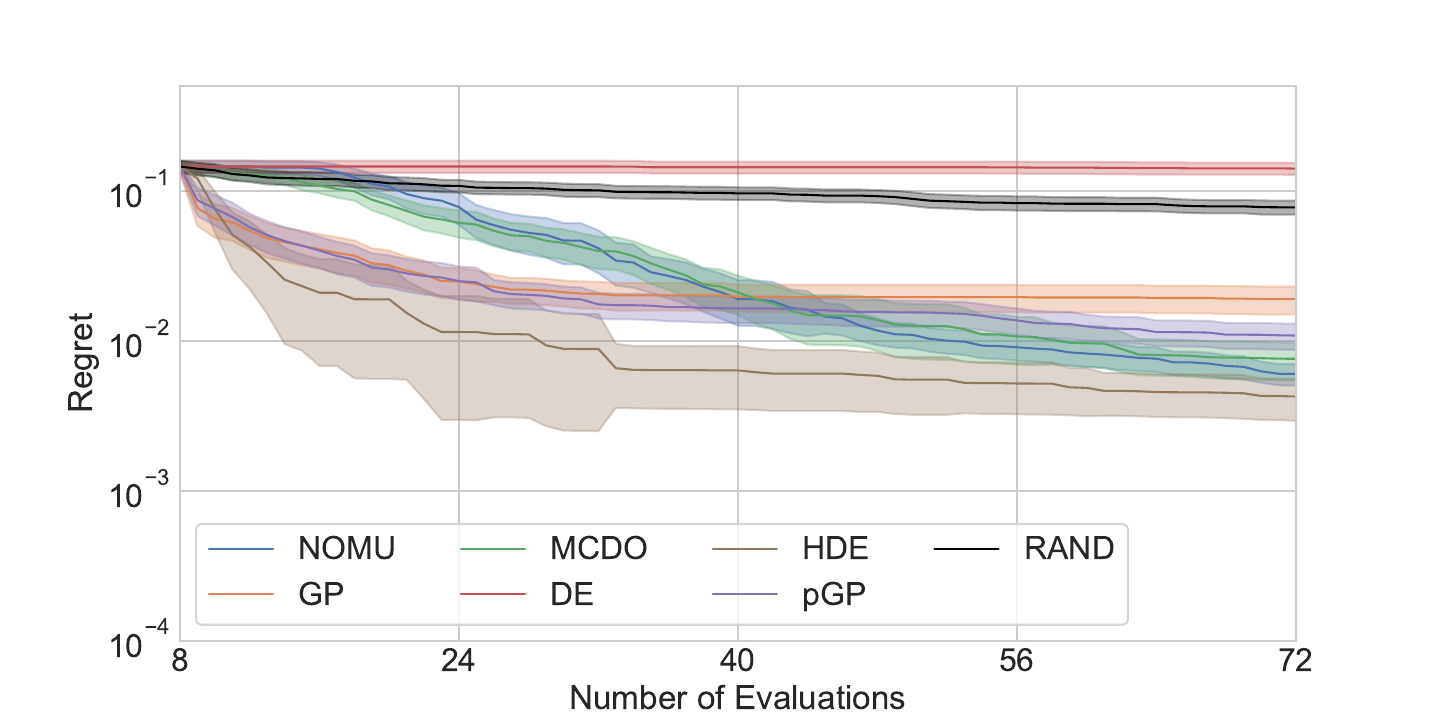}
	}
}
\makebox[\textwidth][c]{%
\subfloat[Regret plot BNN, 0.05 MW\label{subfig:bnn20D_0.05}]{%
		\includegraphics[trim={30 0 60 30},clip,width =.83\columnwidth]{
		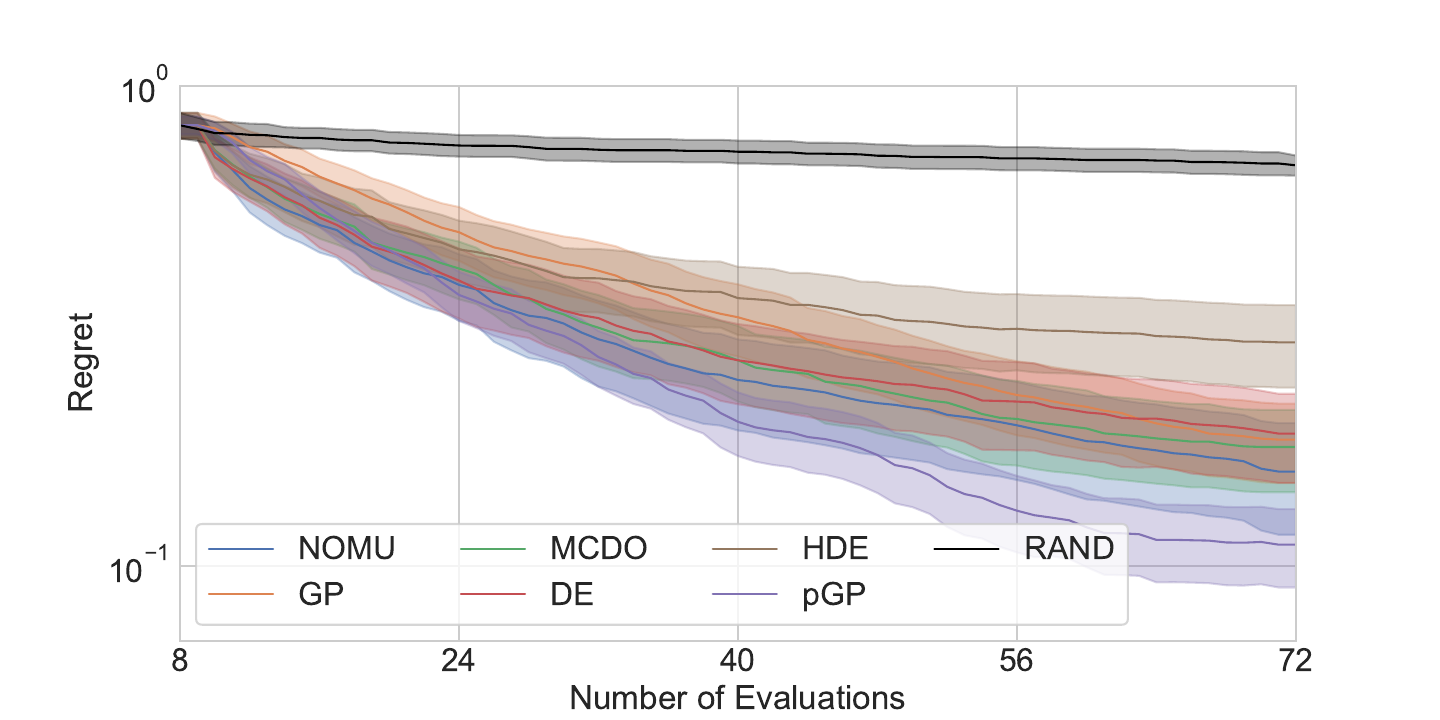}
	}
\hfill
\subfloat[Regret plot BNN, 0.5 MW\label{subfig:bnn20D_0.5}]{%
		\includegraphics[trim={30 0 60 30},clip,width =.83\columnwidth]{
		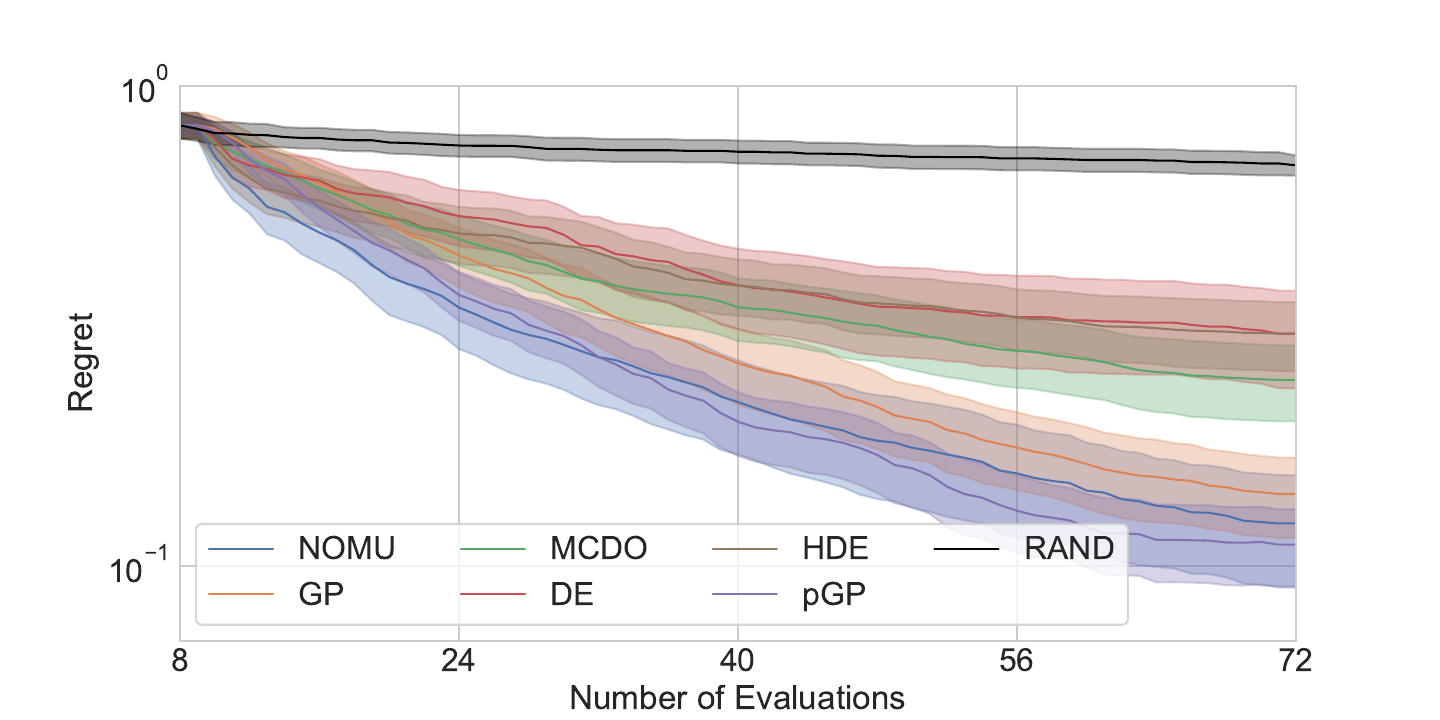}
	}
}
\vskip -0.2cm
	\caption{Regret plots for all 20D test functions and MWs of 0.05 and 0.5, respectively. We show regrets averaged over 100 runs (solid lines) with 95\% CIs.}
	\label{fig:20d_regret_plots}
\end{figure*}

\section{Extensions}\label{sec:Extensions}
\subsection{Incorporating Data Noise}\label{subsec:Incorporating Data Noise Uncertainty}
We now discuss one way to incorporate data noise in NOMU. If $\signoise$ is unknown, one option to learn it, is to add another output~$\signoisehat$ to its architecture and change the loss function to
\begin{align}\label{eq:lmu2datanoise}
  L^\hp(\NN_\theta):=&\sum_{i=1}^{\ntr}\left(\frac{\left(\fhat(\xtr_i)-\ytr_i\right)^2}{2\left(\signoisehat(\xtr_i)\right)^2}+\ln\left(\signoisehat(\xtr_i)\right)\right)\notag\\
  &+ \,\musqr\cdot\sum_{i=1}^{\ntr}\frac{\left(\sigmodelhatraw(\xtr_i)\right)^2}{2\left(\signoisehat(\xtr_i)\right)^2}\\ 
  &+ \muexp\cdot\frac{1}{\lambda_d(\X)}\int_{\X}e^{-\cexp\cdot \sigmodelhatraw(x)}\,dx,\notag
\end{align}

in the case of Gaussian data noise uncertainty.
 In this case we recommend to first train the model prediction~$\fhat$ and the data noise~$\signoisehat$ only and then freeze their parameters and train the $\sigmodelhatraw$-network for capturing model uncertainty~$\sigmodelhat$. Note that in the NOMU loss~\eqref{eq:lmu2} we implicitly assumed a constant very small and negligible data noise $\signoise^2$, absorbed as a factor into the hyperparameter $\muexp$ and the regularization factor $\lambda$. 
 Thus, when using the loss  \eqref{eq:lmu2datanoise} instead of the NOMU loss~\eqref{eq:lmu2}, $\muexp$ and $\lambda$ need to be chosen significantly larger.

In the case of known (heteroscedastic) data noise~$\signoise(x)$, \eqref{eq:lmu2datanoise} can be simplified, replacing $\signoisehat$ by $\signoise$ and dropping the term $\ln\left(\signoise\right)$ (in this case, one can then also drop the $\signoisehat$-output of the NOMU architecture).

\emph{Predictive} UBs are then obtained as
\begin{align}
&\left(\fhat(x)\mp\sqrt{ c_1~\sigmodelhat^2(x)+c_2~\signoisehat^2(x)}\right)
\end{align}
with suitable calibration parameters $c_1,c_2 \in \Rpz$.

In the case of  known normally distributed data noise (and under the assumption that the posterior of each $f(x)$ is Gaussian), it is sufficient to calibrate one calibration parameter $\tilde{c}$ to obtain approximate $\alpha$ predictive bounds
\begin{align}
&\left(\fhat(x)\mp\Phi^{-1}(1-\frac{1-\alpha}{2})\sqrt{ \tilde{c}~\sigmodelhat^2(x)+\signoise^2(x)}\right),
\end{align}
where $\tilde{c}$ relates to the typically unknown prior scale.

\subsection{NOMU for Upwards and Downwards Directed Hypothesis Classes}\label{subsec:NOMU for Upwards and Downwards Directed Hypothesis Class}

As mentioned in \Cref{subsec:connectingNOMUtoPointwiseUncertaintyBounds} and discussed in more detail in \Cref{subsec:A note on Theorem}, often the set~$\HC_{\Dtr}$ is not upwards directed for typical NN-architectures and \cref{eq:IntegralEqualityUB} of \Cref{thm:approxOfUBs} is not fulfilled in general. Therefore, we carefully designed our NOMU algorithm to be able to cope with settings where the set~$\HC_{\Dtr}$ is not upwards and/or downwards directed. The downwards directed property is defined analogously as follows:

\begin{assumption}[Downwards Directed]\label{assumption:downwards direction}
For every $f_1,f_2 \in \HC_{\Dtr}$ there exists an $f \in \HC_{\Dtr}$ such that $f(x)\le \min(f_1(x),f_2(x))$ for all $x \in X$.
\end{assumption}

However, in the following, we discuss a modification of NOMU, which is specifically designed for the case if $\HC_{\Dtr}$ is indeed upwards and/or downwards directed. In this case, by \Cref{thm:approxOfUBs}, we can directly solve 
\begin{align}\label{eq:upperBoundIfUpwardsDirected}
\argmax_{h\in \HC_{\Dtr}}\int_{X}u(h(x)-\fhat(x))\,dx 
\end{align}
to obtain an upper UB and/or
\begin{align}\label{eq:lowerBoundIfDownwardsDirected}
\argmin_{h\in \HC_{\Dtr}}\int_{X}u(h(x)-\fhat(x))\,dx
\end{align}
to obtain a lower UB,
without the need for any modifications as used in the proposed NOMU algorithm (we do not need the dashed connections in the architecture from \Cref{fig:nn_tik}, we do not need a specific choice of $u$ and we do not need to introduce $\sigmodelhatraw$ and the final activation function $\varphi$). The UBs obtained in this way \emph{exactly} coincide then with the pointwise upper and lower UBs defined in \eqref{eq:upperSupBound}, respectively. 

Moreover in this case, $\fhat$ can be even removed from \cref{eq:upperBoundIfUpwardsDirected,eq:lowerBoundIfDownwardsDirected} (as can be seen from the proof of \Cref{thm:gernelazidApproxOfUBs}). Thus, in the following loss formulation, we will remove $\fhat$ in the respective term \termcprime{}.

\subsubsection{The Architecture}
Under \Cref{assumption:upwards direction} (upwards directed)  \emph{and} \Cref{assumption:downwards direction} (downwards directed), we propose an architecture $\NNt_\theta$ consisting of three sub-networks with three outputs: the model prediction $\fhat$, the estimated lower UB $\lUBhat$ and the estimated upper UB $\uUBhat$. In \Cref{fig:nn_tik_extension}, we provide a schematic representation of $\NNt_\theta$.
\begin{figure}[t!]
        \vskip 0.1in
        \begin{center}
        \centerline{
        \resizebox{1\columnwidth}{!}{
                \begin{tikzpicture}
                [cnode/.style={draw=black,fill=#1,minimum width=3mm,circle}]
                \node[cnode=gray,label=180:$\mathlarger{\mathlarger{x \in \X}}$] (x1) at (0.5,-2) {};
                \node[cnode=gray] (x2+3) at (3,3) {};
                \node at (3,2) {$\mathlarger{\mathlarger{\vdots}}$};
                \node[cnode=gray] (x2+1) at (3,1) {};
                \node[cnode=gray] (x2-3) at (3,-3) {};
                \node at (3,-2) {$\mathlarger{\mathlarger{\vdots}}$};
                \node[cnode=gray] (x2-1) at (3,-1) {};
                
                \node[cnode=gray] (x2-5) at (3,-5) {};
                \node at (3,-6) {$\mathlarger{\mathlarger{\vdots}}$};
                \node[cnode=gray] (x2-7) at (3,-7) {};
                
                \draw (x1) -- (x2+3);
                \draw (x1) -- (x2+1);
                \draw (x1) -- (x2-1);
                \draw (x1) -- (x2-3);
                
                \draw (x1) -- (x2-5);
                \draw (x1) -- (x2-7);

                \node[cnode=gray] (x3+3) at (6,3) {};
                \node at (6,2) {$\mathlarger{\mathlarger{\vdots}}$};
                \node[cnode=gray] (x3+1) at (6,1) {};
                \node[cnode=gray] (x3-3) at (6,-3) {};
                \node at (6,-2) {$\mathlarger{\mathlarger{\vdots}}$};
                \node[cnode=gray] (x3-1) at (6,-1) {};
                
                \node[cnode=gray] (x3-5) at (6,-5) {};
                \node at (6,-6) {$\mathlarger{\mathlarger{\vdots}}$};
                \node[cnode=gray] (x3-7) at (6,-7) {};

                \foreach \y in {1,3}
                {   \foreach \z in {1,3}
                        {   \draw (x2+\z) -- (x3+\y);
                        }
                }
                \foreach \y in {1,3}
                {   \foreach \z in {1,3}
                        {   \draw (x2-\z) -- (x3-\y);
                        }
                }
                
                \foreach \y in {5,7}
                {   \foreach \z in {5,7}
                        {   \draw (x2-\z) -- (x3-\y);
                        }
                }

                \node at (7.5,+3) {$\mathlarger{\mathlarger{\ldots}}$};
                \node at (7.5,+2) {$\mathlarger{\mathlarger{\ldots}}$};
                \node at (7.5,+1) {$\mathlarger{\mathlarger{\ldots}}$};
                \node at (7.5,-1) {$\mathlarger{\mathlarger{\ldots}}$};
                \node at (7.5,-2) {$\mathlarger{\mathlarger{\ldots}}$};
                \node at (7.5,-3) {$\mathlarger{\mathlarger{\ldots}}$};
                
                \node at (7.5,-5) {$\mathlarger{\mathlarger{\ldots}}$};
                \node at (7.5,-6) {$\mathlarger{\mathlarger{\ldots}}$};
                \node at (7.5,-7) {$\mathlarger{\mathlarger{\ldots}}$};

                \node[cnode=gray] (x4+3) at (9,3) {};
                \node at (9,2) {$\mathlarger{\mathlarger{\vdots}}$};
                \node[cnode=gray] (x4+1) at (9,1) {};
                \node[cnode=gray] (x4-3) at (9,-3) {};
                \node at (9,-2) {$\mathlarger{\mathlarger{\vdots}}$};
                \node[cnode=gray] (x4-1) at (9,-1) {};
                
                \node[cnode=gray] (x4-5) at (9,-5) {};
                \node at (9,-6) {$\mathlarger{\mathlarger{\vdots}}$};
                \node[cnode=gray] (x4-7) at (9,-7) {};

                \node[cnode=gray,label=360:$\mathlarger{\mathlarger{\uUBhat(x)\in \Rpz}}$] (x5+2) at (11.5,2) {};
                \node[cnode=gray,label=360:$\mathlarger{\mathlarger{\fhat(x)\in \Y}}$] (x5-2) at (11.5,-2) {};
                \node[cnode=gray,label=360:$\mathlarger{\mathlarger{\lUBhat(x)\in \Rpz}}$] (x5-6) at (11.5,-6) {};
                
                \draw (x4+3)--(x5+2);
                \draw (x4+1)--(x5+2);

                \draw (x4-1)--(x5-2);
                \draw (x4-3)--(x5-2);

                \draw (x4-5)--(x5-6);
                \draw (x4-7)--(x5-6);

                \draw [dotted, line width=0.4mm] (2.5,0.5) rectangle (9.5,+3.5);
                \draw [dotted, line width=0.4mm] (2.5,-3.5) rectangle (9.5,-0.5);
                \draw [dotted, line width=0.4mm] (2.5,-7.5) rectangle (9.5,-4.5);
                \node at (6,-0.2) {$\mathlarger{\mathlarger{\fhat}}$-network};
                \node at (1,+2.5) {$\mathlarger{\mathlarger{\uUBhat}}$-network};
                \node at (1,-6.5) {$\mathlarger{\mathlarger{\lUBhat}}$-network};
                \end{tikzpicture}
        }
        }
        \caption{$\NNt_\theta$: a modification of NOMU's original network architecture for upwards and downwards directed hypothesis classes.}
        \label{fig:nn_tik_extension}
        \end{center}
\vskip -0.2in
\end{figure}

\subsubsection{The Loss Function}
From \cref{eq:upperBoundIfUpwardsDirected,eq:lowerBoundIfDownwardsDirected} we can then directly formulate the following modified NOMU loss function $\tilde{L}^\hp$.
\begin{definition}[NOMU Loss Upwards and Downwards Directed]\label{def:NOMU_Loss_Upwards_Directed}
Let  $\hp:=(\musqr,\muexp,\cexp)\in \Rpz^3$ denote a tuple of hyperparameters. Let $\lambda_d$ denote the $d$-dimensional Lebesgue measure. Furthermore, let $u:Y\to\R$ be strictly-increasing and continuous. Given a training set $\Dtr$, the loss function $\tilde{L}^\hp$ is defined as
\begin{align}\label{eq:NOMU_Loss_Extension}
  &\tilde{L}^\hp(\NNt_\theta):=\underbrace{\sum_{i=1}^{\ntr}(\fhat(\xtr_i
  )-\ytr_i)^2}_{\termaprime{}}\\
  &+ \,\musqr\cdot\underbrace{\sum_{i=1}^{\ntr}(\uUBhat(\xtr_i)-\ytr_i)^2 + (\lUBhat(\xtr_i)-\ytr_i)^2}_{\termbprime{}}\\
  &- \muexp\cdot\underbrace{\frac{1}{\lambda_d(\X)}\int_{\X}u(%
  -\lUBhat(x))+u(\uUBhat(x)%
  )\,dx}_{\termcprime{}}.
\end{align}

\end{definition}
The interpretations of the three terms \termaprime{}, \termbprime{} and \termcprime{} are analogous to the ones in the original NOMU loss.

Note that, the three sub-networks: the $\uUBhat$-network, the $\lUBhat$-network and the $\fhat$-network can also be trained independently using the corresponding terms in the loss function. Moreover, if one is only interested in the upper (lower) UB or $\HC_{\Dtr}$ is only upwards (downwards) directed, i.e., fulfills only \Cref{assumption:upwards direction} (\Cref{assumption:downwards direction}), then one can remove the respective sub-network from the architecture as well as the corresponding terms in the loss function. Furthermore, note that, now the obtained UBs can be asymmetric too.\footnote{After the publication of the present paper, \citet{weissteiner2023bayesian} implemented such an upper UB for monotonically non-decreasing functions.}
\section{Discussion of the Desiderata}\label{appendix:Desiderata}
In this section, we discuss in more detail the desiderata proposed in \Cref{subsubsec:Desiderata}. Specifically, we discuss how NOMU fulfills them and thereby prove several propositions.
First, we establish a relation of NOMU to the classical Bayesian approach.

The Bayesian point of view allows for mathematically rigorous estimation of uncertainty.
However, in general a fully Bayesian approach for quantifying uncertainty is very challenging and involves to

\begin{enumerate}[i.,leftmargin=*,topsep=0pt,itemsep=0pt,partopsep=0pt, parsep=0pt]
    \item\label{itm:BayesPriorChallenge} formulate a realistic prior,%
    \item\label{itm:BayesAlgorithmChallenge} use an algorithm to approximate the posterior (challenging to get a good approximation in feasible time for complex models such as NNs),
    \item\label{itm:BayesLastStep} use this approximation of the posterior to obtain UBs.
\end{enumerate}
We follow a different approach by directly approximating \ref{itm:BayesLastStep} based on essential properties of the posterior, e.g., for zero data noise model uncertainty vanishes at, and becomes larger far away from training data.
This can be a reasonable approach in applications since many Bayesian state-of-the-art methods even fail to fulfill these basic properties when implemented in practice \citep{malinin2018predictive} (see \Cref{fig:D41dJakobhfun}).
Since especially for NNs \ref{itm:BayesAlgorithmChallenge} is very costly, we ask ourselves the question, whether in practice it is even true that one has more intuition about the important properties of the prior than about the important properties of the posterior? In other words, can \labelcref{itm:BayesPriorChallenge,itm:BayesAlgorithmChallenge} be skipped by directly approximating \ref{itm:BayesLastStep}? In the case of mean predictions, many successful algorithms following this procedure already exist. These algorithms directly try to approximate the posterior mean by exploiting one's intuition how the posterior of a realistic prior should behave without the need of precisely specifying the prior, e.g.,:
\begin{enumerate}[leftmargin=*,topsep=0pt,partopsep=0pt, parsep=0pt]
     \item \emph{Spline regression}: In many applications it is very intuitive that a good approximation of the posterior mean should not have unnecessarily large second derivative, without explicitly stating the corresponding prior. Even though spline regression can be formulated as the posterior mean of a limit of priors \citep{wahba1978improper}, for a practitioner it can be much more intuitive to decide whether spline regression fits to one's prior beliefs by looking at the smoothness properties of the posterior mean than looking at such complex priors.
     \item \emph{Convolutional Neural Networks (CNNs)}: In image recognition, it is very intuitive to see that a good approximation of the posterior mean should fulfill typical properties, e.g., two pictures that are just slightly shifted should be mapped to similar outputs. CNNs fulfill such typical properties to a large extent and thus have proven to be very successful in practice. Nevertheless, from a Bayesian point of view these properties rely on a yet unknown prior.
 \end{enumerate}
 
 \subsection{Desideratum \ref{itm:Axioms:trivial}}
 Desideratum~\ref{itm:Axioms:trivial} is trivial, since $\sigmodel\geq 0$ per definition. Credible bounds $\underline{C\!B}$ and $\overline{C\!B}$ are lower and upper bounds of an interval $[\underline{C\!B},\overline{C\!B}]$, therefore $\underline{C\!B}\leq\overline{C\!B}$ holds by definition as well. To the best of our knowledge, every established method to estimate model uncertainty yields bounds that fulfil \ref{itm:Axioms:trivial}.
 Furthermore, note that \ref{itm:Axioms:trivial} also holds in the presence of data noise uncertainty.
 
 \subsubsection{How Does NOMU Fulfill \ref{itm:Axioms:trivial}?}
 By definition, NOMU exactly fulfills \ref{itm:Axioms:trivial}.
 \begin{subproposition}\label{prop:NOMUD1}
 For NOMU, $\sigmodelhat\geq0$ and thus $\lUB_c=\fhat-c\sigmodel\le\fhat\le\fhat+c\sigmodel=\uUB_c$ for all $c\ge0$.
 \end{subproposition}
 \begin{proof}
 This holds since $\sigmin\geq0$ in the readout map $\readout$ (see \Cref{eq:modelUncertaintyPrediction}).
 \end{proof}
 
 \subsection{Desideratum~\ref{itm:Axioms:ZeroUncertaintyAtData}}\label{sec:desiderata:InSample}
 In the case of zero data noise, Desideratum~\ref{itm:Axioms:ZeroUncertaintyAtData} holds true exactly.
 
 \begin{subproposition}[Zero Model Uncertainty at Training Points]\label{thm:ZeroModel UncertaintyAtTrainingPoints}
 
 Let $\signoise\equiv0$. Furthermore, let $\Dtr$ be a finite set of training points and consider a prior distribution $\PP[f\in\cdot]$ on the function space $\{f:X\to Y\}$ such that there exists a function in the support of $\PP[f\in\cdot]$ that exactly fits through the training data. Then for the posterior distribution it holds that for all $(\xtr,\ytr)\in \Dtr$ that
 \begin{align}
     &\PP(f(\xtr)=\ytr|\Dtr)=1\\
     &\PP(f(\xtr)\neq\ytr|\Dtr)=0.\label{eq:Pfneqy}
 \end{align}
 In words, there is no model uncertainty at input training points, i.e., $\sigmodel(\xtr)=0$ for all $\xtr\in\{\xtr:(\xtr,\ytr)\in\Dtr\}$.
 \end{subproposition}
\begin{proof}
Intuitively, if the noise is zero, the data generating process is $\ytr=f(\xtr)+0$. Thus, if we observe $(\xtr,\ytr)\in \Dtr$, we know that $f(\xtr)=\ytr$ with zero uncertainty.
More formally, let $(\xtr,\ytr)\in \Dtr$ and define for some $\epsilon>0$
\begin{align*}
    U_\epsilon(\ytr)&:=(\ytr-\epsilon,\ytr+\epsilon)\\
    U_\epsilon(\Dtr)&:=\bigcup_{(x,y)\in\Dtr}U_\epsilon(x,y),
\end{align*}
 where $U_\epsilon(x,y)$ denotes an $\epsilon$-ball around $(x,y)\in\X\times\Y$. Furthermore, let $D$ be a random variable describing the data generating process assuming zero noise.
 Then it holds that
\begin{align*}
     &\PP(f(\xtr)\in U_\epsilon(\ytr)^c|D\in U_\epsilon(\Dtr) )\\
     =&\dfrac{\PP(D\in U_\epsilon(\Dtr)\wedge f(\xtr)\in U_\epsilon(\ytr)^c)}{\PP(D\in U_\epsilon(\Dtr))}\\
     =&\frac{0}{\PP(D\in U_\epsilon(\Dtr))}.
\end{align*}
Note that $\PP(D\in U_\epsilon(\Dtr))>0$ for every $\epsilon>0$, since by assumption there exists a function in the support of the prior that exactly fits through the training data.\footnote{Formally, $\exists f^*$ that fits exactly through $\Dtr$ with the property $\PP[f\in U_{\epsilon}(f^*)]>0$. Since, $0<\PP[f\in U_{\epsilon}(f^*)]<\PP(D\in U_\epsilon(\Dtr))$ (for the canonical $L_\infty$-topology) the claim follows. Given this one can also see that \Cref{thm:ZeroModel UncertaintyAtTrainingPoints} still holds true with the even weaker assumption $\PP(D\in U_\epsilon(\Dtr))>0$.} Defining the posterior
 \begin{align*}
    &\PP(f(\xtr)\neq\ytr|\Dtr):=\\
    &\lim_{\epsilon\to 0} \PP(f(\xtr)\in U_\epsilon(\ytr)^c|D\in U_\epsilon(\Dtr) ),
 \end{align*}
in the canonical way concludes the proof.
\end{proof}

 Even if theoretically, we know that $\sigmodel(\xtr)=0$ at all training points $\xtr$, in practice $\sigmodelhat(\xtr)\approx 0$ can be acceptable (due to numerical reasons).
 
 \subsubsection{Non-Zero Data Noise}\label{subsubsec:Non-Zero Data Noise}
 For non-zero data noise there is non-zero \emph{data-noise induced model uncertainty} at input training points.
 However, also for non-zero but small data noise the model uncertainty at input training points should not be significantly larger than the data noise. In fact, for GPs one can rigorously show that $\sigmodel(\xtr)\leq\signoise(\xtr)$ in the case of known $\signoise$. 
 
 \begin{subproposition}[GPs Model Uncertainty at training points]\label{thm:GPsModelUncertaintyAtTrainingPoints} Let $\Dtr$ be a set of training points.
 For a prior $f\sim \mathcal{GP}(m(\cdot),k(\cdot,\cdot))$ and fixed $\signoise$ it holds that
 \begin{align}
 \sigmodel(\xtr)\leq\signoise(\xtr),
 \end{align}
for all input training points $\xtr$.
\end{subproposition}
 \begin{proof}
 We prove the proposition by induction over the number of training points $\ntr$. For this let \[A^{\ntr}:=\{\xtr: (\xtr,\ytr) \in \Dtr\}.\]
 \begin{itemize}[leftmargin=*,topsep=0pt,partopsep=0pt, parsep=0pt]
     \item \textbf{Base case} $\ntr=1$:
     In this case $A^{\ntr}=A^1=\{\xtr_1\}$. Let $k:=k(\xtr_1,\xtr_1)$. Since
     \begin{align}
     &\sigmodel^2(\xtr_1)\stackrel{\eqref{eq:sigPredGPR}}{=}k-k\frac{1}{k+\signoise^2(\xtr_1)}\cdot k\le\signoise^2(\xtr_1)\\
     &\iff k^2\ge k^2-\signoise^4(\xtr_1),
     \end{align}
     the claim follows.
     \item $\ntr=m$:
     Let $K(A^{m},A^{m})$ be the Gram matrix and let \[P:=\left[K(A^{m},A^{m})+diag(\signoise(A^{m}))\right].\] We then assume for all $\xtr \in A^{m}$ that
     \begin{align}
         &\sigmodel^2(\xtr|A^m)\stackrel{\eqref{eq:sigPredGPR}}{=}\\
         &k(\xtr,\xtr)-k(\xtr,A^{m})^TP^{-1}k(\xtr,A^{m})\le\signoise^2(\xtr)
     \end{align}
     \item \textbf{Inductive step} $\ntr=m+1$:
     We now show that under the inductive assumption for any $\xtr \in A^{m+1}$ we have
     
    {\scriptsize
     \begin{align}
         &\sigmodel^2(\xtr)\nonumber\\
         &\stackrel{\eqref{eq:sigPredGPR}}{=}k(\xtr,\xtr)-\left(\begin{matrix}k(\xtr,A^{m})\\
         k(\xtr,\xtr_{m+1})
         \end{matrix}\right)^T\left(\begin{matrix}P & Q\\
         R & S \end{matrix}\right)^{-1}\left(\begin{matrix}k(\xtr,A^{m})\\
         k(\xtr,\xtr_{m+1})
         \end{matrix}\right)\label{eq:sigmod_m+1}\\
         &\le\signoise^2(\xtr)\nonumber
    \end{align}
    }%
    with
     \begin{align}
         &R:=(k(\xtr_1,\xtr_{m+1}),\ldots k(\xtr_m,\xtr_{m+1}))=k(\xtr_{m+1}, A^m)^T,\nonumber\\
         &Q:=(k(\xtr_1,\xtr_{m+1}),\ldots k(\xtr_m,\xtr_{m+1}))^T=R^T,\nonumber\\
         &S:=k(\xtr_{m+1},\xtr_{m+1})+\signoise^2(\xtr_{m+1}).\nonumber
     \end{align}
     Setting $k:= k(\xtr,\xtr), v:=k(\xtr,A^{m})$, and $w:=k(\xtr,\xtr_{m+1})$ \eqref{eq:sigmod_m+1} can be rewritten as
     \begin{align}
         k-\left(\begin{matrix}v\\
         w
         \end{matrix}\right)^T\left(\begin{matrix}\tilde{P} & \tilde{Q}\\
         \tilde{R} & \tilde{S} \end{matrix}\right)\left(\begin{matrix}v\\
         w
         \end{matrix}\right)\le\signoise^2(\xtr)\label{eq:inductive_step_to_show}
     \end{align}
     with submatrices $\tilde{P},\tilde{Q},\tilde{R},\tilde{S}$ as in \citep[(A.12)]{williams2006gaussian}. Furthermore, with $M:=(S-RP^{-1}Q)^{-1}$ as in \citep[(A.12)]{williams2006gaussian}, we get that  \eqref{eq:inductive_step_to_show} is equivalent to
     
     {\scriptsize
     \begin{align}
     &k-\big({v}^T\tilde{P}v+{v}^T\tilde{Q}w+w\tilde{R}v+w\tilde{S}w\big)\le\signoise^2(\xtr)\nonumber\\
     \iff&k-\big({v}^T{P}^{-1}v+{v}^T{P}^{-1}QMRP^{-1}v-{v}^T{P}^{-1}QMw\nonumber\\
     &\hspace{.5cm}-wMRP^{-1}v+wMw\big)\le\signoise^2(\xtr)\nonumber\\
     \iff &k-{v}^T{P}^{-1}v-(v^TP^{-1}Q-w)M({RP^{-1}v}-w)\le\signoise^2(\xtr)\nonumber\\
     \iff &\underbrace{k-{v}^T{P}^{-1}v}_{=\sigmodel^2(\xtr|A^m)}-(v^TP^{-1}Q-w)^2M\le\signoise^2(\xtr)\label{eq:inductive_step_to_show_2}
     \end{align}
     }%
     
     where the last line follows since $R^T=Q$ and $P$ symmetric. Note that
     \begin{align*}
         M&=\bigg(k(\xtr_{m+1},\xtr_{m+1})+\signoise^2(\xtr_{m+1})\\
         &\hspace{.5cm}-k(\xtr_{m+1},A^m)^TP^{-1}k(\xtr_{m+1},A^m)\bigg)^{-1}\\
         &=\big({\signoise^2(\xtr_{m+1})}+\sigmodel^2(\xtr_{m+1}|A^m)\big)^{-1}.
     \end{align*}
     With this, \eqref{eq:inductive_step_to_show_2} can be further reformulated as
     
     {\scriptsize
     \begin{align}
         &\sigmodel^2(\xtr|A^m)-\underbrace{\frac{\big(k(\xtr,A^m)^TP^{-1}k(\xtr_{m+1},A^m)-k(\xtr,\xtr_{m+1})\big)^2}{{\signoise^2(\xtr_{m+1})}+\sigmodel^2(\xtr_{m+1}|A^m)}}_{\ge 0}\nonumber\\
         &\le\signoise^2(\xtr).\label{eq:inductive_step_to_show_3}
     \end{align}
     }%
     
     First, for $\xtr\in A^m$, \eqref{eq:inductive_step_to_show_3} holds true by assumption. Second, for $\xtr=\xtr_{m+1}$ we obtain
     
     {\scriptsize
     \begin{align*}
        &\sigmodel^2(\xtr_{m+1}|A^m)-\frac{\big( \sigmodel^2(\xtr_{m+1}|A^m)\big)^2}{{\signoise^2(\xtr_{m+1})+\sigmodel^2(\xtr_{m+1}|A^m)}}\le\signoise^2(\xtr_{m+1})\\
        \iff&\sigmodel^4(\xtr_{m+1}|A^m)-\signoise^4(\xtr_{m+1})\le\sigmodel^4(\xtr_{m+1}|A^m)\\
        \iff&-\signoise^4(\xtr_{m+1})\le 0.
     \end{align*}
     }%
     
     Thus, \eqref{eq:inductive_step_to_show} holds true for any $\xtr\in A^{m+1}$.\qedhere
 \end{itemize}
 \end{proof}
 In fact, Proposition~\ref{thm:GPsModelUncertaintyAtTrainingPoints} does not come as a surprise: Even if we only observe one training point~$(\xtr,\ytr)$ and ignore all our prior knowledge by using a flat \enquote{uninformative} improper prior $p(f(\xtr))\propto 1$, this results in $\sigmodel(\xtr)=\signoise(\xtr)$.
 Introducing additional information, e.g., observing more additional training points and introducing additional prior information (such as smoothness assumptions instead of a flat uninformative prior), typically reduces model uncertainty further. Thus, we believe that $\sigmodel(\xtr)\leq\signoise(\xtr)$ holds for most reasonable priors.

Finally, note that Proposition~\ref{thm:ZeroModel UncertaintyAtTrainingPoints} and Proposition~\ref{thm:GPsModelUncertaintyAtTrainingPoints} hold true for \emph{every} prior respectively \emph{every Gaussian} process prior as long as there exists an $f$ in the support of this prior which explains the observed training points (even if this prior is strongly misspecified).
For example this assumption is obviously fulfilled for the prior of Gaussian distributed weights of an overparameterized NN (BNN).

 \subsubsection{Why Does MC Dropout Strongly Violate \ref{itm:Axioms:ZeroUncertaintyAtData} ?}\label{subsubsec:MCDropoutviolatesD2}
 In \Cref{fig:D41dJakobhfun}, MC dropout (MCDO) predicts for every input training point $\sigmodelhat(\xtr_i)>100\signoise$. Thus, if $\sigmodelhat(\xtr_i)$ was correctly calculated as posterior model uncertainty, this would be an practically unobservable event as long as $f$ actually comes from this prior ($\href{https://www.wolframalpha.com/input/?i=\%281+-+CDF\%5BNormalDistribution\%5B0\%2C+1\%5D\%2C+100\%5D\%29+*+2}{\PP[\left|\ytr_i-f(\xtr_i)\right|>100\signoise]< 10^{-2173}}$). Therefore, this is clear statistical evidence that MCDO severely fails to estimate posterior model uncertainty at training points. This can have one of the following three reasons:
 \begin{enumerate}[leftmargin=*,topsep=0pt,itemsep=0pt,partopsep=0pt, parsep=0pt]
     \item\label{itm:MCdropoutCannotApproximatePosterior} MCDO severely fails in correctly approximating the posterior given its prior (i.i.d. Gaussian on weights).
     \item\label{itm:MCdropoutCannotApproximatePosterior2} MCDO's prior does not fit to the data generating process at all.
     \item\label{itm:MCdropoutCannotApproximatePosterior3} During our experiments we very often observed very extreme events that should only happen with probabilities smaller than $10^{-2000}$.
 \end{enumerate}
 We agree with prior work \citep{gal2016dropout,blundell2015weight} that a Gaussian prior on the weights of a NN, i.e., the prior mentioned in \Cref{itm:MCdropoutCannotApproximatePosterior2}, is a very reasonable assumption.  Note that NOMU can also be seen as a heuristic to approximate the posterior model uncertainty given exactly the same prior (see \Cref{subsec:Pointwise Uncertainty Bounds}). 
 Therefore, since \Cref{itm:MCdropoutCannotApproximatePosterior3} can be ruled out, we can conclude that MCDO's problem is \Cref{itm:MCdropoutCannotApproximatePosterior}.

 \paragraph{MC Dropout's Failure in Approximating the Posterior} \Cref{tab:Appendix:bnn_uniform} and \Cref{tab:Appendix:bnn_gauss} show that even though we generate the ground truth function from the \emph{same} prior assumed by MCDO (and also assumed by most BNN algorithms), NOMU significantly outperforms MCDO. This empirically shows (with the help of \Cref{thm:appendix_dkl_approximation}) that (i) NOMU is able to better approximate posterior BNN-credible bounds than MCDO in terms of average Kullback-Leibler divergence $\adkl$ (including further popular variational BNN approximations from \citep{graves2011practical,blundell2015weight,hernandez2015probabilistic}, which themselves are outperformed by MCDO) and (ii) MCDO's variational approximation algorithm severely fails in approximating the targeted posterior.

\subsubsection{Importance of \ref{itm:Axioms:ZeroUncertaintyAtData} in BO} Especially in Bayesian optimization (BO) it is particularly important to fulfill \ref{itm:Axioms:ZeroUncertaintyAtData} as much as possible, since \ref{itm:Axioms:ZeroUncertaintyAtData} helps a lot to prevent the BO-algorithm from getting stuck in local maxima.
For NNs, we often observed that at the $i'$-th step, the mean prediction $\fhat$ is maximized/minimized either at the boundary or exactly at a training point with the largest/smallest function value observed so far $\max_{i\in\fromto{i'}} f(x_i)$/$\min_{i\in\fromto{i'}}f(x_i)$ (see \Cref{fig:1dbounds}). 
In the latter case, without model uncertainty (or with almost constant model uncertainty as is the case in MC dropout), one would query all future function evaluations at exactly this point without learning anything new.
E.g., consider the situation of \Cref{subfig:forrester} when minimizing the \emph{Forrester} function. Each new function evaluation of MC Dropout would be sampled at an already observed training point $x\approx0.4$.
This intuitively explains why estimating model uncertainty precisely at the training data points is especially important in BO and why it can be very problematic in BO, if the model uncertainty does not decrease sufficiently at the training data points.
To summarize, \ref{itm:Axioms:ZeroUncertaintyAtData} strongly influences the acquisition function in a direction that discourages the algorithm from choosing the same point again and \ref{itm:Axioms:ZeroUncertaintyAtData} together with \ref{itm:Axioms:largeDistantLargeUncertainty} can prevent the BO-algorithm from getting stuck (see also \Cref{subsubsec:Calibration}).

 \subsubsection{Dominating Data Noise}\label{subsubsec:Large data noise}
 In the case of dominating data noise uncertainty $\signoise\gg 0$, the model uncertainty $\sigmodel$ should not be small at input training points (only if one observes a very large amount of input training points very close to a input training point $\xtr$ the model uncertainty should become small.) However, in this paper, we do not focus on the case of large data noise uncertainty, but on the case of negligible or zero data noise. In particular, \ref{itm:Axioms:ZeroUncertaintyAtData} is only formulated for this case. 
 
 \subsubsection{How Does NOMU Fulfill \ref{itm:Axioms:ZeroUncertaintyAtData}?}
  Recall, that we train NOMU by minimizing 
  \begin{align}\label{eq:objective}
      \Lmu(\NN_{\theta})+\lambda\twonorm[{\theta}]^2,
  \end{align} 
  where the NOMU loss $L^\hp(\NN_\theta)$ is defined as:
  
  \begin{align}\label{eq:lmu2appendix}
  L^\hp(\NN_\theta):=&\underbrace{\sum_{i=1}^{\ntr}(\fhat(\xtr_i
  )-\ytr_i)^2}_{(a)}+ \,\musqr\cdot\underbrace{\sum_{i=1}^{\ntr}\left(\sigmodelhatraw(\xtr_i)\right)^2}_{(b)}\\ &+\muexp\cdot\underbrace{\frac{1}{\lambda_d(\X)}\int_{\X}e^{-\cexp\cdot \sigmodelhatraw(x)}\,dx}_{(c)}.
\end{align}
 Then, the following proposition holds:
 \begin{subproposition}\label{prop:NOMUD2} Let $\lambda,\muexp,\cexp\in \Rpz$ be fixed and let $\sigmodelhat$ be NOMU's model uncertainty prediction. Then, it holds that $\sigmodelhat(\xtr_i)$ converges to $\sigmin$ for $\musqr\to \infty$ for all input training points $\xtr_i$, where $\sigmin\ge0$ is an arbitrarily small constant modelling a minimal model uncertainty used for numerical reasons.
 \end{subproposition}
 \begin{proof}
 By the definition of $L^\hp(\NN_\theta)$, i.e., since (b) dominates the loss function if $\musqr\to\infty$, it follows that $\sigmodelhatraw(\xtr_i)=0$. More precisely, for the NN $\NN_{\theta^*}=(\fhat^*, \sigmodelhatraw
 ^*)$ with parameters $\theta^*$ that minimize \eqref{eq:objective} it holds that
 \resizebox{\columnwidth}{!}{\parbox{1\columnwidth}{
 \begin{align*}
     &\Lmu(\NN_{\theta^*})+\lambda\twonorm[{\theta^*}]^2\le\Lmu(0)+\lambda\twonorm[0]^2\\
      &\qquad=\sum_{i=1}^{\ntr}(\ytr_i)^2+\muexp\cdot 1\\ 
     \iff& \sum_{i=1}^{\ntr}(\fhat^*(\xtr_i
  )-\ytr_i)^2+ \musqr\cdot\sum_{i=1}^{\ntr}\left(\sigmodelhatraw^*(\xtr_i)\right)^2+\\
  &\muexp\cdot\frac{1}{\lambda_d(\X)}\int_{\X}e^{-\cexp\cdot \sigmodelhatraw^*(x)}\,dx+\lambda\twonorm[{\theta^*}]^2\le \sum_{i=1}^{\ntr}(\ytr_i)^2+\muexp\\
  \iff&\musqr\cdot\sum_{i=1}^{\ntr}\left(\sigmodelhatraw^*(\xtr_i)\right)^2\le\sum_{i=1}^{\ntr}(\ytr_i)^2+\muexp=:C
 \end{align*}
 }
 }%
 
 where for fixed parameters $\lambda,\muexp,\cexp\in \Rpz$, $C>0$ is a constant.
 Assume now that for $\sigmodelhatraw^*$ does not vanish at all training data points for $\musqr$ to infinity, i.e., that there exists an $\epsilon>0$ such that for every $\musqr$ large enough $\sum_{i=1}^{\ntr}\left(\sigmodelhatraw^*(\xtr_i)\right)^2>\epsilon$. This however implies
 \begin{align*}
     \musqr\cdot\epsilon< C \iff \musqr<\frac{C}{\epsilon}\quad \forall \musqr \text{ large enough},
 \end{align*}
 which yields a contradiction. Thus, $\lim_{\musqr\to\infty}\sigmodelhatraw^*(\xtr)=0$ for all training input points $\xtr$.
 Finally, by \cref{eq:modelUncertaintyPrediction} it follows that $\sigmodelhat(\xtr_i)=\sigmin$.
 \end{proof}
 
 Note that even for a finite (sufficiently large) $\musqr$, the raw model uncertainty $\sigmodelhatraw$ converges to zero as $\lambda$ goes to zero ($\frac{\musqr}{\lambda}\to \infty$), if the model is sufficiently over-parameterized. Empirically one can see in \Cref{fig:D41dJakobhfun,fig:1dboundsLevy,fig:irradianceNOMU,fig:1dbounds,fig:2dStyblinski_NOMU,fig:irradiance,fig:D4,fig:2D_D4} that NOMU fulfills \ref{itm:Axioms:ZeroUncertaintyAtData} with a high precision for our choice of hyper-parameters.
 
 \subsection{Desideratum \ref{itm:Axioms:largeDistantLargeUncertainty}}
 We first consider the case of zero (or negligible) data noise $\signoise\approx0$ and then discuss possible extensions to settings with non-zero data noise.
 \subsubsection{Zero Data Noise}\label{subsec:D3ZeroNoise}
The notion of distance used in \ref{itm:Axioms:largeDistantLargeUncertainty} heavily depends on the specific application (i.e., on the prior used in this application). More concretely, there are the following two \enquote{hyperparameters}.
 \begin{enumerate}[leftmargin=*,topsep=0pt,partopsep=0pt, parsep=0pt]
 \item First, the metric\footnote{We use the term \enquote{metric} to describe a general \href{https://en.wikipedia.org/wiki/Pseudometric_space}{pseudometric}.} $d:\X\times\X\to\Rpz$ on $\X$ we use to measure distances can heavily depend on the prior for the specific application. For example, in the case of image recognition, two pictures that are only slightly shifted can be seen as very close to each other even if the Euclidean distance of their pixel-values is quite high.\footnote{For example, if one sees a $1920\times1080$-pixel image, which is perfectly recognizable as a cat, every 10-pixel shift of this picture is also recognizable as a cat with almost no uncertainty (even though this cannot be proven mathematically). Thus, it is very desirable to predict very small model uncertainty~$\sigmodelhat(x)$ for every image $x\in X$ which is only shifted by less than 10 pixel from at least one noiselessly labeled training image $\xtr$.} If one uses a CNN-architecture in NOMU this prior belief on $d$ is approximately captured. The successful generalization properties of many different network architectures can be explained precisely by their use of application-dependent \emph{non-Euclidean} metrics \citep{bronstein2017geometric}.
 (Additionally, instead of fixing $d$ apriori, further aspects of the metric $d$ can be learned from the training data as we discuss in detail in \Cref{subsec:A note on Desideratum D4}.) 
 
 \item\label{itm:dxDtr} Second, even if we can precisely write down a metric $d:\X\times\X\to \Rpz$, a priori it is not clear how to define the distance $\tilde{d}:\X\times 2^{\X}\to\Rpz$ between a point $x$ and the input training points from $\Dtrx:=\left\{\xtr:(\xtr,\ytr)\in\Dtr\right\}$. Both common definitions $\tilde{d}(x,\Dtrx):=\inf_{z\in\Dtrx}d(x,z)$ and $\tilde{d}(x,\Dtrx):=\inf_{z\in\text{Conv}(\Dtrx)}d(x,z)$, where $\text{Conv}(\cdot)$ denotes the convex hull, are inappropriate choices for $\tilde{d}$.\footnote{E.g., for GPs, none of these two classical notions of distance between a point and a set is entirely applicable (see \Cref{eq:sigPredGPR}). An appropriate choice of $\tilde{d}$ should be a compromise between these two notions.} In \Cref{itm:Axioms:largeDistantLargeUncertainty}, we consider a point $x$ closer to the input training points if it is \enquote{surrounded} by input training points in all directions, as opposed to a point $x$ which only has close input training points in some directions and there is a large range of directions without any close input training points. This implies that, for example,
 \begin{enumerate}[(i)]%
 \item\label{itm:closeToNoiselss} very close to noiseless input training points that are surrounded by many other noiseless input training points there is very little model uncertainty.
 \item\label{itm:extrapolationVsInterpolation} for extrapolation one typically has more uncertainty than for interpolation.
 \item\label{itm:farAwayFromConvexHull} far away from the convex hull of the training points model uncertainty is very high. 
 \item\label{itm:Gaps} also within the convex hull of the training data, model uncertainty is high within big gaps in-between training points.
 \end{enumerate}
 \Cref{fig:2dStyblinski_NOMU}
shows how well NOMU fulfills these properties of $\tilde{d}$ similarly to a GP (see \Cref{subfig:2d_Styblinski_GPR}).
\end{enumerate}
 
\subsubsection{Non-Zero Homoscedastic Data Noise} If there is homoscedastic non-zero data noise $\signoise(x)\equiv \signoise>0$, it is important that the \enquote{distance} $\tilde{d}$ of $x$ to the input training points is not minimal if it exactly equals one of the input training points. Instead, one should use a notion of distance $\tilde{d}$ that can even get smaller if there are multiple input training points at $x$ or very close to $x$.

\subsubsection{Non-Zero Heteroscedastic Data Noise} One can also extend \ref{itm:Axioms:largeDistantLargeUncertainty} to heteroscedastic settings. In that case, the used notion of \enquote{distance}~$\tilde{d}$ of $x$ to the input training points needs to be weighted by the precision of the input training points, i.e., if $x$ is close to multiple input training points with low data noise~$\signoise(\cdot)$ you consider $x$ \enquote{closer} to the input training points than if $x$ is close to multiple input training points with high data noise.

\subsubsection{Example for \texorpdfstring{$\tilde{d}$}{tilde d} Based on GPs}
In this section, we give the concrete example of Gaussian process regression (GPR) from \Cref{subsubsec:GaussianProcess} in which  $\tilde{d}$ from \ref{itm:Axioms:largeDistantLargeUncertainty} can be written down explicitly in closed form.

For any arbitrary metric $d$ on $X$, consider for instance the kernel $k(x_i,x_j)=e^{-d(x_i,x_j)^2}$. Then, $\tilde{d}(x,\Dtrx)=\sigmodelhat(x|\Dtrx)$, with posterior model uncertainty $\sigmodelhat$ from \cref{eq:sigPredGPR}. While this is one of the simplest possible ways to define $\tilde{d}$, alternatively one could also define it differently if it shares similar qualitative properties.\footnote{
For instance, $\tilde{d}$ could also be defined 
based on a kernel of the form $k(x_i,x_j)=g(d(x_i,x_j))$ with a monotonically decreasing function $g$, e.g., a Matérn-typed kernel.}

Why do we still consider it interesting to formulate \ref{itm:Axioms:largeDistantLargeUncertainty} \emph{vaguely}, given that there is already such a precise formula as is the case for GPs?
\begin{enumerate}[leftmargin=*,topsep=0pt,partopsep=0pt, parsep=0pt]
\item The GP's formula only holds true for the specific prior of a GP. We however, want to formulate desiderata that capture the most essential properties of credible bounds that
almost all reasonable priors have in common.

\item We want to provide some easy to understand intuition for \ref{itm:Axioms:largeDistantLargeUncertainty}: It might be challenging to see directly from the GP's formula \eqref{eq:sigPredGPR} how the posterior model uncertainty qualitatively behaves as visualized in \Cref{subfig:2d_Styblinski_GPR}.
\end{enumerate}

To summarize, both the exact notion of distances $d, \tilde{d}$ \emph{and} the exact rate of how model uncertainty increases with increasing distance to the input training points depend on one's prior belief. However, \Cref{itm:Axioms:largeDistantLargeUncertainty} gives a qualitative description of properties that most reasonable (generic) priors have in common (see \crefrange{itm:closeToNoiselss}{itm:Gaps}).

\subsubsection{How Does NOMU Fulfill \ref{itm:Axioms:largeDistantLargeUncertainty}?}

Recall, that we train NOMU by minimizing 
  \begin{align}
      \Lmu(\NN_{\theta})+\lambda\twonorm[{\theta}]^2,
  \end{align} where the NOMU loss $L^\hp(\NN_\theta)$ is defined as:
  \begin{align}
  L^\hp(\NN_\theta):=&\underbrace{\sum_{i=1}^{\ntr}(\fhat(\xtr_i
  )-\ytr_i)^2}_{\terma{}}+ \,\musqr\cdot\underbrace{\sum_{i=1}^{\ntr}\left(\sigmodelhatraw(\xtr_i)\right)^2}_{\termb{}}\\ &+\muexp\cdot\underbrace{\frac{1}{\lambda_d(\X)}\int_{\X}e^{-\cexp\cdot \sigmodelhatraw(x)}\,dx}_{\termc{}}.
\end{align}

The interplay of \termb{}, \termc{}, and regularization promotes \ref{itm:Axioms:largeDistantLargeUncertainty} (note that the behaviour of $\sigmodelhatraw$ directly translates to the behaviour of $\sigmodel$):
Term~\termc{} pushes $\sigmodelhatraw$ towards infinity across the whole input space $\X$. However, due to the counteracting force of \termb{} as well as regularization, $\sigmodelhatraw$ increases continuously as you move away from the training data -- see for example \Cref{fig:1dbounds} and \Cref{fig:2dStyblinski_NOMU} (or any other plot showing NOMU, i.e., \Cref{fig:D41dJakobhfun,fig:1dboundsLevy,fig:irradianceNOMU,fig:irradiance,fig:D4,fig:2D_D4}).
In \Cref{fig:2dStyblinski_NOMU}, one can see how NOMU fulfills the properties \ref{itm:closeToNoiselss}\crefrangeconjunction\ref{itm:Gaps} of $\tilde{d}:X\times 2^X\to\Rpz$ mentioned in \Cref{subsec:D3ZeroNoise}.
In \Cref{fig:D4,fig:2D_D4}, one can observe how NOMU behaves when a non-stationary metric $d_\text{\Cref{fig:D4}}(x,x')\neq|x-x'|$ respectively non-stationary non-isotropic metric $d_\text{\Cref{fig:2D_D4}}(x,x')\neq\|x-x'\|_2$ is used (because $d_\text{\Cref{fig:D4}}$ and $d_\text{\Cref{fig:2D_D4}}$ were learned from the data as desired by \ref{itm:Axioms:irregular} in these examples).

The hyperparameters $\muexp$ and $\cexp$ control the size and shape of the UBs. Concretely, the larger $\muexp$, the wider the UBs; the larger $\cexp$, the narrower the UBs at points $x$ with large $\sigmodelhat(x)$ and the wider the UBs at points $x$ with small $\sigmodelhat(x)$.

Finally, we give some intuition that if CNNs are used for the two sub-networks in NOMU's architecture, \ref{itm:Axioms:largeDistantLargeUncertainty} will be fulfilled with respect to an almost shift-invariant metric $d$:
In the noiseless setting, we can choose $\musqr$ large enough such that \ref{itm:Axioms:ZeroUncertaintyAtData} is fulfilled, so that we have $\sigmodelhat(\xtr)\approx 0$ at any training input point $\xtr$. Regularized CNNs have the property that if you slightly shift the input the output barely changes. So if $x$ can be obtained from $\xtr$ by slightly shifting it, the CNN-output $\sigmodelhat(x)\approx\sigmodelhat(\xtr)\approx0$ also does not move too far away from zero. Only if you move further away with respect to the almost shift-invariant metric $d$, the CNN-output~$\sigmodelhat$ is able to move further away from zero. The same principle can also be used for other geometric NNs (e.g., graph neural networks (GNNs)) which correspond to different (non-Euclidean) metrics \citet{bronstein2017geometric}.

\subsection{Desideratum \ref{itm:Axioms:irregular}}\label{subsec:A note on Desideratum D4}

A priori, it is often not obvious which metric $d$ to choose in \hyperref[itm:Axioms:largeDistantLargeUncertainty]{D3} to measure distances. In many applications, it is therefore best to \textit{learn} this metric from the training data (as explained in \cref{footnote:LearningTheMetricISImportant} on \cpageref{footnote:LearningTheMetricISImportant}).

In the following section, we present visualizations of \ref{itm:Axioms:irregular} for all benchmark algorithms in easy to understand, low dimensional settings.

\subsubsection{Visualization of \ref{itm:Axioms:irregular}}\label{subsec:VisualizeAxioms:irregular}

\paragraph{1D} In order to visualize \hyperref[itm:Axioms:irregular]{D4} and show how for NOMU the mean prediction impacts it's model uncertainty prediction we conduct the following experiment. We sample 16 \emph{equidistant} noiseless training points of a trend-adjusted version of \emph{Sine 3}. We then train NOMU (hyperparameters are as in \Cref{subsec:ConfigurationDetailsofBenchmarksRegression} with $\musqr=0.5$, $\sigmin=10^{-4}$, regularization parameter $10^{-4}$ on the $\sigmodelhatraw$-network, and number of training epochs $2^{12}$) and compute the corresponding UBs. \Cref{fig:D4} shows that NOMU UBs are wider (cp. the dotted blue line) in those areas of the input space where small changes of $x$ lead to large variation in the target ($\approx x\ge 0$) compared to areas without large variation in the target ($\approx x\le 0$). This effect is present even though the input training points are sampled from an \emph{equidistant} grid, and thus isolates the effect of \hyperref[itm:Axioms:irregular]{D4}.
\begin{figure}[b!]
    \begin{center}
    \centerline{\includegraphics[width=1\columnwidth]{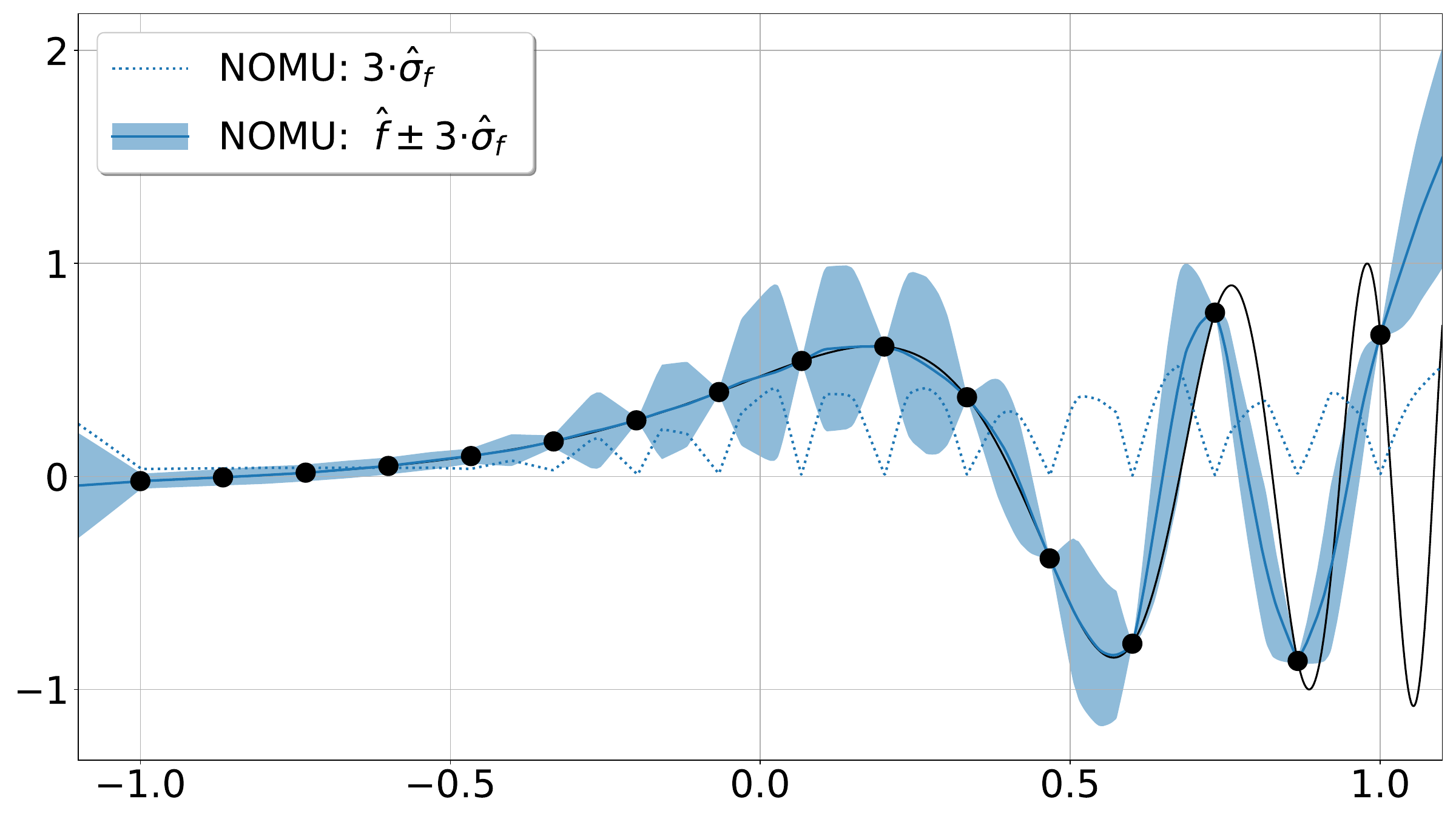}}
    \caption{Visualisation of \ref{itm:Axioms:irregular}.}
    \label{fig:D4}
    \end{center}
\vskip -0.2in
\end{figure}

\paragraph{2D} Analogously, we visualize \hyperref[itm:Axioms:irregular]{D4} for two-dimensional input by training NOMU on 16 training points sampled on an \emph{equidistant} $4{\times}4$-grid and evaluated at the two-dimensional extension of the \emph{Step} function, i.e.,
\begin{align}
f=%
\R^2\to\R:(x_1,x_2)\mapsto
\begin{cases}
-1 &\text{if } x_1<0\\
1 &\text{if } x_1\geq 0.
\end{cases}
\end{align}

Here, \hyperref[itm:Axioms:irregular]{D4} can be interpreted as follows: imagine we do not have any prior knowledge of whether $x_1$ or $x_2$ is more important for predicting the unknown function $f$. However, when NOMU observes the 16 training points it should be able to learn that $x_1$ is more important for the model prediction than $x_2$, \emph{and} that the function is more regular/predictable far away from $\{x_1\approx 0\}$. \hyperref[itm:Axioms:irregular]{D4} requires in this example that feature $x_1$ should have a higher impact than feature $x_2$ also on the model uncertainty prediction.
If a model for UBs did not incorporate \hyperref[itm:Axioms:irregular]{D4}, we would expect the uncertainty in this example to fulfill $\sigmodelhat\left((x_1,x_2)\right)=\sigmodelhat\left((x_2,x_1)\right)$ because of the equidistant grid of the training points (this is indeed the case for GPs, see Figure \ref{subfig:D4_2d_GPR}).

For NOMU however, we have very good control on how strongly we enforce \hyperref[itm:Axioms:irregular]{D4}, e.g., we can strengthen \hyperref[itm:Axioms:irregular]{D4} by increasing the L2-regularization of the hidden layers in the $\sigmodelhatraw$-network and/or decreasing the size of the $\sigmodelhatraw$-network.

\paragraph{NOMU: Visualization of \hyperref[itm:Axioms:irregular]{D4} in 2D} In \Cref{fig:2D_D4}, we present the estimated model uncertainty $\sigmodelhat$ obtained for \emph{different} hyperparameters of the $\sigmodelhatraw$-network with \emph{fixed} $\fhat$-architecture among all four subplots. Thus, \Cref{fig:2D_D4} shows how \hyperref[itm:Axioms:irregular]{D4} realizes in different magnitudes.
In \Cref{fig:D4_2d_same_same}, we use the same hyperparameters for the $\sigmodelhatraw$-network as for the $\fhat$-network.
In \ref{fig:D4_2d_4_same}, we only increase the L2-regularization of the $\sigmodelhatraw$-network. In \Cref{subfig:D4_2d_same_4shallow}, we only decrease the size of $\sigmodelhatraw$-networks. In \Cref{subfig:D4_2d_4_4shallow}, we combine both, i.e., we increase the L2-regularization of the $\sigmodelhatraw$-network \emph{and} decrease the size of the  $\sigmodelhatraw$-network. While \hyperref[itm:Axioms:irregular]{D4} is barely visible in \Cref{fig:D4_2d_same_same}, it is clearly visible in \Cref{fig:D4_2d_4_same,subfig:D4_2d_same_4shallow,subfig:D4_2d_4_4shallow}. In \Cref{fig:D4_2d_4_same,subfig:D4_2d_same_4shallow,subfig:D4_2d_4_4shallow}, we observe that the estimated model uncertainty $\sigmodelhat$ grows faster in horizontal directions (corresponding to changes in $x_1$) than in vertical directions. In \Cref{fig:D4_2d_4_same,subfig:D4_2d_same_4shallow,subfig:D4_2d_4_4shallow}, we further observe that the estimated model uncertainty $\sigmodelhat$ is larger around $\{x_1\approx 0\}$ than far away from this region. The magnitude of both these effects increases from  
\Cref{fig:D4_2d_4_same} to \Cref{subfig:D4_2d_4_4shallow}.
Both of these effects can also be observed for MC dropout (MCDO) and deep ensembles (DE) (see \Cref{subfig:D4_2d_DO} and \Cref{subfig:D4_2d_DE}).

\paragraph{Benchmarks: Visualization of \hyperref[itm:Axioms:irregular]{D4} in 2D} In \Cref{fig:D4_2d_benchmarks}, we present uncertainty plots of all benchmark methods. We can see that deep ensembles (DE) gives high preference to capturing \hyperref[itm:Axioms:irregular]{D4}, even though its estimated model uncertainty still is subject to some randomness with non-uniform patterns for $x_1\in[-0.25,0.25]$ (\Cref{{subfig:D4_2d_DE}}). Moreover, MC dropout (MCDO) also captures higher model uncertainty for $x_1\in[-0.25,0.25]$ as desired by \hyperref[itm:Axioms:irregular]{D4}, but it does not fulfill \hyperref[itm:Axioms:ZeroUncertaintyAtData]{D2} (\Cref{{subfig:D4_2d_DO}}). The Gaussian process (GP) with RBF kernel does not account for \hyperref[itm:Axioms:irregular]{D4} (\Cref{{subfig:D4_2d_GPR}}), which directly follows from the definition. Similarly to deep ensembles (DE), hyper deep ensembles (HDE) and HDE* strongly capture \hyperref[itm:Axioms:irregular]{D4} but show even more random behaviour. This randomness is visible most prominently along $x_1=0$ where one should observe large model uncertainty, whereas their estimated model uncertainty is surprisingly close to $0$.

\begin{figure}[H]
\vskip 0.1cm
\centering
\captionsetup[subfloat]{font=small,labelfont=small}
\subfloat[Same L2-regularization on the $\sigmodelhatraw$-network and $\fhat$-network ($\lambda=10^{-8}$).
\label{fig:D4_2d_same_same}]{%
\includegraphics[trim= 10 10 80 10, clip, width=.9\columnwidth]{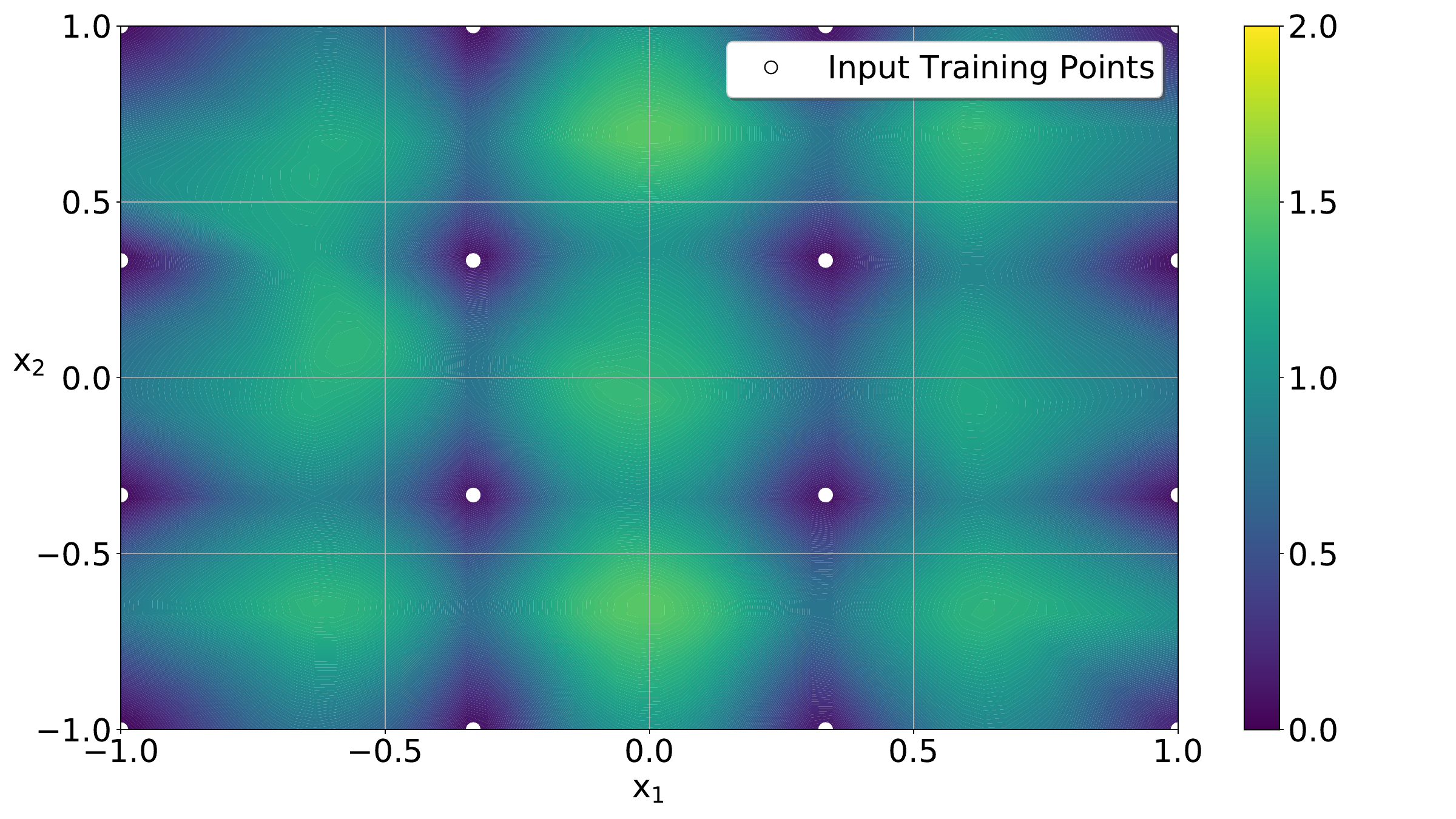}
}%

\subfloat[Larger L2-regularization on the $\sigmodelhatraw$-network ($\lambda=10^{-4}$) than on the $\fhat$-network ($\lambda=10^{-8}$).
\label{fig:D4_2d_4_same}]{%
\includegraphics[trim= 10 10 80 10, clip, width=.9\columnwidth]{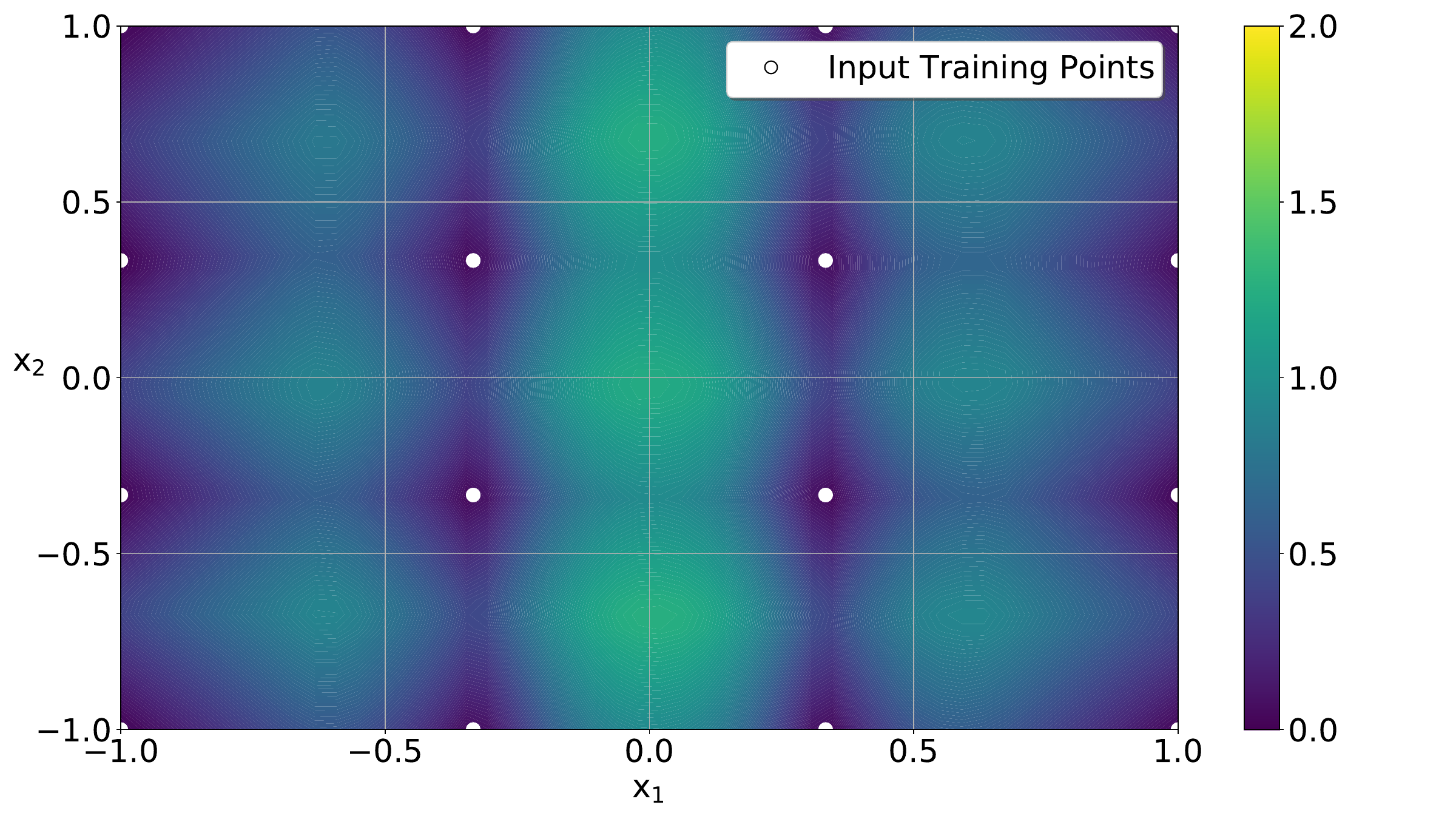}
}%

\subfloat[Shallow $\sigmodelhatraw$-network consisting of 4 hidden nodes.
\label{subfig:D4_2d_same_4shallow}]{%
\includegraphics[trim= 10 10 80 10, clip, width=.9\columnwidth]{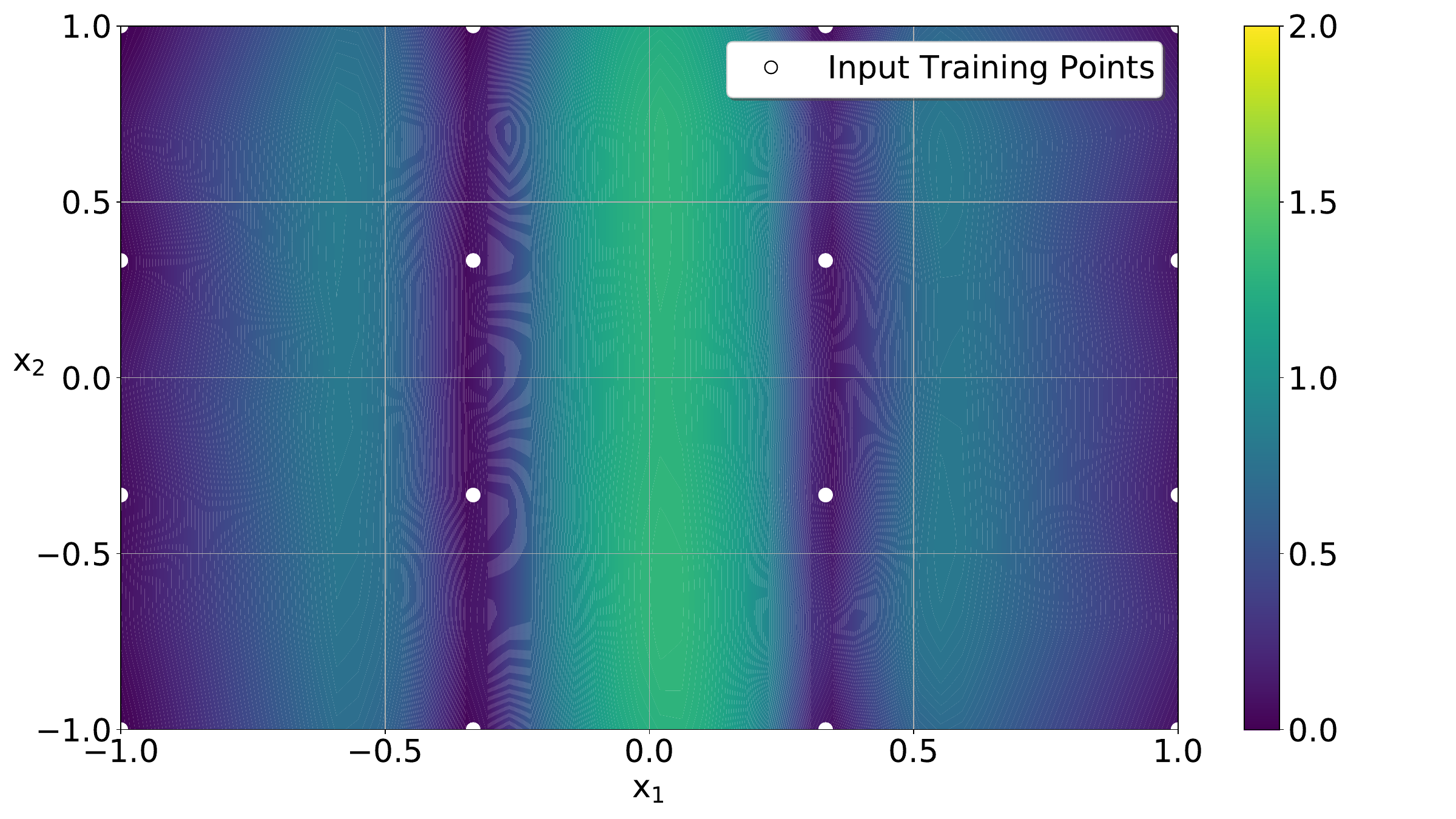}
}%

\subfloat[Shallow $\sigmodelhatraw$-network consisting of 4 hidden nodes \emph{and} larger regularization of $\lambda=10^{-4}$ on $\sigmodelhatraw$-network.
\label{subfig:D4_2d_4_4shallow}]{%
\includegraphics[trim= 10 10 80 10, clip, width=.9\columnwidth]{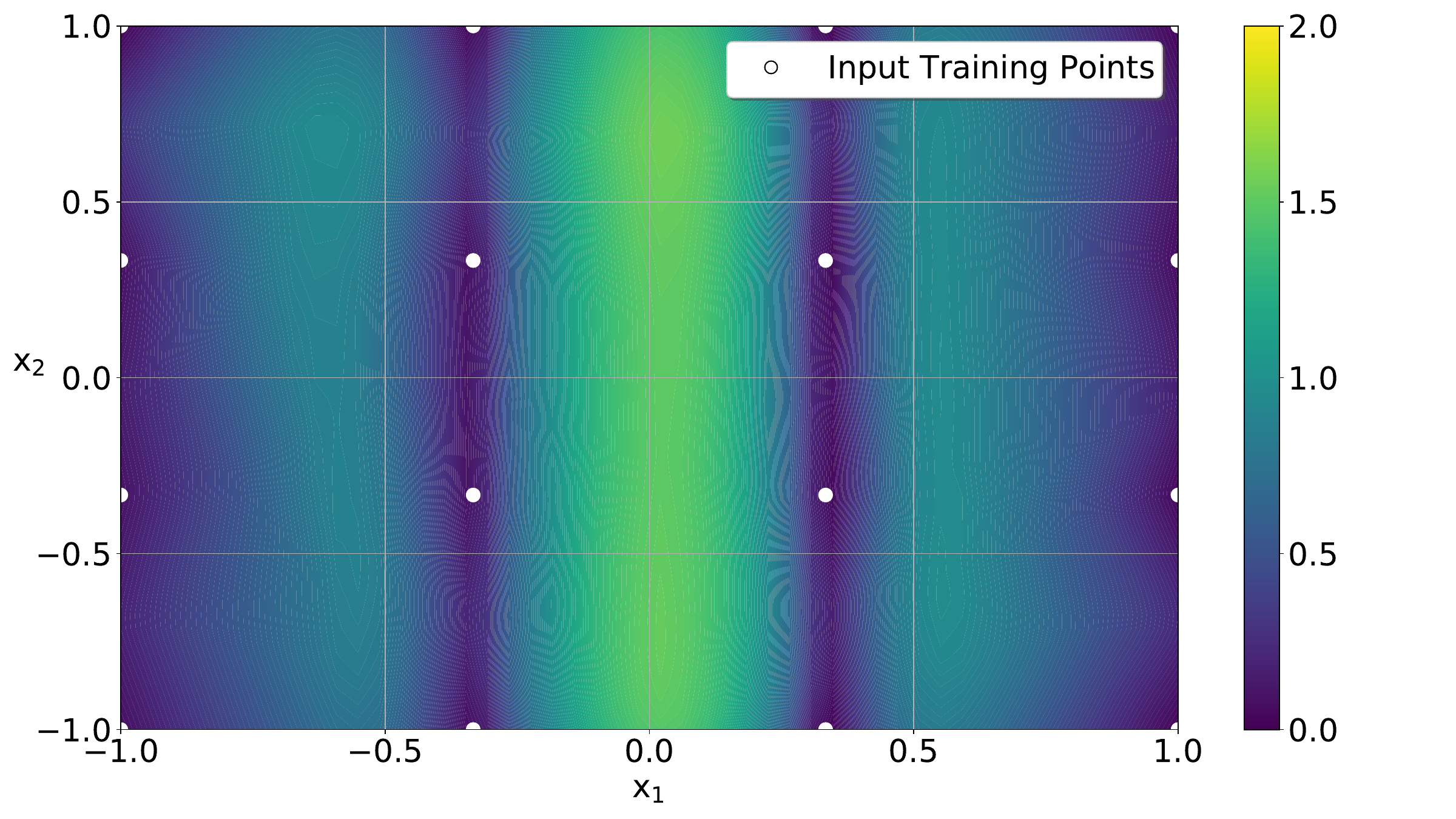}
}%
\caption{Estimated model uncertainty $\sigmodelhat$ of NOMU: visualizing \hyperref[itm:Axioms:irregular]{D4} in 2D.}
\label{fig:2D_D4}
\end{figure}

\begin{figure}[H]
\centering
\captionsetup[subfloat]{font=small,labelfont=small}
\subfloat[GP (c=15)\label{subfig:D4_2d_GPR}]{%
\includegraphics[trim= 10 10 80 10, clip, width=0.9\columnwidth]{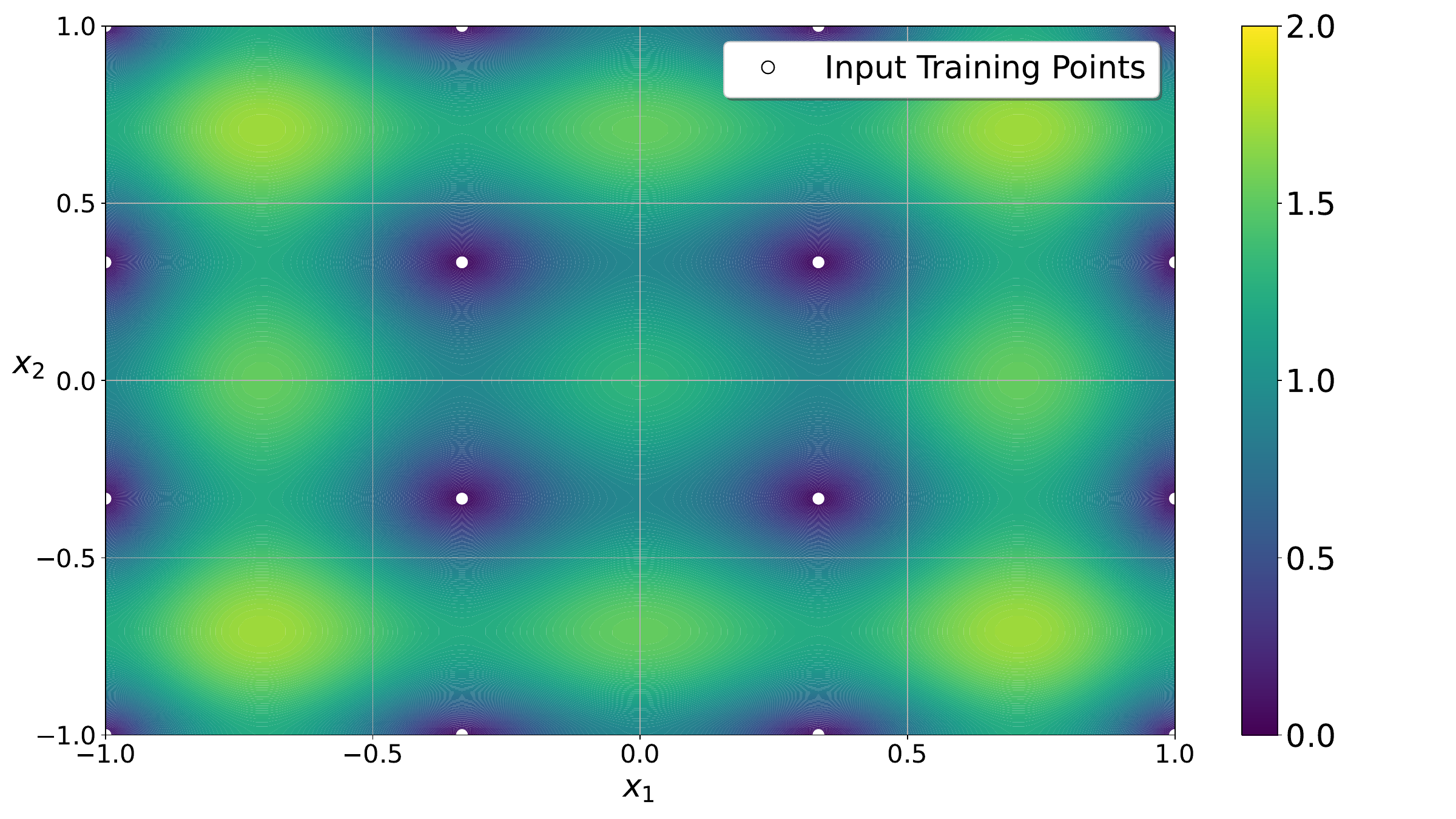}\label{fig:D4GP2d}
}

\subfloat[{MCDO} (c=20)\label{subfig:D4_2d_DO}]{%
\includegraphics[trim= 10 10 80 10, clip, width=.9\columnwidth]{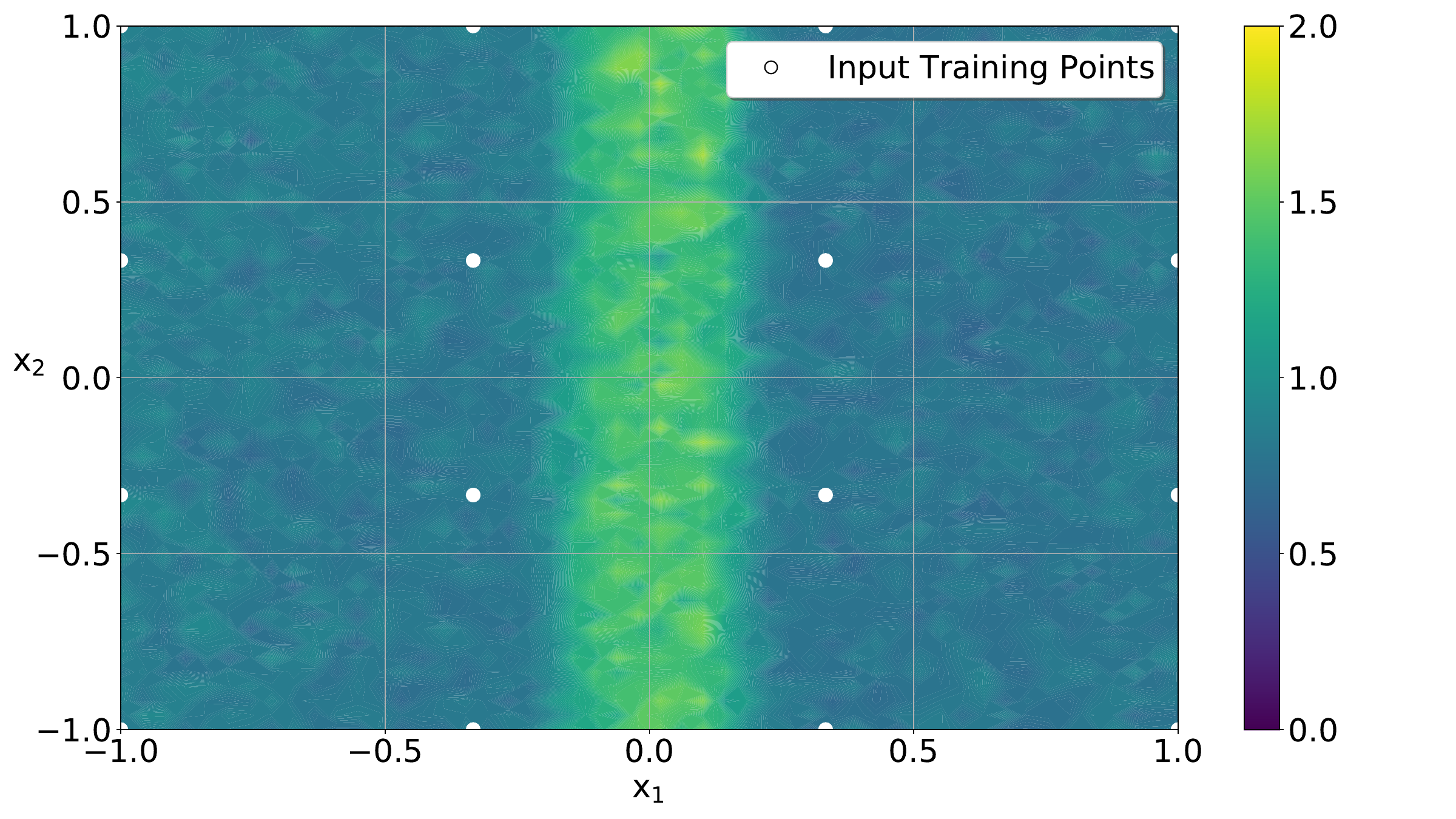}
}

\subfloat[DE (c=30)\label{subfig:D4_2d_DE}]{%
\includegraphics[trim= 10 10 80 10, clip, width=.9\columnwidth]{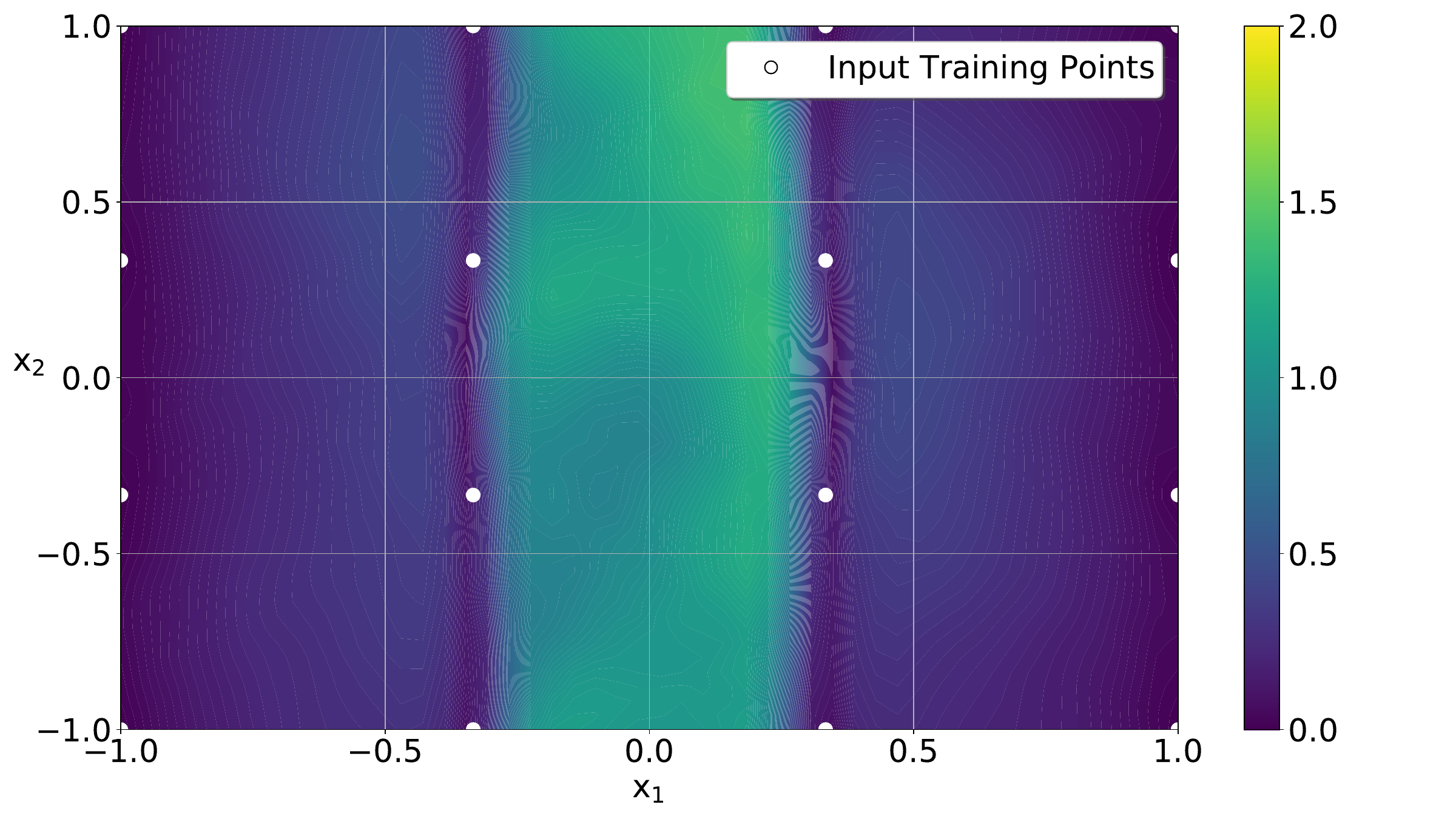}
}

\subfloat[HDE (c=30)\label{subfig:D4_2d_HDE}]{%
\includegraphics[trim= 10 10 80 10, clip, width=.9\columnwidth]{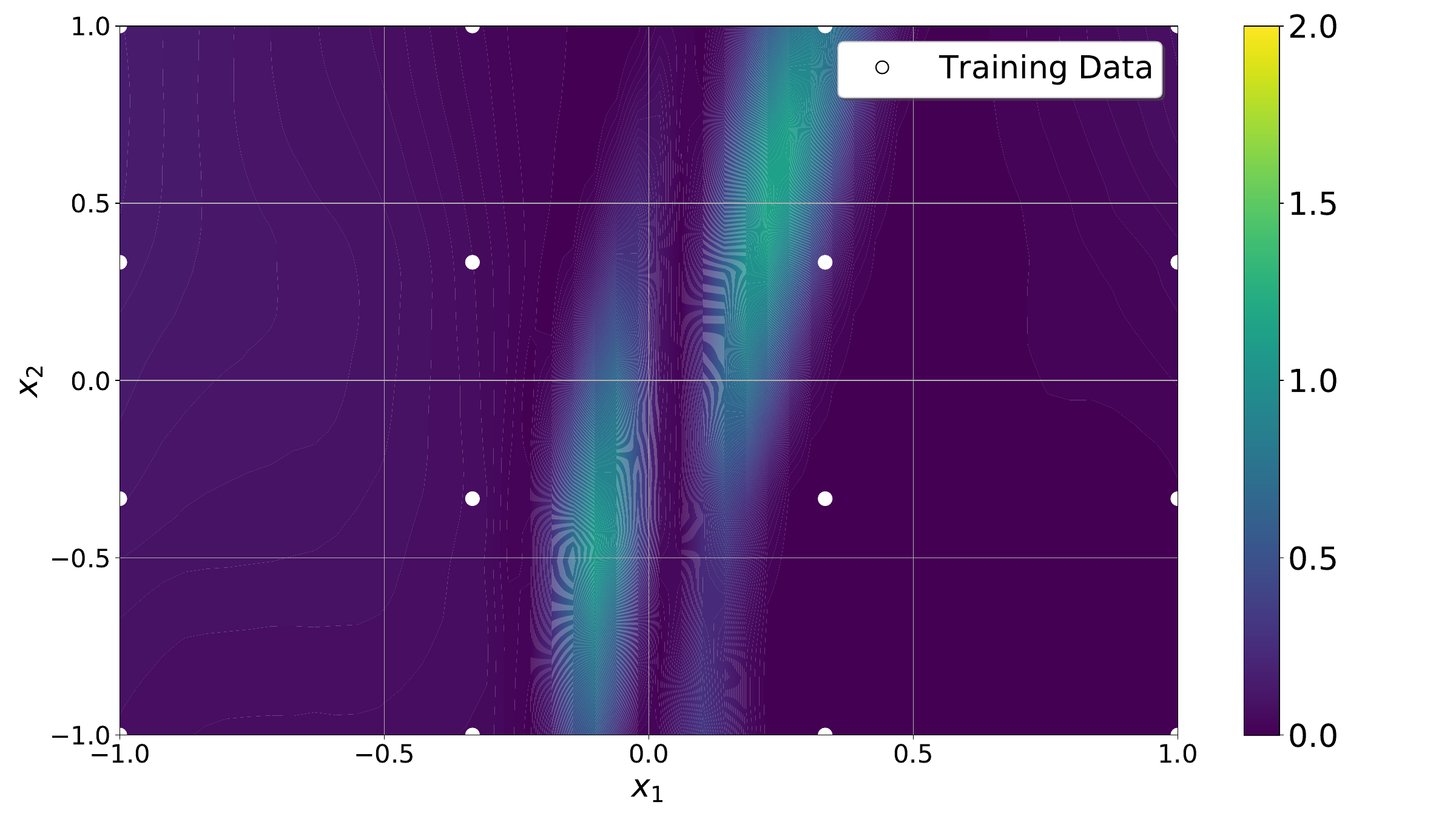}
}

\caption{Estimated model uncertainty of Gaussian process (GP), MC dropout (MCDO), deep ensembles (DE), and hyper deep ensembles (HDE)}
\end{figure}

\begin{figure}[H]
\ContinuedFloat

\subfloat[HDE* (c=30)\label{subfig:D4_2d_HDEstar}]{%
\includegraphics[trim= 10 10 80 10, clip, width=.9\columnwidth]{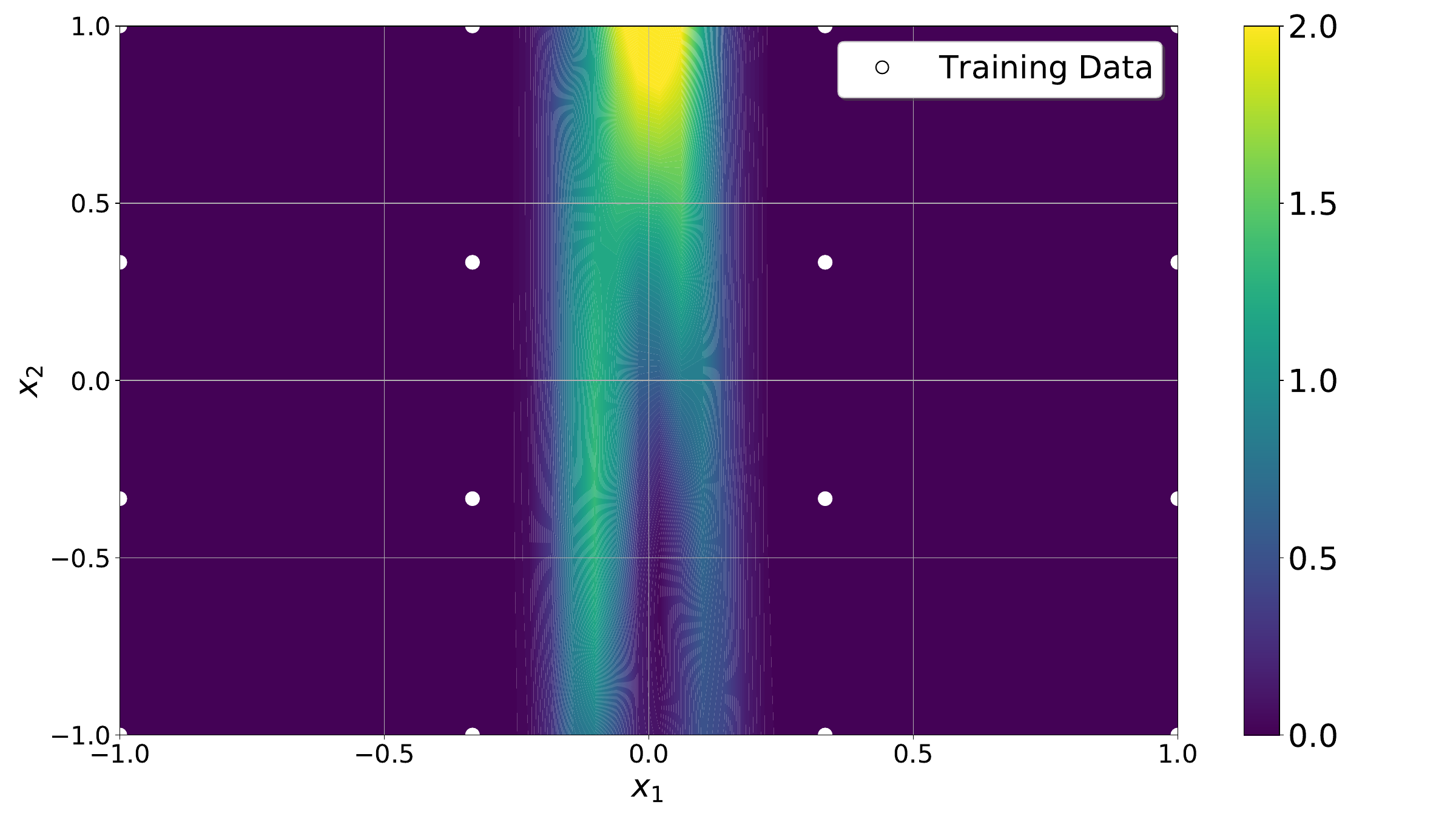}
}
\caption{(cont.) Estimated model uncertainty of HDE*.}
\label{fig:D4_2d_benchmarks}
\end{figure}

\subsubsection{How Does NOMU Fulfill \ref{itm:Axioms:irregular}?}
Recall NOMU's architecture depicted in \Cref{fig:nn_tik} in the main paper.

The $\sigmodelhatraw$-network learns the raw model uncertainty and is connected with the $\fhat$-network through the last hidden layer (dashed lines in Figure \ref{fig:nn_tik}). This connection enables $\sigmodelhatraw$ to re-use features that are important for the model prediction~$\fhat$. 
More theory on how features are reused in L2-regularized deep neural networks can be found in \citep{HeissPart3MultiTask}.
This behaviour directly translates to $\sigmodelhat$ (see \cref{eq:modelUncertaintyPrediction}), implementing Desideratum \ref{itm:Axioms:irregular}.

In \Cref{fig:D4,fig:2D_D4}, one can observe that NOMU fulfills \ref{itm:Axioms:irregular}. Moreover, \Cref{fig:2D_D4} shows how to control the strength of \ref{itm:Axioms:irregular} by varying the L2-regularization and the number of neurons of the $\sigmodelhatraw$-architecture.

\subsection{Desideratum \ref{itm:Axioms:UncertaintyDecreasesWithMoreTrainingPoints}}
First, note that \ref{itm:Axioms:largeDistantLargeUncertainty} already suggest \ref{itm:Axioms:UncertaintyDecreasesWithMoreTrainingPoints}, since for $\ntr\to\infty$ every point $x$ in the support of the input data generating distribution is infinitely close to infinitely many other i.i.d input training points $\xtr$ and thus infinitely close to the input training set.

To the best of our knowledge, \ref{itm:Axioms:UncertaintyDecreasesWithMoreTrainingPoints} is fulfilled by most reasonable algorithms that estimate model uncertainty. Specifically, NOMU, GPs, DE and HDE capture \ref{itm:Axioms:UncertaintyDecreasesWithMoreTrainingPoints}. This can be nicely observed in \Cref{fig:irradiance}, where all these algorithms result in zero model uncertainty in areas with many input training points (up to numerical precision).

MC Dropout can be seen as a variational algorithm for approximating BNNs.
While in theory, BNNs should also fulfill \ref{itm:Axioms:UncertaintyDecreasesWithMoreTrainingPoints}, MC Dropout often struggles to do so, as can be observed in \Cref{fig:irradiance} (see also \ref{subsubsec:MCDropoutviolatesD2} for a discussion why MC Dropout struggles to approximate BNNs).

Finally, \ref{itm:Axioms:UncertaintyDecreasesWithMoreTrainingPoints} is widely accepted \citep{kendall2017uncertainties,malinin2018predictive} and most of the time loosely stated along the lines of \enquote{While data noise uncertainty (aleatoric uncertainty) is irreducible, model uncertainty (epistemic uncertainty) vanishes with an increasing number of training observations.} In other words, the width of credible bounds converges to zero whilst the width of predictive bounds does not converge to zero in the presence of data noise.

However, whilst such statements are qualitatively true, formally, \ref{itm:Axioms:UncertaintyDecreasesWithMoreTrainingPoints} only holds \emph{in the limit} of the number of i.i.d training points~$\ntr$ to infinity and for $x\in \X$ that are \emph{in the support} of the input data generating distribution. Furthermore, note that \ref{itm:Axioms:largeDistantLargeUncertainty} also holds in the presence of data noise uncertainty.

Moreover, note that while \hyperref[itm:Axioms:trivial]{D1}\crefrangeconjunction\hyperref[itm:Axioms:irregular]{D4} are statements on \emph{relative} model uncertainty, i.e., statements that are independent of the calibration parameter $c\ge0$ (see \ref{subsubsec:relativeModelUncertainty}),
\ref{itm:Axioms:UncertaintyDecreasesWithMoreTrainingPoints} is a statement about \emph{absolute} model uncertainty. Thus, \ref{itm:Axioms:UncertaintyDecreasesWithMoreTrainingPoints} only holds for a fixed calibration parameter $c\ge0$ (if $c$ increases sufficiently fast with increasing $\ntr$, model uncertainty does not vanish). 

\subsubsection{Why does \ref{itm:Axioms:UncertaintyDecreasesWithMoreTrainingPoints} only hold in the limit?}
\begin{enumerate}[leftmargin=*,topsep=0pt,partopsep=0pt, parsep=0pt]
\item Even for fixed $c\ge0$, in the case of large \emph{unknown} data noise uncertainty that is simultaneously learned by the algorithm, adding another input training point $x$ close to existing input training points~$\xtr$ whose corresponding target $y$ is very far away from $\ytr$ could lead to an increase in model uncertainty, since this new training point $(x,y)$ would increase the predicted data noise uncertainty and thus increase the data noise induced model uncertainty.

\item Even if there is no data noise uncertainty $\signoise\equiv0$ and $c\ge 0$ is fixed, adding another input training point can increase the model uncertainty, when \ref{itm:Axioms:irregular} is fulfilled.
To see this, consider the following scenario: an already observed set of training points suggest that $f$ is very flat/simple/predictable (e.g., linear) in a certain region. However, adding a new training point $(x,y)$ shows that $f$ is much more irregular in this region than expected. Then, the learned metric can drastically change resulting in increased model uncertainty in this region (outside an $\epsilon$-ball around the new input training point $x$).
\end{enumerate}

\subsubsection{How Does NOMU Fulfill \ref{itm:Axioms:UncertaintyDecreasesWithMoreTrainingPoints}?}

Recall, that we train NOMU by minimizing 
  \begin{align}\label{eq:NOMUOptimizationProblem}
      \Lmu(\NN_{\theta})+\lambda\twonorm[{\theta}]^2,
  \end{align} where the NOMU loss $L^\hp(\NN_\theta)$ is defined as:
  \begin{align}
  L^\hp(\NN_\theta):=&\underbrace{\sum_{i=1}^{\ntr}(\fhat(\xtr_i
  )-\ytr_i)^2}_{\terma{}}+ \,\musqr\cdot\underbrace{\sum_{i=1}^{\ntr}\left(\sigmodelhatraw(\xtr_i)\right)^2}_{\termb{}}\\ &+\muexp\cdot\underbrace{\frac{1}{\lambda_d(\X)}\int_{\X}e^{-\cexp\cdot \sigmodelhatraw(x)}\,dx}_{\termc{}}.
\end{align}

 Then, the following proposition holds:
 \begin{subproposition}\label{prop:NOMUD5} Let $\lambda,\muexp,\cexp,\musqr\in \Rp$ be fixed and let the activation-functions of $\NN_\theta$ be Lipschitz-continuous and let $\sigmodelhat(x)$ be NOMU's model uncertainty prediction. Then, it holds that $\sigmodelhat(x)$ converges in probability to $\sigmin$ for $\ntr\to \infty$ (with $\xtr_i\overset{\text{i.i.d}}{\sim}\PP_X$) for all input points $x$ in the support of the input data generating distribution $\PP_{X}$, i.e., more formally $\forall x \in \text{supp}\left(\PP_{X}\right), \forall \epsilon>0, \forall \delta\ge0: \exists n^0: \forall n\ge n^0:$
 \begin{align}
     \PP\left[|\sigmodelhat(x)-\sigmin|\ge\epsilon\right]<\delta,
 \end{align}
 where $\sigmin\ge0$ is an arbitrarily small constant modelling a minimal model uncertainty used for numerical reasons.
 \end{subproposition}
 \begin{proof}
 Let $x\in \text{supp}\left(\PP_{X}\right)$, $\delta\ge0$ and $\epsilon>0$. 
 
 First, let $ \Lmu(0):=c<\infty$, i.e., the value of the loss function when inserting the constant zero function. Then for the optimal solution $\NN_{\theta^*}$ of \eqref{eq:NOMUOptimizationProblem} it immediately follows that 
 \begin{align}\label{eq:twonormofthetastar}
     \twonorm[\theta^*]^2\le\frac{c}{\lambda}.
 \end{align}
 Using the fact that all activation-functions $\phi$ of $\NN_\theta^*$ are Lipschitz-continuous together with \eqref{eq:twonormofthetastar}, one can show that there exists a constant $L:=L(c,\lambda,\phi,\text{architecture})$ such that the raw model uncertainty prediction
$\sigmodelhatraw^{\theta^*}$ is Lipschitz-continuous with constant $L$.

Next, let $U:=U_{\frac{\epsilon}{4L}}(x)$ be an open ball with radius $\frac{\epsilon}{4L}$ around $x$. Given that the diameter of $U$ is equal to $\frac{\epsilon}{2L}$ and the fact that $\sigmodelhatraw^{\theta^*}$ is Lipschitz-continuous with constant $L$, it follows that
\begin{align}\label{eq:maxmindistance}
    \max\limits_{z\in U }\sigmodelhatraw^{\theta^*}(z)-\min\limits_{z\in U }\sigmodelhatraw^{\theta^*}(z)< \frac{\epsilon}{2}.
\end{align}
Given $U$ let $p:=\PP[x\in U]$. Since $x\in \text{supp}\left(\PP_{X}\right)$, per definition it holds that $p>0$.

In the following let $\Dtr_{n,x}\sim (\PP_{X})^n$ denote the random variable representing a set of $n$ input training points.
Now, let $n^0$ be sufficiently large such that
\begin{align}
    \PP\left[|\Dtr_{n^0,x}\cap U|>4\cdot\frac{c}{\epsilon^2}\right]>1-\delta.
\end{align}
Note that one can explicitly calculate the value of $n^0$, since $|\Dtr_{n^0,x}\cap U|\in \N_0$ is binomial distributed with $p>0$ and $n^0\in \N$.

Finally, we show that $\sigmodelhatraw^{\theta^*}(x)<\epsilon$ by contradiction. For this, assume on the contrary that
\begin{align}
\sigmodelhatraw^{\theta^*}(x)\ge\epsilon.
\end{align}
Using \eqref{eq:maxmindistance} it follows that for all $z\in\Dtr_{n^0,x}\cap U$ it holds that $ \sigmodelhatraw^{\theta^*}(z)>\frac{\epsilon}{2}$ with probability larger than $1-\delta$.
This together with the fact that each summand in the term~\termb{} in the NOMU loss function $\Lmu$ is non-negative implies that
\begin{align}\label{eq:termbInequality}
    \termb{}
    \geq \sum_{\xtr\in\Dtr_{n^0,x}\cap U}\left(\sigmodelhatraw(\xtr)\right)^2
    > \left(\frac{\epsilon}{2}\right)^2\cdot 4\cdot\frac{c}{\epsilon^2}=c.
\end{align}
Putting everything together and using the fact that each term in the NOMU loss is non-negative implies that
\begin{align*}
    \Lmu(\NN_{\theta^*})+\lambda\twonorm[{\theta^*}]^2\ge \termb{}\overset{\eqref{eq:termbInequality}}{>}c=\Lmu(0)+\lambda\twonorm[{0}]^2,
\end{align*}
which is a contradiction for $\NN_{\theta^*}$ being optimal in \eqref{eq:NOMUOptimizationProblem}.

Therefore, we can conclude that $\sigmodelhatraw^{\theta^*}(x)<\epsilon$ with probability larger than $1-\delta$. By definition of $\sigmodelhat$ this implies that $|\sigmodelhat-\sigmin|< \epsilon$ with probability larger than $1-\delta$, which concludes the proof.
\end{proof}

Note that empirically one can see in \Cref{fig:irradiance} (in the areas with many input training points) how well NOMU fulfills \ref{itm:Axioms:UncertaintyDecreasesWithMoreTrainingPoints} in real-world settings. Furthermore, one can see that the statement only holds true and is only desirable for $x$ in the support of the input data generating distribution~$\PP_{X}$ (not in the gaps).

\section{NOMU vs. Prior Networks}\label{sec:NOMUvsPriorNetworks}
In this section, we highlight several differences of NOMU compared to \emph{prior regression networks} that were recently introduced in a working paper by \citet{malinin2020regression}.

While the high level idea of introducing a separate loss term for the in-sample-distribution and \emph{out-of-distribution (OOD)} distribution is related to NOMU, there are several important differences, which we discuss next:
\begin{enumerate}[leftmargin=*,topsep=0pt,partopsep=0pt, parsep=0pt]
\item \citet{malinin2020regression}'s approach focuses on estimating both model and data noise uncertainty. Thus, to properly compare it to NOMU, we consider their approach for known and negligible data noise uncertainty, e.g., for $\signoise=10^{-10}$ we need to set in their paper $(L,\nu)=(I\cdot l^{-1},\frac{1}{\signoise} l)$ with $l \to \infty$, such that the their corresponding model uncertainty prediction is given by $\left(\kappa(x)\Lambda(x)\right)^{-1}\overset{l\to\infty}{=}\frac{\signoise}{\kappa(x)}\cdot I$. In the following, we will consider for simplicity a one-dimensional output, i.e.,  $\sigmodelhat=\frac{\signoise}{\kappa(x)}$.
\item They explicitly define a prior ``target'' distribution independent of $x\in \X$, which is parametrized by $\kappa_0$ (model uncertainty) and $m_0$ (mean prediction) for OOD input points. Specifically, for OOD input points their mean prediction $\fhat$ is pushed towards $m_0$. In many applications the success of classical mean predictions of deep NNs is evident. In none of these applications there was a term that pushed the mean prediction to a fixed predefined prior mean. Therefore, for NOMU we keep the mean prediction untouched.
\item Instead of our loss, their loss (derived from a reverse KL-divergence) is of the form:
\begin{align}\label{eq:priorRegressionNetworksModelUncertainty}
&\underbrace{\sum_{i=1}^{\ntr}\frac{\left(\fhat(\xtr_i)-\ytr_i\right)^2}{2\left(\signoise\right)^2}+\sum_{i=1}^{\ntr}\frac{1}{\kappa(x)}}_{\textrm{in-sample}}+\\
&\underbrace{\int_{\X}\frac{\kappa_0\left(\fhat(x)-m_0\right)^2}{2\left(\signoise\right)^2}+\frac{\kappa_0}{\kappa(x)}-\log\frac{\kappa_0}{\kappa(x)}-1\,d\muOOD(x)}_{\textrm{out-of-distribution}}\notag,
\end{align}
where $m_0,\kappa_0$ are the prior parameters for the mean and model uncertainty prediction and $\muOOD$ is an OOD measure. Specifically, they enforce zero model uncertainty at input training points via the linear term $\frac{1}{\kappa(x)}$, while we use a quadratic term. Moreover, they only use an OOD term and no \emph{out-of-sample (OOS)} term (see below \Cref{itm:OODtermOnly}).
\item\label{itm:OODtermOnly} In their loss formulation in \eqref{eq:priorRegressionNetworksModelUncertainty}, they only use an \emph{out-of-distribution} (OOD) term, while we use an \emph{out-of-sample (OOS)} term. By OOD they refer to input training points \emph{only} far away from the training data, e.g., in \citep[Section~3]{malinin2020regression} $\muOOD$ only has support far away from the convex hull of the input training points. Thus, they do not enforce model uncertainty in gaps between input training points. In contrast, by \emph{out-of-sample (OOS)} we refer to a distribution with no mass on the input training points, i.e., we sample new input points that are not equal to the input training points but come from the same range (we recommend to use a distribution that is similar to the data generating process). Therefore, our loss explicitly also enforces model uncertainty in gaps between input training points.
\item They use a different architecture and train only \emph{one} NN. This implies that their mean prediction $m(x)$ can be influenced in unwanted ways by the model uncertainty prediction $\left(\kappa(x)\Lambda(x)\right)^{-1}$.

\item Their theoretical motivation substantially differs from ours: They only partially specify a prior belief by defining \emph{marginal} distributions for $f(x)$ for each input point $x\in \X$, without specifying a joint prior distribution for $f$.
However, given only marginals no joint distribution, which is the crucial aspect when defining a prior in regression, can be derived without further information (E.g., consider Gaussian processes (GPs); here all one-dimensional marginal distributions are simply given by $\mathcal{N}((m(x)),k(x,x))$.
However, the crucial part is how to define $k(x,x^{'})$ specifying the relation of $x$ and $x^{'}$.
Only defining the marginals does not suffice to fully define GPs, leaving this most crucial part undefined).

However, for NOMU, we give in \Cref{sec:TheoreticalAnalysisAppendix} a theoretical connection to BNNs, with Gaussian prior on the weights. This induces a prior on the function space, i.e., a distribution over $f$ rather than separate \emph{marginal} distributions over $f(x)$ for each $x \in \X$.
\item Parametrizing the model precision instead of the model uncertainty can have negative effects due to (implicit) regularization of NNs in the case of negligible or zero data noise. To get uncertainties in gaps between the input training points (small $\kappa(x)$) while having almost zero uncertainty at these input training points ($\kappa(x)\to \infty$), would imply very high regularization costs for the function $\kappa(x)$ and thus is very hard to learn for a NN. For NOMU, we therefore parameterize directly the model uncertainty (which is always finite) instead of the model precision (that should be infinite at noiseless training data points).
\item Our experimental results suggest that NOMU clearly outperforms DE in BO, whilst DE outperforms \emph{prior regression networks} in their set of experiments.
\end{enumerate}

\section{NOMU vs. Neural Processes}\label{sec:appendix:NOMU vs. Neural Processes}
In this section, we discuss the differences between neural processes (NPs) introduced by \citet{conditional_neural_proc_pmlr-v80-garnelo18a,neural_proc} and NOMU. 
Specifically, we explain in the following why NPs and NOMU are solving very different problems in different settings.

For training NPs, one has to observe data from 1000s of realizations of $f_k$, sampled i.i.d. from the prior distribution (for each $f_k$ one observes $x_{i}$ and $f_k(x_i)$ to train the NP).
This is often mentioned in \citet{conditional_neural_proc_pmlr-v80-garnelo18a,neural_proc}, e.g., \citet[Section~4.1]{conditional_neural_proc_pmlr-v80-garnelo18a}: \textit{''We generate [...] datasets that consist of function\textbf{s} generated from a GP [...] \textbf{At every training step we sample a curve from the GP} [...].``.} For NOMU we consider the \textbf{very different} task of estimating $p(y^\text{test} | x^\text{test};\Dtr)$ based on a \textbf{single dataset}, i.e, generated from a single realization $f=f_1$.

For example in the case of the Boston housing data set from \cref{subsubsec:UCI}, there is only one function $f=f_1$ involved that maps a (multidimensional) input data point $x$ corresponding to a house in Boston to its price $f(x)$. For this data set, NPs would not be well suited, since it only contains data $(x_i,y_i)=(x_i,f(x_i)+\varepsilon_i)$ coming from this specific function $f$. One does not have access to data corresponding to another function $f_2$ that had been sampled from the same prior distribution.

The same is true for the other data sets we consider in this paper (e.g., for the solar irradiance data set we only use the data visible in \Cref{fig:irradianceNOMU} and NOMU does not have access to any data coming from other time series to make its predictions). Thus, NPs cannot be applied to the tasks considered in this paper.

\mypar{Summary} NOMU, GP, MCDO, DE and HDE are designed to be trained on data coming from \emph{one} unknown function~$f$ without having access to data from other functions $f_2,f_3,\dots$. In contrast,
NPs are designed to be trained on \emph{multiple} data sets generated from \emph{multiple} functions $f_1, f_2, f_3,\dots$.

\section{Aleatoric Neural Networks: Aleatoric vs. Epistemic Uncertainty}\label{sec:appendix:Aleatoric Neural Networks: Aleatoric vs. Epistemic Uncertainty}
In this section, we discuss the classical approach of a NN with two outputs, one output for a model prediction and another for aleatoric uncertainty, which is trained using the (scaled) Gaussian negative log-likelihood as introduced by \citet{nix1994estimating}. We will use the terms aleatoric uncertainty and data noise as well as model uncertainty and epistemic uncertainty interchangeably. 

In what follows, we call this method \emph{aleatoric neural network (ANN)}. Within this section, we show that such an ANN does not explicitly estimate model uncertainty (in contrast to all other benchmark methods discussed in this paper), i.e., when using the aleatoric uncertainty output~$\signoisehat$ naively as $\sigmodelhat$, the so obtained $\sigmodelhat:=\signoisehat$ does not represent model uncertainty (epistemic uncertainty).\footnote{The reminder of this section only targets readers who do not directly see that substituting $\sigmodelhat$ by $\signoisehat$ is an extremely bad idea.} First, we give the definition of an ANN.

\begin{definition}[ANN]\label{def:ANN}
An ANN is a fully-connected feed-forward NN $\NNann:\mathbb{R}^d\to \mathbb{R}\times \mathbb{R}_+$ with two outputs: (i) the model prediction $\fhat \in \mathbb{R}$ and (ii) a model uncertainty prediction $\signoisehat\in \mathbb{R}_+$ that is trained for a given set of training points $\Dtr$ using the following loss function:
\begin{align}
&\Lann(\NNann):=\nonumber\\
&\frac{1}{|\Dtr|}\sum_{(x,y) \in \Dtr}\left[\frac{\left(y-\fhat(x)\right)^2}{2 \left(\signoisehat(x)\right)^2} +  \ln\left(\signoisehat(x)\right)\right]
\end{align}
\end{definition}
The idea of ANN is to estimate the noise scale $\signoise(x)=\sqrt{\mathbb{V}[\varepsilon]}=\sqrt{\mathbb{V}[y|x,f(x)]}$. \Cref{fig:appendix:NOMUvsANNc1} shows that, as we would expect, the trained ANN has learned the true $\signoise\equiv0\approx\signoisehat$ quite precisely. However, it as also becomes evident that the ANN does not learn any form of model uncertainty. Very far away from all observed training data points, the ANN does not express any uncertainty about its prediction (in \Cref{fig:appendix:NOMUvsANNc5}, to the right of $x=0.5$, the predictions are very far away from the truth, but $\signoisehat$ does not capture this uncertainty). Therefore, the ANN's aleatoric uncertainty output $\signoisehat$ does not fulfill desideratum \ref{itm:Axioms:largeDistantLargeUncertainty}.
\begin{figure}[t!]
    \begin{center}
    \centerline{\includegraphics[width=\columnwidth]{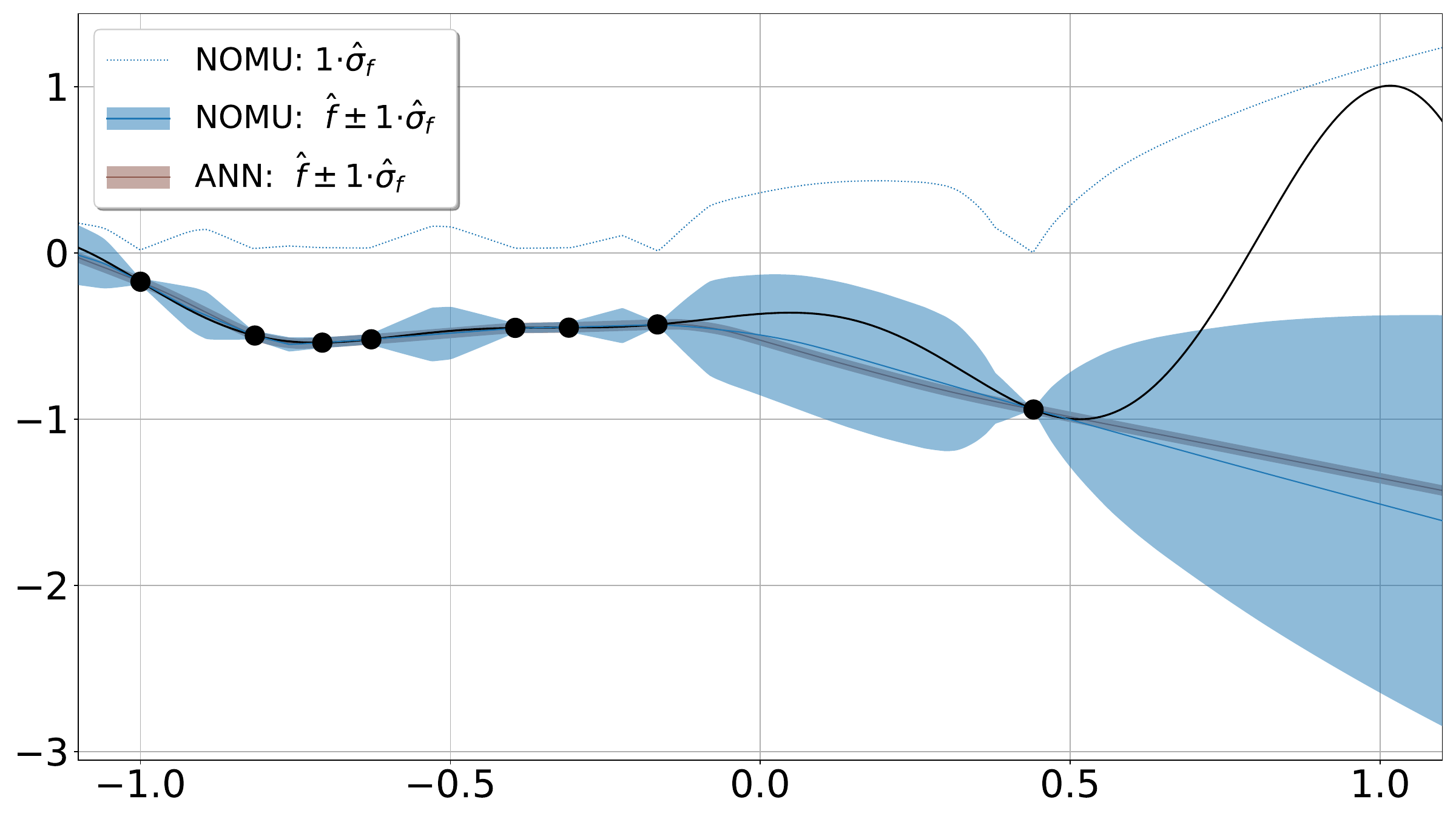}
    }
    \caption{Comparison of UBs resulting from ANN ($c=1$) (where $\signoisehat$ is used as a substitute for $\sigmodelhat$) vs. NOMU for the Forrester function (solid black line). For NOMU, we also show $\sigmodelhat$ as a dotted blue line. Training points are shown as black dots.}
    \label{fig:appendix:NOMUvsANNc1}
    \end{center}
    \vskip -0.2in
\end{figure}
The problem of misusing $\signoisehat$ as substitute for $\sigmodel$ is not that $\signoisehat$ is too small (as we study relative uncertainty in this paper (see \cref{subsubsec:relativeModelUncertainty}), one can always scale the uncertainty by a factor $c$). However,  \Cref{fig:appendix:NOMUvsANNc5} shows that also the scaled uncertainty completely fails to capture the desiderata, i.e., the aleatoric uncertainty output $\signoisehat$ is almost constant for all input points $x$. Thus, UBs of an ANN do not fulfill \ref{itm:Axioms:ZeroUncertaintyAtData} and \ref{itm:Axioms:largeDistantLargeUncertainty}.
In \Cref{fig:appendix:NOMUvsANNc5}, we can see that $5\signoisehat$ is way too underconfident at input training points and at the same time way too overconfident far away from the observed input training points. This would result in very bad $\NLPD{}$ scores on a test set.
(Moreover in noisy settings, i.e., $\signoise\neq 0$, the aleatoric uncertainty output $\signoisehat$ of an ANN does not fulfill \ref{itm:Axioms:UncertaintyDecreasesWithMoreTrainingPoints} either: $\signoisehat$ should converge to $\signoise$ while $\sigmodelhat$ should converge to zero as the amount of training data increases.)

Furthermore, especially in Bayesian optimization (BO) $c\signoisehat$ would be a terribly bad substitute for $\sigmodel$:
Maximizing the upper UB acquisition function $\hat{f}+c\signoisehat$, would be almost equivalent to maximizing $\hat{f}$ since $c\signoisehat$ is almost constant because of the lack of \ref{itm:Axioms:ZeroUncertaintyAtData} and \ref{itm:Axioms:largeDistantLargeUncertainty}.
If one wants to maximize the function in \Cref{fig:appendix:NOMUvsANNc5} on $[-1,1]$, the next BO-step would propose to query an input training point at the left boundary $x=-1$ (even for large $c$). However, one does not learn anything new from evaluating at $x=-1$, because this input training point has already been evaluated in a previous BO-step. All subsequent BO steps would propose to query the same point $x=-1$ without exploring any other region of the input space and one would never find the true maximum at $x\approx 1$.
\begin{figure}[t!]
    \begin{center}
    \centerline{\includegraphics[width=\columnwidth]{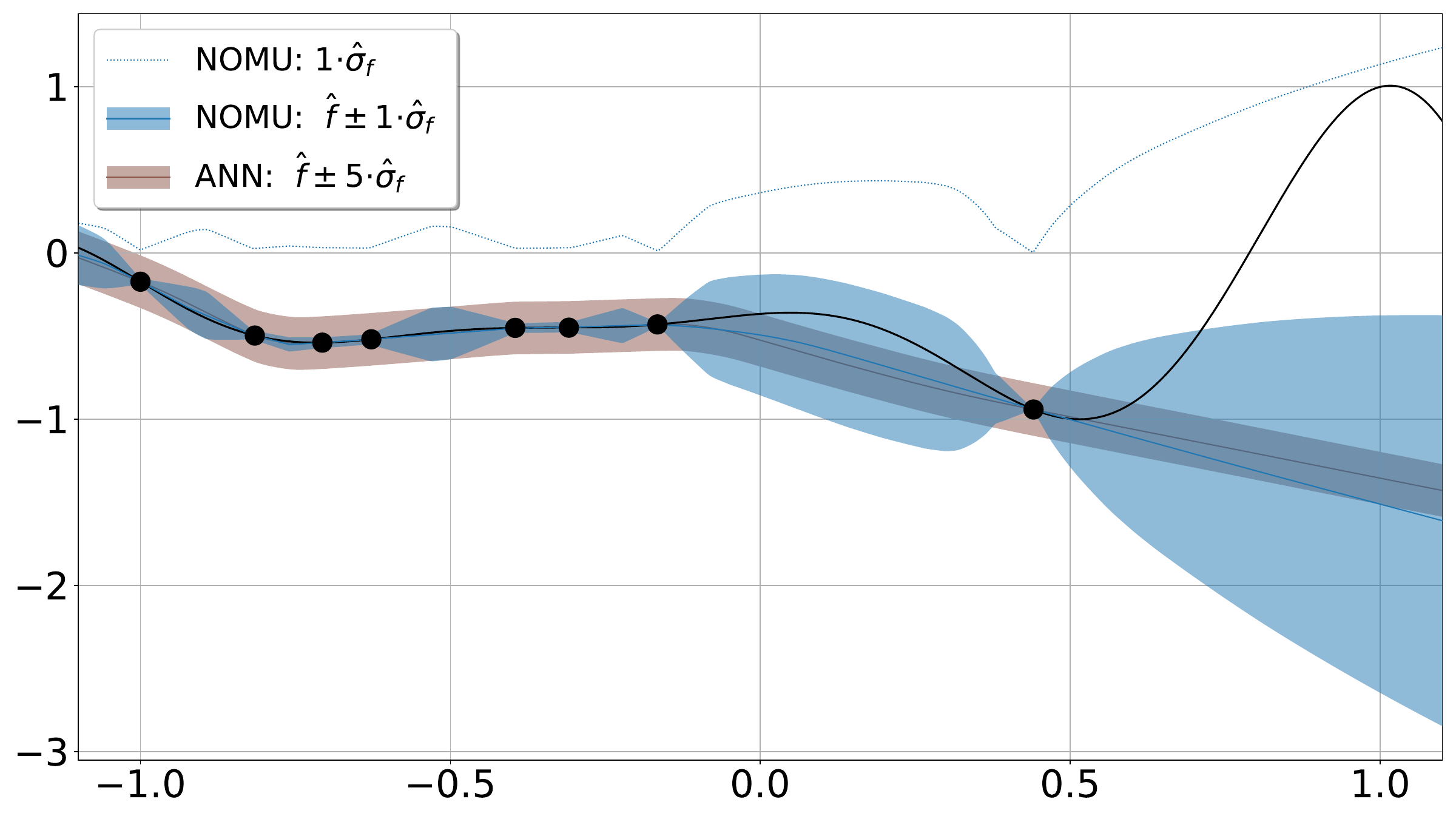}
    }
    \caption{Comparison of UBs resulting from ANN ($c=5$) (where $\signoisehat$ is used as a substitute for $\sigmodelhat$) vs. NOMU for the Forrester function (solid black line). For NOMU, we also show $\sigmodelhat$ as a dotted blue line. Training points are shown as black dots.}
    \label{fig:appendix:NOMUvsANNc5}
    \end{center}
    \vskip -0.2in
\end{figure}
In contrast, a reasonable model for estimating $\sigmodel$ (such as NOMU), would directly (after scaling up $c$ dynamically by a factor $2$ as described in \cref{par:dynamicC}) choose a point in the unexplored right region $x\approx 1$, because the left side $x\approx-1$ is already well explored.

Overall, aleatoric uncertainty $\signoise$ and epistemic uncertainty $\sigmodel$ are two very different objects. Thus, an estimator $\signoisehat$ designed to estimate $\signoise$ is usually a bad estimator for $\sigmodel$. 
\section{Hyperparameter Sensitivity Analysis}\label{sec:appendix:Sensitivity Analysis - Loss HPs}
In this section, we provide a sensitivity analysis with respect to NOMU's loss hyperparameters, i.e., $\musqr$, $\muexp$, $\cexp$, and $\Daug$. First, we present a visual qualitative analysis in 1D showing how each of these hyperparameters affects the shape of NOMU's UBs (\Cref{subsec:appendix:Qualitative Sensitivity Analysis}). Second, we also present an extensive quantitative sensitivity analysis in the generative test-bed setting from \Cref{subsubsec:GenerativeTestbed}, where in addition to the loss hyperparameters we also include the hyperparameters of the readout map $\sigmin$, and $\sigmax$ in our analysis (\Cref{subsec:appendix:Quantitative Sensitivity Analysis}).

\subsection{Qualitative Sensitivity Analysis}\label{subsec:appendix:Qualitative Sensitivity Analysis}
In this section, we consider the setting of \Cref{subsubsec:ToyRegression}, and visualize the effect of increasing or decreasing each of NOMU's loss hyperparameters $\musqr$, $\muexp$, $\cexp$, and $\Daug$ on the example of the 1D Levy function. For reference, \Cref{fig:sensitivity:default} shows NOMU's UBs (with scaling factor $c=2$) for the default loss hyperparameters $\musqr=0.1$, $\muexp=0.01$, $\cexp=30$, and $\Daug=128$ that are used in \Cref{subsubsec:ToyRegression} in the main paper.

\begin{figure}[ht]
\begin{center}
\centerline{\includegraphics[width=\columnwidth]{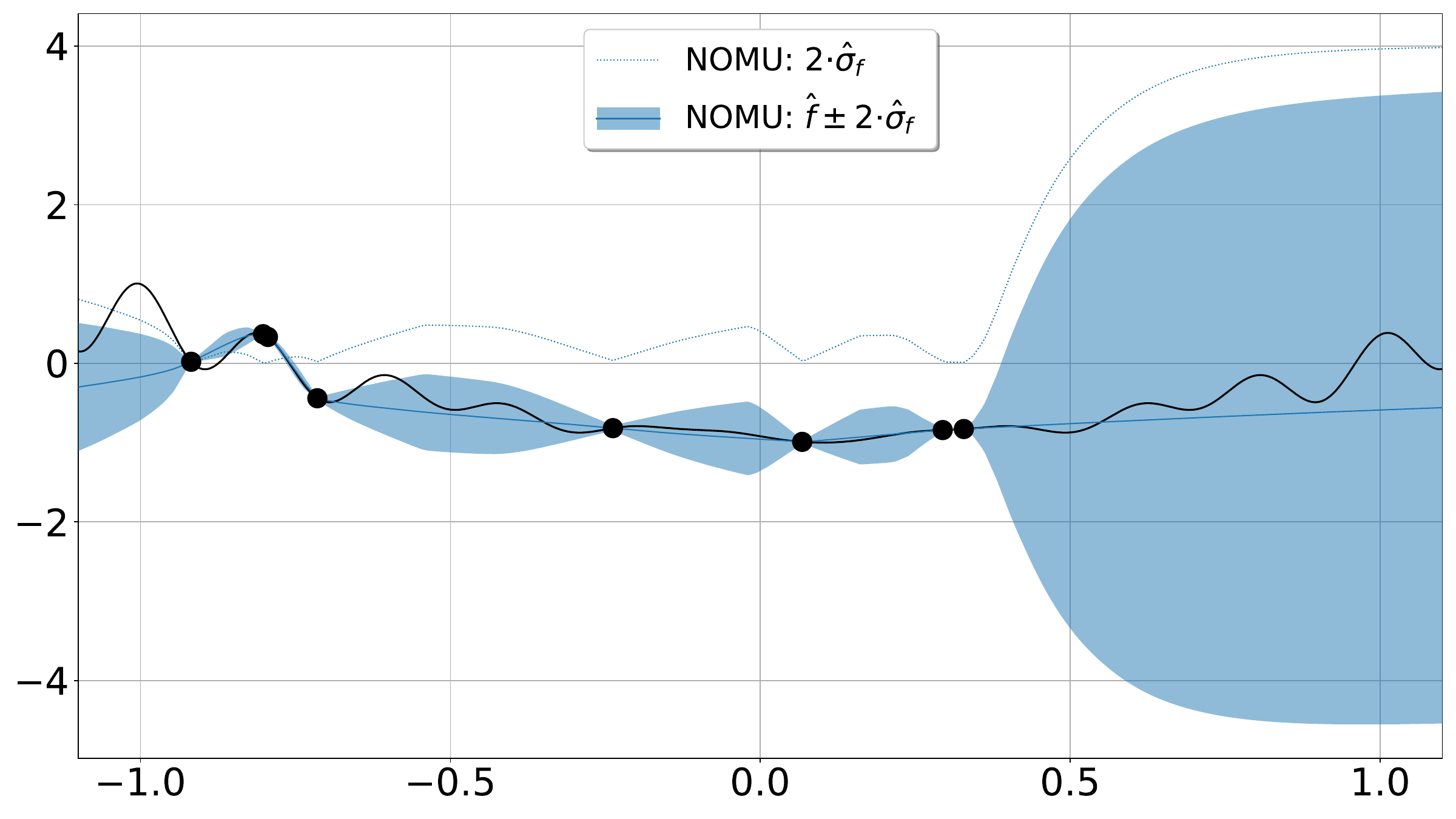}}
\caption{NOMU's UBs (c=2) for the generic loss hyperparameters from \Cref{subsubsec:ToyRegression} in the main paper.}
\label{fig:sensitivity:default}
\end{center}
\end{figure}

For each of $\Daug$ and $\cexp$, we fit two additional NOMU models, where we ceteris paribus de- and increase the hyperparameter's default value by factors $1/s$ and $s$, respectively. The multiplicative factors $\musqr$ and $\muexp$ we treat jointly in our sensitivity analysis: First, we show the effect of ceteris paribus de- and increasing the default value of the product $\muexp\musqr$ by factors $s_l=0.001$ and $s_u=10$. Second, we vary the ratio $\muexp/\musqr$ in the same fashion, with scaling factors $s_l=1/16$ and $s_u=16$.
Within all of the experiments in this section, we sample artificial input points $\Daug$ on an equidistant (deterministic) grid on $[-1.1,1.1]$. This allows us to give a qualitative analysis of the hyperparameters' effects as follows.

\paragraph{Varying $\muexp\musqr$ with Scaling Factors of $s_l=0.001$ and $s_u=10$.}
Decreasing $\muexp\musqr$ by decreasing both $\muexp$ and $\musqr$ leads to more tubular bounds by relaxing the desiderata \ref{itm:Axioms:ZeroUncertaintyAtData} and \ref{itm:Axioms:largeDistantLargeUncertainty}. This can be seen in \Cref{fig:sensitivity:muexpmalmusqr}:
NOMU's blue dashed uncertainty (corresponding to small $\muexp\musqr$) is larger at data points than the orange one of NOMU 2 (corresponding to large $\muexp\musqr$), and it is smaller further away from the training data points.
NOMU's default hyperparameters are in a range where the loss is already at its limit\footnote{For NOMU's default parameters the ratio $\muexp\musqr / \lambda$ is already very large such that the explicit regularization via $\lambda\twonorm[\theta]^2$ is almost negligible. Thus, the UBs are only prevented from having even larger curvature by implicit regularization, i.e., within a given number of epochs the training algorithm cannot reach a function with more curvature, because increasing the amplitude of the loss is (partially) compensated by the adaptivity of the \textsc{Adam} algorithm. Only in ranges where $\muexp\musqr / \lambda$ is small enough for the explicit regularization to actually matter, the UBs become sensitive to the ratio $\muexp\musqr / \lambda$. Then the regularization of $\lambda\twonorm[\theta]^2$ keeps the curvature of $\sigmodelhat$ low, i.e., the UBs become more tubular.} enforcing the desiderata \ref{itm:Axioms:ZeroUncertaintyAtData} and \ref{itm:Axioms:largeDistantLargeUncertainty}. Therefore,
further increasing $\muexp\musqr$ (while keeping their ratio and the other hyperparameters fixed) barely causes the UBs to change as can be seen for the orange UBs in \Cref{fig:sensitivity:muexpmalmusqr}. Increasing $\muexp\musqr$ too much, can lead to numerical instabilities.
\begin{figure}[t]
\begin{center}
\centerline{\includegraphics[width=\columnwidth]{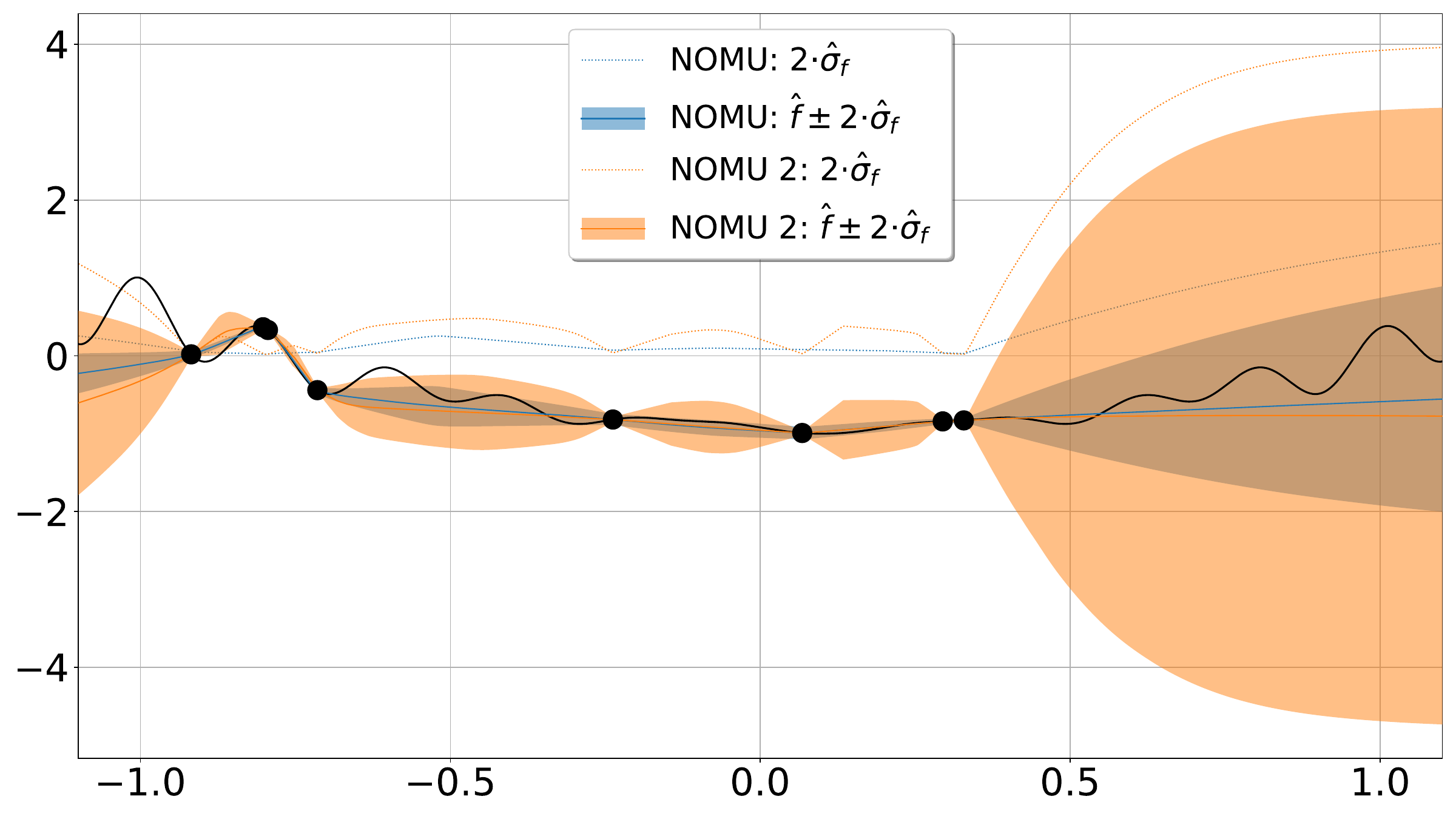}}
    \caption{NOMU's UBs (c=2) for $\musqr=1e-4, \muexp=1e-5$ (blue) and $\musqr=1, \muexp=0.1$ (orange).}
    \label{fig:sensitivity:muexpmalmusqr}
\end{center}
\vskip -0.2in
\end{figure}

\paragraph{Varying $\muexp/\musqr$ with Scaling Factors of $s_l=1/16$ and $s_u=16$.}
Increasing the ratio of $\muexp/\musqr$ (while keeping their product and all other hyperparameters fixed) simply causes NOMU's UBs to uniformly widen across the entire domain. Indeed, in \Cref{fig:sensitivity:muexpdurchmusqr}, the orange UBs of NOMU 2 (corresponding to large $\muexp/\musqr$) are blown up and cover the blue UBs (corresponding to small  $\muexp/\musqr$) throughout the input space.

\begin{figure}[ht]
\vskip 0.2in
\begin{center}
\centerline{\includegraphics[width=\columnwidth]{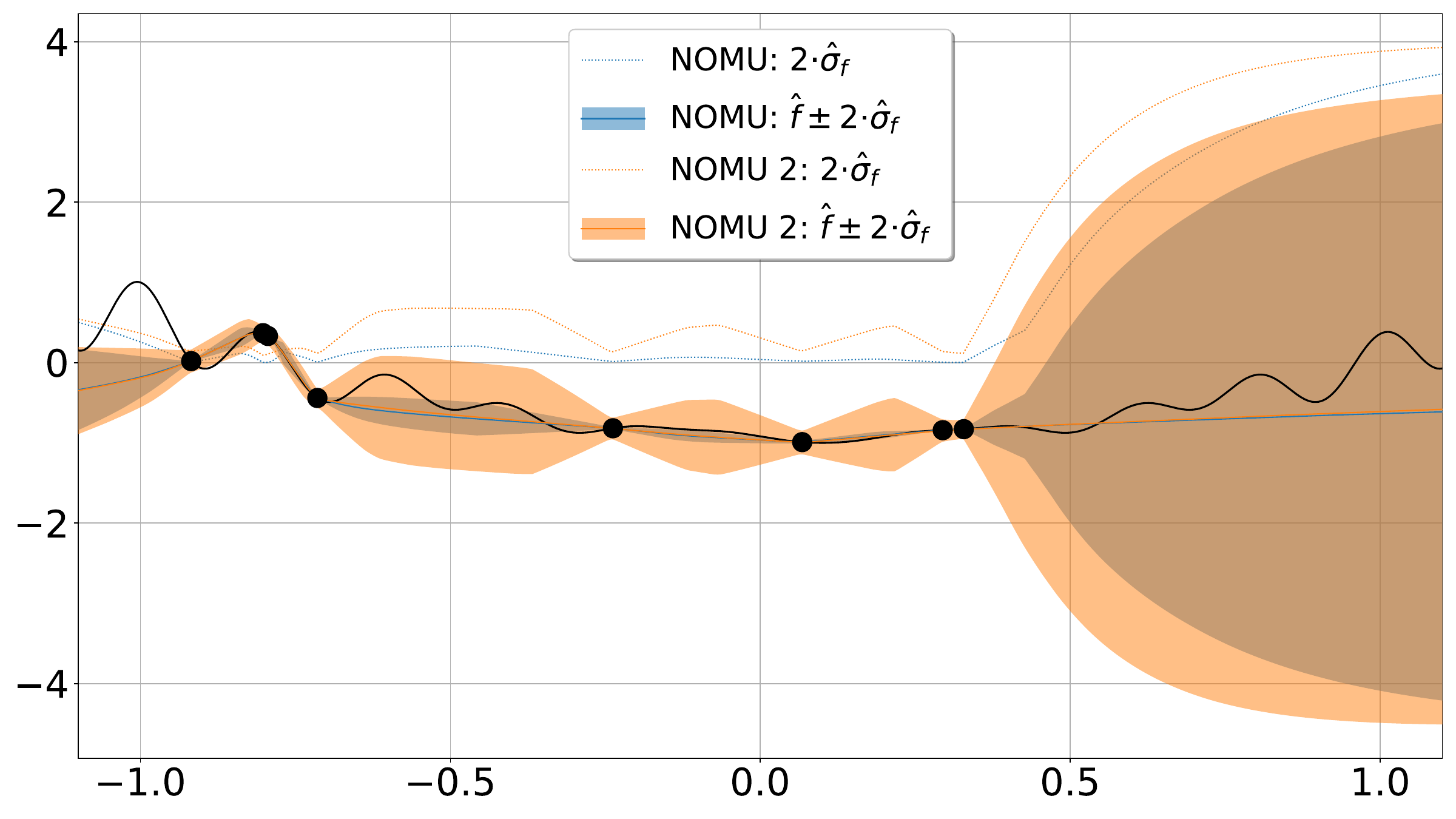}}
    \caption{NOMU's UBs (c=2) for $\muexp=0.0025, \musqr=0.4$ (blue) and $\muexp=0.04, \musqr=0.025$ (orange).}
    \label{fig:sensitivity:muexpdurchmusqr}
\end{center}
\vskip -0.2in
\end{figure}

\paragraph{Varying $\cexp$ with a Scaling Factor of $s=2$.}
Increasing the hyperparameter $\cexp$ causes UBs to shrink in areas of large uncertainty and to widen in areas of small uncertainty. This effect is visualized in \Cref{fig:sensitivity:cexp}: the orange dashed uncertainty line of NOMU 2 (large $\cexp=60$) lies above the blue one of NOMU (small $\cexp=15$) at data points; and in regions of large uncertainty, the orange UBs (corresponding to large $\cexp$) turn out to be more narrow than the blue ones (corresponding to small $\cexp$). Thus, increasing $\cexp$ causes NOMU's UBs to be more tubular.

\begin{figure}[ht]
\begin{center}
\centerline{\includegraphics[width=\columnwidth]{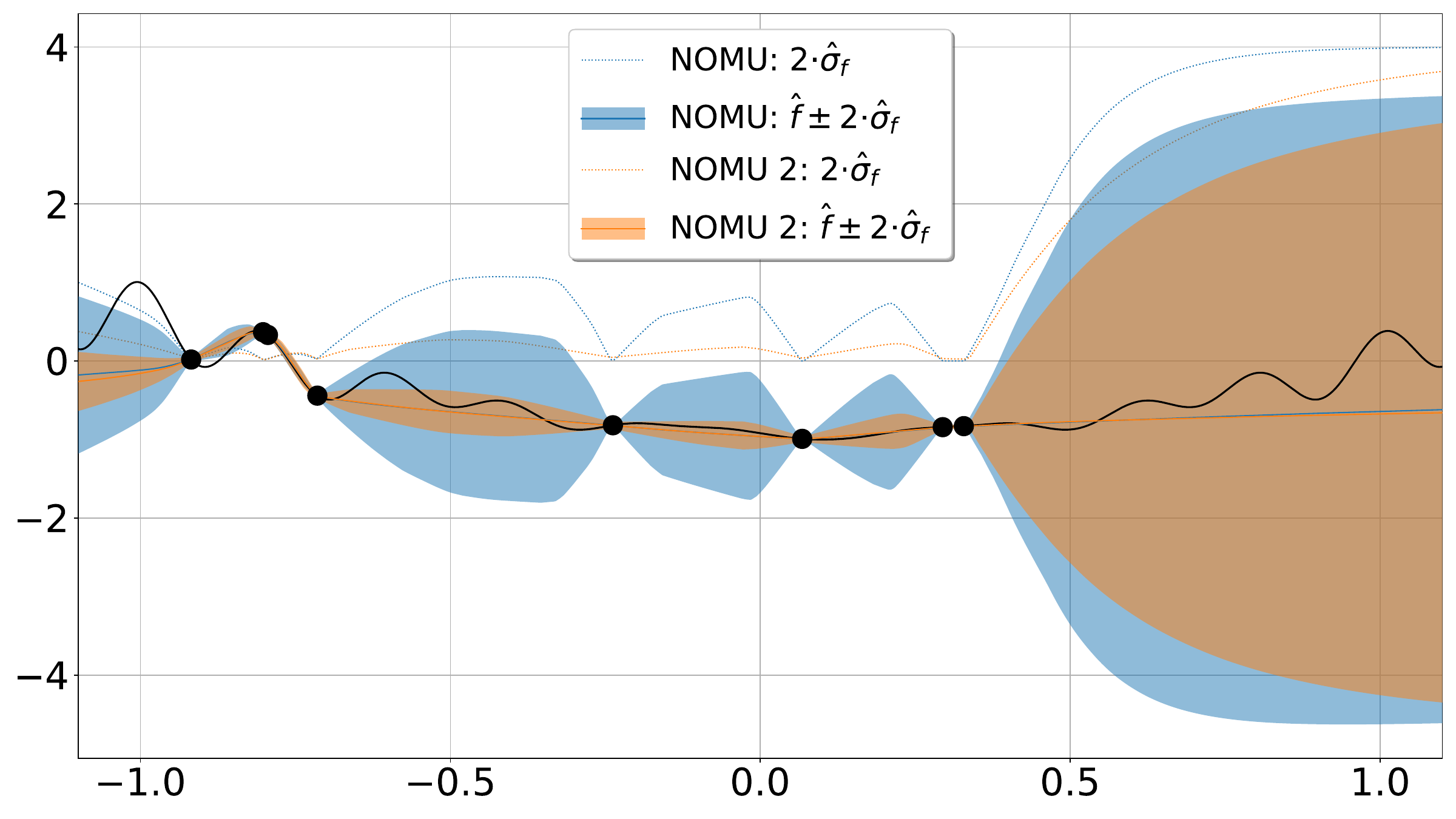}}
    \caption{NOMU's UBs (c=2) for $\cexp=15$ (blue) and $\cexp=60$ (orange).}
    \label{fig:sensitivity:cexp}
\end{center}
\vskip -0.2in
\end{figure}

\paragraph{Varying $\Daug$ with a Scaling Factor of $s=8$.}
Finally, we assess the effect of changing the number of artificial data points $\Daug$ used to approximate the integral \termc{} of NOMU's loss function defined in \Cref{eq:lmu2}. As expected, the UBs behave overall very similarly. However, for very small gaps in between input training points, $\Daug$ can have an influence on the estimated UBs. For example, in the gap between the training input point at $x\approx-0.77$ and the one at $x\approx-0.73$, the $\sigmodelhat$ obtained from the smaller $\Daug$ (blue) vanishes (because of a lack of artificial data points falling in this gap), while $\sigmodelhat$ obtained from the larger $\Daug$ also estimates non-zero model uncertainty in this small gap.

\begin{figure}[ht]
\vskip 0.2in
\begin{center}
\centerline{\includegraphics[width=\columnwidth]{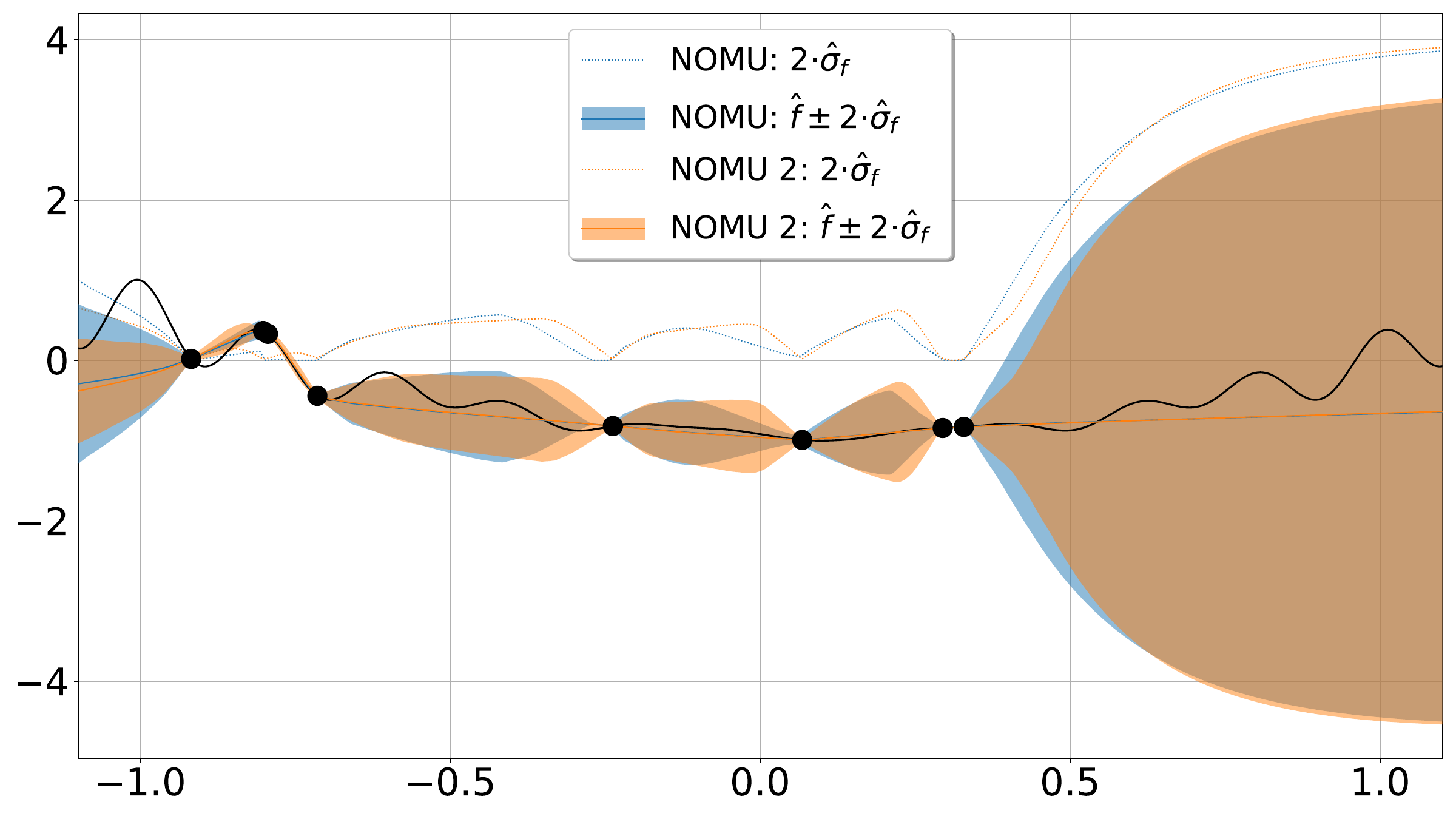}}
    \caption{NOMU's UBs (c=2) for $\Daug=16$ (blue) and $\Daug=1024$ (orange).}
    \label{fig:sensitivity:daug}
\end{center}
\vskip -0.2in
\end{figure}

\paragraph{Varying the architecture and L2-regularization}
In \Cref{subsec:VisualizeAxioms:irregular}, we visualize how certain changes to the architecture and different L2-regularization parameters $\lambda$ for different parts of the network influence NOMU's model uncertainty estimate $\sigmodelhat$. In particular, we show in \Cref{subsec:VisualizeAxioms:irregular} how the choice of the architecture and the L2-regularization determine the degree to which NOMU fulfills desiderata \ref{itm:Axioms:irregular}.

Finally, in \Cref{subsec:appendix:Quantitative Sensitivity Analysis}, we empirically show that NOMU is robust with respect to its hyperparameters within a certain range.

\subsection{Quantitative Sensitivity Analysis}\label{subsec:appendix:Quantitative Sensitivity Analysis}
We now present an extensive quantitative sensitivity analysis of NOMU's loss hyperparameters: $\musqr$, $\muexp$, $\cexp$, and $\Daug$ in the generative test-bed setting (see \Cref{subsubsec:GenerativeTestbed} for details on this setting). Additionally, we also consider in this analysis the hyperparameters corresponding to the readout map, i.e., $\sigmin$ and $\sigmax$. We decided to perform the quantitative sensitivity analysis in the generative test-bed setting, since it offers a particularly rich variety of different test functions and thus exposes each hyperparameter selection to hundreds of different test functions.

\paragraph{Setting} We use the same default hyperparameters as in \Cref{subsubsec:GenerativeTestbed}. This includes NOMU's loss hyperparameters: $\musqr$, $\muexp$, $\cexp$, and $\Daug$, the hyperparameters of the readout map: $\sigmin$ and $\sigmax$ as well as all other hyperparameters. For the following sensitivity analysis, we then vary at each time only a \emph{single} hyperparameter, i.e., one of $\musqr$, $\muexp$, $\cexp$, $\Daug$, $\sigmin$, and $\sigmax$, and set all other hyperparameters to their default values (i.e., we perform a \emph{ceteris paribus} analysis as in \Cref{subsec:appendix:Qualitative Sensitivity Analysis}).

In \Cref{tab:appendix:sensitivity_analysis_grids}, we present for each considered hyperparameter a grid of five different values which we use to test its sensitivity. The \textbf{\textsc{NOMU}} column in \Cref{tab:appendix:sensitivity_analysis_grids} corresponds to the NOMU's default hyperparameters used in the generative test-bed setting (\Cref{subsubsec:GenerativeTestbed}). The columns \textbf{\textsc{NOMU1}} to \textbf{\textsc{NOMU4}} in \Cref{tab:appendix:sensitivity_analysis_grids} then correspond to deviations from these original hyperparameters.

\begin{table}[b!]
    \robustify\bfseries
    \caption{Grid selection for each hyperparameter (HP). The column NOMU corresponds to NOMU's default hyperparameters used in the generative test-bed setting. NOMU1 to NOMU4 correspond to deviations from these default hyperparameters. For $\Daug$, $d$ denotes the input dimension.}
    \label{tab:appendix:sensitivity_analysis_grids}
    \vskip 0.1in
    \begin{center}
    \begin{small}
    \begin{sc}
    \resizebox{1\columnwidth}{!}{
    \setlength\tabcolsep{4pt}
    \begin{tabular}{
    		c
    		c
    		c
    		c
    		c
    		c
    		c
    	}
    	\toprule
    	\textbf{HP} & \multicolumn{1}{c}{\textbf{NOMU1}} &  \multicolumn{1}{c}{\textbf{NOMU2}}  & \multicolumn{1}{c}{\textbf{NOMU}}  &  \multicolumn{1}{c}{\textbf{NOMU3}} & \multicolumn{1}{c}{\textbf{NOMU4}}\\
    	\midrule
        $\musqr$ &0.01&0.02&0.1&0.5&1\\
        $\muexp$ &0.001&0.002&0.01&0.05&0.1\\
        $\cexp$ &10&20&30&45&90\\
        $\Daug$ &$\frac{100\cdot d}{4}$&$\frac{100\cdot d}{2}$&$100\cdot d$&$2\cdot\left(100\cdot d\right)$&$4\cdot\left(100\cdot d\right)$\\
        $\sigmin$ &0.01&0.05&0.1&0.15&0.20\\
        $\sigmax$ &0.50&0.75&1&2.0&4.0\\
    	\bottomrule 
    \end{tabular}
    }
    \end{sc}
    \end{small}
    \end{center}
    \vskip 0.1in
\end{table}
\paragraph{Results} In \Cref{tab:appendix:sensitivity_analysis_1d}, \Cref{tab:appendix:sensitivity_analysis_2d}, and \Cref{tab:appendix:sensitivity_analysis_5d} we present for each of the ceteris paribus runs average \NLPD{} values for input dimensions 1D, 2D, and 5D, respectively. Each cell in those tables represents a single hyperparameter selection where we use NOMU's default hyperparameters except for the hyperparameter of the corresponding row which we choose according to the cell's column, e.g., to obtain the result for the cell ($\muexp$, \textbf{\textsc{NOMU3}}), we use the default NOMU hyperparameters from \Cref{subsubsec:GenerativeTestbed} \emph{except} for $\muexp$ which we set according to column \textbf{\textsc{NOMU3}} in \Cref{tab:appendix:sensitivity_analysis_grids} to $\muexp:=0.05$. 
Overall, we can make the following three main observations:
\begin{enumerate}[leftmargin=*,topsep=0pt,partopsep=0pt, parsep=0pt]
    \item The majority of all cells in \Cref{tab:appendix:sensitivity_analysis_1d}, \Cref{tab:appendix:sensitivity_analysis_2d}, and \Cref{tab:appendix:sensitivity_analysis_5d} are marked in grey. This shows that, their corresponding hyperparameters lead to average \NLPD{} results which are statistically on par with the results obtained via NOMU's default hyperparameters in the \textbf{\textsc{NOMU}} columns. This highlights NOMU's robustness with respect to all considered hyperparameters within the chosen grids. Furthermore, it confirms our claim from the main paper that using generic hyperparameters for NOMU often works well without much hyperparameter-tuning.
    \item Since NOMU with the default hyperparameters already outperforms \emph{all} other considered benchmark methods in this setting, i.e., each \textbf{\textsc{NOMU}} column represents the winning method among benchmarks (see \Cref{tab:bnn_uniform}), we see that all grey marked deviations of hyperparameters lead to results that outperform all other considered benchmark methods too.
    Moreover, all except for one (5D ($\sigmin$, \textbf{\textsc{NOMU1}})) NOMU models corresponding to cells which are not marked in grey, i.e., with hyperparameters that lead to statistically worse results than NOMU's default hyperparameters, are as good or better than the best benchmark methods from \Cref{tab:bnn_uniform}.
    \item By varying NOMU's hyperparameters, we can even obtain better results (i.e., with a smaller average $\NLPD{}$) than the ones reported in \Cref{tab:bnn_uniform} of the main paper, e.g., in 1D with $\sigmin:=0.01$ the average $\NLPD{}=-1.83<-1.65$. While these improvements are not statistically significant, these results still suggest that systematic hyperparameter-tuning could improve the performance of NOMU even further.
\end{enumerate}
\begin{table}[b!]
    \setlength\tabcolsep{2pt}
    \caption{Sensitivity analysis for 1D generative test-bed setting. We present for each hyperparameter (HP) and its five corresponding grid-points the average \NLPD{} {\scriptsize(without const. $\ln(2\pi)/2$)} and a 95\% CI over 200 BNN samples. Results which are statistically on par with NOMU's default HPs, i.e., the column \textbf{\textsc{NOMU}}, are marked in grey. Note that, the best benchmark method for this experiment is GP with $\NLPD=-1.08\pm0.22$ (see \Cref{tab:bnn_uniform}).}
    \label{tab:appendix:sensitivity_analysis_1d}
    \vskip 0.1in
    \begin{center}
    \begin{small}
    \begin{sc}
    \resizebox{1\columnwidth}{!}{
    \begin{tabular}{
    	l
    	c
    	c
    	c
    	c
    	c
    	c
    	c
    	c
    	c
    	c
    }
    \toprule
    \textbf{HP} & \multicolumn{1}{c}{\textbf{NOMU1}} &  \multicolumn{1}{c}{\textbf{NOMU2}}  & \multicolumn{1}{c}{\textbf{NOMU}}  &  \multicolumn{1}{c}{\textbf{NOMU3}} & \multicolumn{1}{c}{\textbf{NOMU4}}\\
    \midrule
    $\musqr$ &\winc-1.59$\pm$\scriptsize0.11&\winc-1.63$\pm$\scriptsize0.10&\winc-1.65$\pm$\scriptsize0.10 &\winc-1.55$\pm$\scriptsize0.13&\winc-1.54$\pm$\scriptsize0.14&\\
    
    $\muexp$ & \winc-1.64$\pm$\scriptsize0.09&\winc-1.67$\pm$\scriptsize0.09&\winc-1.65$\pm$\scriptsize0.10&\winc-1.62$\pm$\scriptsize0.10&\winc-1.60$\pm$\scriptsize0.10\\
    
    $\cexp$ &\winc-1.77$\pm$\scriptsize0.12&\winc-1.73$\pm$\scriptsize0.10&\winc-1.65$\pm$\scriptsize0.10&\winc-1.49$\pm$\scriptsize0.12&-1.11$\pm$\scriptsize0.13\\
    
    $\Daug$ &\winc-1.65$\pm$\scriptsize0.09&\winc-1.62$\pm$\scriptsize0.11&\winc-1.65$\pm$\scriptsize0.10&\winc-1.63$\pm$\scriptsize0.12&\winc-1.65$\pm$\scriptsize0.10\\
    
    $\sigmin$ &\winc-1.83$\pm$\scriptsize0.09&\winc-1.72$\pm$\scriptsize0.10&\winc-1.65$\pm$\scriptsize0.10&\winc-1.58$\pm$\scriptsize0.10&\winc-1.47$\pm$\scriptsize0.12\\
    
    $\sigmax$ &\winc-1.49$\pm$\scriptsize0.12&\winc-1.60$\pm$\scriptsize0.10&\winc-1.65$\pm$\scriptsize0.10&\winc-1.63$\pm$\scriptsize0.16&\winc-1.66$\pm$\scriptsize0.14\\
    \bottomrule 
    \end{tabular}}
    \end{sc}
    \end{small}
    \end{center}
    \vskip -0.1in
    \end{table}
\begin{table}[t!]
    \setlength\tabcolsep{2pt}
    \caption{Sensitivity analysis for 2D generative test-bed setting. We present for each hyperparameter (HP) and its five corresponding grid-points the average \NLPD{} {\scriptsize(without const. $\ln(2\pi)/2$)} and a 95\% CI over 200 BNN samples. Results which are statistically on par with NOMU's default HPs, i.e., the column \textbf{\textsc{NOMU}}, are marked in grey. Note that, the best benchmark method for this experiment is DE with $\NLPD=-0.77\pm0.07$ (see \Cref{tab:bnn_uniform}).}
    \label{tab:appendix:sensitivity_analysis_2d}
    \vskip 0.1in
    \begin{center}
    \begin{small}
    \begin{sc}
    \resizebox{1\columnwidth}{!}{
    \begin{tabular}{
    	l
    	c
    	c
    	c
    	c
    	c
    	c
    	c
    	c
    	c
    	c
    }
    \toprule
    \textbf{HP} & \multicolumn{1}{c}{\textbf{NOMU1}} &  \multicolumn{1}{c}{\textbf{NOMU2}}  & \multicolumn{1}{c}{\textbf{NOMU}}  &  \multicolumn{1}{c}{\textbf{NOMU3}} & \multicolumn{1}{c}{\textbf{NOMU4}}\\
    \midrule
    $\musqr$ &\winc-1.18$\pm$\scriptsize0.04 &\winc -1.18$\pm$\scriptsize0.04 & \winc-1.16$\pm$\scriptsize0.05 & \winc-1.15$\pm$\scriptsize0.04 & \winc-1.15$\pm$\scriptsize0.04\\
    
    $\muexp$ &\winc-1.15$\pm$\scriptsize0.04&\winc-1.15$\pm$\scriptsize0.05&\winc-1.16$\pm$\scriptsize0.05&\winc-1.18$\pm$\scriptsize0.04&\winc-1.18$\pm$\scriptsize0.04\\
    
    $\cexp$ &\winc-1.17$\pm$\scriptsize0.04&\winc-1.19$\pm$\scriptsize0.04&\winc-1.16$\pm$\scriptsize0.05&\winc-1.11$\pm$\scriptsize0.05&-1.00$\pm$\scriptsize0.05\\
    
    $\Daug$ &\winc-1.16$\pm$\scriptsize0.05 &\winc-1.16$\pm$\scriptsize0.05 &\winc-1.16$\pm$\scriptsize0.05 &\winc-1.16$\pm$\scriptsize0.04&\winc-1.15$\pm$\scriptsize0.05\\
    
    $\sigmin$ &\winc-1.07$\pm$\scriptsize0.05&\winc-1.17$\pm$\scriptsize0.04&\winc-1.16$\pm$\scriptsize0.05&\winc-1.14$\pm$\scriptsize0.05&\winc-1.12$\pm$\scriptsize0.05\\
    
	$\sigmax$ &\winc-1.13$\pm$\scriptsize0.05&\winc-1.15$\pm$\scriptsize0.05&\winc-1.16$\pm$\scriptsize0.05&\winc-1.16$\pm$\scriptsize0.04&\winc-1.16$\pm$\scriptsize0.04\\
    
    \bottomrule 
    \end{tabular}
    }
    \end{sc}
    \end{small}
    \end{center}
    \vskip -0.1in
\end{table}
\begin{table}[t!]
    \setlength\tabcolsep{2pt}
    \caption{Sensitivity analysis for 5D generative test-bed setting. We present for each hyperparameter (HP) and its five corresponding grid-points the average \NLPD{} {\scriptsize(without const. $\ln(2\pi)/2$)} and a 95\% CI over 200 BNN samples. Results which are statistically on par with NOMU's default HPs, i.e., the column \textbf{\textsc{NOMU}}, are marked in grey. Note that, the best benchmark method for this experiment is GP with $\NLPD=-0.33\pm0.02$ (see \Cref{tab:bnn_uniform}).}
    \label{tab:appendix:sensitivity_analysis_5d}
    \vskip 0.1in
    \begin{center}
    \begin{small}
    \begin{sc}
    \resizebox{1\columnwidth}{!}{
    \begin{tabular}{
    	l
    	c
    	c
    	c
    	c
    	c
    	c
    	c
    	c
    	c
    	c
    }
    \toprule
    \textbf{HP} & \multicolumn{1}{c}{\textbf{NOMU1}} &  \multicolumn{1}{c}{\textbf{NOMU2}}  & \multicolumn{1}{c}{\textbf{NOMU}}  &  \multicolumn{1}{c}{\textbf{NOMU3}} & \multicolumn{1}{c}{\textbf{NOMU4}}\\
    \midrule
    $\musqr$ &\winc-0.37$\pm$\scriptsize0.03&\winc-0.37$\pm$\scriptsize0.02&\winc-0.37$\pm$\scriptsize0.02&\winc-0.36$\pm$\scriptsize0.02&\winc-0.35$\pm$\scriptsize0.02\\
    
    $\muexp$ &-0.33$\pm$\scriptsize0.02&\winc-0.34$\pm$\scriptsize0.02&\winc-0.37$\pm$\scriptsize0.02&\winc-0.40$\pm$\scriptsize0.02&\winc-0.40$\pm$\scriptsize0.02\\
    
    $\cexp$ &\winc-0.39$\pm$\scriptsize0.02&\winc-0.38$\pm$\scriptsize0.02&\winc-0.37$\pm$\scriptsize0.02&\winc-0.37$\pm$\scriptsize0.02&\winc-0.34$\pm$\scriptsize0.05\\
    
    $\Daug$ &\winc-0.37$\pm$\scriptsize0.02&\winc-0.37$\pm$\scriptsize0.02&\winc-0.37$\pm$\scriptsize0.02&\winc-0.37$\pm$\scriptsize0.02&\winc-0.37$\pm$\scriptsize0.02\\
    
    $\sigmin$ &-0.21$\pm$\scriptsize0.03&-0.33$\pm$\scriptsize0.02&\winc-0.37$\pm$\scriptsize0.02&\winc-0.38$\pm$\scriptsize0.02&\winc-0.39$\pm$\scriptsize0.02\\
    
    $\sigmax$ &\winc-0.39$\pm$\scriptsize0.02&\winc-0.37$\pm$\scriptsize0.02&\winc-0.37$\pm$\scriptsize0.02&\winc-0.35$\pm$\scriptsize0.02&-0.33$\pm$\scriptsize0.02\\
    \bottomrule 
    \end{tabular}
    }
    \end{sc}
    \end{small}
    \end{center}
    \vskip -0.1in
\end{table}
\section{Details on our Notation}\label{sec:DetailsNotation}
In \cref{sec:Bayesian Uncertainty Framework}, we are using a slightly overloaded notation, where we use the same symbol $f$ for different mathematical objects. Sometimes, we use $f$ for a function-valued random variable $F:(\Omega,\Sigma,\PP)\to Y^X$. and sometimes we use $f$ for the specific unknown ground truth function $f_\text{true}:=F(\omega)$ (i.e., $\forall x\in X: f_\text{true}(x)=(F(\omega))(x)$). While we used this slightly overloaded notation for the sake of readability in the main paper, we will now introduce our Bayesian uncertainty framework in its full mathematical detail:

In practice, there exists an unknown ground truth function $f_\text{true}$. In the classical Bayesian paradigm, one assumes that everything unknown (i.e., here $f_\text{true}$) was sampled from a prior distribution. Specifically, $f_\text{true}:=F(\omega)$ is a realization of a random variable $F:(\Omega,\Sigma,\PP)\to Y^X$ distributed according to a prior belief, i.e., a prior distribution. Using this notation, one can mathematically describe in a rigorous way the full Bayesian data generating process as follows:

Let $F:(\Omega,\Sigma,\PP)\to Y^X$ be a function-valued random variable distributed according to a prior belief, i.e., a prior distribution. Moreover, let $(\mathcal{X}_i,\mathcal{E}_i)_{i\in\fromto{\ntr}}$ denote the random variable representing the input data points and the corresponding data noise, i.e., a family off i.i.d random variables independent of $F$, where each random variable $(\mathcal{X}_i,\mathcal{E}_i):(\Omega,\Sigma,\PP)\to X\times \R$ fulfills $\forall x \in X: (\mathcal{E}_i | \mathcal{X}_i=x)\sim\mathcal{N}\left(0,\signoise(x)\right)$. Finally, let $\mathcal{Y}_i$ be the random variable associated to the targets, which we define as $\mathcal{Y}_i:(\Omega,\Sigma,\PP)\to Y,\omega\mapsto \mathcal{Y}_i(\omega):=F(\omega)(\mathcal{X}_i(\omega))+\mathcal{E}_i(\omega)$.

With this notation in place, the objects used in the main paper in \cref{sec:Bayesian Uncertainty Framework}: $x_i=\mathcal{X}_i(\omega)$, $y_i=\mathcal{Y}_i(\omega)$ and $\varepsilon_i=\mathcal{E}_i(\omega)$ are the values one actually observe in practice as training data, i.e., the realizations of the data generating process. 

Therefore, $\sigmodel(x)$ from \Cref{eq:modelUncertainty} should be interpreted for all $x\in \X$ as
\begin{align}
&\sigmodel(x)=\nonumber\\
&\sqrt{\mathbb{V}\left[F(\cdot)(x)\middle| \forall i \in \fromto{\ntr}:(\mathcal{X}_i,\mathcal{Y}_i)=(x_i,y_i)\right]}, 
\end{align}
which, for all $x\in \X$, is defined mathematically even more rigorously via the conditional variance as follows: 
\begin{align}
\sqrt{\mathbb{V}\left[F(\cdot)(x)\middle|(\mathcal{X}_i,\mathcal{Y}_i)_{i\in\fromto{\ntr}}\right](\omega)}.
\end{align}
Throughout the paper it should always be clear from the context if $x_i$ refers to the random variable $\mathcal{X}_i$ or its realization $\mathcal{X}_i(\omega)$ where the same holds true for $y_i$, $\varepsilon_i$ and $f$.

Importantly, note that in the setting which we consider in this paper, one only observes data coming from a \emph{single} function $f=F(\omega)$ and one does not have access to more functions $f_i$. Specifically, we neither have access to other samples of the random variable $F$ nor do we consider multiple i.i.d random variables $F_i$ in contrast to the setting considered for neural processes as described in \Cref{sec:appendix:NOMU vs. Neural Processes}.

\section{Visualization of the Readout Map}\label{sec:readout_map}
In this section, we provide in \Cref{fig:app:readout_map} a visualization of the readout map
\begin{equation}\label{eq:app:readout_map}
\readout(z)=\sigmax\left(1-\exp\left(-\frac{\max(0,z)+\sigmin}{\sigmax}\right)\right),
\end{equation}
for $\sigmin\ge0$ and $\sigmax>0$. The readout map $\readout$ was introduced in \Cref{subsec:The Network Architecture} to transform the \emph{raw} model uncertainty output $\sigmodelhatraw$ into NOMU's \emph{model uncertainty prediction} $\sigmodelhat(x)=\readout(\sigmodelhatraw(x))$ (see \Cref{eq:modelUncertaintyPrediction}).
 \begin{figure}[t!]
    \begin{center}
    \centerline{\includegraphics[width=0.9\columnwidth]{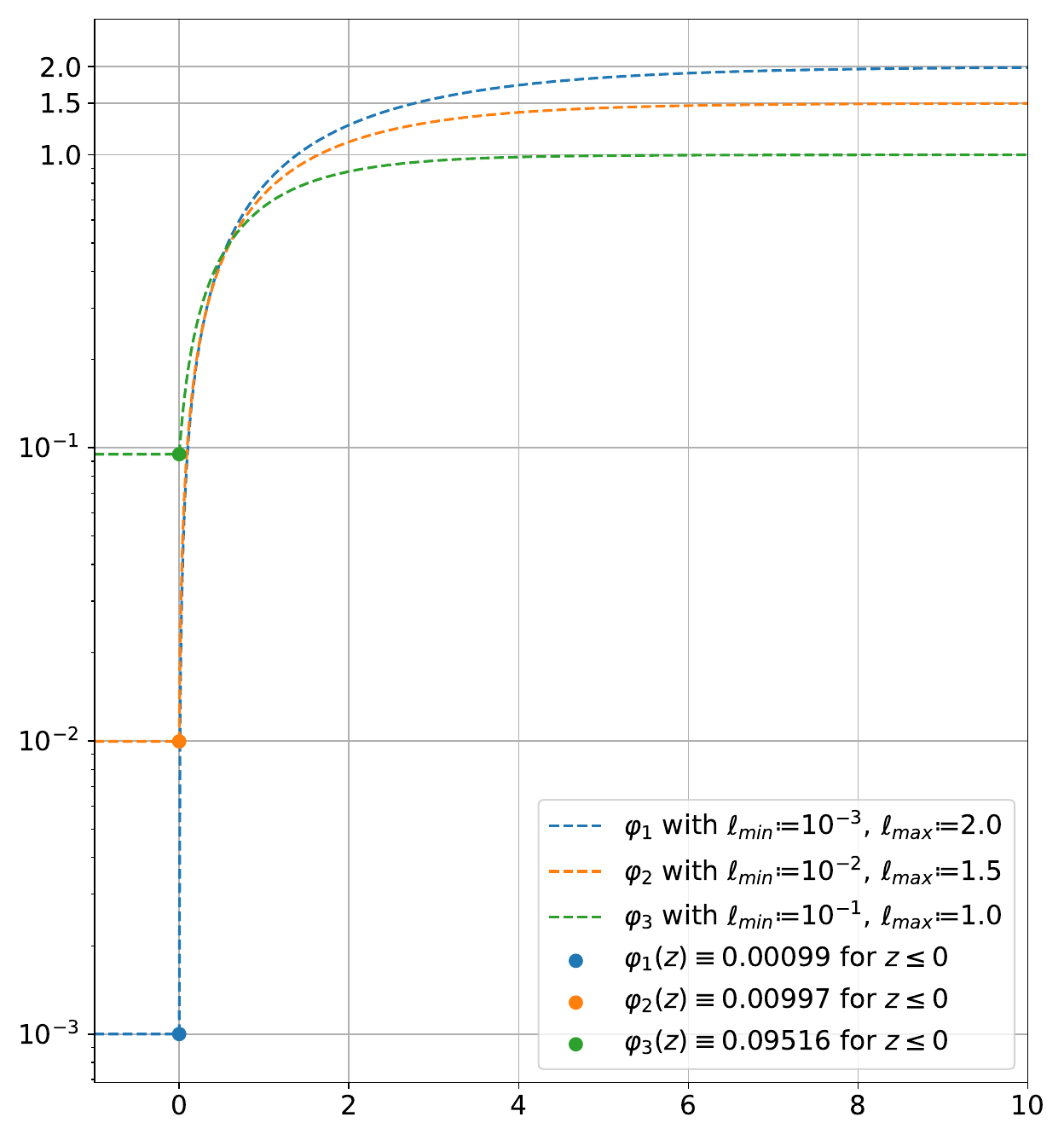}}
    \caption{Visualization of the readout map $\varphi(z)$ for three different $(\sigmin,\sigmax)$-pairs.}
    \label{fig:app:readout_map}
    \end{center}
    \vskip -0.2in
\end{figure}

\Cref{fig:app:readout_map} shows how the readout map $\readout$ interpolates between its minimal value $\sigmax(1-\exp(-\frac{\sigmin}{\sigmax}))\approx\sigmax(1-1+\frac{\sigmin}{\sigmax}))=\sigmin$ (shown in \Cref{fig:app:readout_map} for three different pairs of $(\sigmin,\sigmax)$ as the three dots) and its asymptotic maximal value $\sigmax=\lim_{z\to\infty}\readout(z)$. Specifically, we designed the readout map $\readout$ such that it has a relatively steep increase
starting from $0$ and flattens when asymptotically converging to its maximal value $\sigmax$. The parameter $\sigmin$ is used for numerical stability to prevent $0$ model uncertainty outputs, which would be a very extreme statements, causing metrics such as NLL or AUC to result in infinitely bad values for arbitrarily small numerical deviations of the model predictions.
The parameter $\sigmax$ defines the maximal model uncertainty far away from input training points (similarly to the prior variance for GPs with RBF kernel).
As discussed in \Cref{rem:different_readout}, the readout map $\readout$ can be modified depending on the subsequent use of the estimated UBs or it can even be learned by parameterizing $\readout$ by a NN and  minimizing a valid metric (such as the NLL) on a validation set.
}


\begin{thebibliography}{47}
\providecommand{\natexlab}[1]{#1}
\providecommand{\url}[1]{\texttt{#1}}
\expandafter\ifx\csname urlstyle\endcsname\relax
  \providecommand{\doi}[1]{doi: #1}\else
  \providecommand{\doi}{doi: \begingroup \urlstyle{rm}\Url}\fi

\bibitem[Ashukha et~al.(2020)Ashukha, Lyzhov, Molchanov, and
  Vetrov]{ashukha2020pitfalls}
Ashukha, A., Lyzhov, A., Molchanov, D., and Vetrov, D.
\newblock Pitfalls of in-domain uncertainty estimation and ensembling in deep
  learning.
\newblock In \emph{International Conference on Learning Representations}, 2020.
\newblock URL \url{https://openreview.net/forum?id=BJxI5gHKDr}.

\bibitem[Baptista \& Poloczek(2018)Baptista and Poloczek]{baptista2018bayesian}
Baptista, R. and Poloczek, M.
\newblock Bayesian optimization of combinatorial structures.
\newblock In \emph{International Conference on Machine Learning}, pp.\
  462--471. PMLR, 2018.
\newblock URL
  \url{http://proceedings.mlr.press/v80/baptista18a/baptista18a.pdf}.

\bibitem[Bergstra \& Bengio(2012)Bergstra and Bengio]{bergstra2012random}
Bergstra, J. and Bengio, Y.
\newblock Random search for hyper-parameter optimization.
\newblock \emph{Journal of machine learning research}, 13\penalty0 (2), 2012.
\newblock URL \url{https://www.jmlr.org/papers/v13/bergstra12a.html}.

\bibitem[Blundell et~al.(2015)Blundell, Cornebise, Kavukcuoglu, and
  Wierstra]{blundell2015weight}
Blundell, C., Cornebise, J., Kavukcuoglu, K., and Wierstra, D.
\newblock Weight uncertainty in neural networks.
\newblock In \emph{32nd International Conference on Machine Learning (ICML)},
  2015.
\newblock URL \url{http://proceedings.mlr.press/v37/blundell15.pdf}.

\bibitem[Bronstein et~al.(2017)Bronstein, Bruna, LeCun, Szlam, and
  Vandergheynst]{bronstein2017geometric}
Bronstein, M.~M., Bruna, J., LeCun, Y., Szlam, A., and Vandergheynst, P.
\newblock Geometric deep learning: going beyond euclidean data.
\newblock \emph{IEEE Signal Processing Magazine}, 34\penalty0 (4):\penalty0
  18--42, 2017.
\newblock URL \url{https://arxiv.org/abs/1611.08097}.

\bibitem[Caruana et~al.(2004)Caruana, Niculescu-Mizil, Crew, and
  Ksikes]{caruana2004ensemble}
Caruana, R., Niculescu-Mizil, A., Crew, G., and Ksikes, A.
\newblock Ensemble selection from libraries of models.
\newblock In \emph{Proceedings of the twenty-first international conference on
  Machine learning}, pp.\ ~18, 2004.
\newblock URL
  \url{https://www.cs.cornell.edu/~caruana/ctp/ct.papers/caruana.icml04.icdm06long.pdf}.

\bibitem[Cayton(2005)]{cayton2005algorithms}
Cayton, L.
\newblock Algorithms for manifold learning.
\newblock \emph{Univ. of California at San Diego Tech. Rep}, 12\penalty0
  (1-17):\penalty0 1, 2005.
\newblock URL \url{https://www.lcayton.com/resexam.pdf}.

\bibitem[Foong et~al.(2019)Foong, Li, Hern{\'a}ndez-Lobato, and
  Turner]{foong2019inbetween}
Foong, A.~Y., Li, Y., Hern{\'a}ndez-Lobato, J.~M., and Turner, R.~E.
\newblock \enquote{In-between} uncertainty in bayesian neural networks.
\newblock \emph{arXiv preprint arXiv:1906.11537}, 2019.
\newblock URL \url{https://arxiv.org/abs/1906.11537}.

\bibitem[Fort et~al.(2019)Fort, Hu, and Lakshminarayanan]{fort2019deep}
Fort, S., Hu, H., and Lakshminarayanan, B.
\newblock Deep ensembles: A loss landscape perspective.
\newblock \emph{arXiv preprint arXiv:1912.02757}, 2019.
\newblock URL \url{https://arxiv.org/abs/1912.02757}.

\bibitem[Gal \& Ghahramani(2016)Gal and Ghahramani]{gal2016dropout}
Gal, Y. and Ghahramani, Z.
\newblock Dropout as a bayesian approximation: Representing model uncertainty
  in deep learning.
\newblock In \emph{33rd International Conference on Machine Learning (ICML)},
  pp.\  1050--1059, 2016.
\newblock URL \url{https://proceedings.mlr.press/v48/gal16.html}.

\bibitem[Garnelo et~al.(2018{\natexlab{a}})Garnelo, Rosenbaum, Maddison,
  Ramalho, Saxton, Shanahan, Teh, Rezende, and
  Eslami]{conditional_neural_proc_pmlr-v80-garnelo18a}
Garnelo, M., Rosenbaum, D., Maddison, C., Ramalho, T., Saxton, D., Shanahan,
  M., Teh, Y.~W., Rezende, D., and Eslami, S.~A.
\newblock Conditional neural processes.
\newblock In \emph{International Conference on Machine Learning}, pp.\
  1704--1713. PMLR, 2018{\natexlab{a}}.
\newblock URL \url{https://proceedings.mlr.press/v80/garnelo18a.html}.

\bibitem[Garnelo et~al.(2018{\natexlab{b}})Garnelo, Schwarz, Rosenbaum, Viola,
  Rezende, Eslami, and Teh]{neural_proc}
Garnelo, M., Schwarz, J., Rosenbaum, D., Viola, F., Rezende, D.~J., Eslami, S.,
  and Teh, Y.~W.
\newblock Neural processes.
\newblock \emph{arXiv preprint arXiv:1807.01622}, 2018{\natexlab{b}}.
\newblock URL \url{https://arxiv.org/abs/1807.01622}.

\bibitem[Ghahramani(2015)]{ghahramani2015probabilistic}
Ghahramani, Z.
\newblock Probabilistic machine learning and artificial intelligence.
\newblock \emph{Nature}, 521\penalty0 (7553):\penalty0 452--459, 2015.
\newblock URL \url{https://www.nature.com/articles/nature14541}.

\bibitem[G{\'o}mez-Bombarelli et~al.(2018)G{\'o}mez-Bombarelli, Wei, Duvenaud,
  Hern{\'a}ndez-Lobato, S{\'a}nchez-Lengeling, Sheberla, Aguilera-Iparraguirre,
  Hirzel, Adams, and Aspuru-Guzik]{gomez2018automatic}
G{\'o}mez-Bombarelli, R., Wei, J.~N., Duvenaud, D., Hern{\'a}ndez-Lobato,
  J.~M., S{\'a}nchez-Lengeling, B., Sheberla, D., Aguilera-Iparraguirre, J.,
  Hirzel, T.~D., Adams, R.~P., and Aspuru-Guzik, A.
\newblock Automatic chemical design using a data-driven continuous
  representation of molecules.
\newblock \emph{ACS central science}, 4\penalty0 (2):\penalty0 268--276, 2018.
\newblock URL \url{https://doi.org/10.1021/acscentsci.7b00572}.

\bibitem[Goodfellow et~al.(2014)Goodfellow, Pouget-Abadie, Mirza, Xu,
  Warde-Farley, Ozair, Courville, and Bengio]{goodfellow2014gen}
Goodfellow, I., Pouget-Abadie, J., Mirza, M., Xu, B., Warde-Farley, D., Ozair,
  S., Courville, A., and Bengio, Y.
\newblock Generative adversarial nets.
\newblock \emph{Advances in neural information processing systems}, 27, 2014.
\newblock URL
  \url{https://proceedings.neurips.cc/paper/2014/file/5ca3e9b122f61f8f06494c97b1afccf3-Paper.pdf}.

\bibitem[Graves(2011)]{graves2011practical}
Graves, A.
\newblock Practical variational inference for neural networks.
\newblock In \emph{Advances in neural information processing systems}, pp.\
  2348--2356, 2011.
\newblock URL
  \url{http://papers.nips.cc/paper/4329-practical-variational-inference-for-neural-networks.pdf}.

\bibitem[Gustafsson et~al.(2020)Gustafsson, Danelljan, and
  Schon]{gustafsson2020evaluating}
Gustafsson, F.~K., Danelljan, M., and Schon, T.~B.
\newblock Evaluating scalable bayesian deep learning methods for robust
  computer vision.
\newblock In \emph{Proceedings of the IEEE/CVF Conference on Computer Vision
  and Pattern Recognition Workshops}, pp.\  318--319, 2020.
\newblock URL \url{https://arxiv.org/abs/1906.01620}.

\bibitem[Havasi et~al.(2021)Havasi, Jenatton, Fort, Liu, Snoek,
  Lakshminarayanan, Dai, and Tran]{havasi2021training}
Havasi, M., Jenatton, R., Fort, S., Liu, J.~Z., Snoek, J., Lakshminarayanan,
  B., Dai, A.~M., and Tran, D.
\newblock Training independent subnetworks for robust prediction.
\newblock In \emph{International Conference on Learning Representations}, 2021.
\newblock URL \url{https://openreview.net/forum?id=OGg9XnKxFAH}.

\bibitem[Heiss et~al.(2019)Heiss, Teichmann, and Wutte]{heiss2019implicit1}
Heiss, J., Teichmann, J., and Wutte, H.
\newblock How implicit regularization of relu neural networks characterizes the
  learned function -- part i: the 1-d case of two layers with random first
  layer.
\newblock \emph{arXiv preprint arXiv:1911.02903}, 2019.
\newblock URL \url{https://doi.org/10.3929/ethz-b-000402003}.

\bibitem[Heiss et~al.(2022)Heiss, Teichmann, and Wutte]{HeissPart3MultiTask}
Heiss, J., Teichmann, J., and Wutte, H.
\newblock How infinitely wide neural networks can benefit from multi-task
  learning - an exact macroscopic characterization.
\newblock \emph{arXiv preprint arXiv:2112.15577}, 2022.
\newblock \doi{10.3929/ETHZ-B-000550890}.
\newblock URL \url{https://arxiv.org/abs/2112.15577}.

\bibitem[Hern{\'a}ndez-Lobato \& Adams(2015)Hern{\'a}ndez-Lobato and
  Adams]{hernandez2015probabilistic}
Hern{\'a}ndez-Lobato, J.~M. and Adams, R.
\newblock Probabilistic backpropagation for scalable learning of bayesian
  neural networks.
\newblock In \emph{International Conference on Machine Learning}, pp.\
  1861--1869, 2015.
\newblock URL \url{http://proceedings.mlr.press/v37/hernandez-lobatoc15.pdf}.

\bibitem[Kendall \& Gal(2017)Kendall and Gal]{kendall2017uncertainties}
Kendall, A. and Gal, Y.
\newblock What uncertainties do we need in bayesian deep learning for computer
  vision?
\newblock In \emph{Advances in neural information processing systems}, pp.\
  5574--5584, 2017.
\newblock URL
  \url{https://papers.nips.cc/paper/2017/hash/2650d6089a6d640c5e85b2b88265dc2b-Abstract.html}.

\bibitem[Khosravi et~al.(2010)Khosravi, Nahavandi, Creighton, and
  Atiya]{khosravi2010lower}
Khosravi, A., Nahavandi, S., Creighton, D., and Atiya, A.~F.
\newblock Lower upper bound estimation method for construction of neural
  network-based prediction intervals.
\newblock \emph{IEEE transactions on neural networks}, 22\penalty0
  (3):\penalty0 337--346, 2010.
\newblock URL \url{https://doi.org/10.1109/TNN.2010.2096824}.

\bibitem[Kuleshov et~al.(2018)Kuleshov, Fenner, and
  Ermon]{pmlr-v80-kuleshov18a}
Kuleshov, V., Fenner, N., and Ermon, S.
\newblock Accurate uncertainties for deep learning using calibrated regression.
\newblock In Dy, J. and Krause, A. (eds.), \emph{Proceedings of the 35th
  International Conference on Machine Learning}, volume~80 of \emph{Proceedings
  of Machine Learning Research}, pp.\  2796--2804, Stockholmsmässan, Stockholm
  Sweden, 10--15 Jul 2018. PMLR.
\newblock URL \url{http://proceedings.mlr.press/v80/kuleshov18a.html}.

\bibitem[Lakshminarayanan et~al.(2017)Lakshminarayanan, Pritzel, and
  Blundell]{lakshminarayanan2017simple}
Lakshminarayanan, B., Pritzel, A., and Blundell, C.
\newblock Simple and scalable predictive uncertainty estimation using deep
  ensembles.
\newblock In \emph{Advances in neural information processing systems}, pp.\
  6402--6413, 2017.
\newblock URL
  \url{http://papers.nips.cc/paper/7219-simple-and-scalable-predictive-uncertainty-estimation-using-deep-ensembles.pdf}.

\bibitem[LeCun et~al.(2015)LeCun, Bengio, and Hinton]{LeCun2015}
LeCun, Y., Bengio, Y., and Hinton, G.
\newblock Deep learning.
\newblock \emph{Nature}, 521\penalty0 (7553):\penalty0 436–444, 2015.
\newblock ISSN 1476-4687.
\newblock \doi{10.1038/nature14539}.
\newblock URL \url{https://www.nature.com/articles/nature14539}.

\bibitem[Malinin \& Gales(2018)Malinin and Gales]{malinin2018predictive}
Malinin, A. and Gales, M.
\newblock Predictive uncertainty estimation via prior networks.
\newblock In \emph{Advances in Neural Information Processing Systems}, pp.\
  7047--7058, 2018.
\newblock URL
  \url{https://papers.nips.cc/paper/2018/file/3ea2db50e62ceefceaf70a9d9a56a6f4-Paper.pdf}.

\bibitem[Malinin et~al.(2020{\natexlab{a}})Malinin, Chervontsev, Provilkov, and
  Gales]{malinin2020regression}
Malinin, A., Chervontsev, S., Provilkov, I., and Gales, M.
\newblock Regression prior networks, 2020{\natexlab{a}}.
\newblock URL \url{https://arxiv.org/pdf/2006.11590.pdf}.

\bibitem[Malinin et~al.(2020{\natexlab{b}})Malinin, Mlodozeniec, and
  Gales]{malinin2019ensemble}
Malinin, A., Mlodozeniec, B., and Gales, M.
\newblock Ensemble distribution distillation.
\newblock In \emph{International Conference on Learning Representations},
  2020{\natexlab{b}}.
\newblock URL \url{https://openreview.net/forum?id=BygSP6Vtvr}.

\bibitem[Martinez-Cantin et~al.(2009)Martinez-Cantin, De~Freitas, Brochu,
  Castellanos, and Doucet]{martinez2009bayesian}
Martinez-Cantin, R., De~Freitas, N., Brochu, E., Castellanos, J., and Doucet,
  A.
\newblock A bayesian exploration-exploitation approach for optimal online
  sensing and planning with a visually guided mobile robot.
\newblock \emph{Autonomous Robots}, 27\penalty0 (2):\penalty0 93--103, 2009.
\newblock URL \url{https://doi.org/10.1007/s10514-009-9130-2}.

\bibitem[Mukhoti et~al.(2018)Mukhoti, Stenetorp, and
  Gal]{mukhoti2018importance}
Mukhoti, J., Stenetorp, P., and Gal, Y.
\newblock On the importance of strong baselines in bayesian deep learning.
\newblock \emph{arXiv preprint arXiv:1811.09385}, 2018.
\newblock URL \url{https://arxiv.org/abs/1811.09385}.

\bibitem[Neal(2012)]{neal2012bayesian}
Neal, R.~M.
\newblock \emph{Bayesian learning for neural networks}, volume 118.
\newblock Springer Science \& Business Media, 2012.
\newblock URL \url{https://api.semanticscholar.org/CorpusID:60809283}.

\bibitem[Nix \& Weigend(1994)Nix and Weigend]{nix1994estimating}
Nix, D.~A. and Weigend, A.~S.
\newblock Estimating the mean and variance of the target probability
  distribution.
\newblock In \emph{Proceedings of 1994 ieee international conference on neural
  networks (ICNN'94)}, volume~1, pp.\  55--60. IEEE, 1994.
\newblock URL \url{https://doi.org/10.1109/ICNN.1994.374138}.

\bibitem[Ober \& Rasmussen(2019)Ober and Rasmussen]{ober2019benchmarking}
Ober, S.~W. and Rasmussen, C.~E.
\newblock Benchmarking the neural linear model for regression.
\newblock \emph{arXiv preprint arXiv:1912.08416}, 2019.
\newblock URL \url{https://arxiv.org/abs/1912.08416}.

\bibitem[Osband et~al.(2021)Osband, Wen, Asghari, Ibrahimi, Lu, and
  Van~Roy]{osband2021epistemic}
Osband, I., Wen, Z., Asghari, M., Ibrahimi, M., Lu, X., and Van~Roy, B.
\newblock Epistemic neural networks.
\newblock \emph{arXiv preprint arXiv:2107.08924}, 2021.
\newblock URL \url{https://arxiv.org/abs/2107.08924}.

\bibitem[Ovadia et~al.(2019)Ovadia, Fertig, Ren, Nado, Sculley, Nowozin,
  Dillon, Lakshminarayanan, and Snoek]{ovadia2019can}
Ovadia, Y., Fertig, E., Ren, J., Nado, Z., Sculley, D., Nowozin, S., Dillon,
  J., Lakshminarayanan, B., and Snoek, J.
\newblock Can you trust your model\textquotesingle s uncertainty? evaluating
  predictive uncertainty under dataset shift.
\newblock In \emph{Advances in Neural Information Processing Systems},
  volume~32. Curran Associates, Inc., 2019.
\newblock URL
  \url{https://proceedings.neurips.cc/paper/2019/file/8558cb408c1d76621371888657d2eb1d-Paper.pdf}.

\bibitem[Pearce et~al.(2018)Pearce, Zaki, Brintrup, and Neely]{pearce2018high}
Pearce, T., Zaki, M., Brintrup, A., and Neely, A.
\newblock High-quality prediction intervals for deep learning: A
  distribution-free, ensembled approach.
\newblock In \emph{35th International Conference on Machine Learning (ICML)},
  pp.\  4075--4084, 2018.
\newblock URL \url{https://proceedings.mlr.press/v80/pearce18a.html}.

\bibitem[Srinivas et~al.(2012)Srinivas, Krause, Kakade, and
  Seeger]{Srinivas_2012}
Srinivas, N., Krause, A., Kakade, S.~M., and Seeger, M.~W.
\newblock Information-theoretic regret bounds for gaussian process optimization
  in the bandit setting.
\newblock \emph{IEEE Transactions on Information Theory}, 58\penalty0
  (5):\penalty0 3250–3265, May 2012.
\newblock ISSN 1557-9654.
\newblock \doi{10.1109/tit.2011.2182033}.
\newblock URL \url{http://dx.doi.org/10.1109/TIT.2011.2182033}.

\bibitem[Steinhilber et~al.(2009)Steinhilber, Beer, and
  Fröhlich]{doi:10.1029/2009GL040142}
Steinhilber, F., Beer, J., and Fröhlich, C.
\newblock Total solar irradiance during the holocene.
\newblock \emph{Geophysical Research Letters}, 36\penalty0 (19), 2009.
\newblock \doi{10.1029/2009GL040142}.
\newblock URL
  \url{https://agupubs.onlinelibrary.wiley.com/doi/abs/10.1029/2009GL040142}.

\bibitem[Wahba(1978)]{wahba1978improper}
Wahba, G.
\newblock Improper priors, spline smoothing and the problem of guarding against
  model errors in regression.
\newblock \emph{Journal of the Royal Statistical Society: Series B
  (Methodological)}, 40\penalty0 (3):\penalty0 364--372, 1978.
\newblock URL \url{https://doi.org/10.1111/j.2517-6161.1978.tb01050.x}.

\bibitem[Wang et~al.(2019)Wang, Ren, Zhu, and Zhang]{wang2019function}
Wang, Z., Ren, T., Zhu, J., and Zhang, B.
\newblock Function space particle optimization for bayesian neural networks.
\newblock In \emph{International Conference on Learning Representations}, 2019.
\newblock URL \url{https://openreview.net/forum?id=BkgtDsCcKQ}.

\bibitem[Weissteiner \& Seuken(2020)Weissteiner and
  Seuken]{weissteiner2020deep}
Weissteiner, J. and Seuken, S.
\newblock Deep learning—powered iterative combinatorial auctions.
\newblock \emph{Proceedings of the AAAI Conference on Artificial Intelligence},
  34\penalty0 (02):\penalty0 2284--2293, Apr. 2020.
\newblock \doi{10.1609/aaai.v34i02.5606}.
\newblock URL \url{https://ojs.aaai.org/index.php/AAAI/article/view/5606}.

\bibitem[Weissteiner et~al.(2023)Weissteiner, Heiss, Siems, and
  Seuken]{weissteiner2023bayesian}
Weissteiner, J., Heiss, J., Siems, J., and Seuken, S.
\newblock Bayesian optimization-based combinatorial assignment.
\newblock \emph{Proceedings of the AAAI Conference on Artificial Intelligence},
  37, 2023.
\newblock URL \url{https://arxiv.org/abs/2208.14698}.

\bibitem[Wen et~al.(2020)Wen, Tran, and Ba]{wen2020batchensemble}
Wen, Y., Tran, D., and Ba, J.
\newblock Batchensemble: An alternative approach to efficient ensemble and
  lifelong learning, 2020.
\newblock URL \url{https://arxiv.org/pdf/2002.06715.pdf}.

\bibitem[Wenzel et~al.(2020{\natexlab{a}})Wenzel, Roth, Veeling, Swiatkowski,
  Tran, Mandt, Snoek, Salimans, Jenatton, and Nowozin]{wenzel2020good}
Wenzel, F., Roth, K., Veeling, B., Swiatkowski, J., Tran, L., Mandt, S., Snoek,
  J., Salimans, T., Jenatton, R., and Nowozin, S.
\newblock How good is the {B}ayes posterior in deep neural networks really?
\newblock In \emph{Proceedings of the 37th International Conference on Machine
  Learning}, volume 119 of \emph{Proceedings of Machine Learning Research},
  pp.\  10248--10259. PMLR, 13--18 Jul 2020{\natexlab{a}}.
\newblock URL \url{https://proceedings.mlr.press/v119/wenzel20a.html}.

\bibitem[Wenzel et~al.(2020{\natexlab{b}})Wenzel, Snoek, Tran, and
  Jenatton]{wenzel2020hyperparameter}
Wenzel, F., Snoek, J., Tran, D., and Jenatton, R.
\newblock Hyperparameter ensembles for robustness and uncertainty
  quantification.
\newblock In \emph{Proceedings of the 34th International Conference on Neural
  Information Processing Systems}, NIPS'20, Red Hook, NY, USA,
  2020{\natexlab{b}}. Curran Associates Inc.
\newblock ISBN 9781713829546.
\newblock URL
  \url{https://papers.nips.cc/paper/2020/file/481fbfa59da2581098e841b7afc122f1-Paper.pdf}.

\bibitem[Williams \& Rasmussen(2006)Williams and
  Rasmussen]{williams2006gaussian}
Williams, C.~K. and Rasmussen, C.~E.
\newblock \emph{Gaussian processes for machine learning}.
\newblock MIT press Cambridge, MA, 2006.
\newblock URL \url{https://gaussianprocess.org/gpml/}.

\end{thebibliography}
\end{document}